\documentclass[11pt]{article}    
  \usepackage[margin=0.92in]{geometry} 

\usepackage{kpfonts,comment}   
\usepackage{cpdef}

\usepackage{float}
\usepackage{mathrsfs}
\usepackage{epigraph}
\usepackage{fancyvrb}
\usepackage{mdframed}
\usepackage{authblk}

\RecustomVerbatimCommand{\VerbatimInput}{VerbatimInput}
{fontsize=\footnotesize,
 frame=single,  
 framesep=0.5em,
 labelposition=topline,
}

\usepackage{amsmath,amsfonts,bm}

\def\eqref#1{equation~\ref{#1}}

\def\1{\bm{1}}

\DeclareMathAlphabet{\mathsfit}{\encodingdefault}{\sfdefault}{m}{sl}
\SetMathAlphabet{\mathsfit}{bold}{\encodingdefault}{\sfdefault}{bx}{n}

\newcommand{\softmax}{\mathrm{softmax}}

\DeclareMathOperator*{\argmax}{arg\,max}
\DeclareMathOperator*{\argmin}{arg\,min}

\usepackage{hyperref}
\usepackage{url}

\title{Do LLM Agents  Have Regret? \\  
A Case Study in Online Learning and Games}

 \author[1, *]{Chanwoo Park}
 \author[2, *]{~~~Xiangyu Liu}
 \author[1]{~~~Asuman Ozdaglar}
 \author[2]{~~~Kaiqing Zhang}
 \affil[1]{Massachusetts Institute of Technology}
 \affil[2]{University of Maryland, College Park}
 \affil[*]{Equal Contribution}

\date{}

\newcommand{\kzedit}[1]{{#1}}
\newcommand{\xyedit}[1]{{#1}}
\newcommand{\cpedit}[1]{{#1}}
\newcommand{\xy}[1]{\textcolor{green}{}} 
  
 \newcommand{\arxiv}[1]{{#1}}
 \newcommand{\arxivconfirmed}[1]{{#1}}
 \newcommand{\kz}[1]{{}} 
\newcommand{\icml}[1]{{}}
 \newcommand{\conf}[1]{{}} 
\newcommand{\iclr}[1]{{}}
\newcommand{\safevspace}[1]{\vspace{0mm}}

\begin{document}

\maketitle

\begin{abstract}
Large language models (LLMs)  have been increasingly employed for (interactive) decision-making, via the development of LLM-based autonomous agents.  Despite their emerging successes, the performance of LLM agents in decision-making has not been fully investigated through quantitative metrics---especially in the multi-agent setting when they interact with each other, a typical scenario in real-world LLM-agent applications. To better understand the limits of  
LLM agents in these interactive environments, we propose to study their interactions in benchmark decision-making  settings in online learning and game theory, through the   
performance metric of \emph{regret}. 
We first empirically study the {no-regret} behaviors of LLMs in canonical non-stochastic 
online learning problems, as well as the emergence of equilibria when multiple of them interact through playing repeated games.  
We then provide some theoretical insights into {sublinear regret growth in the cases we observed}, under certain assumptions on {(supervised)} pre-training \kzedit{and the data generation model.}  
Notably, we also identify (simple) cases where advanced LLMs such as GPT-4 fail to be no-regret. To further promote the no-regret behaviors, we propose a novel  \emph{unsupervised} training loss, the \emph{regret-loss}, which, in contrast to the supervised pre-training loss,  does \emph{not}  require the labels of (optimal)  actions. Finally, we establish the \emph{statistical}   guarantee of generalization bound for  regret-loss minimization, {and more importantly,} the \emph{optimization} guarantee that 
minimizing such a loss can lead to known no-regret learning algorithms, when single-layer self-attention models are used. 
Our further experiments demonstrate the effectiveness of our regret-loss, especially in addressing the above  ``regrettable'' cases.        
\end{abstract}

\iclr{\safevspace{-20pt}}
\safevspace{-5pt}

\section{Introduction}

\safevspace{-6pt} 

Large language models (LLMs) have recently exhibited remarkable emerging capabilities \citep{bubeck2023sparks,achiam2023gpt,wei2022chain,yao2023tree}. As a consequence, a burgeoning body of work has been investigating the employment of LLMs as central {controllers} for (interactive) decision-making, through the construction of \emph{LLM-based autonomous agents}  \citep{hao2023reasoning,shen2023hugginggpt,yao2022react,shinn2023reflexion,wang2023describe,Significant_Gravitas_AutoGPT}. Specifically, an LLM agent interacts with the (physical) world in a \emph{dynamic/sequential} way: it 
uses LLMs as an oracle for reasoning and planning, then acts in the environment based on the reasoning/planning  and the feedback it perceives over time. LLM agents have achieved impressive successes in embodied AI  \citep{ahn2022can,huang2022language,wang2023voyager}, \icml{and }natural science \citep{wu2023autogen,swan2023math}\icml{applications. }\arxivconfirmed{, and social science  \citep{park2022social,park2023generative} applications.}

\safevspace{-5pt}

Besides being \emph{dynamic}, another increasingly captivating feature of LLM-based decision-making is the involvement of \emph{strategic} interactions, oftentimes among multiple LLM agents. For example, it has been reported that the reasoning capability of LLMs can be improved by interacting \arxivconfirmed{with each other }through negotiation and/or debate games \citep{fu2023improving,du2023improving}; LLM agents have now been widely used to \emph{simulate} the strategic behaviors for social and economic studies, to understand the emerging behaviors in interactive social systems \citep{aher2022using,park2023generative}.  
Moreover, LLMs have also exhibited remarkable potential in solving various games \citep{meta2022human,mukobi2023welfare}, and in fact, a rapidly expanding literature has employed \emph{repeated games} as a fundamental benchmark to understand the strategic behaviors of LLMs  \citep{brookins2023playing,akata2023playing,fan2023can}. These \arxivconfirmed{exciting }empirical successes call for a rigorous examination and understanding through a theoretical lens of decision-making. 
\iclr{\safevspace{-6pt}}

\safevspace{-5pt}

\emph{Regret}, on the other hand, has been a core metric in (online) decision-making. It measures how ``sorry'' the decision-maker is, in retrospect, not to have followed the best prediction in hindsight 
\citep{shalev2012online}.  It provides not only a sensible way to \emph{evaluate}  the {sophistication} level of online decision-makers, but also a quantitative way to measure their \emph{robustness} against arbitrary (and possibly adversarial) environments.  More importantly, it inherently offers a connection to modeling and analyzing \emph{strategic behaviors}: the long-run interaction of {no-regret} learners leads to certain \emph{equilibrium} when they repeatedly play games  \citep{cesa2006prediction}. In fact, \emph{no-regret} learning has served as a natural framework for modeling and explaining human behavior in strategic decision-making, with experimental evidence \citep{erev1998predicting,nekipelov2015econometrics,balseiro2019learning}. It has thus been posited as an important model of  ``rational behaviors''  in playing games \citep{blum2008regret,roughgarden2015intrinsic}.  
Hence, it is natural to ask:

\iclr{
\safevspace{-6pt}
}
{\iclr{\fontsize{9.5}{11}}
\arxivconfirmed{
\safevspace{-8pt}
\begin{center}
    \emph{Can we 
    examine and better understand the online and strategic decision-making \\
    behaviors 
    of LLMs through the lens of \emph{regret}?}
\end{center}
\safevspace{2pt}} 
\icml{
\fontsize{9.5}{11}
\safevspace{-20pt}
\begin{center}
    \emph{Can we examine and better understand the online and strategic decision-making behaviors 
    of LLMs through the lens of \emph{regret}?}
\end{center}
\safevspace{-7pt}} 
}   
\iclr{\safevspace{-8pt}}

 \safevspace{-10pt}
Acknowledging that LLM(-agent)s are extremely complicated to analyze, 
to gain some insights into the question, 
we focus on benchmark decision-making settings: online learning with convex (linear) loss functions, and playing repeated games. 
We summarize our contributions as follows.

\safevspace{-10pt}

\icml{\safevspace{-14pt}}
\iclr{\safevspace{-8pt}}
\paragraph{Contributions.} 
First, we carefully examine the performance of several representative pre-trained LLMs in the aforementioned benchmark online decision-making settings,  
in terms of \emph{regret}. 
We observe that LLM agents can {achieve regret sublinear in time} in (non-stochastic) online learning settings,  where the loss functions change over time either arbitrarily, or  by following some patterns with bounded variation, and 
in playing both representative and randomly generated repeated games. For the latter,  equilibria will emerge as the long-term behavior of the multi-LLM interactions.  
Second, we provide some theoretical insights into the observed {sublinear regret} behaviors, \kzedit{based on certain assumptions on the \emph{supervised pre-training} procedure, a common practice in training large models for decision-making, and some hypothetical models for  training data generation.} 
In particular, we make a connection of the pre-trained LLMs to the known no-regret algorithm of \emph{follow-the-perturbed-leader} (FTPL) under these assumptions. Third, we also identify (simple) cases where advanced LLMs such as GPT-4 fail to be
no-regret. We thus propose a novel \emph{unsupervised} training loss, \emph{regret-loss}, which, in contrast to the supervised pre-training loss, does not require the \emph{labels}  of (optimal) actions. 
We then establish both \emph{statistical} and \emph{optimization}  guarantees 
for regret-loss minimization, \kzedit{which, in particular, show} that 
minimizing such a loss can  \emph{automatically} lead to the known no-regret learning algorithm of \emph{follow-the-regularized leader}  (FTRL),  \kzedit{under single-layer self-attention  parameterization.}  
Our \arxivconfirmed{further }experiments demonstrate the effectiveness of our new loss, especially in addressing the above ``regrettable'' cases. \kzedit{With the fast development of LLMs, we emphasize that our goal is not to assert whether  (current) LLMs are no-regret learners or not, especially given both the positive and negative observations above. Instead, our hope is to introduce and inspire more rigorous metrics and principles into the current evaluation and development of LLM agents, for online and multi-agent strategic decision-making.} 

\subsection{Related Work}
\paragraph{LLM(-agent) for decision-making.} The impressive capability of LLMs for \emph{reasoning} \citep{bubeck2023sparks,achiam2023gpt,wei2022chain,wei2022emergent,srivastava2023beyond,yao2023tree} has inspired a growing line of research on \emph{LLM for (interactive) decision-making}, i.e., an LLM-based autonomous agent interacts with the environment by taking actions repeatedly/sequentially,  based on the feedback it perceives. Some promises have been shown from a  \emph{planning} perspective \citep{hao2023reasoning,valmeekam2023planbench,huang2022inner,shen2023hugginggpt}. In particular, for embodied AI applications, e.g., robotics, LLMs have achieved impressive performance when used as the controller for decision-making \citep{ahn2022can,yao2022react,shinn2023reflexion,wang2023describe,driess2023palm,Significant_Gravitas_AutoGPT}. However, the performance of decision-making has not been rigorously characterized via the regret metric in these works. Very recently,  \cite{liu2023reason} has proposed a  principled architecture for LLM-agent, with provable regret guarantees in stationary and stochastic decision-making environments, under the Bayesian adaptive Markov decision processes framework. In contrast, our work focuses on online learning and game-theoretic settings, in potentially
adversarial and non-stationary environments.   
Moreover, (first part of) our work focuses on \emph{evaluating} the intelligence level of LLM per se in decision-making (in terms of the regret metric), while \cite{liu2023reason} focused on \emph{developing} a new architecture that uses LLM as an oracle for reasoning,  together with memory and specific planning/acting subroutines, \emph{to achieve} sublinear (Bayesian) regret, in stationary and stochastic environments.   

\paragraph{LLMs in multi-agent environments.} The interaction of multiple LLM agents has garnered significant attention lately.  
For example, \cite{fu2023improving} showed that LLMs can autonomously 
improve each other in a negotiation game by playing and criticizing each other. Similarly, \citep{du2023improving,liang2023encouraging,xiong2023diving,chan2023chateval, li2023prd} showed that multi-LLM \emph{debate} can improve the reasoning and evaluation capabilities of the LLMs.  \cite{qian2023communicative,schick2022peer,wu2023autogen}  demonstrated the potential of multi-LLM interactions and collaboration in software development,  writing, and problem-solving, respectively. \cite{zhang2023building} exhibited a similar potential in embodied cooperative environments. Recently, \cite{park2025maporl} trained multi-agent LLM systems using multi-agent reinforcement learning. More formally,  multi-LLM interactions have also been investigated under a \emph{game-theoretic} framework, to characterize the \emph{strategic} decision-making of LLM agents. \cite{meta2022human,mukobi2023welfare} and \cite{xu2023language,xu2023exploring} have demonstrated the promise of LLMs in playing Diplomacy and WereWolf games, respectively, which are both language-based games with a mixture of competitive and cooperative agents.  
{Note that these works utilized LLM to solve a specific rather than a general game.} Related to our work, \cite{brookins2023playing,akata2023playing,lore2023strategic,brookins2023playing,fan2023can} have also used (repeated) matrix games as a benchmark to evaluate the reasoning capability and rationality of LLM agents. 
In contrast to our work, these empirical studies have not formally investigated LLM agents using the metric of \emph{regret}, nor through the lenses of \emph{online learning} and  \emph{equilibrium-computation}, which are all fundamental in modeling and analyzing strategic multi-agent interactions. Moreover, our work also provides theoretical results to explain and further enhance the no-regret property of LLM agents.

\paragraph{LLMs \& Human/Social behavior.}  LLMs have also been used to \emph{simulate} the behavior of human beings, for social science and economics studies \citep{engel2023integrating}. The extent of LLMs simulating human behavior has been claimed as a way to evaluate the level of its intelligence in a controlled environment \citep{aher2022using,tsai2023can}.  For example, \cite{li2023camel,hong2023metagpt,zhao2023competeai} showed that by specifying different ``roles'' to LLM agents, certain collaborative/competitive behaviors can emerge.  \cite{argyle2023out} showed that LLMs can emulate response distributions from diverse human subgroups, illustrating their adaptability. \cite{horton2023large} argued that an LLM, as a computational model of humans, can be used as \emph{homo economicus} when given endowments, information,
preferences, etc., to gain new economic insights by simulating its interaction with other LLMs.  \cite{park2022social,park2023generative}  proposed scalable simulators that can generate realistic social behaviors emerging in populated and interactive social systems, and emerging behaviors of LLM agents in society have also been consistently observed in \cite{chen2023agentverse,chen2023emergence}.  \cite{li2023you,li2023quantifying} studied the behavioral dynamics of LLM agents on social networks. 
These empirical results have inspired our work, which can be viewed as an initial attempt towards quantitatively understanding the \emph{emerging behavior} of LLMs as computational human models, given the known justification of \emph{equilibrium} being a long-run emerging behavior of \emph{learning dynamics} \citep{fudenberg1998theory} and strategic interactions \citep{young2004strategic,camerer2011behavioral}.

\paragraph{Transformers \& In-context-learning.} LLMs nowadays are predominantly built  upon the architecture of {Transformers}  \citep{vaswani2017attention}. Transformers have exhibited a remarkable capacity of \emph{in-context-learning} (ICL), which can construct new predictors from sequences of labeled examples as input,  without further parameter updates. This has enabled the  \emph{few-shot learning}  capability of Transformers \citep{brown2020language,garg2022can,min2022rethinking}. The empirical successes have inspired burgeoning theoretical studies on ICL. \cite{xie2021explanation} used a Bayesian inference framework to explain how ICL works, which has also been adopted in \cite{wang2023large,jiang2023latent}. \cite{akyurek2022learning, von2023transformers,dai2022can,giannou2023looped} showed (among other results) that ICL comes from that Transformers can implement the gradient descent (GD) algorithm. \cite{bai2023transformers} further established that Transformers can
implement a broad class of machine learning algorithms in context. 
Moreover,  \cite{ahn2023transformers, zhang2023trained,mahankali2023one}  proved that a \emph{minimizer} of the certain training loss 
among single-layer Transformers is equivalent to a single step of GD for linear regression. 
\cite{li2023transformers} established generalization bounds of ICL from a   multi-task learning perspective. \cite{zhang2023and} argued that ICL  implicitly implements Bayesian model averaging, and can be approximated by the attention mechanism. They also established a result on some  \emph{regret} metric. However, the regret notion is not defined for (online) decision-making, and is fundamentally different from ours that is standard in online learning and games.  Also, we provide extensive experiments to validate the no-regret behavior by our definition. More recently, the ICL property has also been generalized to decision-making settings.   \cite{laskin2022context,lee2023supervised,lin2023transformers} investigated the in-context reinforcement learning \arxiv{(RL)} property of Transformers under supervised pre-training, for solving 
stochastic bandits and Markov decision processes.  
In contrast, our work focuses on online learning settings with an arbitrary and \emph{potentially adversarial} nature, as well as \emph{game-theoretic} settings. We also provide a new  \emph{unsupervised}  loss to promote the no-regret behavior in our settings.

\paragraph{Online learning and games.} Online learning has been extensively studied to model the decision-making of an agent who interacts with the environment sequentially, with a potentially arbitrary sequence of loss functions  \citep{shalev2012online,hazan2016introduction}, and has a deep connection to game theory \citep{cesa2006prediction}.  In particular, regret,  the difference between the incurred accumulated loss and the best-in-hindsight accumulated loss,  has been the core performance metric, and a  good online learning algorithm should have regret at most sublinear in time $T$, which is referred to as being \emph{no-regret}. 
Many well-known algorithms can achieve no-regret against \emph{arbitrary} loss sequences, e.g., multiplicative weight updates (MWU)/Hedge  \citep{freund1997decision,arora2012multiplicative}, EXP3 \citep{auer2002nonstochastic},  
and more generally follow-the-regularized-leader (FTRL) \citep{shalev2007primal} and follow-the-perturbed-leader (FTPL) \citep{kalai2005efficient}. In the bandit 
literature \citep{lattimore2020bandit,bubeck2012regret}, such a setting without any statistical assumptions on the losses is also referred to as the \emph{adversarial/non-stochastic}  setting. 
Following the conventions in this literature, the online settings we focus on 
shall not be confused with the stationary and \emph{stochastic}(-bandit)/(-reinforcement learning) settings that have been explored in several other recent works on \emph{Transformers for decision-making}   \citep{lee2023supervised,lin2023transformers}. Centering around the regret metric, our work has also explored the 
non-stationary bandit setting \citep{besbes2014stochastic}, as well as the repeated game setting where the environment itself consists of strategic agents \citep{cesa2006prediction}.

\iclr{\safevspace{-15pt}}
\icml{\safevspace{-14pt}}
\safevspace{-9pt}

\section{Preliminaries}\label{sec:formulation}
\iclr{\safevspace{-9pt}}
\icml{\safevspace{-7pt}}
\safevspace{-10pt}

\iclr{\safevspace{-8pt}}
\icml{\safevspace{-10pt}}

\paragraph{Notation.} 
We use $\NN$ and $\NN^+$ 
to denote the sets of non-negative and positive integers, respectively. 
For a finite set $\cS$, we use $\Delta(\cS)$ to denote the simplex over $\cS$. For $d\in \NN^+$, we define $[d]:= \{1, 2, \dots, d\}$. \arxivconfirmed{For two vectors $x,y\in\RR^d$, we use $\la x,y\rangle$ to denote the inner product of $x$ and $y$. }We define $\pmb{0}_d$ and $\pmb{1}_d$ as a $d$-dimensional zero or one vector\icml{.}\arxivconfirmed{, {and $\pmb{O}_{d\times d}$ and $I_{d \times d}$ as a $d \times d$-dimensional zero matrix and  identity matrix, respectively. We omit $d$ when it is clear from the context.}} We define $e_i$ as a unit vector \arxivconfirmed{(with proper dimension) }whose  $i$-th coordinate equals $1$. For $p \in \RR^d, R>0$ and $C \subseteq \RR^d$ is a convex set, define $B(p, R, \norm{\cdot}):= \{x \in \RR^d \mid \norm{x - p} \leq R \}$,  $\texttt{Proj}_{C, \norm{\cdot}}(p) = \argmin_{x \in C} \norm{x - p}$\icml{.}\arxivconfirmed{ (which is well-defined as $C$ is a convex set), and $\texttt{clip}_R(x):=[ \texttt{Proj}_{B(0, R, \norm{\cdot}_2), \norm{\cdot}_2}(x_i)]_{i \in [d]}$.} Define $\texttt{Softmax}(x):= \left(\frac{e^{x_i}}{\sum_{j \in [d]}e^{x_j}}\right)_{i \in [d]}$\icml{.}\arxivconfirmed{ and $\texttt{ReLU}(x) = \max(0, x)$ for $x \in \RR^d$. For $A \in \RR^{m \times n}$ {with $A_i$ denoting its $i$-th column}, we define $\norm{A}_\texttt{op}:= \max_{\norm{x}_2 \leq 1} \norm{Ax}_2$, $\norm{A}_{2, \infty}:= \sup_{i \in [n]} \norm{A_i}_2$,
$\norm{A}_{F}$
as the Frobenius norm, and    $A_{-1} := A_n$ to denote the last column vector of $A$. } We define $\RR^+:= \{x \mid x\geq 0 \}$. \arxivconfirmed{For a set $\Pi$, define $\text{diam}(\Pi, \norm{\cdot}):= \sup_{\pi_1, \pi_2 \in \Pi} \norm{\pi_1 - \pi_2}$.}  {We define $\mathbbm{1}(\cE):=1$ if $\cE$ is true, and $\mathbbm{1}(\cE):=0$ otherwise.} \arxivconfirmed{For a random variable sequence $(X_n)_{n \in \NN}$ and random  variables $X, Y$, we denote $F_X$ as the cumulative distribution function of a random variable $X$, $X_n \overset{p}{\to} X$  if $\forall \epsilon > 0, \lim_{n \to \infty} \PP(|X_n - X| > \epsilon) = 0$, $X_n \overset{d}{\to} X$  if $\lim_{n \to \infty} F_{X_n}(x) = F_X(x)$ for all $x$ where $F_X(x)$ is continuous, $X \overset{d}{=} Y$ if $F_X(x) = F_Y(x)$ for all $x$, $X_n \overset{a.s.}{\to} X$ if $\PP(\lim_{n \to \infty} X_n = X) = 1$, and $\operatorname{esssup}(X) := \inf \{M \in \mathbb{R} : \PP(X > M) = 0\}$.} {For a random variable $X$, we use $\text{supp}(X)$ to denote its support.} \arxivconfirmed{{For functions $f,g:\RR\to \RR$,} we define $g(x) = \cO(f(x))$ if there exist $x_0, M<\infty$ such that $|g(x)| \leq M|f(x)|$ for {all} $x > x_0$. {We use $f'$ to denote the derivative of $f$.} Let $F: \Omega\rightarrow \RR$ be a continuously-differentiable, strictly convex function defined on a convex set $\Omega$. The Bregman divergence associated with $F$ for points $p$, $q$ is defined as $D_F(p, q):=F(p)-F(q)-\langle \nabla F(q), p-q\rangle$. For a sequence $(\ell_{t})_{t\in[T]}$ for some $T\in\NN^+$, we define $\ell_{a:b}:=(\ell_a, \cdots, \ell_b)$ for $1\le a\le b\le T$. If $a>b$, we define  $\ell_{a:b}=\emptyset$.}

\subsection{Online Learning \& Repeated Games}
\safevspace{-9pt}

\icml{\safevspace{-6pt}}
\iclr{\safevspace{-8pt}}

\paragraph{Online learning.} We consider the {online learning} setting where an agent interacts with the environment for $T$ rounds,  by iteratively making decisions based on the feedback she receives. Specifically, at each time step $t$, the agent chooses her decision policy $\pi_t\in\Pi$ for some bounded domain $\Pi$, and after her commitment to $\pi_t$, a {bounded} 
loss function $f_t:\Pi\to {[-B,B]}$ for some constant $B>0$ is chosen by the environment, potentially in an adversarial fashion. The agent thus incurs a loss of $f_t(\pi_t)$, and will update her decision to $\pi_{t+1}$ using the feedback. We focus on the most basic setting where the agent chooses actions from a finite set $\cA$ every round, which is also referred to as the \emph{Experts Problem} \citep{cover1966behavior,vovk1990aggregating,littlestone1994weighted,hazan2016introduction}, without loss of much generality (cf. \Cref{rem:wolg} for a detailed discussion). In this case,  $\Pi$ becomes the simplex over $\cA$, i.e., $\Pi=\Delta(\cA)$, and $f_t(\pi_t)=\la {\ell}_t,\pi_t \rangle$  for some loss \emph{vector} ${\ell}_t\in\RR^{d}$ that may change over time{, where $d:=|\cA|$.}

\iclr{\safevspace{-4pt}} 
\safevspace{-4pt}

\icml{\safevspace{-4pt}} 

At time step $t\in[T]$, the agent may receive either the full vector ${\ell}_t$, or only the realized loss ${\ell}_{ta_t}$ {(\arxivconfirmed{we sometimes also interchangeably write it as }\icml{or }$\ell_t (a_t)$)}, {the $a_t$th element of $\ell_t$,} for some $a_t\sim \pi_t(\cdot)$, as feedback, which will be referred to as online learning with \emph{full-information feedback}, and that with \emph{bandit feedback}, respectively. The latter is also referred to as the \emph{adversarial/non-stochastic bandit} problem in the multi-armed bandit (MAB) literature. Note that hereafter, we will by default refer to this setting that does \emph{not} make any assumptions on the loss sequence $({\ell}_t)_{t \in [T]}$ simply as \emph{online learning}. 
Moreover, if the loss functions change over time (usually with certain bounded variation), we will refer to it as 
\emph{non-stationary online learning} for short, whose bandit-feedback version is also referred to as the \emph{non-stationary bandit} problem. 

\safevspace{-9pt}

\safevspace{-3pt}

\icml{\safevspace{-2pt}}

\paragraph{Repeated games.} The online learning setting above has an intimate connection to game theory. Consider a normal-form game $\cG=\la N, \{\cA_n\}_{n\in[N]},\{r_n\}_{n\in[N]}\rangle$, where $N$ is the number of players, $\cA_n$ and $r_n:\cA_1 \times \dots \times \cA_N \rightarrow {[-B, B]}$ are the action set and  the payoff  function of player $n$, respectively. The $N$ players repeatedly play the game for $T$ rounds, each player $n$  maintains a strategy $\pi_{n,t}\in\Delta(\cA_n)$ at time $t$, and takes action $a_{n,t}\sim\pi_{n,t}(\cdot)$. The \emph{joint} action $a_t=(a_{1,t},\cdots,a_{N,t})$ determines the payoff of each player at time $t$, $\{r_n(a_t)\}_{n\in[N]}$. From \arxivconfirmed{a single-player's (e.g., player $n$'s)}\icml{player $n$'s} perspective, she encounters an online learning problem with (expected) loss function  ${\ell}_t:=-\EE_{a_{-n,t}\sim \pi_{-n,t}}[r_{n}(\cdot,a_{-n,t})]$  at time $t$, where $-n$ denotes the index for all the players other than player $n$. 
We will  refer to it as the \emph{game setting} for short,  
and use  the terms of ``agent'' and ``player'' interchangeably hereafter.  The key difference between online learning and repeated games is in their interaction dynamics: online learning involves an agent facing a potentially adversarial, changing environment (or sequence of loss functions), while in repeated games, agents interact by playing the same game repeatedly, which might be less adversarial when they follow specific learning algorithms.

\subsection{Self-attention \& Transformers}
\label{ssec:self-tf}
\paragraph{(Linear) Self-attention.} 
One key component in Transformers \citep{vaswani2017attention}, the backbone of modern language models, is the \emph{(self-)attention} mechanism. For simplicity, we here focus on introducing the \emph{single-layer} self-attention architecture.
The mechanism takes a sequence of  vectors \( Z = [z_1, \dots, z_t] \in \RR^{d \times t} \) as input, and  outputs some sequence of $[\hat{z}_1, \dots, \hat{z}_t] \in \RR^{d \times t}$. For each $i \in [t]$ where $i > 1$, the output is generated by  $\hat{z}_i = (Vz_{1:{i-1}}) \sigma((Kz_{1:i-1})^\intercal  (Qz_i))$, where $z_{1:{i-1}}$ denotes the $1$ to $i-1$ columns of $Z$, $\sigma$ is either the $\texttt{Softmax}$ or $\texttt{ReLU}$ activation function, and for the initial output, \( \hat{z}_1 = \pmb{0}_d \). Here, $V, Q, K \in \RR^{d \times d}$ are referred to as  the \emph{Value}, \emph{Query}, and \emph{Key} matrices,  respectively. Following the theoretical framework in \cite{von2023transformers, mahankali2023one}, we exclude the attention score for a token $z_i$ in relation to itself.
For theoretical analysis, we also consider the \emph{linear} self-attention  model, where $\hat{z}_i = (Vz_{1:{i-1}}) ((Kz_{1:i-1})^\intercal  (Qz_i))$. We write this (linear) self-attention layer's output as $\texttt{(L)SA}_{(V, Q, K)}(Z)$. We define an $M$-head self-attention layer with $\theta = \{(V_m, Q_m, K_m)\}_{m\in[M]}$ as $\texttt{M-(L)SA}_{\theta}(Z) := \sum_{m=1}^M \texttt{(L)SA}_{(V_m, Q_m, K_m)}(Z)$. We define $\norm{\cdot}_{\texttt{M-(L)SA}}$ as $\norm{\theta}_{\texttt{M-(L)SA}} := \max_{m \in [M]} \left\{ \norm{Q_m}_{\texttt{op}}, \norm{K_m}_{\texttt{op}} \right\} + \sum_{m=1}^M \norm{V_m}_{\texttt{op}}$.

\paragraph{Transformers.} 
For a multi-layer perceptron (MLP) layer, it takes \( Z = [z_1, \dots, z_t]\in \RR^{d \times t} \) as input, with parameter $\theta = (W_1, W_2) \in \RR^{d' \times d} \times \RR^{d \times d'}$ such that for each $i \in [t]$, the output is $\hat{z_i} := W_2 \sigma(W_1 z_i)$ where $\sigma$ is either $ \texttt{Softmax}$ or $\texttt{ReLU}$. We write the output of an MLP layer with parameter $\theta$ as $\texttt{MLP}_\theta(Z)$. Defining    $\norm{\cdot}_{\texttt{MLP}}$ as $\norm{\theta}_{\texttt{MLP}} := \norm{W_1}_{\texttt{op}} + \norm{W_2}_{\texttt{op}}$ and  $\texttt{ResNet}(f, Z):= Z + f(Z)$, we can define an  $L$-layer Transformer with parameter $\theta = (\theta^{(lm)}, \theta^{(la)})_{l \in [L]}$ as
\$
\texttt{TF}_{\theta}(Z):= Z^{(L)},
\$
where the output $Z^{(L)}$ is defined iteratively from  $Z^{(0)} = \texttt{clip}_R(Z):= \min(R, \max(-R, Z))$ and 
\begin{align*}
    Z^{(l)} = \texttt{clip}_R\left(\texttt{ResNet}\left(\texttt{MLP}_{\theta^{(la)}}, \texttt{ResNet}\left(\texttt{M-(L)SA}_{\theta^{(lm)}}, Z^{(l-1)}\right)\right)\right),
\end{align*} 
for some $R>0$.   
We define a class of Transformers with certain parameters as $\Theta_{d, L, M, d', B_{\texttt{TF}}}:= \{\theta = (\theta^{(lm)}, \theta^{(la)})_{l \in [L], m \in [M]}: \norm{\theta}_{\texttt{TF}}  \leq B_{\texttt{TF}}\}$,  where $M$ is the number of heads of self-attention,  \#\label{equ:def_TF_norm}
\norm{\theta}_{\texttt{TF}} := \max_{l \in [L]}\left\{  \norm{\theta^{(la)}}_{\texttt{M-(L)SA}}  + \norm{\theta^{(lm)}}_{\texttt{MLP}} \right\},
\#
and $B_{\texttt{TF}}>0$ is some constant. When it is clear from the context, we may omit the subscripts and write it as $\Theta$ for simplicity. We assume $R$ to be sufficiently large such  that $\texttt{clip}$ does not take effect on any of our approximation results.

\icml{\safevspace{-11pt}}

\safevspace{-10pt}

\iclr{\safevspace{-9pt}}

\subsection{Performance Metric: Regret}
\safevspace{-7pt}

\iclr{\safevspace{-10pt}}
\icml{\safevspace{-8pt}}

We now introduce \emph{regret}, the core performance metric used in online learning and games. 
For a given algorithm $\mathscr{A}$, 
let \( \pi_{\mathscr{A},t} \) denote the decision policy of the agent at time $t$ generated by $\mathscr{A}$. Then, the regret, which is the difference between the accumulated (expected) loss incurred by implementing $\mathscr{A}$ and that incurred by 
the best-in-hindsight fixed decision, can be defined as  
\safevspace{-18pt}
\iclr{
\safevspace{-23pt}}
\safevspace{-4pt}
\icml{\safevspace{-8pt}}
\$
\text{Regret}_{\mathscr{A}}\left((f_t)_{t\in[T]}\right) :=\sum_{t=1}^T f_t(\pi_{\mathscr{A}, t}) - \inf_{\pi \in \Pi} \sum_{t=1}^T f_t(\pi).
\$ 
\safevspace{-4pt} 
\icml{\safevspace{-12pt}}
\iclr{
\normalsize 
\safevspace{-18pt}}
\safevspace{-12pt}

In the Experts Problem, the definition is instantiated as $\text{Regret}_{\mathscr{A}}(({\ell}_t)_{t \in [T]}):=\sum_{t=1}^T \langle {\ell}_t, \pi_{\mathscr{A}, t} \rangle - \inf_{\pi \in \Pi} \sum_{t=1}^T \langle {\ell}_t,  \pi\rangle$. With bandit-feedback, a common metric may also take further expectation for $\text{Regret}_{\mathscr{A}}$, over the {randomness of the policies $(\pi_{\mathscr{A}, t})_{t\in[T]}$.}  
An algorithm $\mathscr{A}$ is referred to as being \emph{no-regret}, {if $\max_{(f_t)_{t\in[T]}}\text{Regret}_{\mathscr{A}}((f_t)_{t\in[T]})\sim o(T)$}\arxivconfirmed{, i.e., the {worst-case} regret grows sublinearly in $T$}. Known no-regret algorithms include follow-the-regularized-leader \citep{shalev2007primal}, follow-the-perturbed-leader  \citep{kalai2005efficient} (see \Cref{appendix:FTRL} for more details). 
\safevspace{-5pt}

\icml{\safevspace{-8pt}}

In non-stationary online learning, one may also use the metric of \emph{dynamic regret} \citep{zinkevich2003online}, where the \emph{comparator} in the definition also changes over time, as the best decision policy at each time $t$: $\text{D-Regret}_{\mathscr{A}}((f_t)_{t\in[T]}):=\sum_{t=1}^T f_t(\pi_{\mathscr{A},t}) - \sum_{t=1}^T\inf_{\pi \in \Pi}f_t(\pi)$, 
which is a stronger notion than $\text{Regret}_{\mathscr{A}}((f_t)_{t\in[T]})$ in that $\text{Regret}_{\mathscr{A}}((f_t)_{t\in[T]})\leq  \text{D-Regret}_{\mathscr{A}}((f_t)_{t\in[T]})$. 

\iclr{\safevspace{-12pt}}
\icml{\safevspace{-12pt}}

\safevspace{-8pt}
\section{{Do} Pre-Trained LLM{s} Have  Regret? Experimental Validation}
\iclr{\safevspace{-6pt}}
\label{sec:empirical_insights}
\icml{\safevspace{-8pt}}
\safevspace{-11pt}

In this section, we explore the no-regret behaviors of representative LLMs (i.e., \kzedit{mainly} GPT-4 Turbo and GPT-4, \kzedit{together with} GPT-3.5 {Turbo}{, Mixtral-8x7b-instruct, and Llama-3-70B-instruct}),  
in the context of online learning and games.
\arxivconfirmed{All experiments with LLMs are conducted using the public OpenAI \citep{openai2023gpt4} {or LLM Engine \citep{llm-engine} Python API}.}  {We provide some hypothetical intuitions as to why pre-trained LLMs might exhibit no-regret behavior in \Cref{sec:intuition}{, which will be made concrete next}.} 

\safevspace{-11pt}

\icml{\safevspace{-10pt}}
 \iclr{\safevspace{-2pt}}
\paragraph{Interaction protocol.} To enable the sequential interaction with  LLMs, we first describe the setup and objective of our experimental study. At each round, we incorporate the entire history of loss vectors of past interactions into our prompts, as concatenated texts, and ask the LLM agent to determine a policy that guides the decision-making for the next round. Note that since we hope to \emph{evaluate} the sophistication level of pre-trained LLMs through online learning or games, we only provide simple prompts that it should utilize the history information, without providing explicit rules of \emph{how} to make use of the history information, nor asking her to \emph{minimize regret} (in any sense). A detailed description and an ablation study of the prompts are deferred to \Cref{appendix:ablation}{, and an illustration of the protocol for playing repeated games is given in \Cref{fig:game-prompt}.} 

\begin{figure*}[!h]	
\hspace{-25pt}
\centering\includegraphics[width=1.04\textwidth]{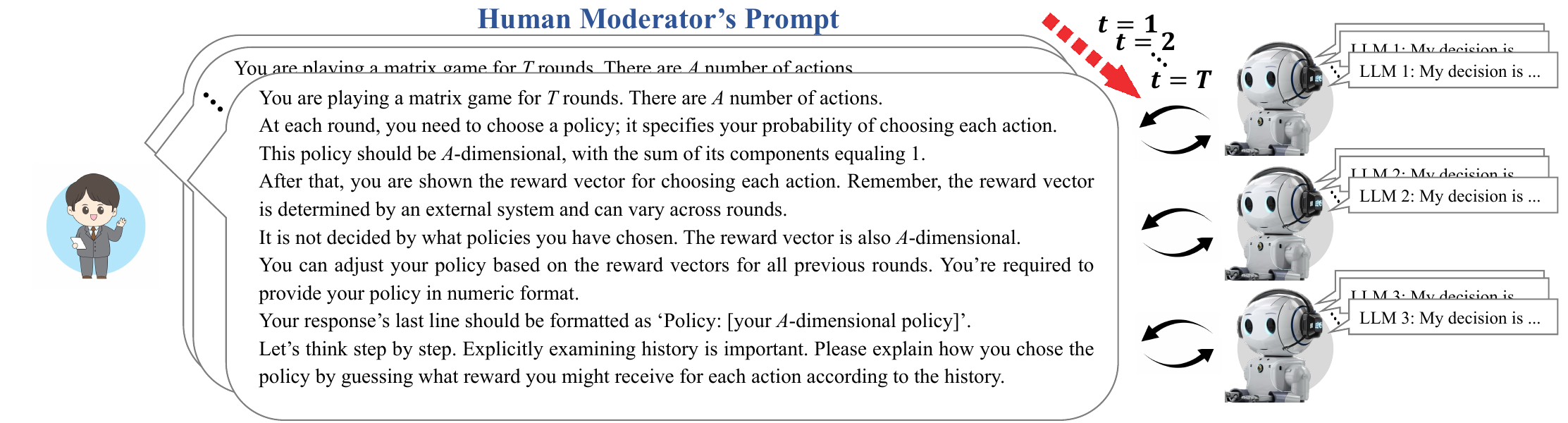}	\caption{Demonstration of the prompts and interaction protocol for multi-player repeated games. A human moderator does not provide the game's payoff matrices to the LLMs. Instead, at each round, the human moderator provides each player's own payoff vector history.}
 \label{fig:game-prompt}
 \end{figure*}

\subsection{Framework for {Sublinear Regret} Behavior Validation}\label{sec:framework_test}

\arxivconfirmed{Before delving into the results, }{we note that to the best of our knowledge, we are not aware of any principled framework for validating {sublinear growth of the regret} with \emph{finite-time}  experimental data. Therefore, }we propose two frameworks below to rigorously validate the no-regret behaviors of algorithms over a \emph{finite} $T${, which might be of independent interest. More details  can be found in  \Cref{appendix:verify}.}
\safevspace{-12pt}

\arxivconfirmed{
\paragraph{Trend-checking framework.}
{We propose a statistical hypothesis test aligned with our objectives:}
\begin{align*}
    H_0&: \text{The sequence } \left(\text{Regret}_{\mathscr{A}}\left((f_\tau)_{\tau\in[t]}\right)/t\right)_{t \in [T]} \text{ does not exhibit a decreasing pattern}
    \\
    H_1&:  \text{The sequence } \left(\text{Regret}_{\mathscr{A}}\left((f_\tau)_{\tau\in[t]}\right)/t\right)_{t \in [T]} \text{ shows a decreasing pattern}.
\end{align*}

\safevspace{-12pt}
Ideally, one should check if $\text{Regret}_{\mathscr{A}}\left((f_\tau)_{\tau\in[t]}\right)/t$ approaches zero (or a negative value) as $t$ goes to infinity. With a finite $T$ value, testing these hypotheses provides a method to quantify this -- whether we reject $H_0$ offers a way to measure it. 
To this end, one needs to count the number of $\text{Regret}_{\mathscr{A}}\left((f_\tau)_{\tau\in[t]}\right)/t - \text{Regret}_{\mathscr{A}}\left((f_\tau)_{\tau\in[t+1]}\right)/(t+1) > 0$, for which we use \Cref{prop:testing} below. We will report the $p$-value of $H_0$, denoted as $p_{trend}$, as the output of this framework.

\begin{restatable}{proposition}{testing}
\label{prop:testing}
    \emph{($p$-value of the null hypothesis)\textbf{.}} Define the event 
    \small
    \$
    \cE(s, T) :=\left\{\text{The~number~of~} \frac{\text{Regret}_{\mathscr{A}}\left((f_\tau)_{\tau\in[t]}\right)}{t} - \frac{\text{Regret}_{\mathscr{A}}\left((f_\tau)_{\tau\in[t+1]}\right)}{t+1} > 0 \text{~~for~~} t = 1, \dots, T \text{~is~at~least~~} s \geq \frac{T-1}{2}\right\}.
    \$
    \normalsize
    Under the assumption that {the null}  hypothesis \( H_0 \) holds, the probability of this event happening is bounded as $\PP_{H_0}(\cE(s, T)) \leq \frac{1}{2^{T-1}} \sum_{t = s}^{T-1}  \begin{pmatrix}
T-1 \\
t
\end{pmatrix}$. 
\end{restatable}

\paragraph{Regression-based framework.} 
We propose an alternative approach by 
fitting the data with regression.  
In particular, one can use the data 
\(
\left\{\left(t, \log \text{Regret}_{\mathscr{A}}\left((f_\tau)_{\tau\in[t]}\right)\right)\right\}_{t\in[T]}
\) to fit a  function $g(t) = \beta_0 \log t + \beta_1$, where the estimate of $\beta_0$, i.e., $\hat{\beta}_0$, satisfying  $\hat{\beta}_0<1$ 
may be used to indicate the no-regret behavior, i.e., the \emph{sublinear}  growth of $\text{Regret}_{\mathscr{A}}\left((f_\tau)_{\tau\in[t]}\right)$ over time.  
While being simple, it cannot be directly used when $\text{Regret}_{\mathscr{A}}\left((f_\tau)_{\tau\in[t]}\right)< 0$. Hence, we set $\log \text{Regret}_{\mathscr{A}}\left((f_\tau)_{\tau\in[t]}\right)$ as $-10$ if this happens. {We define $p_{reg}$ as the $p$-value of the regression parameter $\hat{\beta}_0$, and will report  the pair of $(\hat{\beta}_0,p_{reg})$ as the output of this framework.} 
}

\safevspace{-10pt}

\icml{
\icml{\safevspace{-12pt}}
 \paragraph{Trend-checking framework.}
This framework is built upon non-parametric hypothesis testing.
\iclr{
\safevspace{-3pt}
\noindent\begin{minipage}{.62\textwidth}
\setlength{\fboxsep}{-4pt}
\fbox{%
\safevspace{-3pt}
\icml{\safevspace{-10pt}}
    \parbox{\textwidth}{%
    \begin{align*}
        H_0&: (\text{Regret}(t)/t)_{t \in [T]} \text{ does not exhibit a decreasing pattern;} \\
        H_1&: (\text{Regret}(t)/t)_{t \in [T]} \text{ shows a decreasing pattern.}
    \end{align*}
    }%
\safevspace{-2pt}
\icml{\safevspace{-10pt}}
}
\end{minipage}%
\hfill
\begin{minipage}{.37\textwidth}
{
Ideally, one should check if $\text{Regret}(t)/t$ approaches zero (or a negative value) as $t$ goes to infinity. With ~~finite~~ $T$ ~~values, testing these hypotheses} 
\end{minipage}
\\
}
{
\icml{\safevspace{-25pt}}

    \begin{align*}
        H_0&: (\text{Regret}(t)/t)_{t \in [T]} \text{ does not exhibit a decreasing pattern;} \\
        H_1&: (\text{Regret}(t)/t)_{t \in [T]} \text{ shows a decreasing pattern.}
    \end{align*}
    
    \icml{\safevspace{-13pt}}
Ideally, one should check if $\text{Regret}(t)/t$ approaches zero (or a negative value) as $t$ goes to infinity. With finite $T$ values, testing these hypotheses}
provides a method to quantify this—whether we reject $H_0$ offers a way to measure it.
To this end, one needs to count the number of $R(t)/t - R(t+1)/(t+1) > 0$, for which we have \Cref{prop:testing} to give some understanding of the probability it happens with various counts. {We will report the $p$-value of $H_0$ as the output of this framework.}

\icml{\safevspace{-5pt}}

\safevspace{-7pt}
\paragraph{Regression-based framework.} 
One can use the data $\{(t, \log \text{Regret}(t))\}_{t\in[T]}$ to fit a  function $\log \text{Regret}(t) = \beta_0 \log t + \beta_1$, to validate no-regret behaviors by checking the estimate of $\beta_0$ (i.e., $\hat{\beta}_0$) is less than 1.  We define $p_{reg}$ as the $p$-value of the regression parameter $\hat{\beta}_0$, which will also be reported.
}

\icml{\safevspace{-10pt}}

\subsection{Results: Online Learning}
\label{sec:experiments_online} 

\icml{\safevspace{-8pt}}
\safevspace{-8pt} 

We now present the experimental results of pre-trained LLMs in online learning in: 1) (arbitrarily) changing environments, 2) non-stationary environments, and 3) bandit-feedback environments. {Results for 2) and 3) are deferred to Appendices  \ref{appendix:nonstationary} and \ref{appendix:bandit}.}

\icml{We defer a detailed explanation to \Cref{ssec:environment}. For \textbf{arbitrary changing environments}, the average regret (over multiple randomly generated instances) performance is presented in \Cref{fig:online}, where we compare {GPT-3.5 Turbo/}GPT-4 with well-known no-regret algorithms, FTRL with entropy regularization {and FTPL with Gaussian perturbations} (with tuned parameters). It is seen that these pre-trained LLMs are indeed no-regret, and often have smaller regrets than baselines. For \textbf{non-stationary environments}, the average dynamic regret results are presented in \Cref{table:no_regret_statistics}. It can be seen that GPT-4 achieves sublinear dynamic regret and outperforms Restart FTRL/FTPL. For \textbf{bandit-feedback settings}, we compare the performance with the counterparts of FTRL in the bandit-feedback setting, e.g., EXP3 \citep{auer2002nonstochastic} and the bandit-version of FTPL \citep{abernethy2015fighting} in both \Cref{fig:bandit-feedback} and \Cref{table:no_regret_statistics}, where {GPT-3.5 Turbo/}GPT-4 consistently achieves lower regret.} 

\safevspace{-11pt}

\arxivconfirmed{\paragraph{Changing environments.} 
We first consider the setting with  (arbitrarily)  changing environments, which are instantiated as follows:  
1) \emph{Randomly-generated loss sequences}. At every timestep, we generate a random loss vector ${\ell}_t\sim \operatorname{Unif}(\times_{i=1}^d[\min\{x_i, y_i\}, \max\{x_i, y_i\}])$ for $\{x_i, y_i\sim \operatorname{Unif}(0, 10)\}_{i\in[d]}$ or ${\ell}_t\sim \cN(\pmb{\mu}_d, I)$ with 
clipping {to $[0,10]$} 
to ensure boundedness of the loss, where $\pmb{\mu}_d\sim\operatorname{Unif}([0, 10]^d)$. Note that we 
use this as a way to  \emph{systematically} generate potentially arbitrary loss sequences, and also note that our regret was defined for each \emph{realization} of the \emph{random loss vectors} (instead of their expectations as in the definition of regret in \emph{stochastic bandit}  problems), which can be arbitrarily different across timesteps.} 
2) \emph{Loss sequences with certain trends.} Although many real-world environments may change, they often change by following certain patterns. Therefore, we consider two representative trends, the \emph{linear} trend and the \emph{periodic}  (sinusoid) trend. We sample $a, b\sim\text{Unif}([0, 10]^d)$ and let $\ell_t = (b-a)\frac{t}{T}+a$  for the linear trend and $\ell_t = 5(1+\sin(at+b))$ for the periodic trend. In the experiments, we choose $d=2$.  The average regret (over multiple randomly generated instances) performance is presented in \Cref{fig:online}\footnote{\kzedit{We emphasize that the error bars in the figures are \emph{not} associated with the randomness/variance of the \emph{algorithms/LLM-agents}, but with the randomness/variance of the \emph{generation of environment instances}.}}, 
where we compare GPT-4 with well-known no-regret algorithms, FTRL with entropy regularization {and FTPL with Gaussian perturbations} (with tuned parameters). {It is seen that these pre-trained LLMs can achieve sublinear regret in a large portion of the instances, and have sometimes even lower regret values than baselines}. 

 \safevspace{-5pt}

{\paragraph{Behavioral patterns of LLMs.} To understand how LLMs make decisions at each time step, we provided example outputs of LLMs \emph{reasoning} how they generate their policies in \Cref{sec:LLM_explain}. We find that LLMs tend to use the history of the reward vectors by computing their \emph{sum/average}, and tend to introduce \emph{randomization} in decision-making. These are known to be the keys to achieving no-regret behaviors in online learning \citep{hazan2016introduction,cesa2006prediction}}.  

 \begin{figure*}[!t]
 \safevspace{-16pt}
 \hspace{-9.5pt} 
\centering\includegraphics[width=1.02\textwidth]{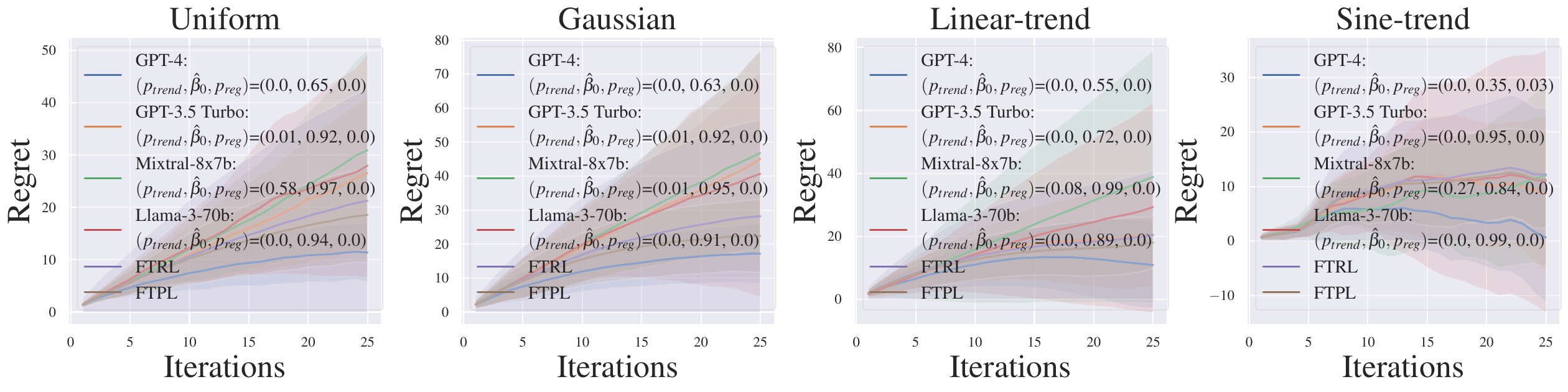}

\caption{Regret of pre-trained LLMs for online learning with full-information feedback. {Notably, both commercial and open-source LLMs {can achieve sublinear regret} as validated by our frameworks and the comparison with FTRL/FTPL, 
 \kzedit{though the performances  of weaker models, GPT-3.5 and open-source ones are worse}.} Interestingly,  the GPT-4 model can even outperform well-known no-regret learning algorithms, FTRL and FTPL. 
}\label{fig:online}

 \end{figure*}

  \begin{figure*}[!t]
 \hspace{-10pt} 
 
\centering\includegraphics[width=1.02\textwidth]{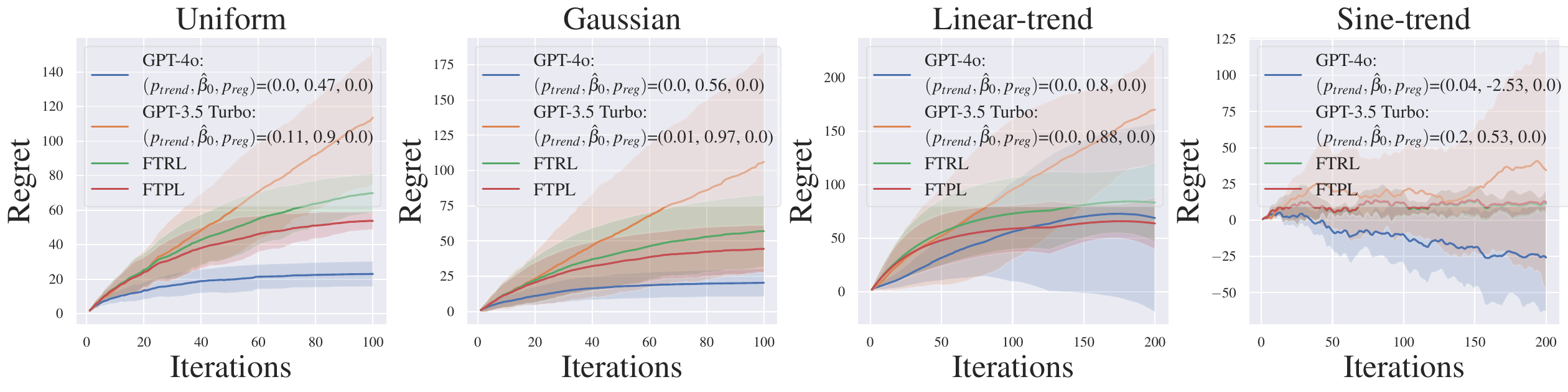}
 
\caption{Regret of pre-trained LLMs for online learning with full-information feedback, with longer horizons of {$T=100$ and $T=200$}. In most cases, the LLMs {can achieve sublinear regret} as validated by our frameworks \kzedit{and the comparison with FTRL/FTPL, though the performances of the weaker model, GPT-3.5, is worse}.}\label{fig:online200}
 \end{figure*}

\paragraph{Longer-horizon results.}  
We also test the robustness and scalability of our empirical findings in more challenging environments. We extend the problem horizon to $T=100$ for the two settings where loss vectors are generated in a stationary way (i.e., \emph{Uniform}  and \emph{Gaussian}), and $T=200$ for the other two non-stationary settings (i.e., \emph{Linear-trend} and \emph{Sine-trend}). Note that since in each round, we need to feed all the previous history to the LLMs, the API costs in fact scale \emph{quadratically} with respect to the horizon $T$.  Therefore, we replace GPT-4 by its cheaper (and more recent) version of GPT-4o. To further scale to even longer-horizon cases with  $T=500$, we \emph{summarize}  the history 
to reduce the prompt length by providing LLMs with the summation of the history loss associated with each action. \kzedit{Similar summary-based input was also used in the concurrent work \cite{krishnamurthy2024can}, where both the \emph{averaged reward} and the \emph{action selection count} of each action were summarized for the (i.i.d.) stochastic bandit setting.} 
The corresponding results are provided in \Cref{fig:online200} and \Cref{table:long-horizon-500}, 
where the LLMs can exhibit no-regret behaviors as validated by our frameworks and the comparison with FTRL/FTPL.

\begin{table}[!h]
    \centering
    \begin{tabular}{lcccc}
        \toprule
        $(p_{trend},\hat{\beta}_o, p_{reg})$ & GPT-4o & FTRL & FTPL \\
        \midrule
        Uniform     & (0.0, 0.85, 0.0) & (0.0, 0.6, 0.0) & (0.0, 0.52, 0.0) \\
        Gaussian    & (0.0, 0.86, 0.0) & (0.0, 0.64, 0.0) & (0.0, 0.68, 0.0) \\
        Linear-trend & (0.02, 0.83, 0.5) & (0.02, 0.76, 0.1) & (0.01, 0.79, 0.0) \\
        Sine-trend  & (0.09, 0.28, 0.0) & (0.01, 0.24, 0.0) & (0.01, 0.26, 0.0) \\
        \bottomrule
    \end{tabular}
    \caption{Longer-horizon ($T=500$). GPT-4o model can still exhibit sublinear regret behaviors 
    as validated by our frameworks and the comparison with FTRL/FTPL.} 
    \label{table:long-horizon-500}

\end{table}

\safevspace{-8pt}

\subsection{Results: Multi-Player Repeated Games}\label{sec:experiments_games}

{\safevspace{-8pt}}

We now consider the setting when multiple  LLMs make online decisions in a \emph{shared} environment repeatedly. 
Specifically, at each round, the loss vectors each agent receives are determined by both her payoff matrix and the strategies of all other agents. Note that the payoff matrix is not directly revealed to the LLM agent, but she has to make decisions in a completely online fashion {based on the payoff vector marginalized by the opponents' strategies} (see \Cref{fig:game-prompt} for an example of the prompt). This is a typical scenario in learning in (repeated) games \citep{cesa2006prediction}. 

{\safevspace{-10pt}}

\arxiv{
\paragraph{{Representative  games.}}
{We first test LLMs on $6$ representative general-sum games (\emph{win-win, prisoner's dilemma, unfair, cyclic, biased,} and \emph{second best}) studied in \cite{robinson2005topology} (cf. \Cref{appendix:six})}. For each type of the game, we conduct $20$ repeated experiments. 

{\safevspace{-10pt}}

\paragraph{{Randomly generated  games.}} 
To further validate the no-regret behaviors of LLMs, we also test on $50$ randomly generated three-player general-sum games, and $50$ randomly generated four-player general-sum games, where each entry of the payoff matrix is sampled randomly from $\operatorname{Unif}([0, 10])$. These are larger and more challenging settings than the structured and representative ones above.  

{\safevspace{-4pt}}

We summarize the experimental results in \Cref{fig:average}, which 
 are similar to the above {in the online setting}: for all types of games, \xyedit{pre-trained LLMs} can achieve sublinear regret, which is often lower than that obtained by FTRL/FTPL for most games. {We provide six instances of three-player general-sum games and six instances of four-player general-sum games in \Cref{fig:333} and \Cref{fig:3333}, respectively. Occasionally, GPT-4 even provides a negative regret value. 

 }

\begin{figure*}[!t]
    \hspace{-7pt}  
    \centering
    \begin{minipage}{0.33\textwidth}
        \centering
        \includegraphics[width=\linewidth]{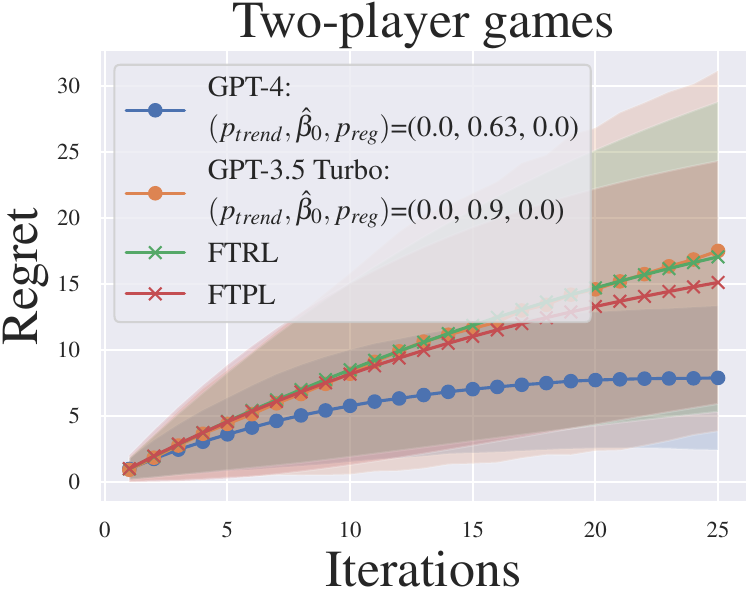}
    \end{minipage}\hspace{-6pt}
    \begin{minipage}{0.33\textwidth}
        \centering
        \includegraphics[width=\linewidth]{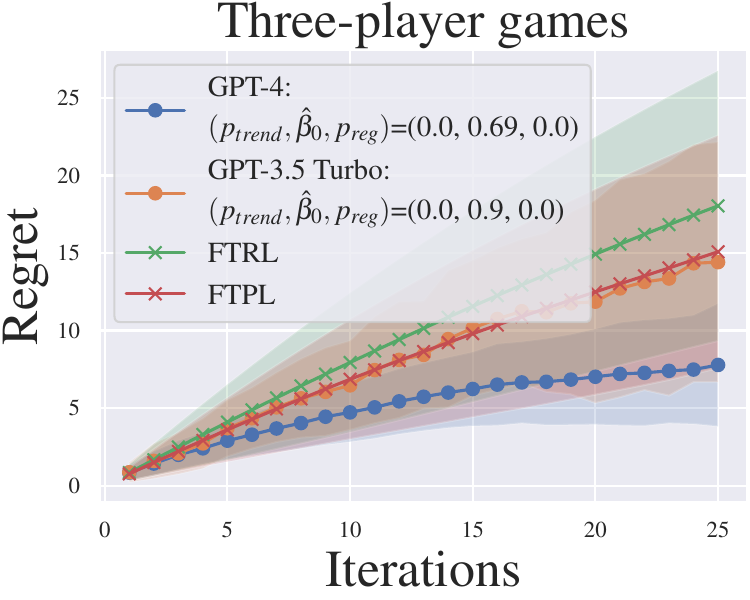}
    \end{minipage}
    \hspace{-9pt}
    \begin{minipage}{0.33\textwidth} 
        \centering
        \includegraphics[width=\linewidth]{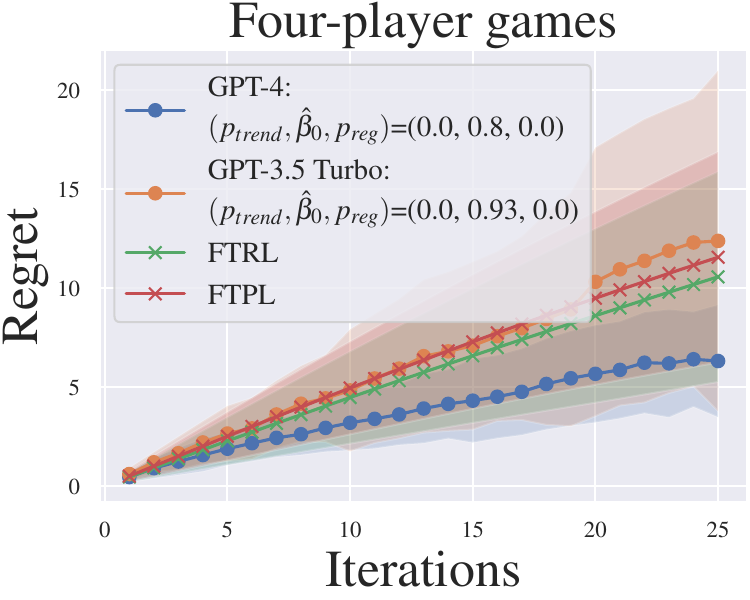}
    \end{minipage}

    \safevspace{-8pt}
\caption{Regret of {pre-trained LLMs}
for repeated games of different sizes. In most cases, both commercial and open-source
LLMs can achieve sublinear regret as validated by our frameworks and the comparison with
FTRL/FTPL. 
{We report the regret of one agent for ease of presentation.}} 
\label{fig:average} 
\safevspace{-8pt}
\end{figure*} 
}

\icml{We introduce the detailed games in \Cref{ssec:environment-game}. The results (\Cref{fig:average}, \ref{fig:333} and \ref{fig:3333}) show that: 1) GPT-4 indeed have no-regret behavior when interacting in repeated games; 2) GPT-4's regrets are comparable with those obtained by the FTRL algorithm, {according to the frameworks in \Cref{sec:framework_test} and the graphic trends.}}

\arxiv{
\begin{figure}[!t]
\centering 
\begin{adjustwidth}{-0.10cm}{-2cm} 
    \begin{subfigure}[b]{0.35\textwidth}
        \includegraphics[height=4cm]{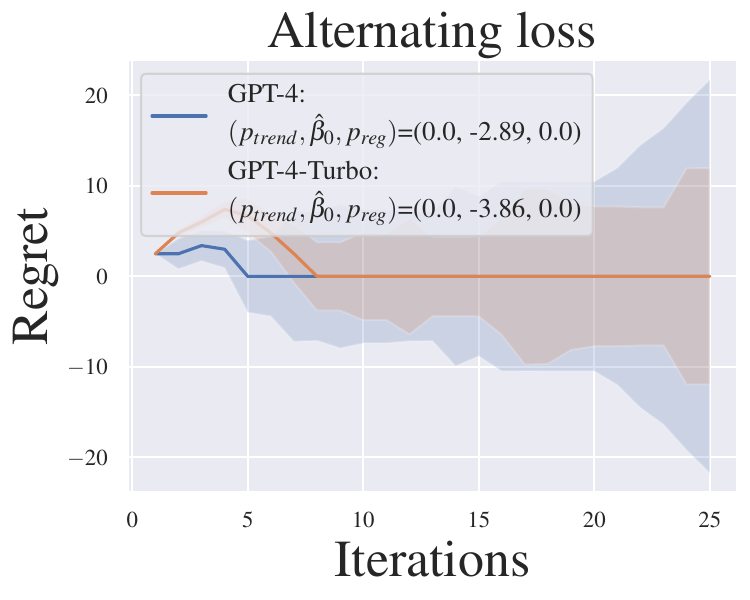}
    \end{subfigure}
    \hspace{-15pt}
    \begin{subfigure}[b]{0.35\textwidth}
        \includegraphics[height=4cm]{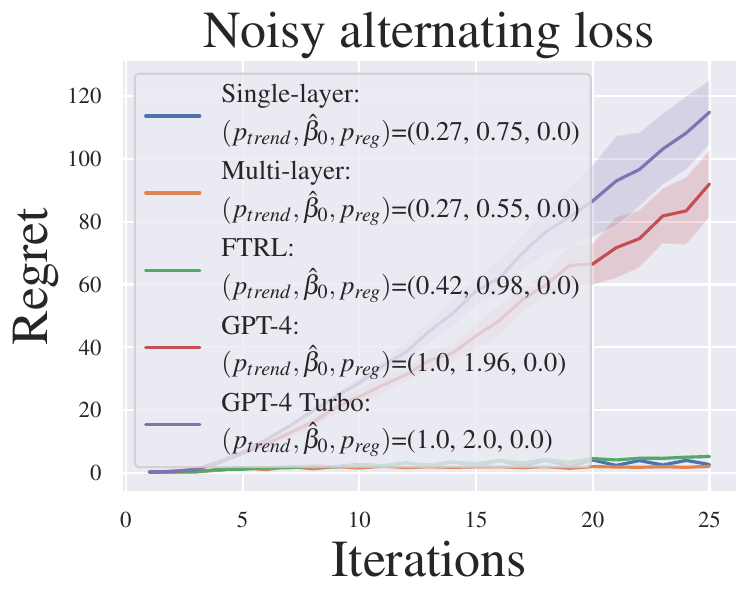}
    \end{subfigure}
   \hspace{-15pt}
    \begin{subfigure}[b]{0.35\textwidth}
        \includegraphics[height=4cm]{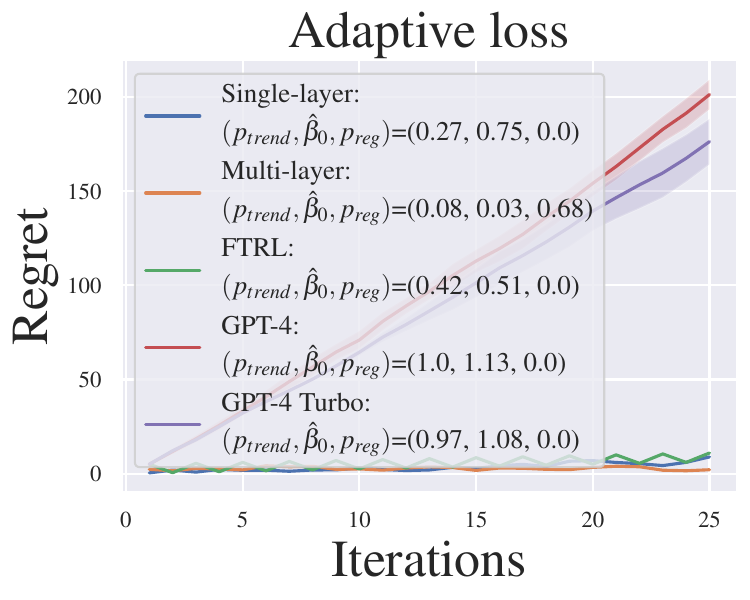}
    \end{subfigure}
    \end{adjustwidth}
\safevspace{-8pt}
\caption{(left) Regret of GPT-4 (Turbo) under the canonical counterexample for FTL  \citep[Chapter 5]{hazan2016introduction}.  (middle, right) Failure of GPT-4 (Turbo)  on two scenarios with regrettable behaviors, while Transformers trained by our new regret-loss {($N=1$)} in \Cref{sec:trained-transformer} {can achieve sublinear regret}.}
    \label{fig:counter}
    \safevspace{-16pt}
\end{figure}  
}

{\safevspace{-9pt}}

\subsection{Pre-Trained LLM Agent{s} Can Still Have Regret}\label{sec:yes_regret_example}

{\safevspace{-6pt}}

The experiments above may suggest the no-regret behaviors of LLMs in online learning and game playing. 
However, is this capability \emph{universal}? We show that the no-regret property can  break for LLM agents if the loss vectors are generated in a more adversarial way.

{\safevspace{-10pt}}

{
\arxiv{
\paragraph{Canonical counterexamples for follow-the-leader.} 
First, we consider two well-known examples that the \emph{follow-the-leader} (FTL) algorithm \citep{shalev2012online} suffers from \emph{linear regret}. 
 
\noindent\textit{\textbf{Example 1:}} $\ell_{1}(1)=5, \ell_{1}(2)=0$ and $\ell_{t}(2-t\%2) = 10, \ell_{t}(1+t\%2) = 0$ for $t\ge 2$ \citep{hazan2016introduction}. 

\noindent\textit{\textbf{Example 2:}}    $\ell_{t}(2-t\%2) = 10, \ell_{t}(1+t\%2) = 0$ for $1 \leq t\leq c$ and $\ell_t (1) = 10$, $\ell_t (2) =0 $ for $c + 1 \leq t \leq  T (= 500)$, for some integer $c$ satisfying $0<c<T$  \citep{feder1992universal}. 

Here, $\%$ denotes the modulo operation. Interestingly, for \textit{\textbf{Example 1}}, GPT-4 agent can easily identify the pattern for the loss sequence that the optimal action \emph{alternates}, thus accurately predicting the loss it will receive and achieving low regret in \Cref{fig:counter}. For \textit{\textbf{Example 2}}, the GPT-4 agent with \emph{raw history} input also provides an impressively lower (negative) regret than FTRL and FTPL  (\Cref{fig:feder}). The GPT-4 agent with \emph{summarized history} input, in contrast,  \kzedit{suffers from much larger  regret than FTRL and FTPL.} {We defer the detailed comparison between using raw history and summarized history to \Cref{fig:feder}, and an explanation of LLMs' behaviors  via predicting the \emph{trend} of the loss instances  to \Cref{ssec:additional_yes_REGRET}.} 
In summary, the GPT-4 agent \kzedit{may predict such worst-case sequences well}, and does not fail in the same way as FTL, which is known to suffer from a lack of randomness in decisions. 

{\safevspace{-6pt}} 

\kzedit{Additionally, the results on \textit{\textbf{Example 2}} also imply that summary-based history input can perform worse than the raw-history-based one in the adversarial setting we consider, while the former was claimed to be the key in succeeding in the i.i.d. stochastic bandit setting \citep{krishnamurthy2024can}. The regret values between these two input types differ significantly, with a $p$-value of $1.2\times 10^{-157}$ under a one-sided independent t-test. 
These results further illustrate the fundamental differences between the settings considered in \cite{krishnamurthy2024can} and ours.} 

{\safevspace{-10pt}}

\paragraph{Noisy alternating loss sequence.} Inspired by the above, we design a new loss sequence that is \emph{similar but less predictable}, by adding some noise to the canonical counterexample. Specifically,
we construct the following (simple) loss sequence with $2$ actions such that ${\ell}_{t}(1+t\%2) = \min(25/t, 10), {\ell}_{t} (2-t\%2)\sim\text{Unif}([9, 10])$  for $t \in [25]$. 

{\safevspace{-10pt}}

\paragraph{Adaptive loss sequence.} We also develop a simpler but more \emph{adaptive} loss sequence that takes the full power of the adversary in our online learning setup. After the GPT-4 agent provides $\pi_t$, we choose ${\ell}_t$ with  ${\ell}_{t}(\argmax_{i}\pi_{ti}) = 10$ and ${\ell}_{t}(3-\argmax_{i}\pi_{ti}) = 0$.

We also report the average regret over $20$ repeated experiments for the latter two settings using GPT-4 and more advanced GPT-4 Turbo in \Cref{fig:counter}, where we cannot reject the hypothesis that GPT-4 (Turbo) has linear regret by either our trend-checking or regression-based framework.}
These observations have thus motivated us to design new approaches to further promote the no-regret behaviors {of the models}{, with additional training,} as to be detailed in \Cref{sec:trained-transformer}. Before it, 
 we first provide some theoretical insights into the observed 
 {sublinear regret behaviors}.

\safevspace{-11pt}
\section{Why Do  Pre-Trained LLM{s} (Not) Have Regret? {A Hypothetical Model and} Some Theoretical Insights} 
\label{sec:theoretical_insights}
\safevspace{-6pt}

{We now provide plausible explanations about the observed no-regret behaviors  of pre-trained LLMs,
which are highly  \emph{hypothetical} by nature, since to the best of our knowledge, the details of pre-training these popular LLMs (e.g., GPT-3.5 {Turbo} and GPT-4), concerning data distribution, training algorithm, etc., have not been revealed. We instead make the explanations based on some existing assumptions 
in {the literature for modeling human behaviors, and} the recent literature on  understanding LLMs and Transformers.}

\safevspace{-3mm}
\arxiv{
\subsection{A (Human)  Decision-Making Model: Quantal Response}\label{sec:human_DM_model_intro}
 \safevspace{-2mm}

A seminal model for human decision-making behaviors is the \emph{quantal response} model,  which assumes that humans are often   {imperfect} decision-makers, and their \emph{bounded rationality} can be modeled through unseen \emph{latent variables} that influence the decision-making process  \citep{mcfadden1976quantal,mckelvey1995quantal}, for which we defer the formal definition and introduction to \Cref{app:quantal-new}. 
In online decision-making, given the \emph{history} information with \emph{multiple} loss vectors, 
we adopt the following generalization of the quantal response model. 
 
\begin{definition}[Quantal response against multiple losses]\label{quantal}
	Given a set of losses \( ({\ell}_i)_{i \in [t]} \), a noise distribution \( \epsilon\sim P_{\text{noise}} \), and $\eta_t>0$, the generalized quantal response against \( ({\ell}_i)_{i \in [t]} \) is defined as
$$  
P_{\text{quantal}}^{\eta_t}\left(a \biggiven  ({\ell}_i)_{i \in [t]}\right):=  P_{\text{quantal}}^{\eta_t}\left( a \bigggiven  \sum_{i=1}^{t} {\ell}_{i}\right)=\PP\left( a \in \argmin_{a'\in\cA}~~z(a')  \right),~\text{where $z=\eta_t \epsilon+\sum_{i=1}^{t} {\ell}_{i} $.}
$$ 
\normalsize
\end{definition} 
\safevspace{-3mm}
{In simpler terms, the generalized quantal response 
is defined as the standard quantal response against the \emph{summation} of the losses. Such a model has been investigated in the learning-in-games and behavioral economics literature (see \Cref{app:quantal-new} for more details). {Such a definition is also aligned with our empirical findings on LLMs' behavioral patterns in \Cref{sec:experiments_online}: i) evaluating the summation/average; ii) introducing randomization in decision-making.}  {To gain more insights into these empirical findings}, we next analyze a case where pre-training under certain assumptions {provably leads to the quantal response behaviors and further yields no-regret guarantees}. 
}
}

\iclr{In \Cref{appendixsec:GQR}, we newly defined quantal response (QR) for the multiple-losses (\Cref{quantal}), which is defined as the standard quantal response against some scaled summation of the losses. We also provided implications of our generalized quantal response from behavioral economics. Moreover, our generalized quantal response is equivalent to the FTPL algorithm \citep{kalai2005efficient} with proper perturbation (\Cref{thm:informal_generalized_quantal_FTPL}).
\safevspace{-8pt}}
\safevspace{-3mm}
\arxiv{\subsection{Case Study: Pre-Training   under Canonical Data Distribution}}
\safevspace{-2mm}
Pre-training of LLMs is predominantly based on  \emph{next-token prediction}. When applying LLMs to sequential decision-making, the model receives the context of the decision-making task as $(x_{1}, x_{2}, \cdots, x_{N})$ and then generates $(x_{N+1}, \cdots, x_{M})$ encoding the \emph{action} for some $N, M\in\NN^+$ and $N<M$, where each $x_i\in\cV$ represents one \emph{natural language token} for $i\in[M]$, and $\cV$ is the finite token set. This process can be conceptualized as \emph{predicting the optimal action} in the form of the next token prediction \citep{yao2022react,shinn2023reflexion, liu2023chain,liu2023reason}. Note that this training procedure
may also appear in the form of \emph{supervised fine-tuning (SFT)} for downstream tasks of decision-making
or question-answering, where optimal action labels may be easier to obtain \citep{cobbe2021training,li2022pre,lewkowycz2022solving}.  
Meanwhile, large models are often {(pre-)trained under several \emph{fixed/stationary} environments} \citep{laskin2022context,lin2023transformers, lee2023supervised,reed2022generalist}, which may limit their ability to handle \emph{arbitrary{/non-stationary}/adversarial} {loss} sequences in online learning. Thus, it is natural to ask: 
\emph{Is it possible to have {no-regret behaviors} emerging as a consequence of this (optimal) action prediction, {under {only} a \emph{fixed} pre-training distribution of the environments?}}  

\begin{wrapfigure}{r}{0.35\textwidth}  
\centering
  \safevspace{-22pt}
  \includegraphics[width=0.35\textwidth]{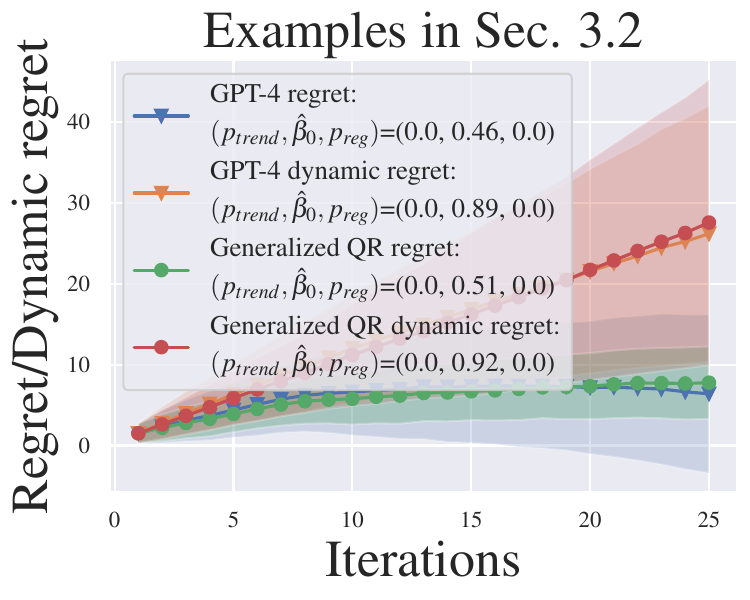}
  \caption{Comparison of GPT-4 with the generalized QR model, 
where the model can very well capture the behavior of the GPT-4 agent for examples in \Cref{sec:experiments_online}.}
  \label{fig:calibration}
\end{wrapfigure}

Here we analyze a standard pre-training objective on a token sequence distribution $x_{1:N_{t+1}}\sim P^{text}_t$ for given $t\in[T]$, 
which is the expected log-likelihood maximization for next-token prediction over $\Theta$, the parameter space of the LLM:  
\small
\safevspace{-6pt} 
\#\label{eq:pretrain}
\max_{\theta\in\Theta}\quad \EE_{x_{1:N_{t+1}}\sim P^{text}_t}\sum_{j=1}^{N_{t+1}}\log\text{LLM}_\theta\left(x_{j}\given x_{1:j-1}\right),
\#
\safevspace{-14pt}
\normalsize 
where we define $\text{LLM}_\theta\left(x_{1}\given x_{1:0}\right)=\text{LLM}_\theta\left(x_{1}\right)$.

For the pre-training distribution, we model it as follows: there exists a latent variable \( z \), representing the loss for the underlying \emph{static} decision-making problem. 
The pre-training dataset, however, only contains \emph{partial observations} $x_{1:N_t}$ (a natural language representation of \( {\ell}_{1:t} \)) of \( z \) due to imperfect data collection, which could be attributed to the fact that $z$ is private to the data-generator (human), representing the actual intention of the human/data-generator. 
Hence, LLM will only be pre-trained with partial and noisy information about \( z \). Meanwhile, we assume that some high-quality action label $x_{N_t+1:N_{t+1}}$ (a natural language representation of \( a \)) with respect to the underlying loss vector $z$ is also available in the dataset{, which could come from user surveys, personal blogs, or data annotation}. We formalize such an assumption:
\begin{assumption}[{Pre-training distribution}]\label{assump:decomp}
Given $T\in\NN^+$, $t\in[T]$, $N_{t+1}\in\NN^+$, there are latent variables $(z, \ell_{1:t})$, $N_1, \cdots, N_t\in[N_{t+1}]$, $N_0=0$, such that 
	$
	\PP(z, \ell_{1:t}, x_{1:N_{t+1}}) =\PP(z, \ell_{1:t})\PP(x_{1:N_t}\given \ell_{1:t})\PP(x_{N_t+1:N_{t+1}}\given z),
	$ 
	and    
	$
P^{text}_t(x_{1:N_{t+1}}):= \PP(x_{1:N_{t+1}})=\int_{z}\int_{\ell_{1:t}} \PP(z, \ell_{1:t}, x_{1:N_{t+1}})d\ell_{1:t} dz.
	$
	Intuitively, tokens $\{x_{N_{i-1}+1:N_{i}}\}_{i\in[t]}$ encode the context, i.e., information for $\ell_{1:t}$, and the user will decode action $a$ from $x_{N_t+1:N_{t+1}}$.
\end{assumption}
\safevspace{-3mm}
To further understand our assumption, we provide an example in \Cref{app:example}, showing how a natural text corpus may satisfy it. {Similar assumptions that suppose the existence of such latent variables in generating the pre-training datasets have also been made recently in  \cite{lee2023supervised, lin2023transformers, liu2023reason}, for understanding the in-context decision-making behaviors of LLMs/Transformers through posterior sampling, for which we defer a detailed comparison to \Cref{app:comparison}. In particular, we show in \Cref{thm:ftpl-main} that if the noise, i.e., $\ell_i-z$ is modeled as Gaussian distributions and $x_{N_t+1:N_{t+1}}$ encodes the optimal action for $z$, the pre-trained LLM provably recovers the prominent human behavior model in \Cref{sec:human_DM_model_intro}, the quantal response model.}

\safevspace{0pt}
\begin{theorem}[Informal: Emergence of no-regret behavior]
    \label{thm:ftpl-main}
    Suppose \Cref{assump:decomp} holds with both the prior distribution of $z$ and the conditional distribution of $\{\ell_i\given z\}_{i\in[t]}$ being Gaussian, and  $x_{N_t+1:N_{t+1}}$ encodes the optimal action for $z$. Then, with the function class of $\text{LLM}_{\theta}$ being expressive enough, and $\theta^\star$ being a  maximizer of \Cref{eq:pretrain}, {the behavior of $\text{LLM}_{\theta^\star}$ follows \Cref{quantal}. Furthermore,} \kzedit{the use of} $\text{LLM}_{\theta^\star}$  can achieve no (dynamic) regret for (non-stationary) online learning with full-information/bandit feedback for arbitrary loss vectors (with bounded variation). 
\end{theorem}

{The formal statement and proof are deferred to \Cref{app:ftpl}. The results show that even when pre-training is conducted solely with loss vectors generated from \emph{stationary} distributions  ($\ell_{1:t}$ are i.i.d. conditioned on $z$), it can still enable the \textit{emergence of no-regret behaviors} in online learning against \emph{potentially adversarial losses}. Key in the proof is a connection of pre-trained LLM models to the online learning algorithm of FTPL.  
Furthermore, \Cref{assump:decomp} can be relaxed to better match the actual LLMs' pre-training data distributions from diverse sources (cf. \Cref{sec:data_relaxation}), and the prior distribution of $z$ could also be replaced by a general distribution (cf.  \Cref{prop:extension}). Finally, we point out its implications for playing games in \Cref{app:game}}.   
 
\safevspace{-8pt}
\paragraph{How well can our hypothetical model class predict actual LLMs' behaviors?}  
{To further verify our theoretically-justified model in  \Cref{thm:ftpl-main}, we propose to \emph{estimate} the parameters of $\{\eta_t\}_{t=0}^{T-1}$ in \Cref{quantal} using the interaction data with actual LLMs, and use the estimated model to predict LLMs' behaviors on some test set. In \Cref{fig:calibration}, we show the averaged regret for the LLMs and our estimated model, {where the generalized quantal response can \emph{very well capture} the behavior of the LLM agent for all problem instances in \Cref{sec:experiments_online}, on which the LLMs {oftentimes achieve sublinear regret}, justifying the applicability of our hypothetical model and assumptions. 
}}

Finally, we acknowledge that for existing pre-trained LLMs like GPT-4, {the canonical assumptions above, though may be further relaxed (cf. \Cref{app:relaxtion}), may not hold in general. More importantly, the \emph{supervision labels}, i.e., the optimal action given $z$, may be {sometimes imperfect or unavailable} in  the dataset. These caveats motivate the study in our next section.}

{\safevspace{-10pt}}
\section{Provably Promoting No-Regret Behavior by a New 
Loss}  
\label{sec:trained-transformer}

{\safevspace{-8pt}}

In light of the observations in \Cref{sec:empirical_insights}, we ask the question:    

{\safevspace{-8pt}}

\begin{center}
    \emph{Is there a way to enhance the no-regret property of the models \textbf{without} (optimal) action labels?}  
\end{center}

{\safevspace{-8pt}}

To address this question, we propose to train models with a new \emph{unsupervised learning} loss that naturally provides no-regret behaviors. \kzedit{We will particularly focus on the \emph{Transformer}  architecture \citep{vaswani2017attention} under this new loss, since the Transformer architecture is the backbone of most modern LLMs.} 

{\safevspace{-8pt}}
\subsection{A New Unsupervised Training Loss: \emph{Regret-Loss}} 
\label{ssec:regret-loss}
{\safevspace{-6pt}}

Intuitively, our new training loss is designed to enforce the trained models to minimize regret under an arbitrary sequence of loss vectors. Specifically, 
we define the training loss as 
\safevspace{-2pt}
\begin{equation}   {\cL(\theta)}:=\max_{\ell_1, \dots, \ell_T} ~~~\text{Regret}_{\text{LLM}_\theta}\left((\ell_t)_{t\in[T]}\right) \label{eqn:max-loss}
\end{equation}
where $\norm{\ell_t}_\infty \leq B$ for $t \in [T]$. 
As discussed in   \cite{kirschner2023regret}, directly minimizing the max regret can be  {computationally} challenging, except for superficially simple problems.   
Moreover, 
\Cref{eqn:max-loss} is not necessarily differentiable with respect to the parameter $\theta$, if it does not satisfy the condition of Danskin's Theorem \citep{danskin1966theory}; or even if it is differentiable (i.e., the maximizer of  $(\ell_t)_{t \in [T]}$ is unique), computation of derivatives can be challenging {since we need to calculate $\argmax_{(\ell_t)_{t \in [T]}} \text{Regret}_{\text{LLM}_\theta}((\ell_t)_{t\in[T]})$ while there is an $\inf$ in the definition of regret.}  
Therefore, we provide a general class of surrogate losses to approximate \Cref{eqn:max-loss}: 
\safevspace{-4pt}
\begin{align}
    &\cL(\theta, k, N):= 
\EE\left[\frac{\sum_{j \in [N]} h(\text{Regret}_{\text{LLM}_\theta}((\ell_t^{(j)})_{t\in[T]})) f(\text{Regret}_{\text{LLM}_\theta}((\ell_t^{(j)})_{t\in[T]}), k)}{\sum_{j \in [N]}f(\text{Regret}_{\text{LLM}_\theta}((\ell_t^{(j)})_{t\in[T]}), k)}\right], \label{eqn:regret-loss}
\end{align}
where
$k \in \NN^+$, $N\in\NN^+$, $h: \RR \to \RR^+$ is a continuous function, with  continuous derivative $h'$, and $f(\cdot, k): \RR \to \RR^+$ is a continuous function for each $k \in \NN^+$, satisfying  $\lim_{k \to \infty} \frac{f(R_1, k)}{f(R_2, k)} = \infty \cdot \mathbbm{1}(R_1 > R_2) + \mathbbm{1}(R_1 = R_2)$, {where we use the convention of $\infty \cdot 0= 0$}. These conditions on $h,f$ will be assumed throughout the paper. 
Examples of such an $f$ include $f(x, k) = x^{k}$ and $\exp(kx)$.
We will sample $N$ trajectories of loss sequences $(\ell_t^{(j)})_{t\in [T], j \in [N]}$ from some continuous 
probability distribution supported on $[-B, B]^{T \times N}$ (without other additional statistical assumptions),  {and the expectation in \Cref{eqn:regret-loss} is thus taken with respect to this distribution}. 
In \Cref{appendix:mathrigor}, we prove  that {under certain regularity conditions of $f$ and $h$,} we have 
$$\lim_{N, k \to \infty}\cL(\theta, k, N) = h\left(\max_{\ell_1, \dots, \ell_T} \text{Regret}_{\text{LLM}_\theta}((\ell_t)_{t\in[T]})\right),$$
and the uniform convergence of $\cL(\theta, k, N)$:  
$\lim\limits_{N, k \to \infty} \sup_{\theta \in \Theta}\Big|h\left(\max\limits_{\ell_1, \dots, \ell_T} \text{Regret}_{\text{LLM}_\theta}((\ell_t)_{t\in[T]})\right) - \cL(\theta, k, N)\Big| = 0,$  
where $\Theta$ is {a compact} set of the model   {parameters}. Hence, one can expect that minimizing the loss function in  \Cref{eqn:regret-loss} with large enough $k$ and $N$ may promote the trained \kzedit{models} to have a small regret value.   We will hereafter refer to \Cref{eqn:regret-loss} as the \textit{regret-loss}. 

\safevspace{-1mm}
\subsection{
Generalization and Regret  Guarantees of Regret-Loss Minimization} \label{ssec:general-case} 

{\safevspace{-6pt}}
{We first establish a \emph{statistical} guarantee under general parameterizations of $\text{LLM}_\theta$ that are Lipschitz with respect to $\theta$, including the Transformer-based models as used in GPT-4 and most existing LLMs ({see} \Cref{prop:transformer-lipschitz} {for an example with a formal statement}).  
This guarantee focuses on their \emph{generalization ability} when trained to minimize the empirical regret loss {({cf.} \Cref{eqn:empirical_regret_loss}), denoted as $\hat{\cL}(\theta, k, N, N_T)$, by replacing the expectation $\EE$ in \Cref{eqn:regret-loss} {with the empirical mean using  $N_T$ samples.}} We  denote $\hat{\theta}_{k,N, N_T}\in\argmin_{\theta\in\Theta}~\hat{\cL}(\theta,k,N, N_T)$, and present the generalization guarantee in   \Cref{thm:generalization-error}.} Thanks to the uniform convergence of $\cL(\theta, k, N)$ (cf. \Cref{appendix:mathrigor}), we further obtain  the following theorem on the regret guarantee of $\text{LLM}_{\hat{\theta}_{k, N,N_T}}$:  

\begin{restatable}{theorem}{corregret}\label{coro:regret_guarantee}
\emph{(Regret){\bf.}}
Suppose\footnote{Note that these conditions on $h,f$ are in addition to those specified after \Cref{eqn:regret-loss}.} {for any $k\in \NN^+$,} $h,f(\cdot, k)$ are non-decreasing,  and $\log f$ is a supermodular function {(i.e., $ \log f(R_1 , k_1) - \log f(R_1, k_2) \geq \log f(R_2, k_1) - \log f(R_2, k_2)$ for $R_1 \geq R_2$ and $k_1 \geq k_2$)}. {Then, with high probability,  
we have}  
\begin{align*} 
    &h\left(\lim_{N \to \infty}\lim_{k \to \infty} \max_{\norm{\ell_t}_\infty \leq B} \emph{Regret}_{\emph{LLM}_{\hat{\theta}_{k, N,N_T}}} \left((\ell_t)_{t\in[T]}\right)\right)  \leq  h\left(\inf_{{\theta \in \Theta}} \max_{\norm{\ell_t}_\infty \leq B} \emph{Regret}_{\emph{LLM}_{{\theta}}} \left((\ell_t)_{t\in[T]}\right)\right) +   \tilde{\cO}\left(\sqrt{\frac{d_\theta}{{N_T}}}\right).  
\end{align*}
\normalsize
\end{restatable}

We defer the proof of the theorem to \Cref{appendix:proof-generalization-error}.  
Therefore, if additionally, the model parameterization (e.g., Transformers) can \emph{realize} a no-regret algorithm {(as to be shown next)}, 
then \Cref{coro:regret_guarantee} means that with a large enough 
$N_T$, the learned $\text{LLM}_{\hat{\theta}_{k, N,N_T}}$ becomes a \emph{no-regret} learner, i.e., $\text{Regret}_{\text{LLM}_{\hat{\theta}_{k, N,N_T}}}\left((\ell_t)_{t\in[T]}\right)=o(T)$. Finally, as a consequence, it is folklore that when multiple such LLMs interact, a coarse correlated equilibrium will emerge in the long-term (cf. \Cref{cor:noregretcce}).

{\safevspace{-6pt}}

\subsection{Regret-Loss Trained Transformers Can be 
Online Learning Algorithms} 
\label{ssec:opt}
{\safevspace{-6pt}}

Despite the generality of the previous results, one cannot use an \emph{infinitely large} $N$ and $k$ in practice. Hence, we next provide results when $N$ is finite, for the architecture of  \emph{Transformer} models   \citep{vaswani2017attention}. 
We focus on single-layer (linear) self-attention models, as in most recent theoretical studies of Transformers \citep{ahn2023transformers, zhang2023trained, mahankali2023one}, and 
{$N =1$. {Note that in this case, the choice of $f$ (and thus $k$) is not relevant.} Thus, throughout this subsection, we drop superscript $(j)$ in \Cref{eqn:regret-loss}}. {We sample $\ell_t$ for $t \in [T]$ as  realizations of some random variable $Z$}, where we assume that $Z$ is {symmetric about zero}, and $\text{Var}(Z) = \Sigma\succ 0$.  

{\safevspace{-6pt}}
We consider the single-layer {\emph{linear}} self-attention model as follows{, for which we can show that the \emph{global optimizer} of our regret-loss can automatically lead to a no-regret learning algorithm}: 
\begin{align} 
    g(Z_t; V, K, Q, v_c, k_c, q_c)= \sum_{i=1}^{t} (V\ell_i + v_c) \left((K\ell_i + k_c)^\intercal \cdot (Qc + q_c)\right).  \label{eqn:unconstrained-single-linear-transformer}
\end{align}
\safevspace{-3.3mm}
\begin{restatable}{theorem}{linear}
\label{thm:single-layer-regret-minimizer} 
{Consider the policy space $\Pi = B(0, R_{\Pi},\norm{\cdot})$ for some $R_{\Pi}>0$.} 
The configuration {of a single-layer linear self-attention model {in \Cref{eqn:unconstrained-single-linear-transformer}} $(V,K,Q,v_c, k_c, q_c)$ such} that $K^\intercal (Qc + q_c) = v_c = \pmb{0}_d$ and  $V =  -2R_{\Pi}\Sigma^{-1} \EE \left( \norm{\sum_{t=1}^T \ell_t} \ell_1\ell_2^\intercal \right) \Sigma^{-1}$ is {a global optimal solution} of \Cref{eqn:regret-loss} with $N=1$, $h(x) = x^2$. 
Moreover, every global optimal configuration of \Cref{eqn:regret-loss} within the parameterization class of \Cref{eqn:unconstrained-single-linear-transformer} has the same output function $g$. Additionally, if $\Sigma$ is a diagonal matrix, {then}   plugging any global optimal configuration into \Cref{eqn:unconstrained-single-linear-transformer}, {and projecting the output with $\texttt{Proj}_{\Pi, \norm{\cdot}}$}  
\xyedit{is equivalent to FTRL with an $L_2$-regularizer.} 
\end{restatable}  
\safevspace{-2mm}

\Cref{thm:single-layer-regret-minimizer} {not only} shows the \emph{capacity} of self-attention {models}: it can {realize} {online learning} algorithms, {but also shows, more importantly, that minimizing our new regret-loss may \emph{automatically} produce it. 
In particular, one does not need to {hard-code}  the parameters of the Transformer to implement no-regret algorithms.} Under single-layer self-attention parameterization (with softmax), we can also show that a \emph{stationary point} of the loss function (\Cref{eqn:regret-loss}) can lead to FTRL (cf.  \Cref{appendix:detailed-softmax}). Some potential generalizations of the results are also discussed in \Cref{appendix:conj6}.

{\safevspace{-4pt}}

\safevspace{-1mm}
\subsection{Experimental Results for {Regret-Loss Trained Transformers}} 
\label{ssec:experiemnt}
 {\safevspace{-6pt}}

{We now provide {experimental results} for minimizing our \emph{regret-loss} with the Transformer models{, and evaluate in the following environments}: 1) randomly-generated loss sequences (\Cref{fig:single-layer-ftrl1}); 2) loss sequences with certain trends (\Cref{fig:trend-att});  3) repeated games (\Cref{fig:game}); and 4) counterexamples for pre-trained LLMs to be regrettable  (\Cref{fig:counter}). Training setup can be found in \Cref{appendix:training-detail}.} \conf{We provide detailed experimental settings in \Cref{detailed-experiment-setting-ch5}.} We also provide an ablation study for optimizing \Cref{eqn:regret-loss}  in \Cref{appendix:nvariation}. 
 
\paragraph{Randomly generated loss sequences.} 
 We use the same loss vectors as those in \Cref{sec:experiments_online} for randomly generated loss functions, and compare the results with that using GPT-4. The results show that with regret-loss, both the trained single-layer self-attention model and  {the} trained Transformers with multi-layer self-attention structures can achieve comparable regrets as FTRL and GPT-4. The results can be found in \Cref{fig:single-layer-ftrl1}.

\begin{figure}[!t]
   \centering
    \begin{subfigure}[b]{0.36\textwidth}
        \includegraphics[width = \linewidth]{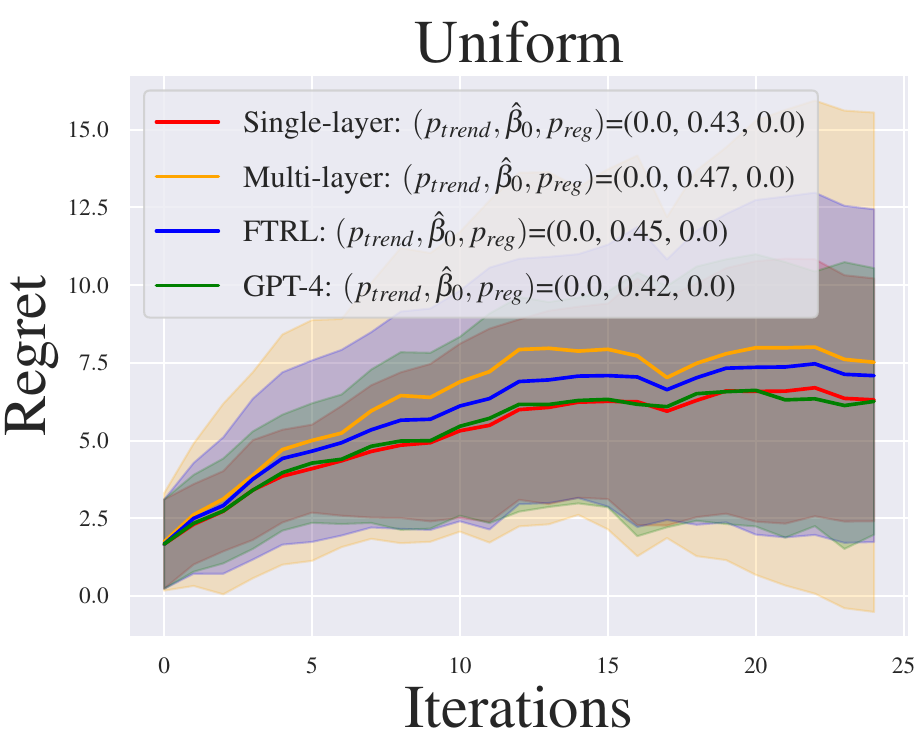}
    \end{subfigure}
    \hspace{60pt}
     \begin{subfigure}[b]{0.36\textwidth}
        \includegraphics[width = \linewidth]{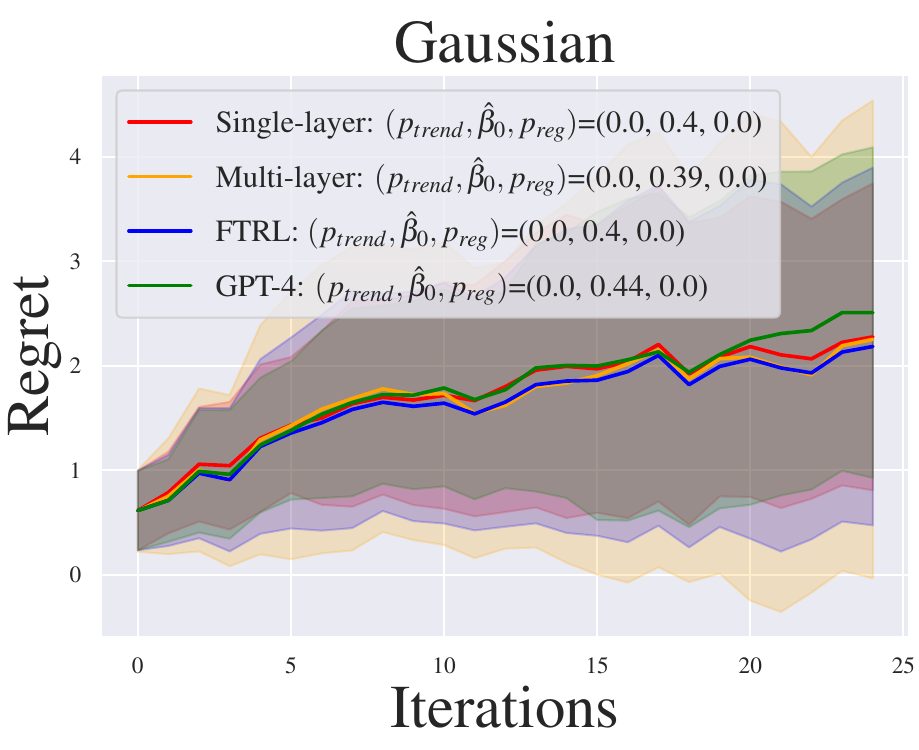}
    \end{subfigure}   
  \caption{Regret performance for the randomly generated loss sequences that are generated by Gaussian with truncation and uniform distribution. 
   No-regret behaviors of single-layer and multi-layer self-attention models are validated by both of our frameworks (low $p$-values and  $\hat{\beta}_0<1$).
  }
          \label{fig:single-layer-ftrl1}
\end{figure}

\begin{figure}[!t]
    \centering
    \hspace{10pt}\begin{subfigure}[b]{0.36\textwidth}
        \includegraphics[width = \linewidth]{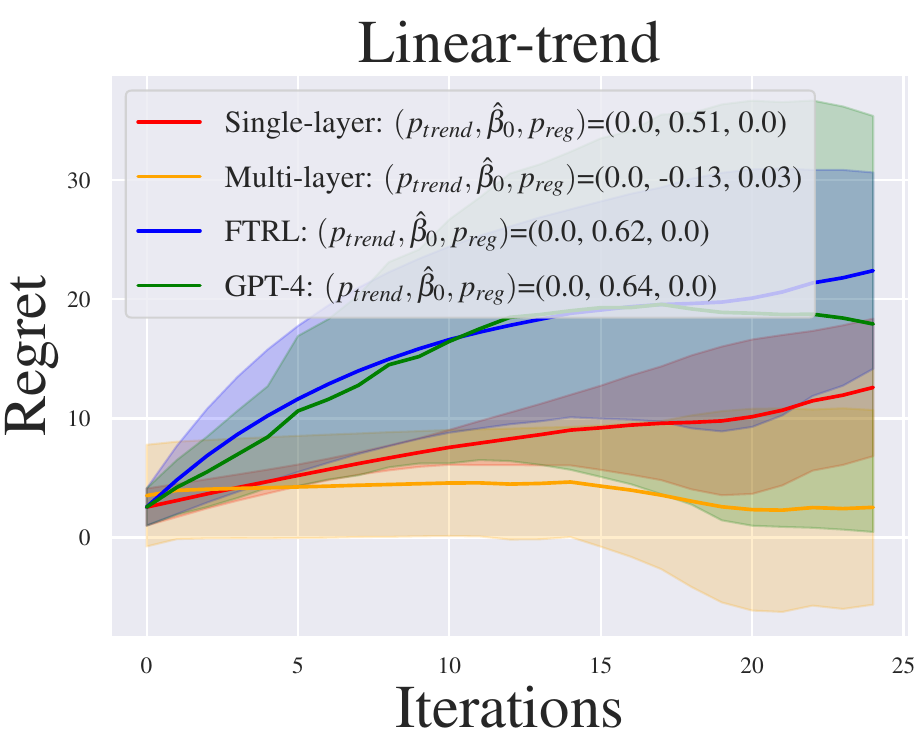}
    \end{subfigure}
    \hspace{56pt}
     \begin{subfigure}[b]{0.36\textwidth}
        \includegraphics[width = \linewidth]{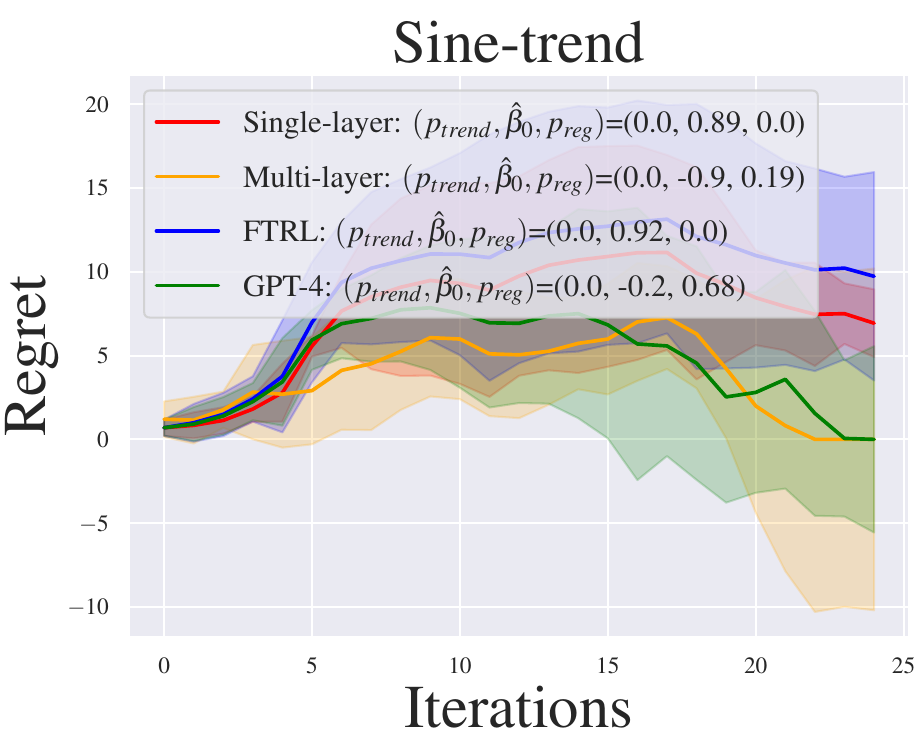}
    \end{subfigure}   
  \caption{Regret performance for the randomly generated loss sequences that are generated by linear-trend and sine-trend.  
   No-regret behaviors of single-layer and multi-layer self-attention models are validated by both of our frameworks (low $p$-values and  $\hat{\beta}_0<1$).}
    \label{fig:trend-att}
\end{figure}

\paragraph{Loss sequences with certain trends.} 
{We investigate the case where the loss sequences have predictable trends such as linear-trend or sine-trend. One might expect that the performance of the trained Transformer would surpass the performance of traditional no-regret learning algorithms such as FTRL, since they may not be an optimal algorithm for the loss sequence with a predictable trend.} We modify the training distribution by changing the distribution of random variable \( Z \) {(which generates the loss vectors $\ell_t$)} to follow two kinds of trends: linear and sine functions. The results, as illustrated in \Cref{fig:trend-att}, show that the trained single-layer self-attention model and the trained Transformer with multi-layer self-attention structures with regret-loss outperformed GPT-4 and FTRL in terms of regret, when the loss sequence is a linear trend. Similarly, \Cref{fig:trend-att} shows that the trained Transformer with multi-layer self-attention structures with regret-loss is comparable to GPT-4 and outperformed FTRL in terms of regret, when the loss sequence is a sine-trend. {Note that the training dataset does not contain the sequence of losses. Nonetheless, by focusing on the overall trend during training, we can attain performance that is either superior to or on par with that of FTRL and GPT-4.}}

\paragraph{Repeated games.}

We then investigate the case of multi-player repeated games. We study 2x2, 3x3x3, 3x3x3x3 games, where each entry of the payoff matrix is sampled randomly from $\operatorname{Unif}([0, 10])$. The results, as illustrated in \Cref{fig:game}, show that the trained single-layer self-attention model and the trained Transformer with multi-layer self-attention structures with regret-loss have a similar performance as that of  FTRL. {However, GPT-4 still outperforms the trained single-layer self-attention model and the trained Transformer with multi-layer self-attention structures in terms of regret. Since {for repeated games (in which the environment faced by the agent can be less adversarial than that in the online setting),} there might be a better algorithm than FTRL (see e.g., \cite{daskalakis2021near}), {while}  our self-attention models have a similar structure as FTRL (\Cref{thm:single-softmax-layer-regret-minimizer} or \Cref{thm:single-layer-regret-minimizer}). {Also, in practical training (with the empirical loss in \Cref{eqn:empirical_regret_loss}), we possibly did not find the exact global minimum or stationary point of the \emph{expected} loss in \Cref{eqn:regret-loss}. Hence, it is possible that GPT-4 may have lower regret than our trained models with the regret-loss.}}

\begin{figure}
\hspace{-12pt}
       \centering
\includegraphics[width=1.02\linewidth]{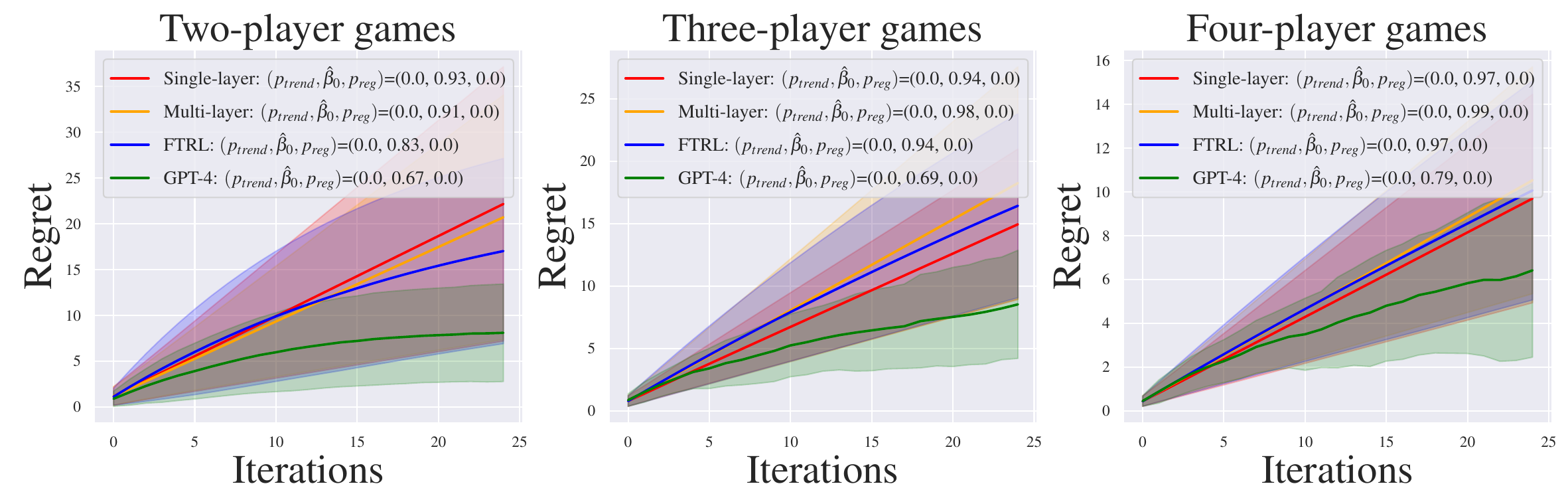}
        \caption{Regret performance for the game with two players, three players, and four players general-sum games. No-regret behaviors of single-layer and multi-layer self-attention models are validated by both of our frameworks (low $p$-values and  $\hat{\beta}_0<1$).}
        \label{fig:game}
\end{figure}

\paragraph{Two scenarios that caused regrettable behaviors of GPT-4.}

 Finally, we investigate the cases that have caused GPT-4 to have regrettable performance in \Cref{sec:experiments_online}. 
 The results, which can be found in \Cref{fig:counter}, show that both the trained single-layer self-attention model and the trained  Transformer with regret-loss can achieve comparable no-regret performance as FTRL, and outperforms that of GPT-4.  This validates that our new unsupervised training loss can address the regrettable cases, as our theory in  \Cref{ssec:general-case} and \ref{ssec:opt} has predicted. 

{\color{black}
\paragraph{Remark on performance discrepancy between single-agent and multi-agent settings.} 
In certain cases above, LLMs can outperform FTRL/FTPL algorithms and single/multi-layer models. This phenomenon is primarily observed when the loss sequence exhibits discernible trends, as seen in the single-agent setting. In \Cref{sec:yes_regret_example}, we explored this behavior using canonical counterexamples for the follow-the-leader algorithm. Specifically, when the loss sequences display obvious or predictable patterns, LLMs can effectively infer the next loss vector based on historical data, enabling near-optimal decisions. This phenomenon can be further formalized through the lens of in-context learning. Conversely, FTRL/FTPL algorithms, constrained by their update rules, tend to produce near-uniform policies in such cases, as do single/multi-layer Transformer models. In \Cref{ssec:additional_yes_REGRET}, we provide ablation studies to support these observations, demonstrating that LLMs leverage trends in the loss sequences by comparing their performance when provided with raw versus summarized historical data. When the loss sequences are summarized (e.g., through aggregation), the resulting loss vectors no longer reflect the trend, leading to significantly diminished performance by the LLMs.  
In contrast, in multi-agent or game settings, the loss sequence trends depend on the behavior of other agents, rendering them inherently less predictable as all agents continually update their behavior policies. This increased unpredictability likely accounts for the comparable or inferior performance of LLMs relative to the FTRL/FTPL algorithms and  single/multi-agent-trained Transformer models in such settings.

\section{Concluding Remarks}\label{sec:conclude}

In this paper, we studied the online decision-making and strategic behaviors of LLMs quantitatively, through the metric of regret. We first examined the no-regret behaviors of several representative pre-trained LLMs in benchmark settings in online learning and game theory. 
We then provided some theoretical insights into the  sublinear regret growth  in the cases we observed, by connecting pre-trained LLMs to the follow-the-perturbed-leader algorithm in online learning, under certain assumptions.  We also identified (simple) cases where pre-trained LLMs failed to be no-regret, and thus proposed a new unsupervised training loss, the \emph{regret-loss}, to provably promote the no-regret behavior of Transformers without the labels of (optimal) actions. We established both experimental and theoretical evidence for the effectiveness of our regret-loss. 

As a first attempt to  rigorously  understand  the online and strategic decision-making behaviors of LLMs through the metric of regret, we  identify the following limitations and list some potential directions for future research: 

\begin{itemize}
    \item There are more than one definitions of (dynamic-)regret in the online learning literature, and we mainly focused on the so-called \emph{external-regret} in the literature. It would be interesting to incorporate other regret metrics, e.g., swap-regret \citep{blum2007external}, which may lead to stronger equilibrium notions in playing repeated games, {and policy regret \citep{arora2012online}, which accounts for adaptive adversaries, into understanding LLM agents' decision-making behaviors.} 
    \item On a related note, no-regret behaviors can sometimes lead to better outcomes in terms of social efficiency \citep{blum2008regret,roughgarden2015intrinsic,nekipelov2015econometrics}. It would thus be interesting to further evaluate the efficiency of no-regret LLM agents in these scenarios, as well as identify new prompts and training losses for LLMs to promote the efficiency of the outcomes. 
    \item {\color{black} Our  \Cref{thm:ftpl-main} towards explaining  why LLMs achieved  sublinear regret is highly hypothetical. Given that LLMs are such complex, random, and black-box systems, there are definitely  behaviors that our \Cref{thm:ftpl-main} cannot fully capture, and there do exist other possible explanations. For example, an alternative in-context-learning-based explanation may be used to account for the \emph{improved} performance of LLMs on specific loss sequences \emph{with trends}  (see \Cref{ssec:additional_yes_REGRET} for more details).  
    However, this explanation may not generalize to the loss sequences \emph{without} obvious trends, complementing our explanations based on the connection to no-regret learning algorithms, which apply to general loss sequences.   
       Hence, it would be interesting to propose and validate other hypotheses for the observed behaviors of LLMs.}  
     \item  Our new regret-loss has exhibited promises in our experiments for training modest-scale Transformers. One limitation is that we have not trained on other larger-scale \emph{language} models and foundation models for decision-making, which is an important ongoing effort.   
    \item To evaluate the performance quantitatively, we focused on online learning and games with \emph{numeric valued}  payoffs. It would be interesting to connect our no-regret-based and game-theoretic framework with existing multi-LLM systems, e.g., debate, collaborative problem-solving, and human/social behavior simulation, with potentially new notions of regret (defined in different spaces) as performance metrics. 
\end{itemize}

\section*{Acknowledgement}
The authors thank Constantinos Daskalakis, Kristian Georgiev, Noah Golowich, Dingwen Kong, Akshay Krishnamurthy, and Aleksander Madry for their helpful feedback. In particular, the authors thank Dingwen Kong for discussing the truncation idea in proving \Cref{lem:asconvergence}, and thank Akshay Krishnamurthy for bringing up a concurrent work that inspired our new experiments for the stochastic bandit setting that strengthened our paper. X.L. and K.Z. acknowledge the support from the U.S. Army Research Laboratory and the U.S. Army Research Office under grant number W911NF-24-1-0085 and NSF CAREER Award-2443704.

\bibliographystyle{iclr2025_conference}
\bibliography{main}

\appendix
\newpage
~\\
\centerline{{\fontsize{13.5}{13.5}\selectfont \textbf{Supplementary Materials for}}}

\safevspace{6pt}
\centerline{\fontsize{12}{12}\selectfont \textbf{
	  ``Do LLM Agents Have Regret? A Case Study in Online Learning and Games''}}

\tableofcontents
\clearpage

\icml{\section{Additional Notation} 
We use $\NN$ and $\NN^+$ 
to denote the sets of non-negative and positive integers, respectively. 
For a finite set $\cS$, we use $\Delta(\cS)$ to denote the simplex over $\cS$. For $d\in \NN^+$, we define $[d]:= \{1, 2, \dots, d\}$. {For two vectors $x,y\in\RR^d$, we use $\la x,y\rangle$ to denote the inner product of $x$ and $y$. }We define $\pmb{0}_d$ and $\pmb{1}_d$ as a $d$-dimensional zero or one vector{, {and $\pmb{O}_{d\times d}$ and $I_{d \times d}$ as a $d \times d$-dimensional zero matrix and  identity matrix, respectively. We omit $d$ when it is clear from the context.}} We define $e_i$ as a unit vector {(with proper dimension) }whose  $i$-th coordinate equal to $1$. For $p \in \RR^d, R>0$ and $C \subseteq \RR^d$ is a convex set, define $B(p, R, \norm{\cdot}):= \{x \in \RR^d \mid \norm{x - p} \leq R \}$,  $\texttt{Proj}_{C, \norm{\cdot}}(p) = \argmin_{x \in C} \norm{x - p}${ (which is well defined as $C$ is a convex set), and $\texttt{clip}_R(x):=[ \texttt{Proj}_{B(0, R, \norm{\cdot}_2), \norm{\cdot}_2}(x_i)]_{i \in [d]}$.} Define $\texttt{Softmax}(x):= \left(\frac{e^{x_i}}{\sum_{i \in [d]}e^{x_i}}\right)_{i \in [d]}${ and $\texttt{ReLU}(x) = \max(0, x)$ for $x \in \RR^d$. For $A \in \RR^{m \times n}$ {with $A_i$ denoting its $i$-th column}, we define $\norm{A}_\texttt{op}:= \max_{\norm{x}_2 \leq 1} \norm{Ax}_2$, $\norm{A}_{2, \infty}:= \sup_{i \in [n]} \norm{A_i}_2$,
$\norm{A}_{F}$
as the Frobenius norm, and    $A_{-1} := A_n$ to denote the last column vector of $A$. } We define $\RR^+:= \{x \mid x\geq 0 \}$. {For a set $\Pi$, define $\text{diam}(\Pi, \norm{\cdot}):= \sup_{\pi_1, \pi_2 \in \Pi} \norm{\pi_1 - \pi_2}$.}  {We define $\mathbbm{1}(\cE):=1$ if $\cE$ is true, and $\mathbbm{1}(\cE):=0$ otherwise.} {For a random variable sequence $(X_n)_{n \in \NN}$ and random  variables $X, Y$, we denote $F_X$ as the cumulative distribution function of a random variable $X$, $X_n \overset{p}{\to} X$  if $\forall \epsilon > 0, \lim_{n \to \infty} \PP(|X_n - X| > \epsilon) = 0$, $X_n \overset{d}{\to} X$  if $\lim_{n \to \infty} F_{X_n}(x) = F_X(x)$ for all $x$ where $F_X(x)$ is continuous, $X \overset{d}{=} Y$ if $F_X(x) = F_Y(x)$ for all $x$, $X_n \overset{a.s.}{\to} X$ if $\PP(\lim_{n \to \infty} X_n = X) = 1$, and $\operatorname{esssup}(X) := \inf \{M \in \mathbb{R} : \PP(X > M) = 0\}$.} {For a random variable $X$, we use $\text{supp}(X)$ to denote its support.} {{For functions $f,g:\RR\to \RR$,} we define $g(x) = \cO(f(x))$ if there exist $x_0, M<\infty$ such that $|g(x)| \leq M|f(x)|$ for {all} $x > x_0$. {We use $f'$ to denote the derivative of $f$.} Let $F: \Omega\rightarrow \RR$ be a continuously-differentiable, strictly convex function defined on a convex set $\Omega$. The Bregman divergence associated with $F$ for points $p$, $q$ is defined as $D_F(p, q):=F(p)-F(q)-\langle \nabla F(q), p-q\rangle$. For a sequence $(\ell_{t})_{t\in[T]}$ for some $T\in\NN^+$, we define $\ell_{a:b}:=(\ell_a, \cdots, \ell_b)$ for $1\le a\le b\le T$. If $a>b$, we define  $\ell_{a:b}=\emptyset$.}}

\section{Comparison with Concurrent Work  \cite{krishnamurthy2024can}}

After submitting the first version of our manuscript, we were aware of a concurrent work  \cite{krishnamurthy2024can}, which considered using LLMs to solve multi-arm \emph{stochastic}  bandit problems entirely in-context, 
with a focus on the \emph{exploration} behaviors of LLMs. Specifically, \cite{krishnamurthy2024can} claimed that LLMs may not show robust exploratory behaviors under a variety of prompt configurations, although there does exist some successful prompt configuration  that enabled satisfactory exploratory behaviors. We here provide a detailed comparison between  \cite{krishnamurthy2024can} and the first experimental part of our
paper, i.e., \Cref{sec:empirical_insights} and related appendices. 

\begin{itemize}
    \item \textbf{(Focused settings).} We mainly considered the \emph{full-information} online learning setting with potentially \emph{adversarial} loss vectors, as well as the multi-agent \emph{repeated-game} setting. In contrast, \cite{krishnamurthy2024can} focused on the \emph{stochastic}  setting with 
      \emph{bandit}  feedback,  where the loss vectors at different rounds are drawn i.i.d. from a \emph{fixed}  distribution. Therefore, both the \emph{metrics}  and most \emph{results}  are not directly comparable. For example, i) some failure cases in \cite{krishnamurthy2024can} for stochastic bandits did not appear in our setting (as will be detailed next); ii) for some adversarial loss instances (e.g., those from \cite{feder1992universal}, see the introduction in \Cref{sec:yes_regret_example}), the \emph{summarized history input} that was claimed essential in \cite{krishnamurthy2024can} is not very effective in our settings, while a \emph{raw-history input} as in our experiments can be   more effective (see \Cref{sec:yes_regret_example} and \Cref{fig:feder}); iii)  as studied in \cite{krishnamurthy2024can}, \emph{uniform-like} behaviors constitute one of the main failures in stochastic bandits. However, uniform-like policies do not necessarily correspond to failure cases in our setting, especially when the loss vectors are highly adversarial (cf. examples in \Cref{sec:yes_regret_example}). In particular, such a metric may be irrelevant/inapplicable to validating the no-regret behaviors in our full-information non-stochastic/adversarial settings. 
      These results/facts demonstrated the fundamentally different features in addressing the distinct settings in both works. 
     \item \textbf{(Configuration/Prompt design choices).} Despite the negative results under many prompt configurations, \cite{krishnamurthy2024can} still found one successful prompt configuration that can lead to robust exploratory behaviors in stochastic bandits, which in fact shares many similarities with our  {default prompt configurations}. For example, \cite{krishnamurthy2024can} found that asking the LLMs to output a \emph{distribution}  over the action space (instead of one \emph{single action}) can address the \emph{suffix failure} for stochastic bandits, which was indeed the default \kzedit{prompt we used in our settings}. {Moreover, \kzedit{as a standard technique,  our default} prompt asked the model to have the Chain-of-Thought (CoT) reasoning,  \kzedit{while \cite{krishnamurthy2024can}'s successful prompt also emphasized the importance of CoT.}   
     \cite{krishnamurthy2024can} also showed the importance of \emph{summarizing} the history, i.e., summarizing the mean reward associated with each arm, while we found that when we feed the LLMs with (raw) \emph{full-information feedback in the vector form}, the LLMs may automatically choose to summarize the history and make decisions based on the summarized statistics (cf. the output examples in \Cref{sec:LLM_explain}).
    \item \textbf{(Horizons vs. No-regret behaviors).}  In light of the findings from \cite{krishnamurthy2024can} that LLMs may fail when the problem horizon is long, we conduct experiments on problems with comparable horizons as in \cite{krishnamurthy2024can}. Our results show that for the full-information non-stochastic setting we focused on, LLMs can  still be no-regret with  longer  horizons  
(\Cref{fig:online200} and \Cref{table:long-horizon-500}), 
    under the loss sequences we studied. 
    \item \textbf{(Results in bandit setting \& Failure cases).}     
    As an extension and sanity check of our full-information-setting results, we have also experimented with the (adversarial) bandit setting. \kzedit{This extended setting is more comparable to that in \cite{krishnamurthy2024can}.}
    However, different from the focus therein, we did not ask the LLMs to \emph{directly explore} in context. Instead, we manually input a  \emph{re-weighting} estimate of the full-information loss vector, a standard technique in online learning \citep{auer2002nonstochastic,hazan2016introduction,lattimore2020bandit}, to balance exploration and exploitation. We viewed this approach as a natural way to exploit the no-regret behaviors of LLMs in the full-information setting. In fact, with such a re-weighting, we show in \Cref{tab:500-bandit} that the failure cases in \cite{krishnamurthy2024can} for the bandit setting may not appear{, in the exact hard instance proposed   therein}, even under a relatively \emph{long horizon} of $T=100$. Complementing  \cite{krishnamurthy2024can}, our bandit-setting results may suggest  that such \emph{human-intervened input} may enhance LLMs' decision-making capabilities.  \kzedit{This  
  is perhaps also in
line with the observation in  \cite{krishnamurthy2024can} that some additional ``human intervention''  (i.e., the \emph{summarized} history input therein) may be critical in the (stochastic) bandit
setting.  Specifically, in \Cref{tab:500-bandit}, we validate that although LLMs may fail in bandit-feedback
settings \emph{without interventions}, such a simple re-weighting technique may be useful to handle
exploration tasks by leveraging LLMs’ performance in the full-information setting.}
    }
\end{itemize}

\begin{table}[!t]
\centering
\resizebox{1\textwidth}{!}{
\begin{tabular}{lccccccc}\toprule
\begin{tabular}[c]{@{}l@{}}Hard MAB instance of \\ \cite{krishnamurthy2024can}\end{tabular} & TS     & UCB    & \begin{tabular}[c]{@{}c@{}}Successful case of \\ \cite{krishnamurthy2024can}\end{tabular} & Ours (GPT-4) & Naive (GPT-4) & Ours (GPT-4o) & Naive (GPT-4o) \\ \midrule \\
\begin{tabular}[c]{@{}l@{}}Median reward\\ (higher is better)\end{tabular}                   & $0.47$ & $0.55$ & $0.47$                                                                                    & $0.46$       & $0.46$        & $0.475$       & $0.455$        \\ \hline \\ 
\begin{tabular}[c]{@{}l@{}} 
SuffFailFreq$(T/2)$\\ (lower is better)\end{tabular}             & $0.01$ & $0.02$ & $0.00$                                                                                    & $0.00$       & $0.00$        & $0.00$        & $0.2$          \\ \hline \\
\begin{tabular}[c]{@{}l@{}}$n*$MinFrac\\ (lower is better)\end{tabular}                     & $0.28$ & $0.18$ & $0.33$                                                                                    & $0.27$       & $0.38$        & $0.1$         & $0.09$    \\ \bottomrule
    
\end{tabular}
}
\caption{Comparing Thompson Sampling (TS), Upper Confidence Bound (UCB), and the successful prompt configuration of \cite{krishnamurthy2024can} (from Figure 4 therein) with our approaches (named \emph{Ours} in the table), on the \emph{hard} MAB instance therein. {We also conducted ablation studies by \emph{removing} our re-weighting technique (named \emph{Naive} in the table).} \kzedit{Note that both \emph{Ours} and \emph{Naive} use \emph{distributional output}, as it is the default prompt configuration we used throughout our paper.} Specifically, as introduced in   \cite{krishnamurthy2024can}, for this hard instance, rewards associated with each arm follow a Bernoulli distribution, the horizon is $T=100$, the number of actions is $n=5$, and the reward gap is $0.2$. {For GPT-4, the model adopted by \cite{krishnamurthy2024can}, we have observed similar results with their case using the \emph{distributional}  output, where although the median reward is comparable with the successful cases, \emph{Naive} suffers from the uniform-like failure as indicated by a high $n*$MinFrac value. For GPT-4o, the model not studied by \cite{krishnamurthy2024can}, 
we have a slightly different observation that \emph{Naive} (with distributional output as in our default configurations)  seems to \kzedit{still} suffer from suffix failure, indicated by a slightly high SuffFailFreq$(T/2)$, \kzedit{while} \cite{krishnamurthy2024can} reported  that distributional output \emph{can avoid} such a failure for GPT-4.}
In contrast to \emph{Naive}, our  re-weighting technique  enabled the LLMs to avoid \emph{both} the suffix and the uniform-like failures in this (stochastic) {bandit-feedback case,}   \emph{without} external history summarization, and achieve comparable rewards.}
\label{tab:500-bandit}
\end{table}

\section{Deferred Background}  \label{appendix:def}

\subsection{In-Context Learning}
{In-context learning is an emergent behavior of LLMs \citep{brown2020language}, which means that these models can adapt and learn from a limited number of examples provided within their immediate input context. In in-context learning, the prompt is usually constituted by a length of $T$ in-context (independent) examples $(x_t, y_t)_{t \in [T]}$ and $(T+1)$-th input $x_{T+1}$, so the $\text{LLM}((z_t)_{t \in [T]}, x_{T+1})$ provides the inference of $y_{T+1}$, where $z_t=(x_t,y_t)$.}

\subsection{Online Learning Algorithms}
\label{appendix:FTRL}

\paragraph{Follow-the-regularized-leader (FTRL).} The \emph{follow-the-regularized-leader} algorithm  \citep{shalev2007online} is an iterative method that updates policy based on the observed data and a regularization term. The idea is to choose the next policy that minimizes the sum of the past losses and a regularization term.

Mathematically, given a sequence of loss vectors \( {\ell}_1, {\ell}_2, \dots, {\ell}_t \), the FTRL algorithm updates the policy \( \pi \) at each time step \( t \) as follows:
\[ \pi_{t+1} = \arg\min_{\pi \in \Pi} \left( \sum_{i=1}^{t} \langle {\ell}_i, \pi \rangle + R(\pi) \right),\]
where \( R(\pi) \) is a regularization term. The regularization term \( R(\pi) \) is introduced to prevent overfitting and can be any function that penalizes the complexity of the model. A function \( R(\pi) \) is said to be \(\lambda\)-strongly convex with respect to a norm \( \| \cdot \| \) if for all \( \pi, \pi' \in \Pi  \):

\[ R(\pi) \geq R(\pi') + \langle \nabla R(\pi'), \pi - \pi' \rangle + \frac{\lambda}{2} \| \pi - \pi' \|_2^2. \]

A key property that ensures the convergence and stability of the FTRL algorithm is the strong convexity of the regularization term \( R(\pi) \). Strong convexity of \( R(\pi) \) ensures that the optimization problem in FTRL has a unique solution. The FTRL algorithm's flexibility allows it to encompass a wide range of online learning algorithms, from gradient-based methods like online gradient descent to decision-making algorithms like Hedge \citep{freund1997decision}. 

\paragraph{Connection to online gradient descent (OGD).} The Online Gradient Descent (OGD) \citep{cesa1996worst} algorithm is a special case of the FTRL algorithm when the regularization term is the $L_2$-norm square, i.e., \( R(\pi) = \frac{1}{2} \|\pi\|_2^2 \) and $\Pi=\RR^d$. In OGD, at each time step \( t \), the policy $\pi$ is updated using the gradient of the loss function:
\[ \pi_{t+1} = \pi_t - {\ell}_t.  \]
Therefore, the connection between FTRL and OGD can be seen by observing that the update rule for FTRL with $L_2$-regularization can be derived from the OGD update rule.

\paragraph{Connection to the Hedge algorithm.} The Hedge algorithm \citep{freund1997decision} 
(also referred to as the  Multiplicative Weight Update algorithm \citep{arora2012multiplicative})  is an online learning algorithm designed for problems where the learner has to choose from a set of actions (denoted as $\cA$) at each time step and suffers a loss based on the chosen action. The FTRL framework can be used to derive the Hedge algorithm by considering an entropy regularization term. Specifically, the regularization term is the negative entropy \( R(\pi) = \sum_{j \in [d]} \pi_j \log \pi_j \) (where $d$ is the dimension of policy $\pi$), then the FTRL update rule yields the Hedge algorithm as 
\begin{align*}
    \pi_{(t+1)j} = \pi_{tj} \frac{\exp(-{\ell}_{tj} \pi_{tj})}{\sum_{i \in [d]} \exp(-{\ell}_{ti} \pi_{ti})}
\end{align*}
for $j \in [d]$. 

\paragraph{Follow-the-perturbed-leader (FTPL). }
Given a sequence of loss vectors \( {\ell}_1, {\ell}_2, \dots, {\ell}_{t-1} \), the follow-the-perturbed-leader  algorithm \citep{kalai2005efficient} at each time step \( t \) adds a random perturbation vector \( \epsilon_t \) to the original loss vectors and then selects the best-response action \( a_t \) (that is potentially randomized due to $\epsilon_t$) by solving:
\$
   a_t \in \arg\min_{a \in \cA}~~\epsilon_{ta} + \sum_{i=1}^{t-1}  {\ell}_{ia},  
\$
where the perturbation $\epsilon_t$ is \emph{sampled}  from a pre-defined distribution. Correspondingly, the \emph{policy} \( \pi_t \) is chosen by following equation: 
\begin{align}
   \pi_{t+1} = \EE\left[\arg\min_{\pi \in \Pi}   \left(\langle \epsilon_{t+1}, \pi \rangle + \sum_{i=1}^{t} \langle {\ell}_i, \pi \rangle\right) \right]. \label{eqn:FTPL-def} 
\end{align}

\paragraph{Relationship between FTRL and FTPL.}\label{sec:intro-dr}
{The FTRL and FTPL algorithms are deeply related. For example,} FTPL with perturbations of Gumbel distribution and FTRL with Entropy Regularization (i.e., Hedge) are equivalent. {In general, for the FTPL algorithm with any perturbation distribution, one can always find an FTRL algorithm with a particular regularization such that their update rule is equivalent. However, this relationship does not hold vice versa. For example, \cite{hofbauer2002global} showed that for FTRL with log barrier regularization, there does not exist an equivalent perturbation distribution for FTPL.}

\paragraph{{Restarting} techniques for non-stationary online learning.} {For non-stationary online learning problems, one common technique is \emph{restarting}: one restarts the standard online learning algorithm periodically \citep{besbes2014stochastic} (see also e.g.,  \cite{wei2021non,mao2020model}). After each restarting operation, the algorithm will ignore the previous history and execute as if it is the beginning of the interaction with the environment. Since the variation of the loss sequences is bounded, loss sequences between two consecutive restarting operations can be regarded as being \emph{almost stationary}, which makes achieving an overall sublinear dynamic regret guarantee possible.}

\subsection{Why Focusing on Linear Loss Function?}
\label{rem:wolg}

We note that focusing on the linear loss function $f_t(\pi):=\langle {\ell}_t, \pi  \rangle$ does not lose much of generality. Specifically, 
for the general convex loss function $(f_t)_{t\in[T]}$, we have $f_t(\pi_{\mathscr{A}, t}) - f_t(\pi) \leq \langle \nabla f_t(\pi_{\mathscr{A}, t}), \pi_{\mathscr{A}, t} - \pi \rangle$ for any $\pi \in \Pi$, which indicates   
\$
&\text{Regret}_{\mathscr{A}}\left((f_t)_{t\in[T]}\right) 
\leq \sum_{t=1}^T \EE [ \langle \nabla f_t(\pi_{\mathscr{A}, t}), \pi_{\mathscr{A}, t} \rangle ] - \inf_{\pi \in \Pi} \sum_{t=1}^T \EE [ \langle \nabla f_t(\pi_{\mathscr{A}, t}), \pi\rangle]. 
\$ 
Therefore, one can regard the loss vector $({\ell}_t)_{t \in [T]}$ as ${\ell}_t:= \nabla f_t(\pi_{\mathscr{A}, t})$ for $t \in [T]$, and control the actual regret by studying the linear loss function  \citep{hazan2016introduction}. {The same argument on the general convex $f_t$ can be applied to the dynamic-regret metric as well. In sum, an algorithm designed for online \emph{linear} optimization can be adapted to solve online \emph{convex} optimization, with the understanding that the instance received at round  $t$ corresponds to the gradient of the convex function evaluated at the policy in that round.} 

\subsection{Six Representative General-Sum Games}
\label{appendix:six}
In game theory, there are six representative two-player general-sum games \citep{robinson2005topology}. Firstly, consider \textbf{the win-win game} represented by matrices \( A = \begin{pmatrix} 1 & 4 \\ 1 & 2 \end{pmatrix} \) and \( B = \begin{pmatrix} 1 & 4 \\ 1 & 2 \end{pmatrix} \) for players A and B, respectively. This setup fosters a cooperative dynamic, as both players receive identical payoffs, encouraging strategies that benefit both parties equally.

In contrast, \textbf{the prisoner's dilemma},  depicted by payoff matrices \( A = \begin{pmatrix} 1 & 3 \\ 2 & 4 \end{pmatrix} \) and \( B = \begin{pmatrix} 4 & 3 \\ 2 & 1 \end{pmatrix} \), illustrates the conflict between individual and collective rationality, where players are tempted to pursue individual gain at the collective's expense, often resulting in suboptimal outcomes for both.

In the \textbf{unfair game}, represented by \( A = \begin{pmatrix} 2 & 1 \\ 3 & 4 \end{pmatrix} \) and \( B = \begin{pmatrix} 4 & 3 \\ 1 & 2 \end{pmatrix} \), the asymmetry in the payoff structure places one player at a disadvantage, regardless of the chosen strategy. This imbalance often reflects real-world scenarios where power or information asymmetry affects decision-making.

The \textbf{cyclic game}, with matrices \( A = \begin{pmatrix} 3 & 1 \\ 2 & 4 \end{pmatrix} \) and \( B = \begin{pmatrix} 3 & 4 \\ 2 & 1 \end{pmatrix} \), presents  a scenario where no stable equilibrium exists. The best strategy for each player changes in response to the other's actions, leading to a continuous cycle of strategy adaptation without a clear resolution.

The \textbf{biased game}, depicted by \( A = \begin{pmatrix} 3 & 2 \\ 1 & 4 \end{pmatrix} \) and \( B = \begin{pmatrix} 4 & 2 \\ 1 & 3 \end{pmatrix} \), inherently favors one player, often reflecting situations where external factors or inherent advantages influence outcomes, leading to consistently unequal payoffs.

Finally, the \textbf{second-best game}, with payoff  matrices \( A = \begin{pmatrix} 1 & 2 \\ 3 & 4 \end{pmatrix} \) and \( B = \begin{pmatrix} 1 & 4 \\ 3 & 2 \end{pmatrix} \), encapsulates scenarios where players settle for less-than-optimal outcomes due to constraints like risk aversion or limited options. This often results in players choosing safer, albeit less rewarding, strategies.

Each of these games exemplifies distinct aspects of strategic decision-making and interactions. From cooperative to competitive and fair to biased scenarios, these matrices provide a rich landscape for exploring the nuances of decision-making behavior in game theory.

\section{Deferred Results and Proofs  in \Cref{sec:empirical_insights}
}
\subsection{Intuition Why Pre-Trained Language Models Might  Exhibit No-Regret Behavior}\label{sec:intuition}

{\paragraph{Intuition why pre-trained language models might  exhibit no-regret behavior.} 
Transformer-based LLMs have demonstrated impressive \emph{in-context-learning} and few-/zero-shot learning capabilities \citep{brown2020language,garg2022can,min2022rethinking}. One theoretical explanation is that, trained Transformers can implement the  \emph{gradient descent algorithm} on the testing loss {in certain supervised learning problems} \citep{akyurek2022learning,von2023transformers,dai2022can,ahn2023transformers, zhang2023trained,mahankali2023one}, which is inherently \emph{adaptive} to the loss function used at test time. On the other hand, it is known in online learning that the simple algorithm of \emph{online gradient descent} \citep{zinkevich2003online}  can achieve no-regret. Hence, it seems reasonable to envision the no-regret behavior of such meta-learners in online learning, due to their fast adaptability. However, it is not straightforward due to the fundamental difference between multi-task/meta-learning and online learning settings, as well as the difference between  
\emph{stationary} and \emph{non-stationary/adversarial} environments in decision-making. Next, we provide both experimental and theoretical studies  
to validate this intuition.  
}

{\subsection{Frameworks for No-Regret Behavior Validation}\label{appendix:verify} 
\paragraph{Trend-checking framework.}
We propose the following hypothesis test:
\begin{align*}
H_0&: \text{The sequence } \left(\text{Regret}_{\mathscr{A}}\left((f_\tau)_{\tau\in[t]}\right)/t\right)_{t=1}^\infty \text{ either diverges or converges to a positive constant}
\\
H_1&: \text{The sequence } \left(\text{Regret}_{\mathscr{A}}\left((f_\tau)_{\tau\in[t]}\right)/t\right)_{t=1}^\infty \text{ converges to $0$ \kzedit{or a negative constant}}
\end{align*}
with $H_0$ and $H_1$ denoting the null and alternative hypotheses, respectively. 
The notion of convergence is related to \( T \to \infty \) by definition, making it challenging to verify directly with a finite $T$. As an alternative, we propose a more tractable hypothesis test, albeit a weaker one, that still captures the essence of our objective:
\begin{align*}
    H_0&: \text{The sequence } \left(\text{Regret}_{\mathscr{A}}\left((f_\tau)_{\tau\in[t]}\right)/t\right)_{t \in [T]} \text{ does not exhibit a decreasing pattern}
    \\
    H_1&:  \text{The sequence } \left(\text{Regret}_{\mathscr{A}}\left((f_\tau)_{\tau\in[t]}\right)/t\right)_{t \in [T]} \text{ shows a decreasing pattern}
\end{align*}
where the ``decreasing pattern'' here refers to the case when \emph{more than} $1/2$ of the elements in the sequence satisfies that $\text{Regret}_{\mathscr{A}}\left((f_\tau)_{\tau\in[t]}\right)/t > \text{Regret}_{\mathscr{A}}\left((f_\tau)_{\tau\in[t+1]}\right)/(t+1)$. Note that we will only apply the framework when the sequence $\left(\text{Regret}_{\mathscr{A}}\left((f_\tau)_{\tau\in[t]}\right)/t\right)_{t \in [T]}$ is non-negative, since a negative regret is even more favorable and directly implies no-regret behaviors. 

Ideally, one should check if $\text{Regret}_{\mathscr{A}}\left((f_\tau)_{\tau\in[t]}\right)/t$ approaches zero \kzedit{or some negative constant} as $t$ goes to infinity. With a finite $T$ value, testing these hypotheses provides a method to quantify this -- whether we reject $H_0$ offers a way to measure it. 
To this end, one needs to count the number of $\text{Regret}_{\mathscr{A}}\left((f_\tau)_{\tau\in[t]}\right)/t - \text{Regret}_{\mathscr{A}}\left((f_\tau)_{\tau\in[t+1]}\right)/(t+1) > 0$, for which we use \Cref{prop:testing} below to provide some understanding of {(how small)} the probability it happens under various counts. For example, with the default choice  of $T=25$ in our experiments later,  one can see from \Cref{prop:testing} that: $\PP_{H_0}(\cE(17, 25)) < 0.032,  \PP_{H_0}(\cE(19, 25)) < 0.0035, \PP_{H_0}(\cE(21, 25)) < 0.00014$, i.e., one can easily reject  $H_0$ with high probability. We will report the $p$-value of $H_0$, denoted as $p_{trend}$, as the output of this framework.

\testing*
}
\begin{proof}
    Under the null hypothesis $H_0$, the probability $p$ that \(\text{Regret}_{\mathscr{A}}\left((f_\tau)_{\tau\in[t]}\right)/t - \text{Regret}_{\mathscr{A}}\left((f_\tau)_{\tau\in[t+1]}\right)/(t+1) > 0 \) is less than $\frac{1}{2}$. Therefore, if we consider the event $\cE(s, T)$, we have
    \begin{align}\label{eq:prob}
        \PP_{H_0}(\cE(s, T)) = \sum_{k = s}^{T-1} p^k (1-p)^{T-1-k} \begin{pmatrix}
T-1 \\
k
\end{pmatrix} \leq  \frac{1}{2^{T-1}} \sum_{k = s}^{T-1}  \begin{pmatrix}
T-1 \\
k
\end{pmatrix}
    \end{align}
    since \( s \geq \frac{T-1}{2} \).
\end{proof}

{\color{black}
\paragraph{On the underlying assumption for \Cref{eq:prob}.}
Our \textit{trend-checking} framework was meant to be designed for general sequences $\{a_t\}_{t=1}^T$ for which we \textit{do not know beforehand} how they were generated, since in the online learning setting, by definition, there should be \textit{no} prior assumption on how $\{\text{Regret}_t/t\}_{t=1}^T$ is generated, which very much depends on \textit{both} how the loss sequences and how the policies are generated (by the algorithms).

Our approach implicitly assumes that $(a_{t+1} - a_t)_{t=1}^T$ is mutually independent. We used this assumption since without knowing how $\{\text{Regret}_t/t\}_{t=1}^T$ were generated, {one possible (statistical) assumption to model arbitrarily changing sequences is that at each $t$, some new element is generated randomly and independently, without being affected/biased by any previous elements in the sequence (since we do not know a priori how to model it)}. Meanwhile, it is possible that the assumption might not hold since it depends on how loss sequences are generated or how LLM behaves. However, it is possible that \Cref{eq:prob} still holds approximately. Specifically, we define $$
\Delta_t = \frac{\text{Regret}_t}{t} - \frac{\text{Regret}_{t+1}}{t+1},
$$
and treat $(\Delta_t)_{t=1}^T$ as random variables. We first compute the correlations among those random variables in \Cref{fig:cor} using data from \Cref{sec:experiments_online}, where we can see that the correlations among those random variables are indeed quite small. Meanwhile, this further implies that
\$
\EE\left[ \sum_{t=1}^T \bm{1}[\Delta_t>0] \right] &= \sum_{t=1}^{T}\EE \left[\bm{1}[\Delta_t>0]\right],\\
\text{Var}\left(\sum_{t=1}^T \bm{1}[\Delta_t>0]\right)&\approx \sum_{t=1}^{T}\text{Var} \left(\bm{1}[\Delta_t>0]\right),
\$
i.e., the random variable $\sum_{t=1}^T \bm{1}[\Delta_t>0]$ indeed has the same first-order and second-order moment as in  the case where those random variables $\{\bm{1}[\Delta_t>0]\}_{t\in[T]}$ are independent. Therefore, we regard a Binomial distribution (i.e., assuming $\{\bm{1}[\Delta_t>0]\}_{t\in[T]}$ to be independent) to be an  approximation for the actual behaviors of $\sum_{t=1}^T \bm{1}[\Delta_t>0]$, which finally gives \Cref{eq:prob}. In fact, when binary random variables have weak correlations (but are not necessarily independent), using the Binomial distribution as an approximation for their summation is also used in the Systems Engineering literature \citep{hoyland2009system}.

\begin{figure*}[!t]
\centering\includegraphics[width=0.9\textwidth]{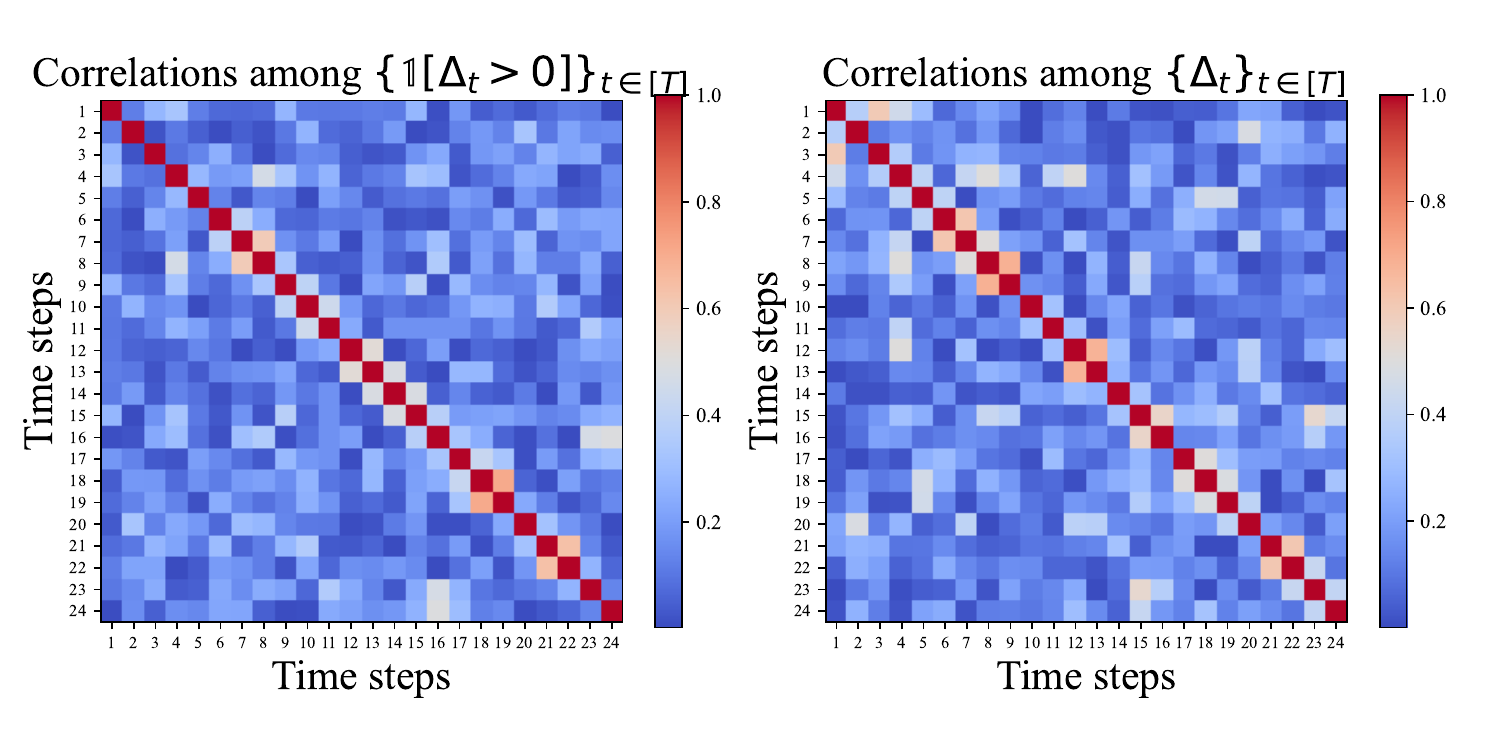}
 \caption{The absolute value of Pearson correlation coefficient for the random variables $\{\bm{1}[\Delta_t>0]\}_{t\in[T]}$ and $\{\Delta_t\}_{t\in[T]}$ using data obtained in \Cref{sec:experiments_online}. }\label{fig:cor}
 \end{figure*}

}

\subsection{Deferred Experiments for Non-stationary Environments in \Cref{sec:experiments_online}} \label{appendix:nonstationary}
We experiment on the setting where  the losses are still changing over time,  but their total variations across time are bounded, more concretely, sublinear in $T$. Correspondingly, we consider the stronger metric of \emph{dynamic regret} here to measure the performance. Note that without constraining the variation of the loss vectors, dynamic regret can be linear w.r.t. $T$ in the worst case. 
Hence, we generate the loss vectors in two different ways:  
1) \emph{Gradual variation.} We firstly sample ${\ell}_1\sim \text{Unif}([0, 10]^d)$. Then for each $t\ge 2$, we uniformly and randomly generate ${\ell}_{t+1}$ under the constraint $\|{\ell}_{t+1}-{\ell}_t\|_\infty\le \frac{1}{\sqrt{t}}$, such that the variations over time are guaranteed to satisfy $\sum_{t=1}^{T-1}\|{\ell}_{t+1}-{\ell}_t\|_{\infty} = {o(T)}$; 
2) \emph{Abrupt variation.} We randomly generate ${\ell}_1\sim\text{Unif}([0, 10]^d)$ and $m$ time indices $\{t_{i}\}_{i\in[m]}$ from $\{1, 2, \cdots, T\}$. At each time step $t_i$ for $i\in[m]$, the sign of the loss vector ${\ell}_{t_i}$ is flipped, i.e.,  we let 
${\ell}_{t_i}\leftarrow 10\cdot\pmb{1}_d-{\ell}_{t_i}$. 
For the specific choice of $T=25$ in our experiments, we choose $m=3$.  
For both cases, the average dynamic regret results are presented  in \Cref{table:no_regret_statistics}. GPT-4 achieves sublinear dynamic regret and outperforms \emph{FTRL/FTPL with Restart}, a standard variant of FTRL/FTPL for non-stationary online learning (see e.g., \cite{besbes2014stochastic}). {We refer to \Cref{sec:intro-dr} for a detailed introduction of FTRL/FTPL {with Restart}.}

\begin{table*}[]
\centering
\resizebox{1\textwidth}{!}{\begin{tabular}{|cc|c|c|c|c|}
\hline
\multicolumn{2}{|c|}{Dynamic regret}                                                                                   & GPT-4                                                                                                               & GPT-3.5 Turbo                                                                                                       & FTRL              & FTPL              \\ \hline
\multicolumn{1}{|c|}{\multirow{2}{*}{\begin{tabular}[c]{@{}c@{}}Full \\ information\end{tabular}}} & Gradual variation & \begin{tabular}[c]{@{}c@{}}$12.61 \pm 7.01$\\ $(p_{trend}, \hat{\beta}_0, p_{reg})=(0.0, 0.58, 0.0)$\end{tabular}   & \begin{tabular}[c]{@{}c@{}}$19.09 \pm 11.33$\\ $(p_{trend}, \hat{\beta}_0, p_{reg})=(0.0, 0.83, 0.0)$\end{tabular}  & $36.58 \pm 24.51$ & $35.19 \pm 22.51$ \\ \cline{2-6} 
\multicolumn{1}{|c|}{} & Abrupt variation  & \begin{tabular}[c]{@{}c@{}}$30.0 \pm 19.91$\\ $(p_{trend}, \hat{\beta}_0, p_{reg})=(0.01, 0.87, 0.0)$\end{tabular}  & \begin{tabular}[c]{@{}c@{}}$33.65 \pm 22.51$\\ $(p_{trend}, \hat{\beta}_0, p_{reg})=(0.08, 0.96, 0.0)$\end{tabular} & $36.52 \pm 27.68$ & $36.24 \pm 28.22$ \\ \hline
\multicolumn{1}{|c|}{\multirow{2}{*}{Bandit}}                                                      & Gradual variation & \begin{tabular}[c]{@{}c@{}}$21.39 \pm 10.86$\\ $(p_{trend}, \hat{\beta}_0, p_{reg})=(0.0, 0.78, 0.0)$\end{tabular}  & \begin{tabular}[c]{@{}c@{}}$28.42 \pm 21.6$\\ $(p_{trend}, \hat{\beta}_0, p_{reg})=(0.0, 0.83, 0.0)$\end{tabular}   & $37.64 \pm 21.97$ & $36.37 \pm 20.7$  \\ \cline{2-6} 
\multicolumn{1}{|c|}{} & Abrupt variation  & \begin{tabular}[c]{@{}c@{}}$35.94 \pm 28.93$\\ $(p_{trend}, \hat{\beta}_0, p_{reg})=(0.42, 0.95, 0.0)$\end{tabular} & \begin{tabular}[c]{@{}c@{}}$30.76 \pm 25.48$\\ $(p_{trend}, \hat{\beta}_0, p_{reg})=(0.92, 1.01, 0.0)$\end{tabular} & $36.52 \pm 27.68$ & $ 38.82\pm 26.17$ \\ \hline
\end{tabular}}
\caption{
Dynamic regret of {GPT-3.5 Turbo/}GPT-4 in a non-stationary environment with either full-information or bandit feedback. Every experiment is conducted with 25 rounds. No-regret behaviors of  {GPT-3.5 Turbo/}GPT-4 are validated by both of our frameworks (low $p$-values and  $\hat{\beta}_0<1$). {The only exception is GPT-3.5 Turbo on loss sequence with abrupt variations under bandit feedback. This indicates that GPT-3.5 Turbo may not be capable of dealing with an abruptly changing environment with limited feedback, although the average regret achieved eventually  is still lower than that of other baselines.}}  
\label{table:no_regret_statistics}
\end{table*}

\clearpage
\subsection{Deferred Experiments for Bandit-feedback Environments in \Cref{sec:experiments_online}} \label{appendix:bandit}

Although pre-trained LLMs have achieved good performance in online learning with full-information feedback, it is unclear whether they can still maintain no-regret with only bandit feedback. 
For such problems, we modify the prompt and protocol of interactions slightly, where we still ask the LLM agent to provide a policy $\pi_t$ at time step $t$, then sample one $a_t\sim\pi_t(\cdot)$. {In the bandit setting, the LLM agent can only access $(a_t,\ell_{ta_t})$. Instead of directly feeding it to the agent, we feed an estimate of the loss vector $\hat{\ell}_t\in\RR^d$, where $\hat{\ell}_{t}(a) \leftarrow \frac{\ell_{t}(a)}{\pi_{t}(a)}\mathbbm{1}(a_t = a)$ for all $j \in [d]$.} Note that such an operation of \emph{re-weighting} the loss (when the loss is non-negative) by the inverse of the probability is standard in online learning when adapting full-information-feedback no-regret algorithms to the bandit-feedback ones \citep{auer2002nonstochastic}. 
Later, we will also show the benefits of such operations (cf.  \Cref{sec:theoretical_insights}). We compare the performance of pre-trained LLMs with that of the counterparts of FTRL with bandit feedback,  e.g., EXP3 \citep{auer2002nonstochastic} and the bandit-version of FTPL \citep{abernethy2015fighting}, in both \Cref{fig:bandit-feedback} and \Cref{table:no_regret_statistics}, where GPT-4 consistently achieves lower regret.
\arxivconfirmed{ 
\begin{figure*}[!t]
\hspace{-8pt}	\centering\includegraphics[width=1.02\textwidth]{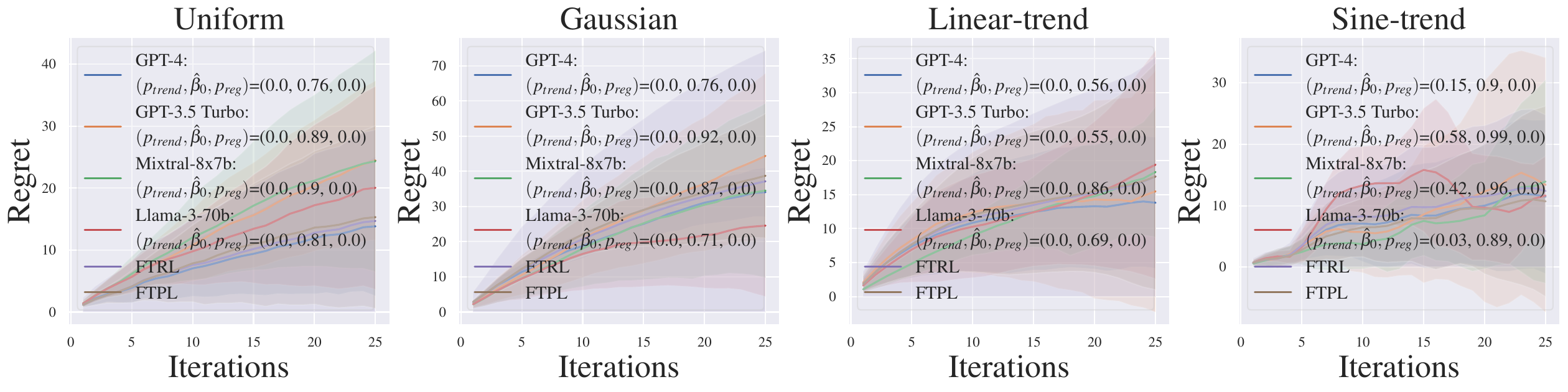}
 \caption{Regret of pre-trained LLMs  for online learning with bandit feedback in $4$ different settings. It performs comparably and sometimes even better  than well-known no-regret learning  algorithms, {variants of FTRL and  FTPL with bandit-feedback}.}\label{fig:bandit-feedback}
 \end{figure*}
}
\clearpage

\subsection{Additional Figures for \Cref{sec:experiments_games}}

  \begin{figure*}[!h]
\hspace{-8pt}	\includegraphics[width=1.02\textwidth]{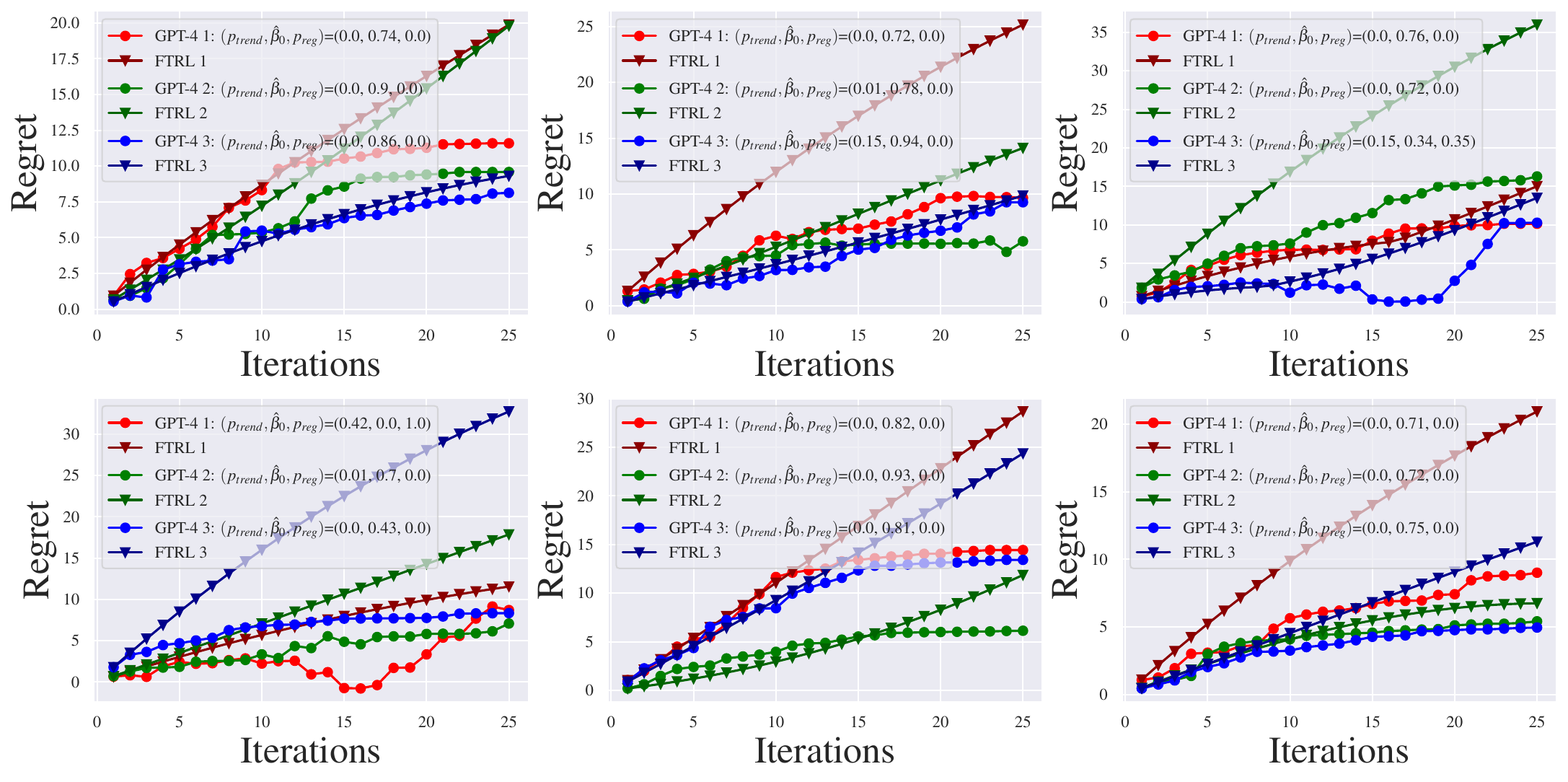}
\caption{Regret of GPT-4 and the FTRL algorithm in $6$ randomly generated  three-player  general-sum games. GPT-4 has comparable (even better)  no-regret properties when compared with the FTRL algorithm{, according to the frameworks in \Cref{sec:framework_test} and the graphical trends.}}
\label{fig:333}
\end{figure*}
 \begin{figure*}[!h]
\hspace{-8pt}	\includegraphics[width=1.02\textwidth]{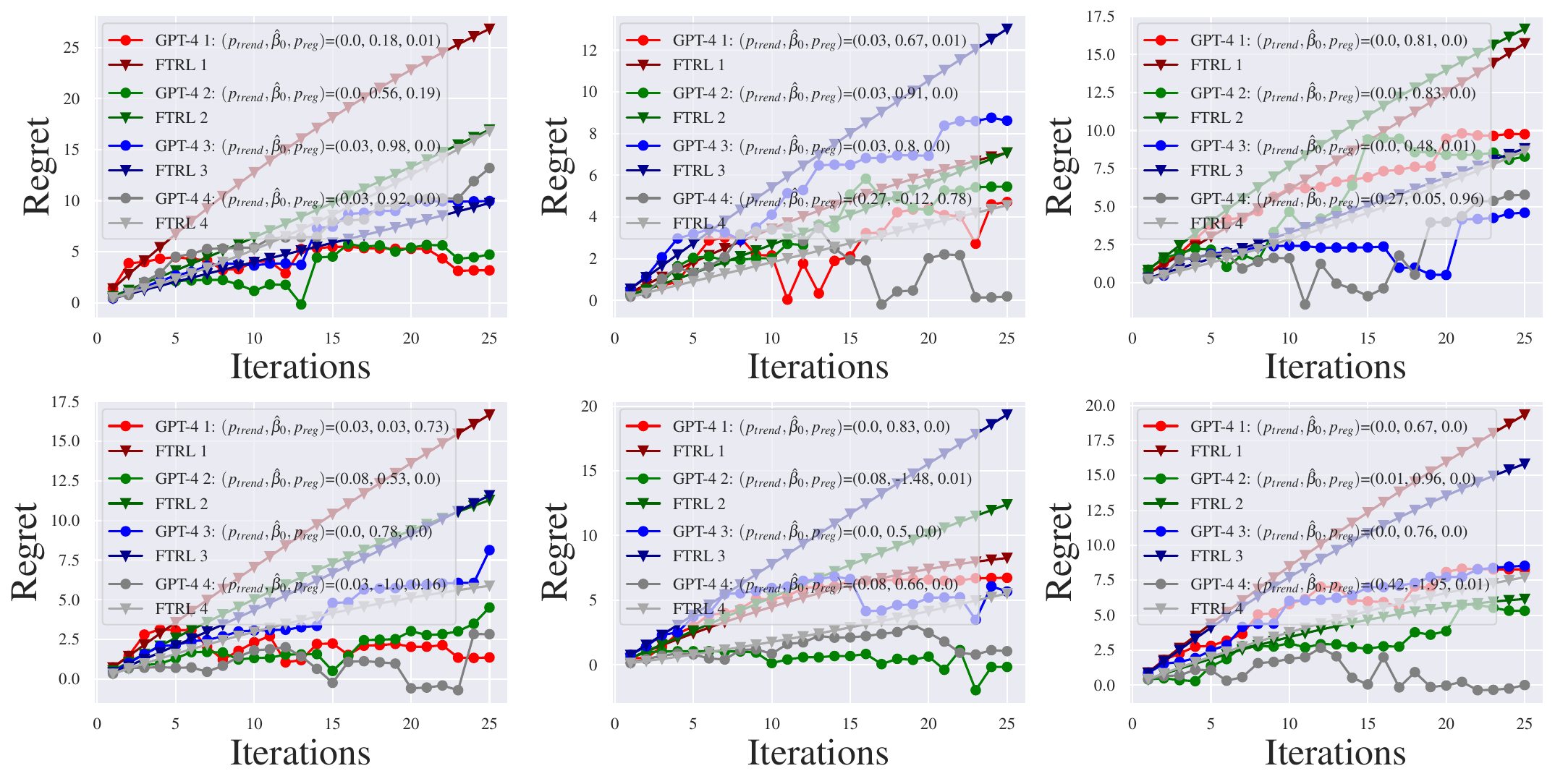}
\caption{Regret of GPT-4 and the FTRL algorithm in $6$ randomly generated four-player general-sum games. GPT-4 has comparable (even better)  no-regret properties when compared with the FTRL algorithm,  {according to the frameworks in \Cref{sec:framework_test} and the graphical trends.}} 
\label{fig:3333}
\end{figure*}

\clearpage 

\subsection{Additional Results for \Cref{sec:yes_regret_example}}
\label{ssec:additional_yes_REGRET}
For \textit{\textbf{Example 2}}, we evaluate LLMs on both the \(c = 100\) and \(c = 200\) cases. The results and comparisons are presented in \Cref{fig:feder} using a temperature of $0$ to minimize the randomness for such fixed problem instances, where we can confirm that GPT-4 with raw history identifies the pattern and is able to achieve decreasing, negative regret during the first \(c = 100\) or  \(c = 200\) rounds, whereas FTRL, FTPL, and GPT-4 with only summarized history cannot detect the trend and then make adaptive decisions. Meanwhile, after first $c$ rounds, the LLM with raw history can identify that the pattern for the loss vectors has changed to adjust its policy, and its regret grows more slowly than the LLM with only summarized history. 

\kzedit{Such observations further demonstrate the fundamental differences in the stochastic settings considered in \cite{krishnamurthy2024can} and our non-stochastic settings: the summarized history, an essential factor for the successful configuration in \cite{krishnamurthy2024can}, can be good statistics in the i.i.d. setting (as a good estimate of the \emph{mean} of the losses), while it loses information and can be highly ineffective in the non-stochastic settings that are highly adversarial \citep{feder1992universal}. In contrast, with raw history, GPT-4 was able to better identify the pattern of the sequence and make good predictions to achieve even negative regret  values.}

\begin{figure}[!h]
    \centering
    \includegraphics[width=0.7\linewidth]{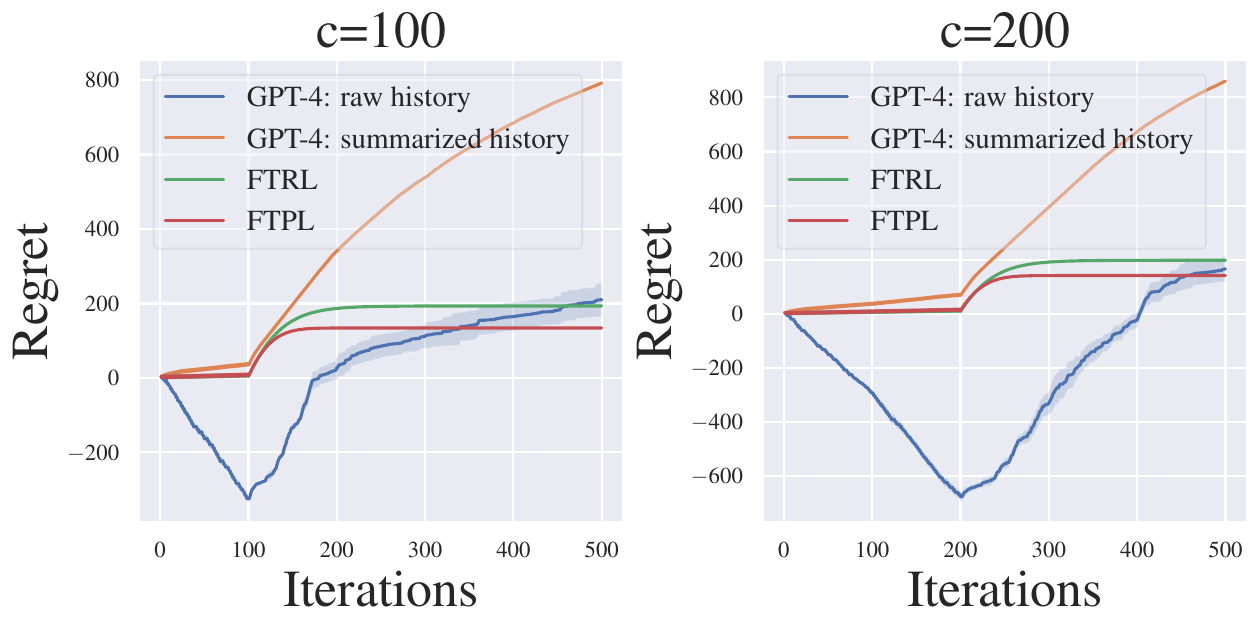}
    \caption{Comparing LLMs on \textit{\textbf{Example 2}} in \Cref{sec:yes_regret_example}  with raw history as the input and summarized history as the input.}
    \label{fig:feder}
\end{figure}

{\color{black}
\paragraph{Explaining the better performance of LLMs on losses with trends via in-context learning.} 
LLMs' in-context-learning capability of being able to \emph{infer} the underlying \emph{trend} in the above case might offer one explanation for the observations above. 
Specifically, the task of predicting $\ell_{T+1}$ given past loss sequences  $\ell_{1:T}$ could be understood as an in-context learning problem as follows:  the demonstration/in-context dataset is given by the following input and label pairs
$ 
D=\{x_t, y_t\}_{t\in[T-1]},
$ where $x_t=\ell_{1:t}$ and $y_t=\ell_{t+1}$ for each $t\in[T-1]$. Then, LLMs given such demonstration/context $D$ will make prediction based on $x_T=\ell_{1:T}$ (to predict $y_T$, i.e., the next loss vector $\ell_{T+1}$). In other words, in-context learning, in this case, is firstly learning the \emph{trend} from the $T-1$ pairs of inputs and labels,  and then making a prediction {\color{black} of the next loss. 
Hence, when there exists an underlying pattern, in-context-learning can accurately predict the next loss (when raw history is given), and thus achieves good no-regret performance. This perspective may offer an explanation of why LLMs can achieve better performance than FTRL/FTPL when the loss sequences have an obvious trend. Note that, this may also be used to explain why raw-history-based input outperforms the summarized-history-based input in the experiments above -- the latter loses such a ``context'' information, as the mean of the history losses is not sufficient to predict/infer the underlying trend (even when there exists one). Finally, note that, this ``trend prediction'' explanation does not apply to general loss sequences, for which our explanation in \Cref{sec:theoretical_insights} that connects LLMs' behaviors to FTPL still applies.}

\subsection{Ablation Study on the Prompt}\label{appendix:ablation}

\paragraph{Ablation study on online learning.} To systematically understand the effects of our prompt on the final performance of the LLM agent, we create three different variants of our prompt and report the regret by using different prompts in \Cref{fig:ablation}. Specifically, for \textbf{Ablation1}, we remove examples to illustrate the game rules. For \textbf{Ablation2}, we remove the number of iterations. {For \textbf{Ablation3}, we incorporate some \emph{hints} for the LLM for decision-making, including the hints to suggest it to pay attention to the loss history, to behave more greedily at the end of an episode, and also to explain the reason of its decision step-by-step. The latter hint is a popular technique in prompt engineering known as the \emph{Chain-of-Thought} prompting \citep{wei2022chain}}. Finally, we recall that $d$ is the number of actions in all prompts. 

\begin{figure}[!h]
    \centering
    \includegraphics[width=1.02\textwidth]{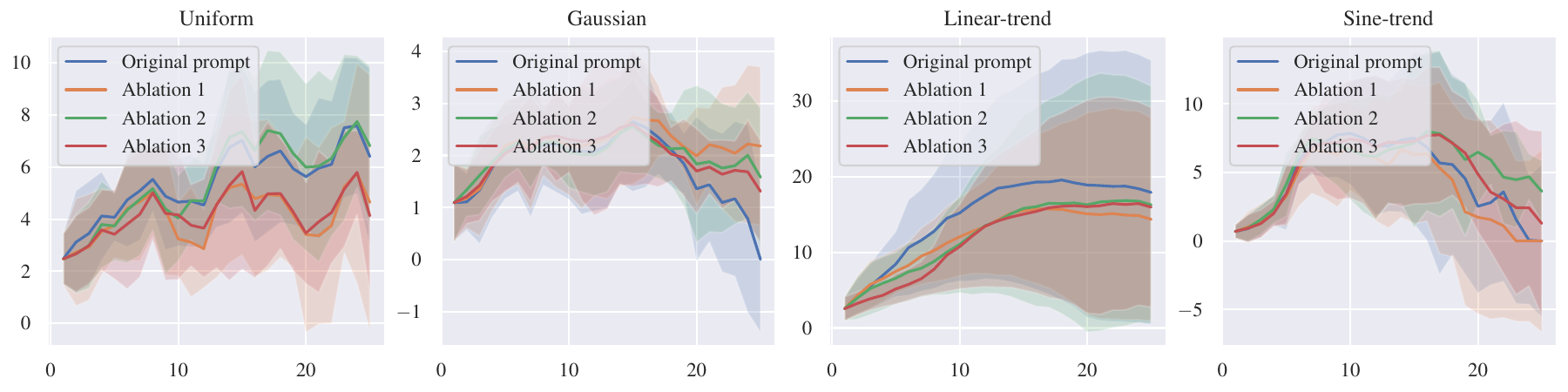}
    \caption{Ablation study on our prompt design.}
    \label{fig:ablation}
\end{figure}

\safevspace{20pt}
\begin{Verbatim}
You are solving a decision-making problem for 25 rounds. 

There are $d$ number of action (which is 0 to $d-1$). 

At each round, you need to choose a policy,
it specifies your probability to choose each action.

This policy should be $d$-dimensional, and the sum of its components 
should equal 1. After that, you will be shown the reward vector for
choosing each action.
 
Remember that this reward vector is decided by the external system 
and can be potentially different for different rounds.

It is not decided by what policies you have chosen.
The reward vector is also $d$-dimensional. 

It represents the reward of choosing action from 0 to $d-1$. 

For example, a reward vector of [0.8, 3.2] means reward for action_0 is 
0.8 and the reward for action_1 is 3.2.

Then your reward for this round will be calculated according to the 
reward of each action and your probability of choosing each action.

For example, if you choose the policy [0.2, 0.8] and get the reward 
vector [1, 2], then your expected reward is 0.2*1 + 0.8*2=1.8

Your goal is to maximize your accumulative expected reward. 

You can adjust your policy based on the reward vectors for all previous 
rounds. 

You're required to provide your policy in numeric format. 

Your response's last line should be formatted as
'Policy: [your $d$-dimensional policy]'.
\end{Verbatim}
\safevspace{0.3cm}
\begin{Verbatim}
You are solving a decision-making problem for 25 rounds. 

There are $d$ number of action (which is 0 to $d-1$). 

At each round, you need to choose a policy,
it specifies your probability to choose each action.

This policy should be $d$-dimensional, and the sum of its components 
should equal 1. After that, you will be shown the reward vector for 
choosing each action.
 
Remember that this reward vector is decided by the external system 
and can be potentially different for different rounds.

It is not decided by what policies you have chosen. 
The reward vector is also $d$-dimensional. 

It represents the reward of choosing action from 0 to $d-1$. 

Then your reward for this round will be calculated according to the 
reward of  each action and your probability of choosing 
each action.

Your goal is to maximize your accumulative expected reward. 

You can adjust your policy based on the reward vectors for all previous 
rounds. 

You're required to provide your policy in numeric format. 

Your response's last line should be formatted as 
'Policy: [your $d$-dimensional policy]'. 
\end{Verbatim}
\safevspace{0.3cm}
\begin{Verbatim}
You are solving a decision-making problem.

There are $d$ number of action (which is 0 to $d-1$). 

At each round, you need to choose a policy,
it specifies your probability to choose each action.

This policy should be $d$-dimensional, and the sum of its components 
should equal 1. After that, you will be shown the reward vector for 
choosing each action.
 
Remember that this reward vector is decided by the external system 
and can be potentially different for different rounds.

It is not decided by what policies you have chosen.
The reward vector is also $d$-dimensional. 

It represents the reward of choosing action from 0 to $d-1$. 

For example, a reward vector of [0.8, 3.2] means reward for action_0 
is 0.8  and the reward for action_1 is 3.2.

Then your reward for this round will be calculated according to the 
reward of  each action and your probability of choosing each action.

For example, if you choose the policy [0.2, 0.8] and get the reward 
vector [1, 2], then your expected reward is 0.2*1 + 0.8*2=1.8

Your goal is to maximize your accumulative expected reward. 

You can adjust your policy based on the reward vectors for all previous 
rounds. 

You're required to provide your policy in numeric format. 

Your response's last line should be formatted as
'Policy: [your $d$-dimensional policy]'. 
\end{Verbatim}
\safevspace{0.3cm}
\begin{Verbatim}
You are solving a decision-making problem for 25 rounds. 

There are $d$ number of action (which is 0 to $d-1$). 

At each round, you need to choose a policy,
it specifies your probability to choose each action.

This policy should be $d$-dimensional, and the sum of its components 
should equal 1.  After that, you will be shown the reward vector for 
choosing each action.
 
Remember that this reward vector is decided by the external system
and can be potentially different for different rounds.

It is not decided by what policies you have chosen.
The reward vector is also $d$-dimensional. 

It represents the reward of choosing action from 0 to $d-1$. 

For example, a reward vector of [0.8, 3.2] means reward for action_0 is
0.8 and the reward for action_1 is 3.2.

Then your reward for this round will be calculated according to the 
reward of each action and your probability of choosing each action.

For example, if you choose the policy [0.2, 0.8] and get the reward 
vector [1, 2], then your expected reward is 0.2*1 + 0.8*2=1.8

Your goal is to maximize your accumulative expected reward. 

You can adjust your policy based on the reward vectors for all previous 
rounds. 

You're required to provide your policy in numeric format. 

Your response's last line should be formatted as
'Policy: [your $d$-dimensional policy]'. 

Let's think step by step. Explicitly examining history is important. 

Please explain how you chose the policy by guessing 
what reward you might receive for each action according to the history.

You should explore for first several rounds and behave greedily for 
later rounds, for example, choosing one action with probability more 
than 0.99. 

Please also explain whether you are behaving more greedily and less 
greedily  by explicitly considering the policy you just used for 
last round.
\end{Verbatim}

We can see in \Cref{fig:ablation} that the performances of LLM agents are consistent under different variants of the prompts.

\paragraph{Ablation study on repeated games.} For the game setting, we also investigate whether explicitly informing LLM agents that they are \texttt{``playing a repeated matrix game with some other opponents''} would affect the performance. Therefore, we evaluate three different prompts by informing LLM agents that they are \texttt{playing a matrix game}, \texttt{solving multi-arm bandit}, or solving general \texttt{decision-making} problems,  in the first line of the prompt. We show the performance of such three prompts in \Cref{fig:ablation-game}, where it is seen that LLM agents' performance on repeated games is consistent among these variants of the prompts. 

\begin{figure*}[!t]
\centering\includegraphics[width=0.36\textwidth]{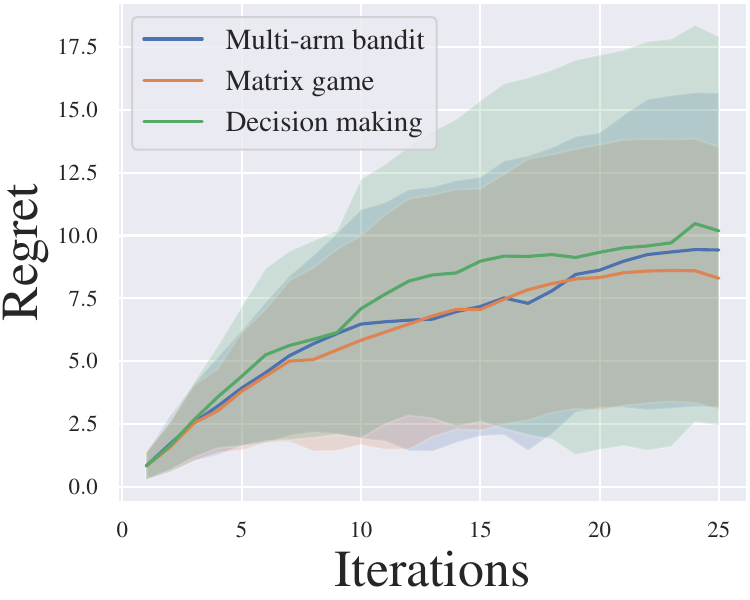}
 \caption{Regret of GPT-4 for repeated games under $3$ different prompt ablations. Its performance is consistent among three different prompts.}\label{fig:ablation-game}
 \end{figure*}

\subsection{Results for GPT-4 Turbo}
{When preparing our paper, we are aware of a more advanced version of GPT-4, GPT-4 Turbo. Although GPT-4 has successfully exhibited no-regret behaviors in most problems, we also  evaluate and report the performance of its more advanced version here. Specifically, we compare GPT-4 and GPT-4 Turbo for the repeated game scenario in \Cref{fig:gpt-4-turbo}, where we validate that GPT-4 Turbo performs either comparably with or better than GPT-4.}

\begin{figure*}[!t]
    \centering
    \begin{minipage}{0.33\textwidth}
        \centering
        \includegraphics[width=\linewidth]{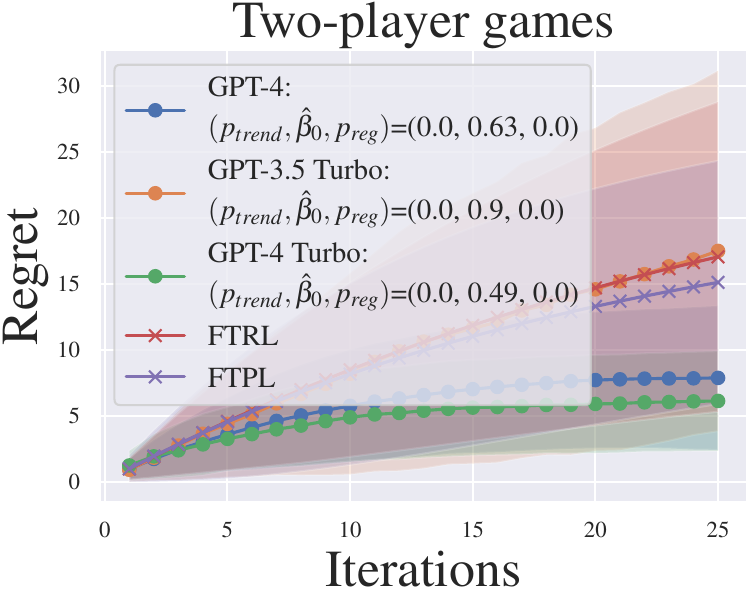}
    \end{minipage}\hfill
    \begin{minipage}{0.33\textwidth}
        \centering
        \includegraphics[width=\linewidth]{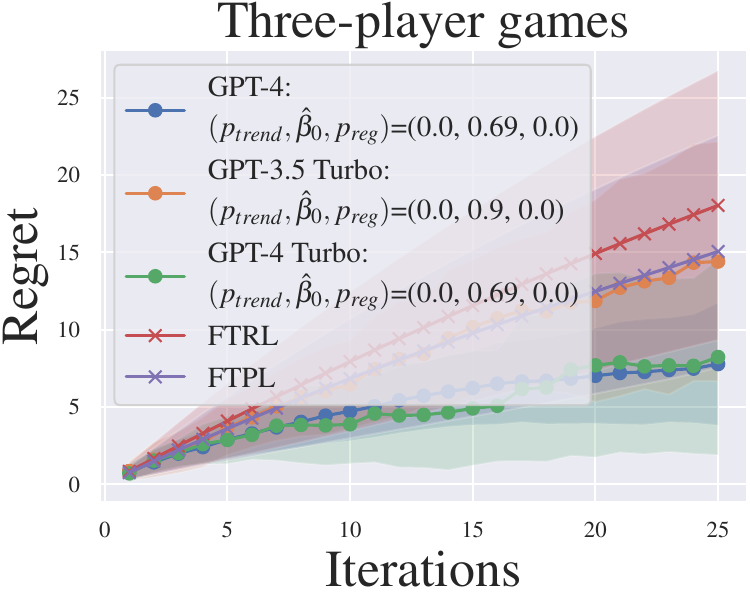}
    \end{minipage}
    \begin{minipage}{0.33\textwidth}
        \centering
        \includegraphics[width=\linewidth]{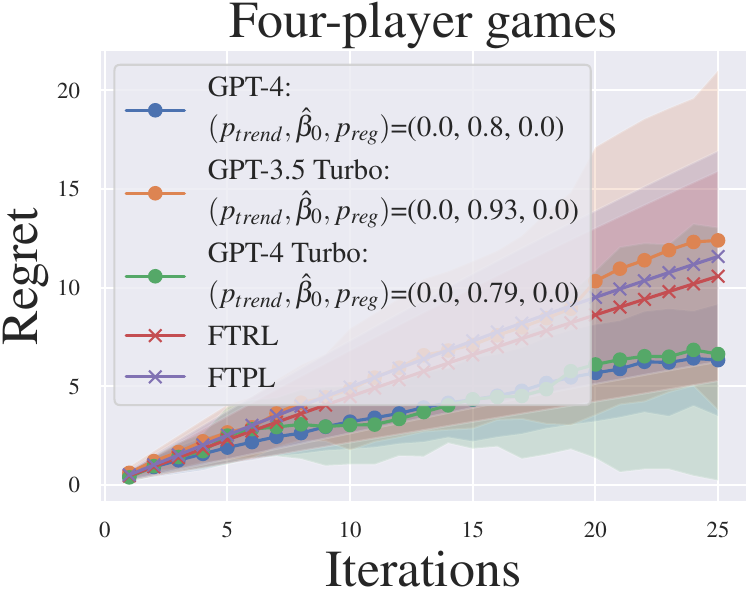}
    \end{minipage}
\caption{Comparisons of GPT-4 Turbo with GPT-4 and GPT-3.5 Turbo. {Due to the symmetry of agents in randomly generated games, we report the regret of one agent for ease of presentation. We can see that GPT-4 Turbo performs on par with or better than GPT-4.}}
\label{fig:gpt-4-turbo}    
\end{figure*}

\subsection{LLM Agents' Explanation  on Their Output Policies}\label{sec:LLM_explain}

We here present two examples of LLMs' explanations of their output policies, showing how they \emph{reason} in their decision-making. 
\safevspace{16pt}
\begin{Verbatim}
Looking at the history of the reward vectors, it is clear to see that 
the reward for choosing action 0 is consistently higher than that 
for choosing action 1.  This pattern is evident across all 5 rounds 
of the game.

In the first round, the difference in rewards is 3.52, this remained 
relatively stable with the difference in the last round being 3.52 as 
well. It's indicative that action 0 consistently brings a higher return 
than action 1.

Therefore, I'll lean more towards selecting action 0 in my policy. 

However, acting deterministically, that is always choosing action 0, 
would introduce the risk of becoming too predictable and letting the 
opponent exploit this predictability

So, in order to maintain some unpredictability while still capitalizing 
on the  higher average return of action 0, I'll set my policy 
as follows: select action 0 with a probability of 0.8 and action 1 
with a probability of 0.2.

This way, the majority of the time I'm choosing the action that gives 
a higher reward, but I'm also occasionally choosing the 
other action to 
keep the opponent guessing, and avoid  any predictable patterns that 
can be exploited

So, my policy is as follows: Policy: [0.8, 0.2]
\end{Verbatim}
\safevspace{16pt} 
\begin{Verbatim}
The total reward from action_1 is 130.92 and the average reward is 
around 10.07.

On the other hand, the total reward from action_0 is 33.418 and the 
average reward is around 8.354.

Although action_0 has a less average reward, its highest reward is 
26.977, which is considerably high and shows its potential.

For this round, I want to continue to explore action_0 but also 
significantly exploit action 1 as it reliably gives higher reward. 

So, I'd like to decrease the probability of selecting action_0 compared 
to last time but still keep it in the policy as a possible option.

The policy for the previous round was not provided but assuming 
it was more balanced, I'm now going to lean more towards action_1.

So, the policy will be: [0.3, 0.7]. Here, I'm behaving more greedily 
compared  to earlier rounds, as we are choosing the action which has 
been giving us higher rewards with more probability.
\end{Verbatim}

As observed from the output, the LLM tends to use the history of the reward vectors by looking at the ``sum/average of the accumulated rewards'', and tends to introduce randomization in decision-making, in order to be ``less predictable and exploitable''. These are several key components in achieving no-regret in online learning and games \citep{cesa2006prediction}, explaining the empirical evidence we had in \Cref{sec:empirical_insights}.

{\color{black} 
\subsection{Case Studies on Real-world Applications} 
In this subsection, we evaluate the sequential decision-making abilities of LLMs in realistic scenarios from the perspective of regret and dynamic regret. While several studies have explored sequential decision-making using synthetic scenarios \citep{krishnamurthy2024can,wu2023smartplay, xia2024beyond,akata2023playing} or real-world data scenarios \citep{liu2023chatgpt,wang2023recmind,wu2024surveylargelanguagemodels}, none have explicitly analyzed regret or dynamic regret. As a result, the \emph{theoretical optimality} of such a sequential decision-making process remains unclear.

Our first case study investigates single-agent sequential decision-making using real-world data, leveraging the same dataset and experimental setup as \citep{liu2023chatgpt}. The second case study explores a two-player negotiation scenario, providing insights into dynamic interactions and their impact on decision-making performance.

\subsubsection{Sequential Recommendation}
We consider the task of sequential recommendation, a task that people have been employing LLMs to solve with success \citep{liu2023chatgpt, wang2023recmind, wu2024surveylargelanguagemodels}. Note that how existing literature \citep{liu2023chatgpt} uses LLMs to solve this task fits exactly into our online learning framework, where humans feed a history of items the user have interacted with to the LLM and then ask the LLM to recommend the item (or several items) the user may want to interact next. The entire process carries on repeatedly. 

Formally, the problem is as follows. Given a sequence of history items the user has interacted with $(x_{1}, x_2, \cdots, x_{t-1})$, where each $x_i\in D$ for $i\in[t-1]$ and $D$ is the collection of all items, the LLM needs to recommend $n$ items that the user might interact with in the next step $t$. Typically, the LLM should also give a priority on the $n$ items it recommends. For simplicity here, we here assume they are of equal priority. In other words, at step $t$, the LLM will take an action $a_t\subseteq D$ with $|a_t|=n$, hoping what the user will interact at step $t$ belongs to $a_t$. Hence, the loss is given by $\ell_t(a_t, x_t) :=\bm{1}[x_t\not\in a_t]$. Correspondingly, the regret by our definition is given by 
\$
\text{Regret}(x_{1:T})=\sum_{t=1}^T \ell_t(x_t, a_t) -\min_{a} \sum_{t=1}^T \ell_t(x_t, a).
\$

We refer to \citep{liu2023chatgpt} for a more detailed introduction. Meanwhile, we use the real-world data and follow the experimental setup of \citep{liu2023chatgpt}. 

In the left one of \Cref{fig:rec}, we can observe that LLMs can achieve expressively low and sublinear regret on such a real-world application with real-world data. As a comparison, in the right one of \Cref{fig:rec}, we replace the real-world data with synthetic data generated in a uniformly random way (it is worth mentioning that the prompt setting still follows the setup of sequential recommendation of \cite{liu2023chatgpt}), where we can see that LLMs can still be no-regret. However, interestingly, LLMs perform better on real-world data, which validates that real-world applications can exhibit certain trends/structures, for which LLMs can exploit and achieve superior performance as we have shown in our paper through synthetic problems with trends.

\begin{figure*}[!t]
\centering\includegraphics[width=0.8\textwidth]{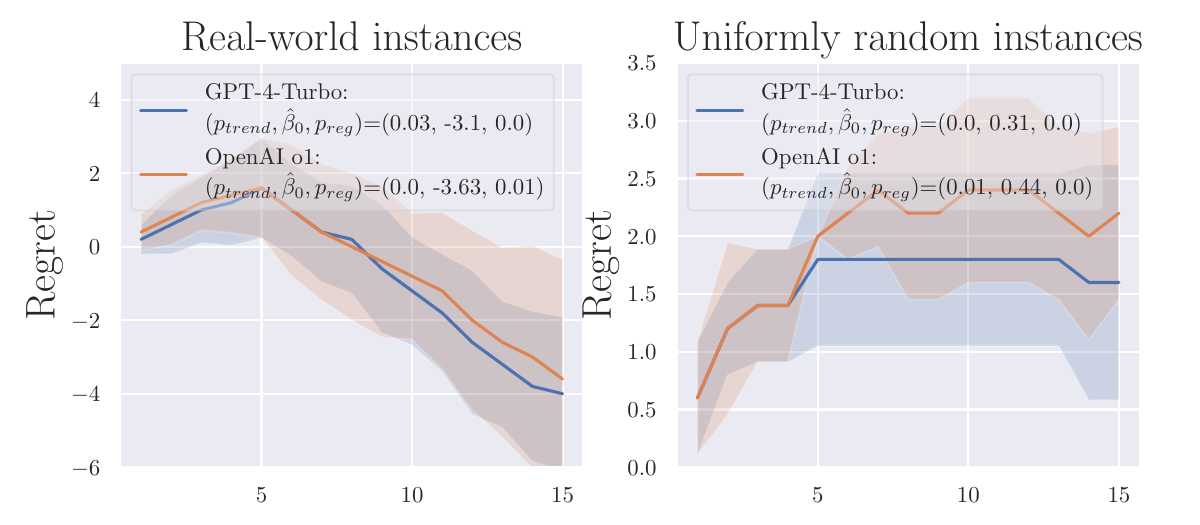}
 \caption{{\color{black}We evaluate GPT-4-Turbo and OpenAI o1 on both real-world data and uniformly random synthetic data, where we can see both models can still achieve sublinear regret.}}\label{fig:rec}
 \end{figure*}

\subsubsection{Interactive  Negotiation}
The experiment was designed to simulate negotiation scenarios between two LLMs, designated as LLM A and LLM B, across multiple turns. The primary objectives were to analyze multi-agent sequential decision-making processes and quantify regret. For each repetition, an LLM generated unique negotiation topics. Based on these topics, the LLM also created the context, objectives, and relevant background information to design engaging and interactive negotiation scenarios.  

\paragraph{Negotiation Process.} The negotiation process was executed in a turn-based manner, with each turn comprising three steps:

\begin{enumerate}
    \item \textbf{Intention Generation}: Each LLM defined its goal for the turn, specifying what it aimed to achieve with its response.
    
    \item \textbf{Response Generation}: Based on the defined intention and the dialogue history, each LLM generated a response.
    
    \item \textbf{Alternative Response Generation}: Three distinct alternative replies were produced for each original response. These alternatives represented diverse negotiation strategies or perspectives while preserving the original intention.
\end{enumerate}

\paragraph{Response Evaluation.} After the dialogue concluded, all responses—both original and alternatives—were evaluated using a scoring scale from 1 to 10 based on the following criteria for each turn:

\begin{itemize}
    \item \textbf{Clarity}: How clear and understandable the reply is.
    \item \textbf{Relevance}: How pertinent the reply is to the negotiation topic and the defined intention.
    \item \textbf{Engagement}: How engaging or persuasive the reply is in fostering further dialogue.
    \item \textbf{Alignment with the Stated Intention}: How well the conversation aligns with the turn's stated intention following the reply. For alternative replies, this was assessed by hypothetically replacing the original reply with an alternative and evaluating the alignment based on the entire conversation.
\end{itemize}

Each response was scored using an LLM as the evaluator. Although human evaluation would be preferable, the use of an LLM as a scorer was chosen for scalability. This approach is common in the LLM domain and is sometimes referred to as G-eval (where ``G'' stands for GPT) \citep{liu2023g}.

\paragraph{Dynamic Regret Analysis. }
Finally, dynamic regret was calculated to measure suboptimality by comparing the scores of the original replies against the highest-scoring alternative responses. Since calculating regret typically requires hindsight knowledge of the best possible responses, which requires rollout of every possible dialogues, we decide to analyze on dynamic regret. Dynamic regret analysis provided a quantitative measure of decision-making effectiveness across turns. This analysis offered insights into how regret dynamics can inform improved decision-making strategies in real-world negotiation contexts.

\paragraph{Example.}
Here is an example from our simulation:

{\color{black} \textbf{Step 1: Generate Topics and Backgrounds.} Topics and backgrounds were generated using a language model.}

\begin{mdframed}[backgroundcolor=gray!10, linewidth=1pt]
{\color{black}\textit{Topic}: The Trade-Off Negotiation Between Eco-Tech Innovator and Traditional Manufacturing Tycoon

\textit{Background of Player A: }
Eco-Tech Innovator (Jordan Green). Jordan Green is the CEO of a rapidly growing startup, EcoWave Technologies, which specializes in developing sustainable energy solutions and eco-friendly manufacturing processes. With a background in environmental science and engineering, Jordan is passionate about reducing carbon footprints and promoting renewable energy sources. Their innovative products, such as biodegradable materials and energy-efficient machinery, have garnered attention and accolades within the green tech community. However, despite the startup's promise, EcoWave faces challenges in scaling production and reaching wider markets due to limited financial resources and manufacturing capabilities.

\textit{Background of Player B: }
Traditional Manufacturing Tycoon (Robert Steele). Robert Steele is the owner of Steele Industries, a well-established manufacturing company known for its mass production of consumer goods. With decades of experience in the industry, Robert has built a reputation for efficiency and profitability, but his company has been criticized for its environmental impact and reliance on outdated practices. Facing increasing pressure from regulators and consumers to adopt sustainable practices, Robert is aware that failing to adapt could threaten the future of his business. However, he is also concerned about the costs and risks associated with transitioning to new technologies and processes, which could disrupt his current operations.

\textit{Detailed Situation:} Jordan and Robert have agreed to meet to discuss a potential partnership. Robert is interested in integrating EcoWave's sustainable technologies into his manufacturing processes, which could help Steele Industries improve its environmental footprint and public image. However, he is also hesitant about the costs associated with adopting these new technologies and is particularly concerned about how this might affect his established workforce and production timelines. On the other hand, Jordan sees this as a golden opportunity to scale EcoWave's innovative solutions and gain access to Robert's extensive distribution networks. However, Jordan is wary of compromising on key values regarding environmental integrity and sustainable practices, fearing that a partnership with a traditional manufacturer could dilute EcoWave's mission. The negotiation aims to balance Robert's need for cost-effectiveness and operational continuity with Jordan's commitment to sustainability and innovation. Both players must navigate their differing priorities, explore potential compromises, and consider the long-term implications of their decisions on their businesses and the environment.}
\end{mdframed}
{\color{black} \textbf{Step 2: Generate Conversation.}} We iterated between {\color{black} \textbf{Step 2-1}} and {\color{black} \textbf{Step 2-2}} to create a conversation with a length of 20 turns. In this setup, each LLM (A and B) contributed 10 responses based on the evolving conversation history.

{\color{black} \textbf{Step 2-1: Intention Generation and Reply Generation.}} We provided the initial turn’s intention for LLM A: 

\begin{mdframed}[backgroundcolor=gray!10, linewidth=1pt]
{\color{black}\textit{Intention}: In my next response, I aim to outline potential trade-offs that both I and Robert might consider in order to reach a mutually beneficial agreement. I will also suggest ways we can ensure that the partnership aligns with their core values of sustainability and operational efficiency. This will help set the stage for a constructive dialogue that addresses our respective concerns and priorities. 

\textit{Reply of Player A based on the intention: }  Thank you for meeting with me, Robert. I believe we have a unique opportunity to collaborate that could benefit both EcoWave and Steele Industries. I understand your concerns about the costs and risks of integrating new technologies, but I think we can explore some trade-offs. For instance, if we can demonstrate the long-term cost savings and efficiency gains from adopting sustainable practices, it might help alleviate some of your concerns. Additionally, I'm open to discussing a phased implementation approach that allows for gradual integration without disrupting your current operations. How do you feel about that, and what trade-offs might you be willing to consider to make this partnership work?}
\end{mdframed}

{\color{black} \textbf{Step 2-2: Alternative Generation.}} The LLM also generated three alternative responses based on the aforementioned intention. 

\begin{mdframed}[backgroundcolor=gray!10, linewidth=1pt]
{\color{black}\textit{Alternative 1: } I appreciate your concerns, Robert, and it's great to see your commitment to a smooth transition. Instead of jumping straight into a pilot project, how about we organize a joint workshop with your key team members? This would allow us to collaboratively explore areas where EcoWave's technologies can fit into your existing operations. We can analyze potential impacts, address any workforce concerns, and identify specific processes that could benefit from our innovations. This way, we ensure that everyone is on board and feels included in the transition. What are your thoughts on setting up a workshop to discuss this further?

\textit{Alternative 2: } I appreciate your concerns, Robert, and I think it's important for us to find a starting point that feels comfortable for you. Instead of a full-scale integration right away, how about we explore a collaborative approach where we can jointly assess some specific production lines that may benefit from our technologies? We could conduct a thorough analysis of these areas together, identifying potential efficiency improvements and cost benefits through a detailed study. This way, we can base our decisions on solid data while ensuring that your existing operations remain stable. Additionally, I'd be more than willing to work closely with your team to address any training needs that may arise. Would this kind of collaborative assessment resonate with you?

\textit{Alternative 3: } I appreciate your insights, Robert, and I recognize the importance of ensuring a smooth transition. Instead of jumping straight into a pilot project, how about we start by conducting a joint assessment of your current operations? This would allow us to pinpoint specific areas where EcoWave's technologies could fit seamlessly. By collaborating on this assessment, we can address your concerns and tailor our solutions to your unique processes. Additionally, I'd be happy to share case studies from other manufacturers who have successfully integrated our technologies, demonstrating the benefits they achieved without significant disruptions. What do you think about this collaborative approach?}
\end{mdframed}

\begin{wrapfigure}{r}{0.35\textwidth}  
\centering
  \includegraphics[width=0.35\textwidth]{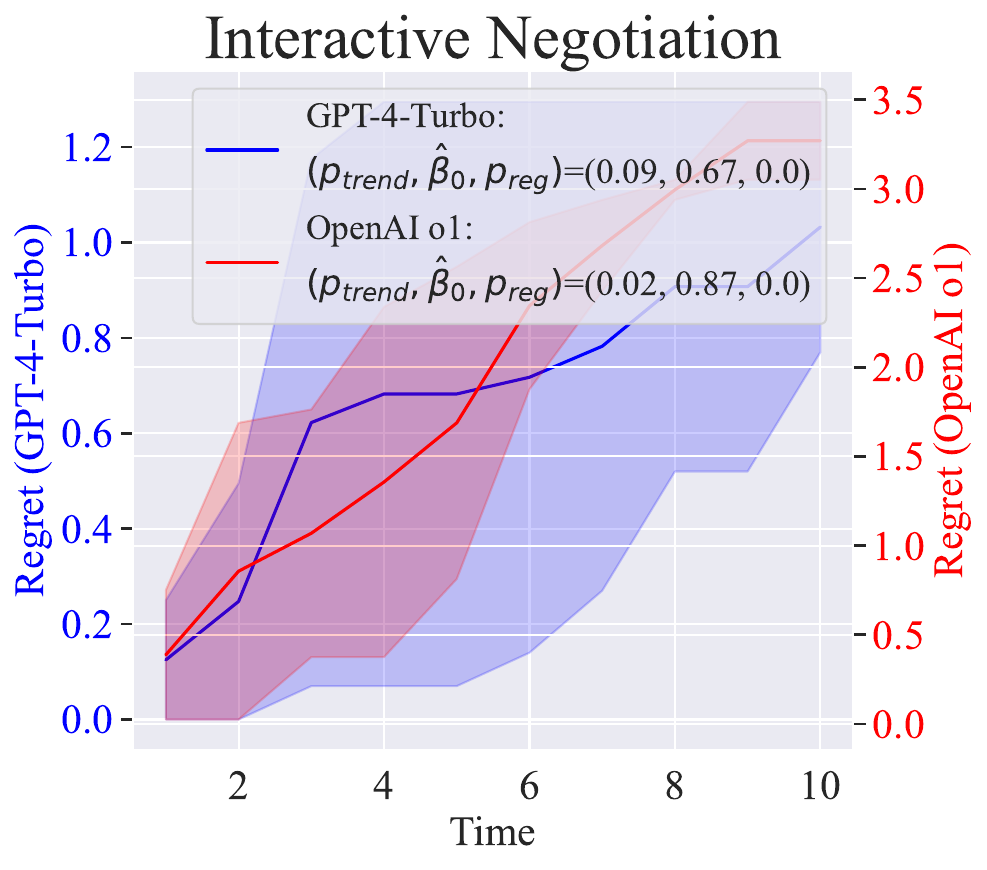}
  \vspace{-10mm}
  \caption{{\color{black}We evaluate GPT-4-Turbo and OpenAI o1 on two-player Negotiation, where we can see both models can achieve no-regret in the regression framework.}}
  \label{fig:nego}
\end{wrapfigure}

\textbf{Step 3: Evaluation.} Each reply and its alternatives were evaluated using G-eval.

\paragraph{Experiment Setting.} We set the maximum number of turns to 10 and evaluated two language models: GPT-4-Turbo and OpenAI’s o1 model. For each model, we tested performance across 10 different scenarios. 

\paragraph{Results.} Interestingly, both models demonstrated no-dynamic-regret behavior within the regression framework. Both models also exhibited low \(p\)-values, with GPT-4-Turbo achieving 0.09 and OpenAI o1 achieving 0.02. The detailed results are shown in \Cref{fig:nego}.

}
}

\safevspace{80pt}

\icml{\subsection{Detailed Environment Settings for the Experiments in \Cref{sec:experiments_online}}
\label{ssec:environment}

We now present the experimental results on the no-regret behavior of pre-trained LLMs in online learning in: 1) arbitrarily changing environments, 2) non-stationary environments, and 3) bandit-feedback environments. 

\paragraph{Online learning in arbitrarily changing environment.} 
We first consider the setting with  arbitrarily changing environments, which are instantiated as follows:  
1) \emph{Randomly-generated loss sequences}. At every timestep, we generate a random loss vector ${\ell}_t\sim \operatorname{Unif}([0, 10]^d)$ or ${\ell}_t\sim \cN(5\pmb{1}_d, I)$ 
with clipping {to $[0,10]$} 
to ensure the boundedness, such that the loss vectors of different timesteps can be arbitrarily different; 
2) \emph{Loss sequences with a predictable trend.} Although many real-world environments may change, they often change following certain patterns. Therefore, we consider two representative trends, a \emph{linear} trend and a \emph{periodic}  (sinusoid) trend. For the linear trend, we sample $a, b\sim\text{Unif}([0, 10]^d)$ and let $\ell_t = (b-a)\frac{t}{T}+a$  for each $t\in[T]$. For the periodic trend, we sample $a, b\sim\text{Unif}([0, 10]^d)$ and let $\ell_t = 5(1+\sin(at+b))$ for each $t\in[T]$. In the experiments, we choose $d=2$.  The average regret (over multiple randomly generated instances) performance is presented in \Cref{fig:online},  
where we compare GPT-4 with well-known no-regret algorithms, FTRL with entropy regularization {and FTPL with Gaussian perturbations} (with tuned parameters). It is seen that these pre-trained LLMs can indeed achieve no-regret and often have smaller regrets than these baselines.

\paragraph{Online learning (in non-stationary environment).} 
We then experiment on the setting where  the losses are still changing over time,  but their total variations across time are bounded, more concretely, sublinear in $T$. Correspondingly, we consider the stronger metric of \emph{dynamic regret} here to measure the performance. Note that without constraining the variation of the loss vectors, dynamic regret can be linear w.r.t. $T$ in the worst case. 
Hence, we generate the loss vectors in two different ways:  
1) \emph{Gradual variation.} We firstly sample ${\ell}_1\sim \text{Unif}([0, 10]^d)$. Then for each $t\ge 2$, we uniformly and randomly generate ${\ell}_{t+1}$ under the constraint $\|{\ell}_{t+1}-{\ell}_t\|_\infty\le \frac{1}{\sqrt{t}}$, such that the variations over time are guaranteed to satisfy $\sum_{t=1}^{T-1}\|{\ell}_{t+1}-{\ell}_t\|_{\infty} = {o(T)}$; 
2) \emph{Abrupt variation.} We randomly generate ${\ell}_1\sim\text{Unif}([0, 10]^d)$ and $m$ time indices $\{t_{i}\}_{i\in[m]}$ from $\{1, 2, \cdots, T\}$. At each time step $t_i$ for $i\in[m]$, the sign of the loss vector ${\ell}_{t_i}$ is flipped, i.e.,  we let 
${\ell}_{t_i}\leftarrow 10\pmb{1}_d-{\ell}_{t_i}$. 
For the specific choice of $T=25$ in our experiments, we choose $m=3$.  
For both cases, the average dynamic regret results are presented  in \Cref{table:no_regret_statistics}. It can be seen that GPT-4 achieves sublinear dynamic regret and outperforms \emph{FTRL/FTPL with Restart}, a standard variant of FTRL/FTPL for non-stationary online learning (see e.g., \cite{besbes2014stochastic}). {We refer to \Cref{sec:intro-dr} for a more detailed introduction of FTRL/FTPL {with Restart}.} 

{
\paragraph{Extension to bandit-feedback settings.} 
Although pre-trained LLMs have achieved good performance in online learning with full-information feedback, it is unclear whether they can still maintain no-regret with only bandit feedback. 
For such problems, we modify the prompt and protocol of interactions slightly, where we still ask the LLM agent to provide a policy $\pi_t$ at time step $t$, then sample one $a_t\sim\pi_t(\cdot)$. {In the bandit setting, the LLM agent can only access $(a_t,\ell_{ta_t})$. Instead of directly feeding it to the agent, we feed an estimate of the loss vector $\hat{\ell}_t\in\RR^d$, where $\hat{\ell}_{t}(a) \leftarrow \frac{\ell_{t}(a)}{\pi_{t}(a)}\mathbbm{1}(a_t = a)$ for all $j \in [d]$.} Note that such an operation of \emph{re-weighting} the loss by the inverse of the probability is standard in online learning when adapting full-information-feedback no-regret algorithms to the bandit-feedback ones \citep{auer2002nonstochastic}. 
Later, we will also show the benefits of such operations (cf.  \Cref{sec:theoretical_insights}). We compare the performance of pre-trained LLMs with that of the counterparts of FTRL with bandit feedback,  e.g., EXP3 \citep{auer2002nonstochastic} and the bandit-version of FTPL \citep{abernethy2015fighting}, in both \Cref{fig:bandit-feedback} and \Cref{table:no_regret_statistics}, where GPT-4 consistently achieves lower regret.}
{ 
\begin{figure*}[!t]
\hspace{-8pt}	\centering\includegraphics[width=1.02\textwidth]{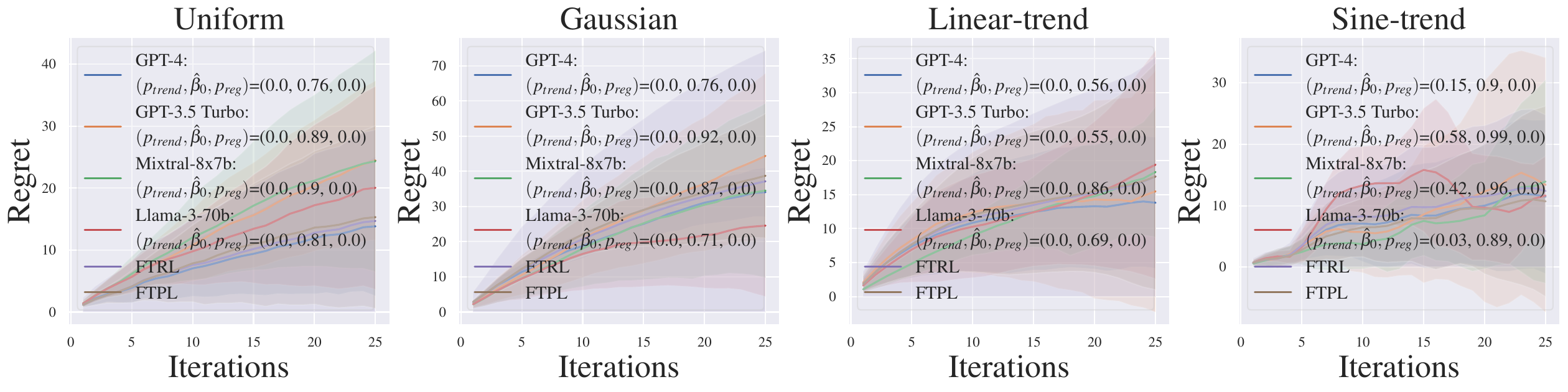}
 \caption{Regret of {GPT-3.5 Turbo/}GPT-4 for online learning with bandit feedback in $4$ different settings. It performs comparably and sometimes even better  than well-known no-regret learning  algorithms, {variants of FTRL and  FTPL with bandit-feedback}.}\label{fig:bandit-feedback}
 \end{figure*}
}
}

\conf{\subsection{Detailed Environment Settings for the Experiments in \Cref{sec:experiments_online}}
\label{ssec:environment}
\paragraph{Online learning in arbitrarily changing environment.} 
We first consider the setting with arbitrarily changing environments, with the following instantiations:  
1) \emph{Randomly-generated loss sequences}. At every timestep, we generate a random loss vector ${\ell}_t\sim \operatorname{Unif}([0, 10]^d)$ or ${\ell}_t\sim \cN(5\cdot\bm{1}_d, I)$ 
with clipping {to $[0,10]$} 
to ensure the boundedness, such that the loss vectors of different timesteps can be arbitrarily distinct; 
2) \emph{Loss sequences with predictable trend.} Although real-world environments can change arbitrarily, they could often exhibit certain patterns. Therefore, we consider two representative trends, a \emph{linear} trend and a \emph{periodic}  (sinusoid) trend. For the linear trend, we sample $a, b\sim\text{Unif}([0, 10]^d)$ and let $\ell_t = (b-a)\frac{t}{T}+a$ for each $t\in[T]$. For the periodic trend, we sample $a, b\sim\text{Unif}[0, 10]^d$ and let $\ell_t = 5(1+\sin(at+b))$ for each $t\in[T]$. In the experiments, we choose $d=2$.  The average regret (over multiple randomly generated instances) performance is presented in \Cref{fig:online}, 
where we compare GPT-4 with well-known no-regret algorithms, FTRL with entropy regularization {and FTPL with Gaussian perturbations} (with tuned parameters). It is seen that these pre-trained LLMs can indeed achieve no-regret and often have smaller regrets than baselines.

\paragraph{Online learning (in non-stationary environment).} 
We then experiment on the setting when the losses are still changing over time but their variations across time are bounded, more concretely, sublinear in $T$. Correspondingly, we consider the stronger metric of \emph{dynamic regret} here to measure the performance. Note that without constraining the variation of the loss vectors, dynamic regret can be linear w.r.t $T$ in the worst case. 
Hence, we generate the loss vectors in two different ways:  
1) \emph{Gradual variation.} We firstly sample ${\ell}_1\sim \text{Unif} [0, 10]^d$. Then for each $t\ge 2$, we uniformly and randomly generate ${\ell}_{t+1}$ under the constraint $\|{\ell}_{t+1}-{\ell}_t\|_\infty\le \frac{1}{\sqrt{t}}$, such that the variations over time are guaranteed to satisfy $\sum_{t=1}^{T-1}\|{\ell}_{t+1}-{\ell}_t\|_{\infty} = \cO(\sqrt{T})$; 
2) \emph{Abrupt variation.} We randomly generate ${\ell}_1\sim\text{Unif}[0, 10]^d$ and $m$ time indices $\{t_{i}\}_{i\in[m]}$ from $\{1, 2, \cdots, T\}$. At each time step $t_i$ for $i\in[m]$, the sign of the loss vector ${\ell}_{t_i}$ is flipped, i.e.,  we let 
${\ell}_{t_i}\leftarrow 10-{\ell}_{t_i}$. 
For the specific choice of $T=25$ in our experiments, we choose $m=3$.  
For both cases, the average dynamic regret results are presented { in \Cref{table:no_regret_statistics}. It can be seen that GPT-4 achieves sublinear dynamic regret and outperforms Restart FTRL/FTPL.} \kz{if some results are NOT GOOD, we should mention here also. just dont forget to explain everything.} \cp{which results are bad btw?}

{
\paragraph{Extension to bandit-feedback settings.} 
Although pre-trained LLMs have achieved good performance in online learning with full-information feedback, it is unclear whether they can still maintain no-regret with only bandit feedback. 
For such problems, we modify the prompt and protocol of interactions slightly, where we still ask the LLM agent to provide a policy $\pi_t$ at time step $t$, but manually sample one $a_t\sim\pi_t$ and then inform the agent of the sampled action $a_t$, together with the loss corresponding to  that action, i.e., {$\ell_{tj} \leftarrow \frac{\ell_{tj}}{\pi_{tj}}\mathbbm{1}(a_t = j)$ for all $j \in [d]$ instead of providing $\ell_{ta_t}$.} Note such an operation of \emph{re-weighting} the loss by the inverse of the probability is standard in online learning when adapting full-information-feedback no-regret algorithms to the bandit-feedback ones. 
Later, we will also show the provable benefits of such operations (cf.  \Cref{sec:theoretical_insights}). We compare the performance with the counterparts of FTRL in the bandit-feedback setting, e.g., EXP3 \citep{auer2002nonstochastic} and the bandit-version of FTPL \citep{abernethy2015fighting} in both \Cref{fig:bandit-feedback} and \Cref{table:no_regret_statistics}, where GPT-4 consistently achieves lower regret.}\cp{if we change more I will copy and paste to appendix - just lmk.}.} 

\conf{\subsection{Detailed Environment Settings for the Experiments in \Cref{sec:experiments_games}}
\label{ssec:environment-game}
\paragraph{{Representative  games.}}
{We first test on all $6$ representative two-player general-sum games (\emph{win-win, prisoner's dilemma, unfair, cyclic, biased,} and \emph{second best}) studied in \cite{robinson2005topology} (see a detailed introduction of these games in \Cref{appendix:six})}. For each type of the game, we conduct $20$ repeated experiments.

\paragraph{{Randomly generated  games.}} 
To further validate the no-regret behavior of LLM agents, we also test on $50$ randomly generated three-player general-sum games, and $50$ randomly generated four-player general-sum games, where each entry of the payoff matrix is sampled randomly from $\operatorname{Unif}[0, 10]$. These are larger and more challenging settings than the two-player and structured cases above.  

We summarize experimental results in \Cref{fig:average}, which 
 are similar to the above: for all types of games, GPT-4 agents achieve sublinear regret, which is comparable with that obtained by FTRL for most games (See \Cref{fig:333} and \Cref{fig:3333} for
more results).   

\begin{figure*}[!t]
	\includegraphics[width=0.93\textwidth]{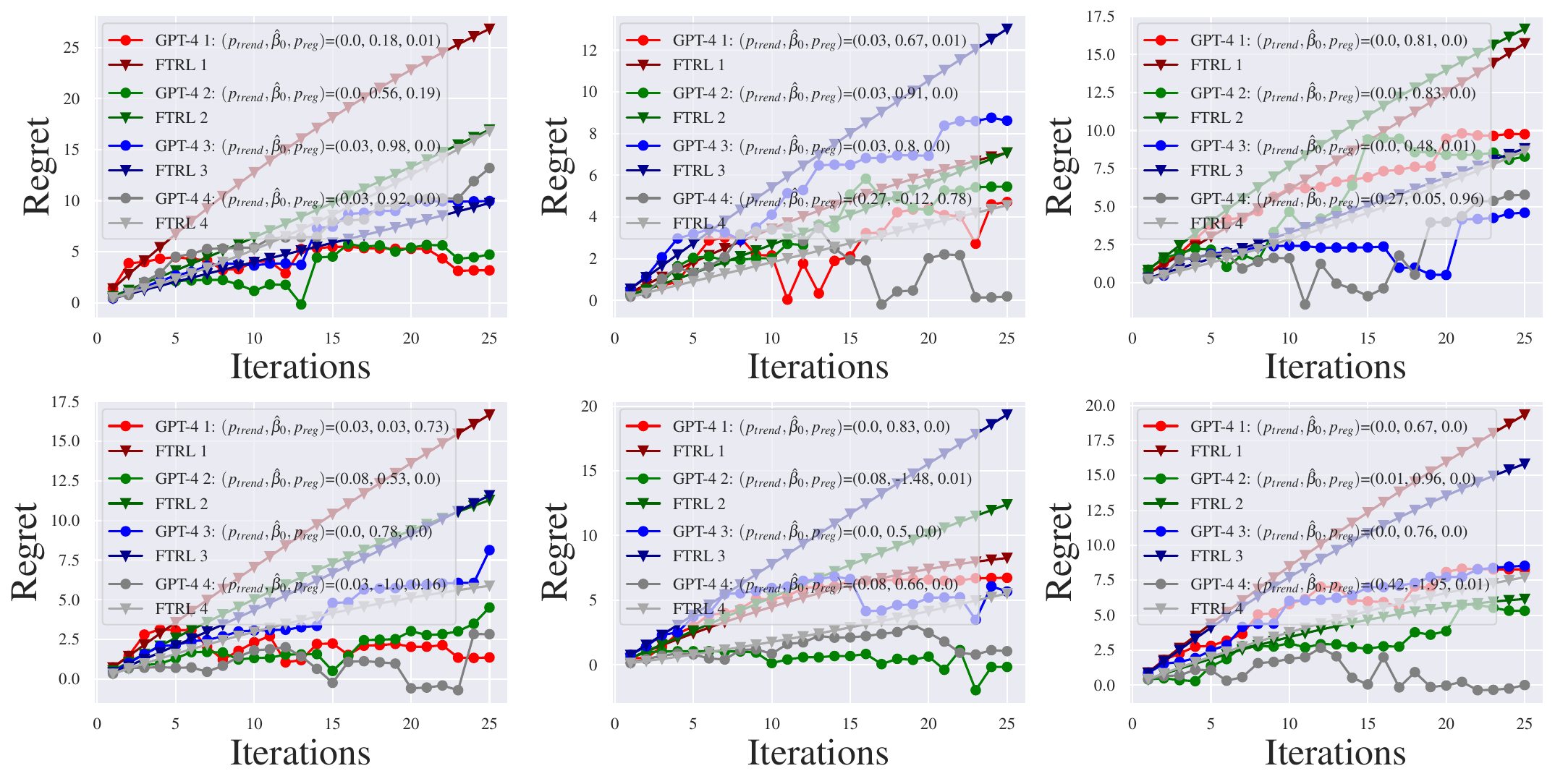}
\caption{Regret of GPT-4 agents and the FTRL algorithm in $6$ randomly generated four-player general-sum games. GPT-4 agents have comparable (even better)  no-regret properties when compared with the FTRL algorithm. }
\label{fig:3333}
\end{figure*}

\begin{figure*}[!t]	\centering\includegraphics[width=0.95\textwidth]{NeurIPS/figs/game_022724.pdf}
	\caption{Demonstration of the prompts used for multi-player repeated games. A human moderator does not provide the game's payoff matrices to the LLM agents. Instead, at each round, the human moderator provides each player's own payoff vector history.}
 \label{fig:game-prompt}
 \end{figure*}
 \begin{figure*}[!t]
	\includegraphics[width=0.93\textwidth]{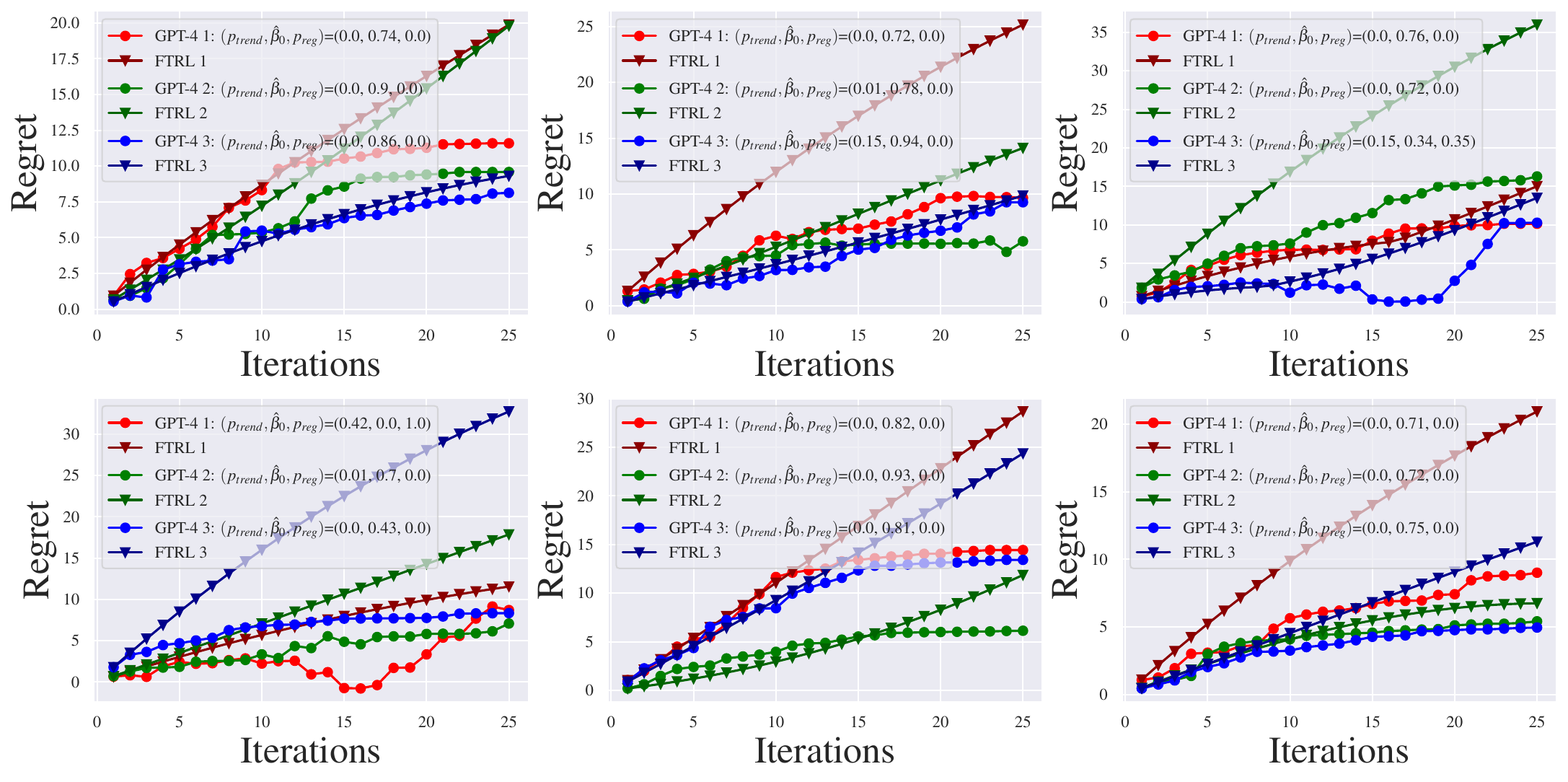}
\caption{Regret of GPT-4 agents and the FTRL algorithm in $6$ randomly generated  three-player  general-sum games. GPT-4 agents have comparable (even better)  no-regret properties when compared with the FTRL algorithm. \kz{larger legend/axis fonts, the same as caption. check/change overall.} \cp{we need to change the number of the experiment we did. (add additional results)}}
\label{fig:333}
\end{figure*}

}
\conf{\subsection{Detailed Environment Settings for the Experiments in \Cref{sec:yes_regret_example}}
\label{ssec:environment-counter}
To begin with, we consider a well-known example that \emph{follow-the-leader} (FTL) algorithm \cite{shalev2012online} suffers from linear regret \cite{hazan2016introduction}, where $\ell_{11}=5, \ell_{12}=0$ and $\ell_{t(2-t\%2)} = 10, \ell_{t(1+t\%2)} = 0$ for $t\ge 2$ where $\%$ is the modulo operation. Interestingly, GPT-4 agent can easily identify the pattern for the loss sequence that the optimal action alternates, thus accurately predicting the loss it will receive and achieving near zero regret in \Cref{fig:FTL}. 

 \begin{figure}[!t]
\centering\includegraphics[width=0.45\textwidth]{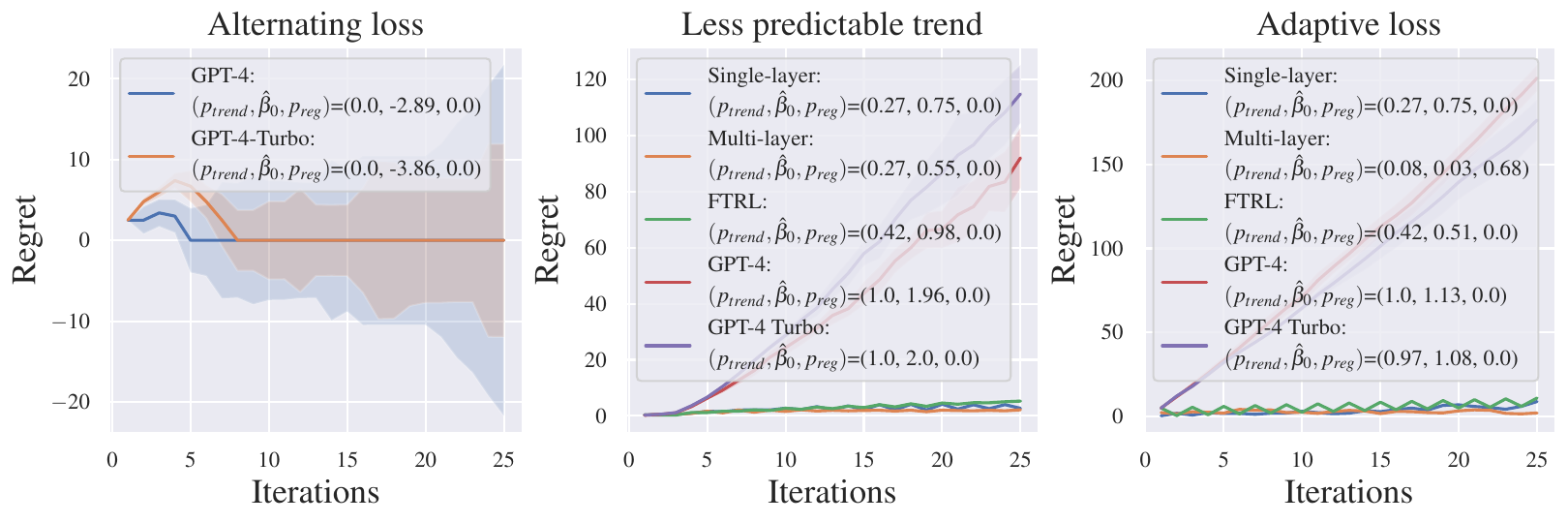}
	\caption{\kz{move this back?} Regret of both GPT-4 and GPT-4-Turbo under the seminal counter-example for FTL.}\label{fig:FTL}
 \end{figure}

\paragraph{Less predictable loss sequence. } Inspired by FTL observation, we design a new loss sequence that is \emph{similar but less predictable}. Specifically,
we construct the following (simple) loss sequence with $2$ actions such that ${\ell}_{t(1+t\%2)} = \min(25/t, 10), {\ell}_{t(2-t\%2)} \sim\text{Unif}[9, 10]$  for $t \in [25]$. 

\paragraph{Adaptive loss sequence.} We also develop a simpler but more \emph{adaptive} loss sequence that takes the full power of the adversary in the online learning setup. After the GPT-4 agent provides $\pi_t$, we choose ${\ell}_t$ such that ${\ell}_{t(\argmax_{i}\pi_{ti})} = 10$ and ${\ell}_{t(3-\argmax_{i}\pi_{ti})} = 0$. We report the regret averaged for the later two settings over $20$ repeated experiments using GPT-4 and more advanced GPT-4 Turbo in \Cref{fig:counter}, where linear regret is confirmed by both trend-checking and regression-based frameworks.

\begin{figure*}[!t]
	\centering\includegraphics[width=0.99\textwidth]{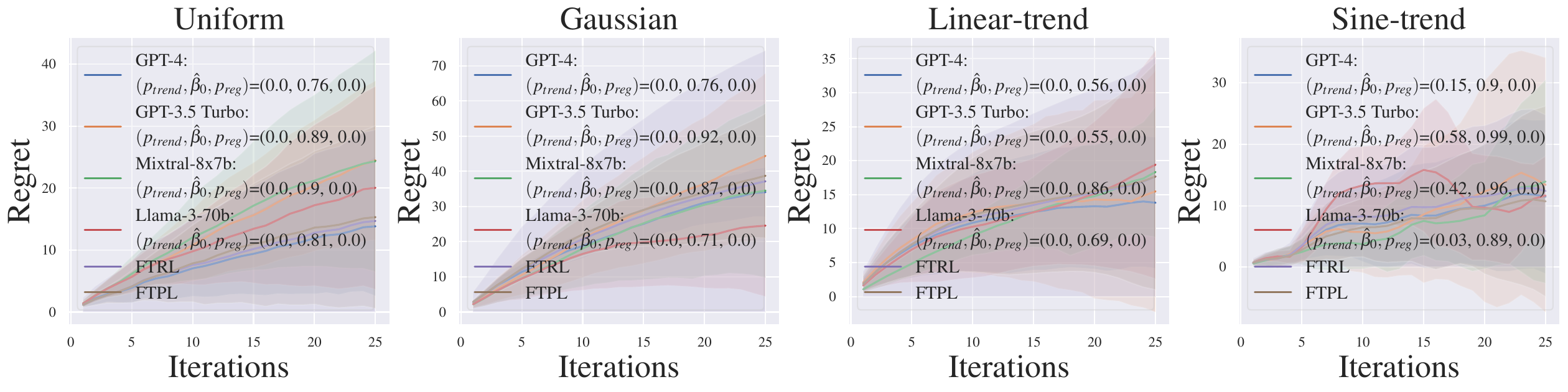}
	\caption{We report regret of GPT-4 for online learning with bandit feedback in $4$ different settings. It performs comparably or better than the bandit version of well-known no-regret algorithms, FTRL, FTPL.}\label{fig:bandit-feedback}
 \end{figure*}}

\vspace{5mm}
\section{Deferred Results and Proofs   in \Cref{sec:theoretical_insights}}
\label{appendix:explanation_theoretical_insights}
\conf{\subsection{Pre-Trained LLMs Have  Similar  Regret as Humans (Who Generate Data)}

Recently, a growing literature has evidenced that the intelligence level 
of LLM agents are determined by, and in fact mimic, those of human beings who generate the data for pre-training the models \citep{park2022social,argyle2023out,horton2023large}.  
The key rationale was that, LLMs (with Transformer parameterization) can approximate the \emph{pre-training data distribution} very well  \citep{xie2021explanation,zhang2023and,lee2023supervised}. 
In such a context, one can expect that LLM agents can achieve similar regret as human decision-makers who generate the pre-training data, as we formally state below.

\begin{restatable}{observation}{regretnotraining}
\label{thm:regret-notraining}
    An LLM {agent} is said to be {pre-}trained with an \emph{$\epsilon$-decision error} if, for any arbitrary $t$ and loss sequences $({\ell}_i)_{i \in [t]}$, the following condition holds: 
    \begin{align*}
        \sup_{\pi \in \Pi} \big| P_{\text{data}}(\pi \given  ({\ell}_i)_{i \in [t]}) - P_{\text{LLM}}(\pi \given  ({\ell}_i)_{i \in [t]})\big|  \leq \epsilon, 
    \end{align*}
    where $P_{\text{data}}$ and $P_{\text{LLM}}$ are the {pre-training data distribution and the pre-trained LLM model}, respectively. 
Then the regret of {an LLM agent}  with $\epsilon$-decision error is bounded as: 
{\begin{align*}
    &\text{(D-)Regret}_{\text{LLM}}\left(({\ell}_t)_{t \in [T]} \right)
    \\
    &\in  \left[\text{(D-)Regret}_{\text{data}}\left(({\ell}_t)_{t \in [T]}\right) 
    \pm \epsilon \\  orm{{\ell}_t} \sup_{\pi \in \Pi} \\  orm{\pi}\right],
\end{align*} 
    where $[a \pm b] := [a-b, a+b]$. }
\end{restatable} 
{
\Cref{thm:regret-notraining} shows that the pre-trained LLM-agent's regret can be controlled by that of the pre-training dataset and the decision error $\epsilon$. A small $\epsilon$ can be achievable if LLM is constructed with the Transformer architecture \citep{zhang2023and,lin2023transformers}. 
}
}

\subsection{Pre-Trained LLMs Have Similar  Regret as Humans (Who Generate Data)}\label{app:observa}

{We first provide a direct observation based on some existing speculation on the capability of Transformer-based LLMs.} 
Recently, a growing literature has evidenced that the intelligence level 
of LLM agents are determined by, and in fact mimic, those of human beings who generate the data for pre-training the models \citep{park2022social,argyle2023out,horton2023large}.  
The key rationale was that, LLMs (with Transformer parameterization) can approximate the \emph{pre-training data distribution} very well  \citep{xie2021explanation,zhang2023and,lee2023supervised}. 
In such a context, one can expect that LLM agents can achieve similar regret as human decision-makers who generate the pre-training data, as we formally state below. 

\begin{restatable}{observation}{regretnotraining}
\label{thm:regret-notraining}
    An LLM {agent} is said to be {pre-}trained with an \emph{$\epsilon$-decision error} if, for any arbitrary $t$ and loss sequences $({\ell}_i)_{i \in [t]}$, the following condition holds:
    \begin{align*}
        \sup_{\pi \in \Pi} \quad \big| P_{\text{data}}(\pi \given  ({\ell}_i)_{i \in [t]}) - P_{\text{LLM}}(\pi \given  ({\ell}_i)_{i \in [t]})\big|  \leq \epsilon, 
    \end{align*}
    where $P_{\text{data}}$ and $P_{\text{LLM}}$ are the {pre-training data distribution and the decision policy distribution of the  pre-trained LLM}, respectively. 
Then, the regret of {an LLM agent}  with $\epsilon$-decision error is bounded as: 
\begin{align*}
&\text{(D-)Regret}_{\text{LLM}}\left(({\ell}_t)_{t \in [T]} \right) \in  \left[\text{(D-)Regret}_{\text{data}}\left(({\ell}_t)_{t \in [T]}\right) 
    \pm \epsilon \norm{{\ell}_t} \sup_{\pi \in \Pi} \norm{\pi}\right],
\end{align*} 
    where $[a \pm b] := [a-b, a+b]$. 
\end{restatable}
\Cref{thm:regret-notraining} shows that the pre-trained LLM-agent's regret can be controlled by that of the pre-training dataset and the decision error $\epsilon$. A small $\epsilon$ can be achieved if LLM is constructed by a rich function class, e.g., the Transformer architecture \citep{zhang2023and,lin2023transformers}. 

\begin{proof}[Proof of \Cref{thm:regret-notraining}]
    For given $(\ell_{t})_{t \in [T]}$,  
    \begin{align*}
        \sum_{t=1}^T  \int_{\pi_t \in \Pi} P_{\text{LLM}}(\pi_t \mid (\ell_{i})_{i \in [t-1]})\langle \ell_t,   \pi_t \rangle d\pi_t  \leq \sum_{t=1}^T  \int_{\pi_t \in \Pi} \left(P_{\text{data}}(\pi_t \mid (\ell_{i})_{i \in [t-1]}) + \epsilon\right)\langle \ell_t,  \pi_t \rangle  d\pi_t 
    \end{align*}
    holds, where we use the convention of  $P_{\text{LLM}}(\pi_t \mid (\ell_{0})):=P_{\text{LLM}}(\pi_t)$ and $P_{\text{data}}(\pi_t \mid (\ell_{0})):=P_{\text{data}}(\pi_t)$. Hence,  
    \begin{align*}
        \text{Regret}_{\text{LLM}}((\ell_{t})_{t \in [T]}) &=\sum_{t=1}^T  \int_{\pi_t \in \Pi} P_{\text{LLM}}(\pi_t \mid (\ell_{i})_{i \in [t-1]}) \langle \ell_t,  \pi_t \rangle d\pi_t - \inf_{\pi \in \Pi} \sum_{t=1}^T \langle \ell_t, \pi \rangle
        \\& \leq\sum_{t=1}^T  \int_{\pi_t \in \Pi} \left(P_{\text{data}}(\pi_t \mid (\ell_{i})_{i \in [t-1]}) + \epsilon\right)\langle \ell_t,  \pi_t \rangle d\pi_t     -  \inf_{\pi \in \Pi} \sum_{t=1}^T \langle \ell_t, \pi \rangle
        \\
        &=  \sum_{t=1}^T  \int_{\pi_t \in \Pi} \left(P_{\text{data}}(\pi_t \mid (\ell_{i})_{i \in [t-1]})\right)\langle \ell_t,  \pi_t \rangle d\pi_t   -  \inf_{\pi \in \Pi} \sum_{t=1}^T \langle \ell_t, \pi \rangle +\sum_{t=1}^T  \int_{\pi_t \in \Pi}\langle \ell_t,  \epsilon \pi_t\rangle d\pi_t 
        \\
        &\leq  \text{Regret}_{\text{data}}((\ell_{t})_{t \in [T]}) + \epsilon \|\ell\|_p \|\pi\|_q T
    \end{align*}
where $\frac{1}{p} + \frac{1}{q} = 1$ and $p, q \geq 1$. Similarly, we can establish the lower bound for $\text{Regret}_{\text{LLM}}((\ell_t)_{t \in [T]})$. To prove the result for the  dynamic-regret case, we can simply change the term $\inf_{\pi \in \Pi} \sum_{t=1}^T \langle \ell_t, \pi \rangle$ in the above derivation to $\sum_{t=1}^T \inf_{\pi \in \Pi}\langle \ell_t, \pi \rangle$.
\end{proof}

\subsection{Background and Motivations for (Generalized) Quantal Response}\label{app:quantal-new}

Formally, the quantal response is defined as follows:
\begin{definition}[Quantal response]
	Given a loss vector \( \ell \in \mathbb{R}^d \), a noise distribution \( \epsilon \sim P_{\text{noise}} \), and $\eta>0$, the quantal response is defined as 
	\$ 
	P_{\text{quantal}}^\eta\big(a \biggiven \ell\big) = \PP\left( a \in \argmin_{a'\in\cA}~~z(a')  \right),\qquad \text{where~~ $z=\ell + \eta \epsilon$.}
	\$
	In essence, this implies that humans are rational but with respect to (w.r.t.) the latent variable \(z\), a perturbed version of \( \ell \), instead of \( \ell \) per se. This addition of noise to the actual loss vector characterizes the {bounded rationality} of humans in decision-making.
\end{definition}

\paragraph{Further motivations for generalized quantal response.} {Note that a \emph{dynamic} version of quantal response in \Cref{quantal} also has implications from behavior economics, and has been recently used to model human behaviors in sequential decision-making \citep{ding2022myopic} (in stochastic and stationary environments).} {Indeed, such a response against multiple loss vectors is believed to be natural, and has also been widely adopted in well-known no-regret learning algorithms of \emph{smooth/stochastic fictitious play} \citep{fudenberg1993learning} and \emph{follow-the-perturbed-leader}  \citep{kalai2005efficient}, 
whose formal definitions can be found in  \Cref{appendix:FTRL}. Finally, note that the response model in \Cref{quantal} does not necessarily involve a \emph{sequential} decision-making process, i.e., the set of losses may not come from the history of an online learning process. 
\subsection{The Example Instantiating \Cref{assump:decomp} }\label{app:example}
\arxiv{  
\begin{example}[An example instantiating   \Cref{assump:decomp}]\label{example:data}
We consider a common decision-making task that may generate the training data, \emph{recommender systems}. An instance of the text data could be: ``On September 29, 2023, user X clicked movie A three times, movie B eight times, and movie C five times''. This sentence corresponds to $x_{N_{i-1}+1:N_{i}}$ for some $i\in[t]$ and serves as a natural language depiction of the numerical $\ell_i$. 
The corresponding label $x_{N_{t}+1:N_{t+1}}$ can be obtained by some user survey: ``User X's favorite movie is movie B''.  Meanwhile, $z$ represents user X's latent, genuine preference for each movie -- information that is private to the user, and cannot be observed or collected in the pre-training dataset. In this example, \Cref{assump:decomp} suggests that  $x_{1:N_{t}}$, which records the frequency of interactions with each movie, serves as an \emph{imperfect estimate} of the user's latent, genuine preference for the movies, while the text   $x_{N_t+1:N_{t+1}}$  depicts the user's favorite movie only based on her latent $z$. 
\end{example}}

\conf{\subsection{Deferred Definition of Generalized Quantal Response in \Cref{sec:theoretical_insights}}
\label{appendixsec:GQR}

\paragraph{Generalized quantal response gives rise to follow-the-perturbed-leader.} 
A seminal model for human behaviors is the \emph{quantal response} model, which assumes that humans are often not \emph{perfect} decision-makers, and their bounded rationality can be modeled through unseen latent variables that influence the decision-making process  \citep{mcfadden1976quantal,mckelvey1995quantal}. Formally, the quantal response is defined as follows:

\begin{definition}[Quantal response]
	Given a loss vector \( l \in \mathbb{R}^d \), a noise distribution \( \epsilon \sim P_{\text{noise}} \), and $\eta>0$, the quantal response is defined as 
	\$ 
	P_{\text{quantal}}^\eta(a \given \ell) = \PP \left( a \in \argmin_{i\in[d]}~~(\ell + \eta \epsilon)[i] \right).
	\$
	In essence, this implies that humans are rational but with respect to the latent variable \( \ell + \eta\epsilon \) instead of \( \ell \). This addition of noise to the actual loss vector characterizes the {bounded rationality} of humans in decision-making.
\end{definition}

Traditional quantal response formulations primarily focused on scenarios with a single loss vector. In online learning, given the \emph{history} information, the decision-maker (either human or  LLM agent) at each time $t$ is faced with \emph{multiple} loss vectors. Hence, we propose the following generalization to model human behavior in online decision-making. 
 
\begin{definition}[Quantal response against multiple losses]\label{quantal} 
	Given a set of loss vectors \( ({\ell}_i)_{i \in [t]} \), a noise distribution \( P_{\text{noise}} \), and $\eta_t>0$, the generalized quantal response is defined as
$P_{\text{quantal}}^{\eta_t}\left(a \given  ({\ell}_i)_{i \in [t]}\right):=  P_{\text{quantal}}^{\eta_t}\left( a \bigggiven  \sum_{i=1}^{t} {\ell}_{i}\right).$ 
\end{definition} 

In simpler terms, the generalized quantal response 
{is defined as} the standard quantal response against {some scaled} summation of the losses. {Note that such a \emph{dynamic} version of quantal response also has implications from behavior economics, and has been recently used to model human behaviors  in sequential decision-making \citep{ding2022myopic} (in stochastic and stationary environments).} Indeed, there is a direct relationship between our Definition \ref{quantal} and a well-known no-regret learning algorithm in online learning, \emph{follow-the-perturbed-leader}  \citep{kalai2005efficient}, whose formal definition can be found in  \Cref{appendix:FTRL}.  

 \begin{restatable}{observation}{FTPL}
 \label{thm:informal_generalized_quantal_FTPL}
	Suppose at each time step $t$, the {decision-maker} {(i.e., human or LLM agent)} response follows \Cref{quantal}, then the decision-making process is equivalent to using the FTPL algorithm with proper perturbation.
 \end{restatable}}

\subsection{Alignment of \Cref{assump:decomp}  with Quantal Response}\label{app:quantal}
Before presenting the technical lemma, based on \Cref{assump:decomp}, we denote the (potentially unkown) mappings that decode semantic information in \Cref{assump:decomp} {into numeric values} as $f$, $g$, such that $f(x_{N_{i-1}+1:N_{i}})=\ell_i\in\RR^d$ for each $i\in[t]$ and $g(x_{N_t+1:N_{t+1}})=a\in\cA$.
\begin{restatable}{lemma}{languagequantal}\label{lem:quantal}
Fix $t\in[T]$, $\sigma>0$. If we model the noise of data collection to be i.i.d. Gaussian distribution in the numeric value space, i.e., 
$$
\PP\left(\left\{f(x_{N_{i-1}+1:N_{i}})\right\}_{i\in[t]}\biggiven  z\right) = \prod_{i=1}^{t} \PP\left(f(x_{N_{i-1}+1:N_{i}})\biggiven z\right)\propto \prod_{i=1}^{t} \exp\left(-\frac{\|f(x_{N_{i-1}+1:N_{i}})- z\|_2^2}{2\sigma^2}\right),
$$
the prior distribution of the latent variable $z$ is also Gaussian, i.e., $z\sim \cN(\pmb{0}_d, \sigma^2 I)$, and the text labels satisfy that
$\PP(g(x_{N_t+1:N_{t+1}})\given z) = \mathbbm{1}\left(g(x_{N_t+1:N_{t+1}})\in\argmin_{a\in\cA}z_a\right)$, then we have 
$$
\PP\left(g(x_{N_t+1:N_{t+1}})\biggiven x_{1:N_t}\right)=P_{quantal}^{\sigma\sqrt{t+1}}\left(g(x_{N_t+1:N_{t+1}})\biggiven \left\{f(x_{N_{i-1}+1:N_{i}})\right\}_{i\in[t]}\right),$$
with $P_{noise}=\cN(\pmb{0}_d, I)$ in \Cref{quantal}, i.e., the action $a=g(x_{N_t+1:N_{t+1}})$ extracted from the text  $x_{N_t+1:N_{t+1}}$ is a quantal response w.r.t. the loss vectors $\left(f(x_{N_{i-1}+1:N_{i}})\right)_{i\in[t]}$.
\end{restatable} 

\begin{proof}
	Note that 
	\$\PP(z\given x_{1:N_t}) = \int_{\ell_{1:t}}\PP(z, \ell_{1:t}\given x_{1:N_t})d{\ell_{1:t}}= \int_{\ell_{1:t}}\PP( \ell_{1:t}\given x_{1:N_t})\PP(z\given x_{1:N_t}, \ell_{1:t})d\ell_{1:t}.
	\$
	For $\PP( \ell_{1:t}\given x_{1:N_t})$, since we have assumed the existence of function $f$ to  decode $\ell_{1:t}$ from $x_{1:N_t}$, it holds that
	\$
	\PP( \ell_{1:t}\given x_{1:N_t})=\prod_{i=1}^t\delta\left(\ell_{i}-f(x_{N_{i-1}+1:N_{i}})\right),
	\$
	where we use $\delta$ to denote the $d$-dimensional Dirac-delta function. 
	For $\PP(z\given x_{1:N_t}, \ell_{1:t})$, by \Cref{assump:decomp}, it holds that
	\$
	\PP(z, x_{1:N_t}, \ell_{1:t})=\PP(z, \ell_{1:t})\PP(x_{1:N_t}\given \ell_{1:t}),
	\$ 
	which leads to $\PP(x_{1:N_t}\given \ell_{1:t})=\PP(x_{1:N_t}\given \ell_{1:t}, z)$ by Bayes rule. This implies that the random variable $x_{1:N_t}$ and $z$ are independent conditioned on $\ell_{1:t}$. Therefore, it holds that $\PP(z\given x_{1:N_t}, \ell_{1:t})=\PP(z\given\ell_{1:t})$. Finally, we can compute 
	\$
	\PP(z\given x_{1:N_t}) &= \int_{\ell_{1:t}}\PP(z, \ell_{1:t}\given x_{1:N_t})d{\ell_{1:t}}= \int_{\ell_{1:t}}\prod_{i=1}^t\delta(\ell_{i}-f(x_{N_{i-1}+1:N_{i}}))\PP(z\given \ell_{1:t})d\ell_{1:t}\\
	&=\PP\left(z\given \left(\ell_i=f(x_{N_{i-1}+1:N_{i}})\right)_{i\in[t]}\right).
	\$
 Based on this, we conclude that 
	\$
	\PP(g(x_{N_t+1:N_{t+1}})\given x_{1:N_t})&=\int_z\PP(g(x_{N_t+1:N_{t+1}})\given z, x_{1:N_t})\PP(z\given x_{1:N_t})dz\\
	&=\int_z\PP(g(x_{N_t+1:N_{t+1}})\given z)\PP(z\given \{\ell_i=f(x_{N_{i-1}+1:N_{i}})\}_{i\in[t]})dz\\
	&=\PP\left(g(x_{N_t+1:N_{t+1}})\given \left(\ell_i=f(x_{N_{i-1}+1:N_{i}})\right)_{i\in[t]}\right)
	\$ 
	where the first equality is by the independence between $x_{N_{t}+1:N_{t+1}}$ and $x_{1:N_t}$ conditioned on $z$, due to \Cref{assump:decomp}. Therefore, it suffices to consider the probability of $\PP(a\given \ell_{1:t})$ only, in order  to analyze $\PP(g(x_{N_t+1:N_{t+1}})\given x_{1:N_t})$, where we recall the definition that $a=g(x_{N_t+1:N_{t+1}})$. 
	Since $z \sim \cN(\pmb{0}_d, \sigma^2 I)$, and ${\ell}_i\mid z \sim \cN(z, \sigma^2 I)$, we have 
\begin{align}\label{eq:normal-pos}
z \mid \ell_{1:t} \sim \cN\left(\frac{1}{t+1}\sum_{i\in[t]}{\ell}_i, \frac{\sigma^2}{t+1}I\right),
\end{align}
by the posterior distribution of  Gaussian distribution. Now we conclude that  
\$
\PP(a\given \ell_{1:t}) &= \int_z \PP(a\given z, \ell_{1:t})\PP(z\given \ell_{1:t})dz=\int_z \PP(a\given z)\PP(z\given \ell_{1:t})dz\\
 &=\int_z \mathbbm{1}(a\in \argmin_{a^\prime\in\cA} z_{a^\prime})\PP(z\given \ell_{1:t})dz=\int_z \mathbbm{1}\left(a\in \argmin_{a^\prime\in\cA} \left(\frac{\sigma}{\sqrt{t+1}}\epsilon+\frac{1}{t+1}\sum_{i\in[t]}\ell_i\right)_{a^\prime}\right)\PP(\epsilon)d\epsilon\\
 &=\int_z \mathbbm{1}\left(a\in \argmin_{a^\prime\in\cA} \left(\sigma\sqrt{t+1}\epsilon+\sum_{i\in[t]}\ell_i\right)_{a^\prime}\right)\PP(\epsilon)d\epsilon=\PP\left(a\in \argmin_{a^\prime\in\cA} \left(\sigma\sqrt{t+1}\epsilon+\sum_{i\in[t]}\ell_i\right)_{a^\prime}\right)\\
 &=P_{quantal}^{\sigma\sqrt{t+1}}(a\given \ell_{1:t}),
\$
where $\PP(\epsilon)=\cN(\pmb{0}_d, I)$. This completes the proof.
\end{proof}

\subsection{Relationship between FTPL and \Cref{quantal}}
\begin{fact}\label{fact:quantal}
    Performing generalized quantal response of \Cref{quantal} at every iteration $t\in[T]$ w.r.t. history loss vectors $\ell_{1:t-1}$ is essentially executing an FTPL algorithm.
\end{fact} 
\begin{proof}
Before we move to the proof, we will define the random variable which has distribution $P_{\text{noise}}$ as $Z_{\text{noise}}$. {Note that at round {$t\geq 2$ (as the policy at round $t=1$ is fixed)}, we have} 
    
    \begin{align}\label{eq:ftpl-quantal}
        P_{\text{quantal}}^{\eta_{t-1}}(a \given  \ell_{1:t-1}) := \PP\left( a \in \argmin_{a' \in \cA} ~\left(\sum_{i=1}^{t-1} {\ell}_{i} + \eta_{t-1} \epsilon\right)(a')\right)
    \end{align}
    which is exactly the case when $\epsilon_t$ in \Cref{eqn:FTPL-def} 
    satisfies $\epsilon_t \overset{d}{=} \eta_{t-1}\epsilon$.
\end{proof}

\subsection{Formal Statement and Proof of \Cref{thm:ftpl-main}}\label{app:ftpl}
\begin{restatable}{theorem}{ftplmain}\emph{(Emergence of no-regret behavior){\bf.}} 
    \label{thm:ftpl-main-full}
Under the assumptions of Lemma \ref{lem:quantal}, suppose the function class of $\text{LLM}_{\theta}$ is  expressive enough {such that {for all $t\in[T]$,}  $\max_{\theta\in\Theta} \EE_{x_{1:N_{t+1}}\sim P^{text}_t}\sum_{j=1}^{N_{t+1}}\log\text{LLM}_\theta\left(x_{j}\given x_{1:j-1}\right)=\max_{\left\{q_j\in \left\{\cV^{j-1}\rightarrow\Delta(\cV)\right\}\right\}_{j\in[N_{t+1}]}}\EE_{x_{1:N_{t+1}}\sim P^{text}_t}\sum_{j=1}^{N_{t+1}}\log q_j\left(x_{j}\given x_{1:j-1}\right)$, where we define $q_1(x_1\given x_{1:0}):=q_1(x_1)$}, and 
$\theta^\star$ maximizes \Cref{eq:pretrain}. Then,   
	there exist (simple) algorithms using $\text{LLM}_{\theta^\star}$ to achieve no (dynamic) regret for (non-stationary) online learning with full-information/bandit feedback. 
\arxiv{To be specific, for (2) and (4), by defining the variation bound $\sum_{t=1}^{T-1}\|\ell_{t+1}-\ell_t\|_\infty\le V_T$ such that $V_T\le T$ and $V_T=\Theta(T^{\rho})$ for some $\rho\in(0, 1)$, it holds that for large enough $T$, $d$: 
\safevspace{6pt}

	 (1) For online learning with full-information feedback, $\text{Regret}_{\text{LLM}{_{\theta^\star}}} \left(({\ell}_t)_{t \in [T]}\right) \leq \cO\left(\sqrt{T \log d}\right)$;  
	 
	 (2) For non-stationary online learning  with full-information feedback, $\text{D-Regret}_{\text{LLM}{_{\theta^\star}}}\left(({\ell}_t)_{t \in [T]}\right) \leq \cO\left((\log d\ V_T)^{1/3} T^{2/3}\right);$   
	 
	 (3) For online learning with bandit feedback, $\EE\left[\text{Regret}_{\text{LLM}{_{\theta^\star}}}\left(({\ell}_t)_{t \in [T]}\right)\right] \le \cO\left((\log d)^{1/2} dT^{1/2 + 1/\log T}\log T\right);$ 
	 
	 (4) For non-stationary online learning  with bandit feedback, 
	 $\EE\left[\text{D-Regret}_{\text{LLM}{_{\theta^\star}}}\left(({\ell}_t)_{t \in [T]}\right)\right]\leq$\\$  \cO\left((T^2d^2 V_T)^{1/3}(\log d)^{1/2}T^{1/\log T}\log T \right).$}  
\end{restatable}

\conf{To be specific,

	 (1) For online learning with full-information feedback, $\text{Regret}_{\text{LLM}{_{\theta^\star}}} (({\ell}_i)_{i \in [T]}) \leq \cO(\sqrt{T \log d})$; 
	 
	 (2) For non-stationary online learning  with full-information feedback, $\text{D-Regret}_{\text{LLM}{_{\theta^\star}}}(({\ell}_i)_{i \in [T]}) \leq  \cO((d\log d\ V_T)^{1/3} T^{2/3})$;
	 
	 (3) For adversarial bandits, \(\EE[\text{Regret}_{\text{LLM}{_{\theta^\star}}}(({\ell}_i)_{i \in [T]})] \le \cO((\log d)^{1/2}T^{1/2 + 2\log\log T/\log T}) \);
	 
	 (4) For non-stationary bandits,
	 \(\EE[\text{D-Regret}_{\text{LLM}{_{\theta^\star}}}(({\ell}_i)_{i \in [T]})]\) \(\leq  \cO((d\log d\ V_T)^{1/3} T^{2/3+2\log\log T/\log T})\),
	 
	 where we define $V_T := \sum_{t=1}^{T-1}\|\ell_{t+1}-\ell_t\|_\infty$.}
\begin{proof}
Note that
\$
&\max_{\left\{q_j\in \left\{\cV^{j-1}\rightarrow\Delta(\cV)\right\}\right\}_{j\in[N_{t+1}]}}\EE_{x_{1:N_{t+1}}\sim P^{text}_t}\sum_{j=1}^{N_{t+1}}\log q_j\left(x_{j}\given x_{1:j-1}\right)\\
&=\max_{q\in \Delta(\cV^{N_{t+1}})}\EE_{x_{1:N_{t+1}}\sim P^{text}_t}\log q(x_{1:N_{t+1}})\\
&=\max_{q\in \Delta(\cV^{N_{t+1}})}-\text{KL}( P_t^{text}\given \given q)+\EE_{x_{1:N_{t+1}}\sim P^{text}_t}[P_t^{text}(x_{1:N_{t+1}})],
\$
where $\text{KL}(q\given \given p)$  denotes the Kullback–Leibler divergence between two distributions $p,q$. Now we define $\text{LLM}_\theta(x_{1:N_{t+1}})=\prod_{t=1}^{N_{t+1}}\text{LLM}_\theta(x_j\given x_{1:j-1})$. 
It is easy to verify that $\text{LLM}_\theta(x_{1:N_{t+1}})\in\Delta(\cV^{N_{t+1}})$, i.e., it also defines a valid joint distribution over tokens. Therefore, we have 
\$
\max_{\theta\in\Theta} \EE_{x_{1:N_{t+1}}\sim P^{text}_t}\sum_{j=1}^{N_{t+1}}\log\text{LLM}_\theta\left(x_{j}\given x_{1:j-1}\right)=\max_{\theta\in\Theta} \EE_{x_{1:N_{t+1}}\sim P^{text}_t}\log\text{LLM}_\theta(x_{1:N_{t+1}}).
\$
Now, due to our assumption that
\$
&\max_{\theta\in\Theta} \EE_{x_{1:N_{t+1}}\sim P^{text}_t}\sum_{j=1}^{N_{t+1}}\log\text{LLM}_\theta\left(x_{j}\given x_{1:j-1}\right)\\
&=\max_{\left\{q_j\in \left\{\cV^{j-1}\rightarrow\Delta(\cV)\right\}\right\}_{j\in[N_{t+1}]}}\EE_{x_{1:N_{t+1}}\sim P^{text}_t}\sum_{j=1}^{N_{t+1}}\log q_j\left(x_{j}\given x_{1:j-1}\right),
\$
we conclude that
\$
\min_{\theta\in\Theta}\text{KL}( P_t^{text}\given\given \text{LLM}_\theta)=\min_{q\in\Delta(\cV^{N_{t+1}})}\text{KL}(P_t^{text}\given\given q)=0,
\$
which implies that $\text{LLM}_{\theta^\star}=P_t^{text}$. Correspondingly, if we define $\text{LLM}_{\theta^\star}(x_{N_t+1: N_{t+1}}\given x_{1:N_t})$ to be the distribution induced by the joint distribution $\text{LLM}_{\theta^\star}(x_{1:N_{t+1}})$, it holds that
\$
\text{LLM}_{\theta^\star}(x_{N_t+1: N_{t+1}}\given x_{1:N_t})=\PP(x_{N_t+1: N_{t+1}}\given x_{1:N_t}).
\$
In other words, intuitively, $\text{LLM}_{\theta^\star}$ has learned the corresponding \emph{pre-training}  distribution perfectly. Note that this has been a common assumption in the Bayesian perspective of ICL \citep{xie2021explanation,lee2023supervised,zhang2023and}. Therefore, to analyze the actions taken by $\text{LLM}_{\theta^\star}$, it suffices to consider $\PP(g(x_{N_t+1: N_{t+1}})\given x_{1:N_t})$, which is equal to $P_{quantal}^{\sigma\sqrt{t+1}}\left(g(x_{N_t+1:N_{t+1}})\biggiven \left\{f(x_{N_{i-1}+1:N_{i}})\right\}_{i\in[t]}\right)$ by \Cref{lem:quantal}. Therefore, we proved that $\text{LLM}_{\theta^\star}$ is essentially mimicking the well-known no-regret algorithm, FTPL with {perturbation distribution as $\cN(\pmb{0}_d, \sigma^2tI)$ for round  $t\in[T]$,  according to \Cref{eq:ftpl-quantal} of \Cref{fact:quantal}}, for which we can establish the corresponding regret guarantee for each case:%

(1) Combining the above result with \Cref{lem:ftpl-full}, we can derive the regret bound for online learning with full-information feedback. 

(2) Combining the above result with \Cref{lem:ftpl-full} and \Cref{lem:nonstationary}, we get that
\$
\text{D-Regret}_{\text{LLM}{_{\theta^\star}}}(({\ell}_i)_{i \in [T]}) \le \min_{\Delta_T\in[T]}\frac{2T}{\Delta_T}C\sqrt{\Delta_T\log d} + 2\Delta_T V_T,
\$
for some constant $C$. We firstly consider the following problem
\$
\min_{u>0}\frac{2T}{u}C\sqrt{u\log d} + 2u V_T,
\$
where the optimal solution is $u^\star = \left(\frac{C^2T^2\log d}{4V_T^2}\right)^{1/3}$. Therefore, if we have $u^\star\in [1, T]$, we can choose $\Delta_T=\lceil u^\star \rceil$, which results in a regret bound of 
\$
\text{D-Regret}_{\text{LLM}{_{\theta^\star}}}(({\ell}_i)_{i \in [T]}) \le \frac{2T}{\sqrt{u^\star}}C\sqrt{\log d} + 4u^\star V_T=\cO\left((\log d\ V_T)^{1/3} T^{2/3}\right).
\$
Now we check the conditions for $u^\star\in[1, T]$. It is direct to see that since $V_T\le T$, $u^\star\ge 1$ holds as long as $d$ is sufficiently large. To ensure $u^\star\le T$, we get the condition $V_T\ge C\sqrt{\frac{ \log d}{4T}}$, which holds as long as $T$ is large enough.

(3) Combining the above result with \Cref{lem:ftpl-bandit}, we can prove a regret guarantee for online learning with bandit feedback. 
 
(4) Combining this result with \Cref{lem:ftpl-bandit} and \Cref{lem:nonstationary}, it holds that
\$
\EE[\text{D-Regret}_{\text{LLM}{_{\theta^\star}}}(({\ell}_i)_{i \in [T]})] \le \min_{\Delta_T\in[T]}\frac{2T}{\Delta_T}C(\log d)^{\frac{1}{2}}d \Delta_T^{\frac{1}{2} + \frac{1}{\log T}}\log \Delta_T + 2\Delta_T V_T,
\$
for some constant $C$. By adopting a similar analysis as that of (2), we choose $u^\star = \left(\frac{C^\prime T^2 d^2}{V_T^2}\right)^{1/3}$ for some constant $C^\prime$. If $u^\star\in[1, T]$, we choose $\Delta_T=\lceil u^\star \rceil$ and derive the following regret:
\$
\EE[\text{D-Regret}_{\text{LLM}{_{\theta^\star}}}(({\ell}_i)_{i \in [T]})] \le \cO\left((T^2d^2 V_T)^{1/3}(\log d)^{1/2}T^{1/\log T}\log T\right). 
\$ 
Now we check the condition of $u^\star\in[1, T]$. Note that since $V_T\le T$, $u^\star\ge 1$ holds as long as $d$ is sufficiently large. For $u^\star\le T$, we have $V_T\ge \sqrt{\frac{C^\prime d^2}{T}}$, which holds as long as $T$ is large enough.

Now, we present \Cref{lem:ftpl-full} - \Cref{lem:nonstationary}. Before proceeding, we assume $\|\ell_t\|_\infty \le B=1$ for simplicity of presentations hereafter. The results and proof are not affected by the constant bound $B$. 

\begin{lemma}[Regret guarantee of FTPL with full-information feedback]\label{lem:ftpl-full}
	Suppose the noise distribution of FTPL satisfies that $\epsilon_t\sim \cN(\pmb{0}_d, \zeta_t^2 I)$ in \Cref{eqn:FTPL-def} and $\zeta_t = \sigma\sqrt{t}$, then for online learning with full-information feedback, 
	\$
	\text{Regret}_{\text{FTPL}}(({\ell}_i)_{i \in [T]})\le 4\left(\sigma + \frac{1}{\sigma}\right)\sqrt{T\log d }=\cO(\sqrt{T\log d }).
	\$
\end{lemma}

\begin{proof}
	By Theorem 8 of \cite{abernethy2014online}, we have
 \begin{align*}
	\text{Regret}_{\text{FTPL}}(({\ell}_i)_{i \in [T]})\le \sqrt{2\log d}\left(\eta_T + \sum_{t=1}^T\frac{1}{\eta_t} \|\ell_t\|_\infty^2\right).     
 \end{align*}
Therefore, plugging $\zeta_t = \sigma\sqrt{t}$ and $\|\ell_t\|_\infty^2 \le 1$ provides 
 \begin{align*}
    \text{Regret}_{\text{FTPL}}(({\ell}_i)_{i \in [T]})\le \sqrt{2\log d}\left(\sigma\sqrt{T} + \sum_{t=1}^T\frac{1}{\sigma\sqrt{t}} \right)\le 4\left(\sigma + \frac{1}{\sigma}\right)\sqrt{T\log d }, 
 \end{align*} 
 completing the proof.
\end{proof}
\begin{lemma}[Regret guarantee of FTPL with bandit feedback]\label{lem:ftpl-bandit}
		Suppose the noise distribution of FTPL satisfies that $\epsilon_t\sim \cN(\pmb{0}_d, \zeta_t^2 I)$ in \Cref{eqn:FTPL-def} and $\zeta_t = \sigma\sqrt{t}$, then for online learning with bandit feedback,
  \begin{align*}
	\EE[\text{Regret}_{\text{FTPL}}(({\ell}_i)_{i \in [T]})]\le \cO((\log d)^{\frac{1}{2}}d T^{\frac{1}{2} + \frac{1}{\log T}}\log T).    
  \end{align*}
\end{lemma}
\begin{proof}
	The proof of the bandit problem is more complex. We first define the following notation. We denote $G_t = \sum_{t^\prime=1}^t -\ell_{t^\prime}$, $\hat{G}_t= \sum_{t^\prime=1}^t -\hat{\ell}_{t^\prime}$,  $\Phi(G) = \max_{\pi}\langle \pi, G\rangle$, $\Phi_t(G) = \EE_{\epsilon\sim\cN(\pmb{0}_d, I)} \Phi(G+\zeta_t\epsilon)$, and $D_{\Phi_t}$ to be the Bregman divergence with respect to $\Phi_{t}$, {where we recall the construction of the empirical estimator $\hat{\ell}_{t^\prime}$ of $\ell_{t^\prime}$ in \Cref{sec:experiments_online}}. By \cite{li2017beyond}, $\pi_t = \nabla\Phi_t(\hat{G}_t)$.
	 Now due to the convexity of $\Phi$,
  \begin{align*}
	\Phi(G_T) = \Phi(\EE[\hat{G}_T])\le \EE[\Phi(\hat{G}_T)].
  \end{align*}
	 Therefore,
\begin{align*}
	\EE[\text{Regret}_{\text{FTPL}}(({\ell}_i)_{i \in [T]})] &= \Phi(G_T)-\EE\left[\sum_{t=1}^T\langle \pi_t, -\hat{\ell}_t\rangle\right]\le \EE\left[\Phi(\hat{G}_T)-\sum_{t=1}^T\langle \pi_t, -\hat{\ell}_t\rangle \right].    
\end{align*} 
	 By recalling the definition of the Bregman divergence, we have
  \begin{align*}
	 -\sum_{t=1}^{T}\langle \pi_t, -\hat{\ell}_t\rangle &= -\sum_{t=1}^T\langle \nabla\Phi_t(\hat{G}_t), -\hat{\ell}_t\rangle=-\sum_{t=1}^T\langle \nabla\Phi_t(\hat{G}_t), \hat{G}_t-\hat{G}_{t-1}\rangle\\
	 &=\sum_{t=1}^T D_{\Phi_t}(\hat{G}_t, \hat{G}_{t-1}) + \Phi_t(\hat{G}_{t-1}) -\Phi_t(\hat{G}_{t}).      
  \end{align*}
	 Therefore,
  \begin{align*}
	 &\EE\left[\text{Regret}_{\text{FTPL}}(({\ell}_i)_{i \in [T]})\right] 
  \\
  &\le \underbrace{\EE\left[\sum_{t=1}^T D_{\Phi_t}(\hat{G}_t, \hat{G}_{t-1})\right]}_{(i)} + \underbrace{\EE\left[\sum_{t=1}^T\Phi_t(\hat{G}_{t-1})-\Phi_{t-1}(\hat{G}_{t-1})\right]}_{(ii)} + \underbrace{\EE\left[\Phi(\hat{G}_T)-\Phi_T(\hat{G}_T)\right]}_{(iii)},         
  \end{align*}
	where $(iii) \leq 0$ due to the convexity of $\Phi$. For $(ii)$, we use  Lemma 10 of \cite{abernethy2014online} to obtain   
$$\EE\left[\sum_{t=1}^T \Phi_t(\hat{G}_{t-1})-\Phi_{t-1}(\hat{G}_{t-1})\right]\le \zeta_T\EE_{\epsilon}[\Phi(\epsilon)]\le \cO(\sqrt{2T\log d}).$$ 
	 For $(i)$, by Theorem 8 of \cite{li2017beyond}, for any $\alpha\in(0, 1)$, the following holds:
	 \begin{align*}
	 \EE\left[\sum_{t=1}^T D_{\Phi_t}(\hat{G}_t, \hat{G}_{t-1})\right]&\le \sum_{t=1}^T\zeta_t^{\alpha-1}\frac{4d}{\alpha(1-\alpha)}
  \\
  &\le \frac{4d}{\alpha(1-\alpha)}\cO(T^{\frac{1+\alpha}{2}}).
	 \end{align*}
	 By tuning $\alpha = \frac{2}{\log T}$, we proved that $\EE[\text{Regret}_{\text{FTPL}}(({\ell}_i)_{i \in [T]})]\le \cO((\log d)^{\frac{1}{2}}d T^{\frac{1}{2} + \frac{1}{\log T}}\log T)$.
\end{proof}
\begin{lemma}\label{lem:nonstationary}
	Denote the variation of loss vectors as $L_T = \sum_{t=1}^{T-1}\|\ell_{t+1}-\ell_t\|_\infty$. Suppose there exists an algorithm $\mathscr{A}$ for online learning with full-information feedback with regret guarantee that $\text{Regret}_{\mathscr{A}}((\ell_i)_{i\in[T]})\le f(T, d)$ for some function $f$, where $T$ denotes the horizon and $d$ denotes the policy dimension. Then, there exists another algorithm $\mathscr{A}^\prime$ that can achieve
	\begin{align*}
	\text{D-Regret}_{\mathscr{A}^\prime}(({\ell}_i)_{i \in [T]})\le \min_{\Delta_T\in[T]}\left(\frac{T}{\Delta_T}+1\right)f(\Delta_T, d)+2\Delta_T L_T.	    
	\end{align*}
	Similarly, suppose there exists an algorithm $\mathscr{B}$ for online learning with bandit feedback   with regret guarantee that $\EE\left[\text{Regret}_{\mathscr{B}}((\ell_i)_{i\in[T]})\right]\le g(T, d)$  for some function $g$; then there exists another algorithm $\mathscr{B}^\prime$ that can achieve 
 \begin{align*}
	\EE[\text{D-Regret}_{\mathscr{B}^\prime}(({\ell}_i)_{i \in [T]})]\le \min_{\Delta_T\in[T]}\left(\frac{T}{\Delta_T}+1\right)g(\Delta_T, d)+2\Delta_T L_T.    
 \end{align*}
\end{lemma}
\begin{proof}
We denote $\mathscr{A}^\prime$ as the algorithm that restarts $\mathscr{A}$ every $\Delta_T$ iterations. We break the time index $[T]$ into $m$ batches $\cT_{1:m}$ of size $\Delta_T$ (except for, possibly the last batch). Denote $\ell_i^\star := \min_{j\in[d]}\ell_{ij}$. By Equation (6)  of \cite{besbes2014stochastic}, it holds that for each $k\in[m]$
\$
\min_{j\in[d]}\left(\sum_{t\in\cT_k} \ell_t\right)_{j}-\sum_{t\in\cT_k}\ell^\star_t\le 2\Delta_T L_k,
\$
where we define $L_k=\sum_{t\in\cT_k}\|\ell_{t+1}-\ell_t\|_\infty$.
Therefore, we have 
\#
\label{eq:d-regret}\text{D-Regret}_{\mathscr{A}^\prime}(({\ell}_i)_{i \in [T]})&\le \min_{j\in[d]}\left(\sum_{t\in[T]} \ell_t\right)_{j}-\sum_{t\in[T]}\ell^\star_t  + \sum_{k\in[m]}\text{Regret}_{\mathscr{A}}((\ell_i)_{i\in[\cT_k]})\\
&\le 2\Delta_T(\sum_{k\in[m]} L_k) + (T/\Delta_T + 1)g(\Delta_T, d)\nonumber.
\#
By Equation (4) of \cite{besbes2014stochastic} that $\sum_{k\in[m]} L_k\le L_T$ and this inequality holds for any $\Delta_T\in[T]$, we proved $\text{D-Regret}_{\mathscr{A}^\prime}(({\ell}_i)_{i \in [T]})\le \min_{\Delta_T\in[T]}\left(\frac{T}{\Delta_T}+1\right)f(\Delta_T, d)+2\Delta_T L_T$.

Similarly, if we take the expectation for \Cref{eq:d-regret}, it holds that 
\$
\EE[\text{D-Regret}_{\mathscr{B}^\prime}(({\ell}_i)_{i \in [T]})]&\le \min_{j\in[d]}\left(\sum_{t\in[T]} \ell_t\right)_{j}-\sum_{t\in[T]}\ell^\star_t + \sum_{k\in[m]}\EE[\text{Regret}_{\mathscr{B}}((\ell_i)_{i\in[\cT_k]})]\\
&\le \min_{\Delta_T\in[T]}\left(\frac{T}{\Delta_T}+1\right)g(\Delta_T, d)+2\Delta_T L_T,
\$
thus completing the proof.
\end{proof}
Combining the results above completes the proof for 
\Cref{thm:ftpl-main}. 
\end{proof}
\subsubsection{Implications of \Cref{thm:ftpl-main} for Repeated Games}\label{app:game}
\begin{remark}[Implication for playing repeated games]\label{remark:cce}
First, we note that the no-regret guarantee in the online setting is stronger than and thus implies that in the game setting, since regret by definition handles arbitrary/adversarial environments, while in playing games the opponents are not necessarily as adversarial.  Second, it is folklore that if all players in the repeated game follow no-regret learning algorithms, then the time-average policies of all players during learning constitute an approximate {\bf \emph{coarse correlated equilibrium}} of the game \citep{cesa2006prediction}. Hence, the results (1) and (2) in \Cref{thm:ftpl-main} imply that a coarse correlated equilibrium will emerge in the long run from the interactions of the LLM agents (under certain assumptions as in the theorem). 
\end{remark}

\conf{
\begin{remark}
\kz{either update or remove, due to the same remark in the main paper?}\cp{I will do it after everything is done. I will make conference version again}Although \cite{lee2023supervised,lin2023transformers} have shown pre-trained LLM agents can solve stochastic bandit provably in light of the equivalence to posterior sampling, it cannot be used for adversarial bandit since posterior sampling can perform almost as badly as a worst-performing agent in some non-stationary environments \citep{liu2023nonstationary}. In contrast, due to the equivalence to FTPL, our approach solves adversarial bandit problems with simple modifications. 
\end{remark}}
\conf{
\begin{remark}[\kzedit{Pre-training with relaxed data assumptions}]\label{app:relaxtion}
\kzedit{Note that the pre-training (text) data are so far assumed to be related to \emph{decision-making} problems (though not necessarily sequential ones), see  \Cref{assump:decomp} and \Cref{example:data} for instance. It can also be generalized to the text datasets involving \emph{Question-Answering} (Q-A), a typical task in natural language processing, where the \emph{true/fact} answer, sampled answers from different human users (with possibly wrong or biased answers), correspond to the latent $z$ (and associated maximizer $a$) and $\ell_{1:t}$, respectively. Moreover, in practice, the pre-training data may also involve  \emph{non-decision-making/Q-A} texts, given the diversity of the datasets. 
For such scenarios, we will make the assumptions on the data distribution \emph{conditioned on the prompt for decision-making}. Specifically, when interacting with the LLM, human users will provide prompts (see e.g., our \Cref{fig:game-prompt}), to induce it to make decisions. This will query the \emph{conditional}  distribution of $\PP\left(g(x_{N_t+1:N_{t+1}})\biggiven x_{1:N_t},\text{decision-making prompt}\right)$ to generate the control action. 
Correspondingly, \Cref{assump:decomp} will thus only need to be made on 
$$
	\PP\left(z, \ell_{1:t}, x_{1:N_{t+1}},\text{decision-making prompt}\right), 
 $$ 
while we do not need to make such assumptions on other prompts, e.g., corpora that are not related to decision-making.}\kz{one question -- do we also need to change the pre-training loss? I think no? since that loss is for GENERAL next-token prediction, which will eventually become the ``joint distribution'' of ALL tokens (including the non-decision-making ones)? given this, we can always recover such ``conditional'' distributions, right? just to confirm.}\cp{exactly!}
}
\conf{
\cpedit{
\begin{remark}[Illustration of \Cref{assump:decomp} in the form of natural languages]For simplicity, we have initially considered the sequence $(\ell_{i}){i\in[T]}$ and the variable $a$ to be numerical values. However, in practical scenarios, these inputs and outputs are frequently represented in natural language. Despite this, we can still convert $(\ell{i})_{i\in[T]}$ and $a$ into textual data through the utilization of an embedding layer. Consider a common application within decision-making tasks, such as a recommendation system. An instance of structured text could be: ``On September 29, 2023, user X interacted with movie A three times, movie B eight times, and movie C five times.'' This sentence serves as a natural language depiction of the numerical data $\ell_i$. By employing an embedding layer, we can transform these textual interactions into numerical $\ell_i$ values. Similarly, we can express $z$ in a natural language format. The corresponding true preference $a^\star$ might be derived from a user survey, where user X states, ``User X's favorite movie is movie B." Here, $z$ encapsulates the user X's latent, genuine preferences for each movie—information that is personal to each user and absent from the pre-training dataset. Under this framework, \Cref{assump:decomp} suggests that $a^\star$ mirrors the user's preferred movie based on $z$, while $\ell_i$, which accounts for the frequency of interactions with each movie, serves as an imperfect indicator of the user's latent, genuine preferences for the movies. This can be also converted to the numerical values by an embedding layer. 
\end{remark}}
}

\subsection{Extending \Cref{thm:ftpl-main} with Relaxed Assumptions}\label{sec:data_relaxation}

\subsubsection{Relaxation under More General Data Distributions}
 
We first remark on the possibility of relaxing the Gaussian assumptions on the data distributions. 

\begin{remark}[{Relaxing the Gaussian  distribution assumption}] 
In the proof of \Cref{lem:quantal}, to obtain the result that the action is a quantal response w.r.t. $\ell_{1:T}$, one  does not necessarily require  \emph{both} the prior distribution of $z$ and the conditional distribution of 
${\ell}_i$ given $z$ to be Gaussian. Instead, for \emph{any}  joint distribution  $\PP(z, \ell_{1:T})$, as long as its posterior distribution satisfies \Cref{eq:normal-pos}, it would suffice. It is a combined effect of  both the prior and the conditional distributions.
\end{remark}

More formally, we can extend \Cref{thm:ftpl-main} to the case with a much more general prior task distribution than the Gaussian one,  where the key is that \Cref{eq:normal-pos} only needs to hold approximately. 

\begin{theorem}\label{prop:extension}
In \Cref{thm:ftpl-main}, we can relax the assumption on $\PP(z)$ to one where we only require $\PP(z)$ to be i.i.d for each coordinate of $z$ and $0<\PP(z_j)<\infty$, $|\nabla\PP(z_j)|<\infty$ for any $j\in[d]$, $z_j\in \RR$, and the bounds for (1) and (2) of \Cref{thm:ftpl-main} still hold, with only a degradation of  $\cO(d^2\log T)$. 
\end{theorem}
The key idea of the proof is that when $t$ is large enough, the prior distribution does not affect the posterior distribution, which is also referred to as the \emph{Bernstein–von Mises theorem} \citep{van2000asymptotic}.

\begin{proof}
Since we extend \Cref{thm:ftpl-main} to settings with general task prior distribution only requiring the coordinates to be i.i.d, from now on, we consider the  $j$-th coordinate only. To begin with, fix $t\in[T]$, we define the log-likelihood of the posterior as 
\$
L_t(z_j) := \log \prod_{i=1}^t \frac{1}{\sigma^d(2\pi)^{d/2}}e^{-\frac{1}{2\sigma^2}(\ell_{ij}-z_j)^2}=-t\log \sigma-\frac{t}{2}\log2\pi-\sum_{i=1}^t\frac{1}{2\sigma^2}(\ell_{ij}-z_j)^2.
\$
Then, the MLE estimator $\hat{z}_{j, t}$ is defined as 
\$
\hat{z}_{j, t} :=\arg\max_{z_j\in\RR}L_t(z_j)=\frac{1}{t}\sum_{i=1}^t \ell_{ij}.
\$
We also define $\hat{J}_t:\RR\to\RR$ as:
\$
\hat{J}_t(z_j):=-\frac{\nabla^2 L_t(z_j)}{t} = \frac{1}{\sigma^2}.
\$

For Assumption 1 of \cite{kasprzak2022good} to hold, any $\delta>0$, $M_2>0$ suffices. 

For Assumption 2 of \cite{kasprzak2022good} to hold, we can choose $\hat{M}_1= \max_{z_j\in[-\delta, 1+\delta]}\frac{1}{\PP(z_j)}$
 
For Assumption 7 of \cite{kasprzak2022good} to hold, we choose $\delta$ to be $\sigma$.

For Assumption 8 of \cite{kasprzak2022good} to hold, one can choose $M_2=\frac{\sigma}{2}$.

For Assumption 9 of \cite{kasprzak2022good} to hold, we have 
\$
\kappa\le - \sup_{(z_j-\hat{z}_j)^2\ge \delta}\frac{L_t(z_j)-L_t(\hat{z}_{j, t})}{t}=-\frac{1}{2\sigma^2 t}\sup_{(z_j-\hat{z}_{j, t})^2\ge \delta}\sum_{i=1}^t(\ell_{ij}-\hat{z}_{j, t})^2- (\ell_{ij}-z_j)^2=\frac{1}{4\sigma}.
\$

For Assumption 10 of \cite{kasprzak2022good} to hold, we choose $M_1 = \sup_{z_j\in[-\delta, 1+\delta]}\left|\frac{\nabla\PP(z_j)}{\PP(z_j)}\right|$, $\tilde{M}_1=\sup_{z_j\in[-\delta, 1+\delta]}|\PP(z_j)|$ since we have assumed that $0<\PP(z_j)<\infty$, $|\nabla\PP(z_j)|<\infty$.

By Theorem 6.1 of \cite{kasprzak2022good}, we have
\begin{align*}
\int_{z_j}| &\PP(z_j/\sqrt{t}+\hat{z}_j\given (\ell_{ij})_{i\in[t]}) -Ce^{-\frac{1}{2\sigma^2}z_j^2}|dz_j 
\\
&= \sqrt{t}\int_{z_j} | \PP(z_j\given (\ell_{ij})_{i\in[t]}) -\cN(\hat{z}_j, \frac{\sigma^2}{t}) |dz_j\le D_1t^{-1/2} + D_2t^{1/2}e^{-t\kappa} + 2\hat{\cD}(t, \delta),    
\end{align*}
where $C$ is the normalization constant and 
\$
D_1 &= \frac{\sqrt{\tilde{M}_1\hat{M}_1}}{\sigma}\left(\frac{\sqrt{3}\sigma^2}{2\left(1-\sqrt{\hat{\cD}(t, \delta)}\right)} M_2 + M_1\right)\\
D_2 &= \frac{2\hat{M_1} \hat{J}_t^p(\hat{z}_j, \delta)}{(2\pi)^{1/2}(1-\hat{\cD}^p(t, \delta))}\\
\hat{\cD}(t, \delta)&=e^{-\frac{1}{2} (\sqrt{t}-1)^2}\\
\hat{J}_t^p(\hat{z}_j, \delta) &= \frac{1}{\sigma^2} + \frac{\delta M_2}{3}.
\$
Therefore, we conclude that the  TV distance  between $z$ (conditioned on $(\ell_{i})_{i\in[t]}$) and $\cN\left(\hat{z}, \frac{\sigma^2}{t}\right)$ satisfies  that 
\$
\int_{z} \bigg| \PP(z\given (\ell_{i})_{i\in[t]}) -\cN\left(\hat{z}, \frac{\sigma^2}{t}\right)\bigg|dz\le \sum_{j=1}^d\int_{z_j} \bigg| \PP(z_j\given (\ell_{ij})_{i\in[t]}) -\cN\left(\hat{z}_j, \frac{\sigma^2}{t}\right) \bigg|dz_j \le \cO(d/t),
\$
due to the independence of $(z_j)_{j\in[d]}$ conditioned on $\ell_{1:t}$. Now we denote algorithm $\hat{\text{FTPL}}$ to be the FTPL algorithm w.r.t. the noise distribution $\PP(z\given (\ell_{i})_{i\in[t]})$, and $\text{FTPL}$ to be the algorithm w.r.t. the noise distribution $\cN(\hat{z}, \frac{\sigma^2}{t})$.
Therefore, we have
\$  
&\big|\text{Regret}_{\text{FTPL}}((\ell)_{i\in[T]})-\text{Regret}_{\hat{\text{FTPL}}}((\ell)_{i\in[T]})\big|\le \sum_{t=1}^T d\|\pi_t-\hat{\pi}_t\|_\infty\\
&\quad  \le d\sum_{t=1}^T\int_{z} \big| \PP(z\given (\ell_{i})_{i\in[t]}) -\cN(\hat{z}, \frac{\sigma^2}{t}) \big|dz=\cO(d^2\log T).
\$ 
In other words, using $\PP(z\given (\ell_{i})_{i\in[t]})$ as the noise distribution only increases the regret by $\cO(d^2\log T)$. Similarly, it is easy to see that
\$
\big|\text{D-Regret}_{\text{FTPL}}((\ell)_{i\in[T]})-\text{D-Regret}_{\hat{\text{FTPL}}}((\ell)_{i\in[T]})\big|\le \cO(d^2\log T),
\$
which completes the proof.
\end{proof}

\subsubsection{Relaxation under Decision-Irrelevant Pre-Training Data}\label{subsec:relax_decision_irrelevant}

We then remark on the possible relaxation when the training data may not all come from decision-making tasks. 

\begin{remark}[{Pre-training with relaxed  data assumptions}]\label{app:relaxtion}
{Note that the pre-training (text) data are so far assumed to be related to \emph{decision-making} problems (though not necessarily \emph{sequential ones}), see  \Cref{assump:decomp} and \Cref{example:data} for instance. It can also be generalized to the text datasets involving \emph{Question-Answering} (Q-A), a typical task in natural language processing, where the \emph{true/fact} answer, sampled answers from different human users (with possibly wrong or biased answers), correspond to the latent $z$ (and associated maximizer $a$) and $\ell_{1:t}$, respectively. Moreover, in practice, the pre-training data may also involve  \emph{non-decision-making/Q-A} texts, given the diversity of the datasets. For such scenarios, we will make the assumptions on the data distribution \emph{conditioned on the prompt for decision-making}. Specifically, when interacting with the LLM, human users will provide prompts (see e.g., our \Cref{fig:game-prompt}), to induce it to make decisions. This will query the \emph{conditional}  distribution of 
$$\PP\left(g(x_{N_t+1:N_{t+1}})\biggiven x_{1:N_t},\text{decision-making prompt}\right)$$ 
to generate the control action. Correspondingly, \Cref{assump:decomp} will thus only need to be made on 
$$
	\PP\left(z, \ell_{1:t}, x_{1:N_{t+1}},\text{decision-making prompt}\right), 
 $$ 
while we do not need to make such assumptions on other prompts, e.g., corpora that are not related to decision-making.
}
\end{remark}

\conf{
\subsection{Comparison to Recent Literature}
\label{appendix-ssec-comparison}
 \cpedit{Similar assumptions and objectives have also been considered in the very recent work of  \cite{lee2023supervised, lin2023transformers} {for studying in-context reinforcement learning (RL) property of Transformers under supervised pre-training}. \cite{lee2023supervised} established its equivalence to \emph{posterior sampling} \citep{osband2013more}{, an important RL algorithm with provable regret guarantees when the environments are \emph{stationary}}, and \cite{lin2023transformers} generalized the study of the settings with algorithm distillation as in \cite{laskin2022context}.} However, their results cannot imply the no-regret guarantee in our online learning setting, due to the known facts that posterior sampling can perform poorly under potentially \emph{adversarial} or \kzedit{\emph{non-stationary}}  environments 
 \citep{zimmert2021tsallis,liu2023nonstationary}. 
 In contrast, we here establish the equivalence {of the pre-trained LLM}  to the FTPL algorithm {(under different pre-training distribution specifications)}, with the ability to handle arbitrary loss sequences, even though the LLMs are only trained on a fixed distribution of \emph{stationary} {online learning} problems.}  
 
 \subsection{Comparison with  \cite{lee2023supervised, lin2023transformers, liu2023reason}}\label{app:comparison}
 \arxiv{
{Similar assumptions and pre-training objectives have also been considered in the very recent work of  \cite{lee2023supervised, lin2023transformers,liu2023reason} {for studying in-context reinforcement learning \iclr{(RL)} property of Transformers/LLM-agents under supervised pre-training}. \cite{lee2023supervised} established its equivalence to \emph{posterior sampling} \citep{osband2013more}{, an important RL algorithm with provable regret guarantees when the environments are \emph{stationary}}, and \cite{lin2023transformers} generalized the study to the setting of algorithm distillation as in   \cite{laskin2022context}.} {\cite{liu2023reason} adopted the similar data generation assumption as \cite{lee2023supervised} without assuming optimal labels are available in the pre-training datasets, but leverages external oracles for \emph{planning}. Consequently, the resulting LLM agent would still perform the posterior sampling algorithm.} However, these results cannot directly imply the no-regret guarantee in our online learning setting, due to the known fact that posterior sampling can perform poorly under potentially \emph{adversarial} or {\emph{non-stationary}}  environments \citep{zimmert2021tsallis,liu2023nonstationary}.  
In contrast, we here establish the equivalence {of the pre-trained LLM}  to the FTPL algorithm {(under different pre-training data distribution specifications)}, with the ability to handle arbitrary loss sequences, even though the LLMs are only trained on a fixed/stationary distribution of texts (tasks).} 

 \subsection{Details of Estimating the Parameters of Our Hypothetical Model}
 
 {To further validate our model and data distribution assumptions, we also propose to estimate the parameter $\{\eta_t\}_{t\in[T-1]}$ in \Cref{quantal}, using data from interacting with LLMs {(following the same protocol as before)},  
with $P_{{noise}}$ being a standard normal distribution {(note that we do not need to estimate $\eta_0$ by \Cref{quantal})}. 
\arxiv{Specifically, given $n$ episodes of the LLM agent's behavior $\{({\ell}_{t}^{(j)},\pi_t^{(j)})_{t\in[T]}\}_{j\in [n]}$, motivated by our \Cref{lem:quantal} and \Cref{thm:ftpl-main}, 
we estimate $\{\eta_t\}_{t\in[T-1]}$  by solving the following problem  
\$
\sigma^\star \in \arg\min_{\sigma>0}\quad  \sum_{t\in[T-1]}\sum_{j\in[n]}\Big\|\pi_{t+1}^{(j)}-P^{\sigma\sqrt{t+1}}_{quantal}\Big(\cdot\bigggiven {\ell}_{1:t}^{(j)}\Big)\Big\|_1,\qquad \quad 
\eta_t^\star =  \sigma^\star\sqrt{t+1},\quad  \forall t\in[T-1].
\$
We solve this single-variable optimization problem by grid search over $[0, 10]$.} We then} run the generalized quantal response model with the estimated $\{\eta_t^\star\}_{t\in[T-1]}$ on another \emph{unseen test set}, and compare it with the behavior of the actual LLM agents. {We use all the interaction data from \Cref{sec:experiments_online} and split it in half for training and testing.

We also use the same framework to understand the regrettable behaviors in \Cref{sec:yes_regret_example}. This analysis uses all the  data from \Cref{sec:yes_regret_example}. We first find that such fitting procedures do not yield good predictions for LLMs on those counter-examples. Therefore, we resort to a more expressive model by directly fitting each $\eta_{t}$ as
\$\eta_{t}^\star \in \arg\min_{\eta_{t}>0}\sum_{j\in[n]}\Big\|\pi_{t+1}^{(j)}-P^{\eta_{t}}_{quantal}\Big(\cdot\bigggiven {\ell}_{1:t}^{(j)}\Big)\Big\|_1\$ {separately for each $t\in[T-1]$}. 
{Even under the expressive model, LLMs fail to follow the generalized quantal response for the counter-examples with noisy alternating or adaptive loss sequences, as \Cref{fig:calibration} shows the gap between GPT-4 (dynamic) regret and the our model's (dynamic) regret. }

\conf{
\subsection{Details on Calibration}\label{sec:calibration}
Given $N$ episodes of the LLM agent's behavior $\{({\ell}_{t}^j, \pi_t^j)_{t\in[T]}\}_{j\in [N]}$, we propose to calibrate $\{\eta_t\}_{t\in[T]}$ by solving the following problem for each $t\in[T]$
\$
\eta_t^\star \in \arg\min_{\eta^t}\sum_{j\in[N]}\Big\|\pi_t^j-P^{\eta_t}_{quantal}\Big(\cdot\given \sum_{t^\prime=1}^{t-1}{\ell}_{t^\prime}\Big)\Big\|_1.
\$
We solve this single-variable optimization problem by grid search over $[0, 10]$.}
\section{Deferred Results and Proofs   in  \Cref{sec:trained-transformer}} 

\subsection{Basic Lemmas}
\begin{lemma}[Double iterated limit]
\label{lem:double-seq} 
    {For a sequence $(a_{mn})_{m,n\in\NN^+}$,} suppose that $\lim_{m,n \to \infty} a_{mn} = L$. Then the following are equivalent:
    \begin{itemize}
        \item For each $m$, $\lim_{n \to \infty} a_{mn}$ exists;
        \item $\lim_{m\to \infty} \lim_{n \to \infty} a_{mn} = L$.
    \end{itemize}
\end{lemma}
\begin{lemma}[Hoeffding's inequality]\label{lem:Hoeffeding}
Let \( X_1, X_2, \ldots, X_n \) be independent random variables bounded by the intervals \( [a_i, b_i] \),  respectively. Define \( \bar{X} = \frac{1}{n} \sum_{i=1}^{n} X_i \) and let \( \mu = \mathbb{E}[\bar{X}] \) be the expected value of \( \bar{X} \). Then, for any \( t > 0 \), 
\[
\PP(|\bar{X} - \mu| \geq t) \leq 2 \exp\left(-\frac{2n^2t^2}{\sum_{i=1}^{n}(b_i - a_i)^2}\right).
\]
\end{lemma}

\begin{lemma}[Uniform convergence  $\Longrightarrow$ Interchanging limit and infimum]
\label{lem:change-lim-inf}
If $(f_n: X \to \RR)_{n \in \NN^+}$ is a sequence of continuous functions that uniformly converge to a function $f: X \to \RR$ on the domain $X$, then $\lim_{n \to \infty} \inf_{x \in X} f_n(x) = \inf_{x \in X} f(x)$ holds. 
\end{lemma}

\subsection{{Deferred Proof for} the Arguments in  \Cref{ssec:regret-loss}} 
\label{appendix:mathrigor}

In this section, we prove some properties of $\cL(\theta, k, N)$ {under certain  regularity conditions of $f,h$}. Throughout this subsection, we will assume the following condition holds.

\begin{condition}
\label{cond}
For $h: \RR \to \RR^+$ and $f: \RR \times \NN^+ \to \RR^+$, suppose $h(\cdot)$ and $f(\cdot, k)$ are both continuous and non-decreasing functions for any $k \in \NN^+$. The derivative  $h':\RR \to \RR$ is also a continuous function. {Moreover,  $f$ satisfies that $ \log f(R_1 , k_1) - \log f(R_1, k_2) \geq \log f(R_2, k_1) - \log f(R_2, k_2)$ for $R_1 \geq R_2$ and $k_1 \geq k_2$}{, i.e., $\log f$ is supermodular}.  Lastly, $f$ is a function such that $\lim_{k \to \infty} \frac{f(R_1, k)}{f(R_2, k)} = \infty \cdot \mathbbm{1}(R_1 > R_2) + \mathbbm{1}(R_1 = R_2)$, with the convention of $\infty \cdot 0= 0$. Lastly, $(\ell_t^{(j)})_{t\in [T], j \in [N]}$  {are continuous random variables supported on $[-B, B]^{T \times N}$}.     
\end{condition}

\begin{claim}[Iterated limit of $\cL(\theta, k, N)$ is the same as  double limit of $\cL(\theta, k, N)$] It holds that: 
\begin{align*}
    \lim_{N \to \infty} \lim_{k \to \infty}\cL(\theta, k, N) = \lim_{N, k \to \infty} \cL(\theta, k, N) = \lim_{k \to \infty} \lim_{N \to \infty}\cL(\theta, k, N) =  h\left(\max_{\ell_1, \dots, \ell_T} \emph{Regret}_{\emph{LLM}_\theta} ((\ell_t)_{t\in[T]})\right).
    \end{align*}
\end{claim}
\begin{proof}
\textbf{{\color{blue}Step 1.} Proving}  $\lim_{N \to \infty} \lim_{k \to \infty} \cL(\theta, k, N) = h\left(\max_{\ell_1, \dots, \ell_T} \text{Regret}_{\text{LLM}_\theta}((\ell_t)_{t\in[T]})\right)$. 

Firstly, as both $h$ and $f$ are non-negative (\Cref{cond}), and $\EE_{(\ell_t^{(j)})_{t\in [T], j \in [N]}} \left[ h(\max_{j \in [N]}\text{Regret}_{\text{LLM}_\theta}((\ell_t^{(j)})_{t\in[T]})) \right]$ exists, we have by dominated convergence theorem that 
\begin{align*}
    \lim_{k \to \infty} \cL(\theta, k, N) &= \EE \lim_{k \to \infty} \left[\frac{\sum_{j \in [N]} h(R_{\text{LLM}_\theta}((\ell_t^{(j)})_{t\in[T]})) f(R_{\text{LLM}_\theta}((\ell_t^{(j)})_{t\in[T]}), k)}{\sum_{j \in [N]}f(R_{\text{LLM}_\theta}((\ell_i^{(j)})_{t\in[T]}), k)}\right] 
    \\
    &= \EE_{(\ell_t^{(j)})_{t\in [T], j \in [N]}} \left[ h(\max_{j \in [N]}R_{\text{LLM}_\theta}((\ell_t^{(j)})_{t\in[T]})) \right]
\end{align*}
where {$R_{\text{LLM}_\theta}$ denotes an abbreviation of $\text{Regret}_{\text{LLM}_\theta}$.} 
By \cite[Chapter 11]{ahsanullah2013introduction},  
we have  $h(\max_{j \in [N]}\text{Regret}_{\text{LLM}_\theta}((\ell_t^{(j)})_{t\in[T]}))  \overset{p}{\to} h(\max_{\ell_1, \dots, \ell_T} \text{Regret}_{\text{LLM}_\theta}((\ell_t)_{t\in[T]}))$ when $N \to \infty$. Hence, we have 
$\lim_{N \to \infty} \lim_{k \to \infty} \cL(\theta, k, N) = h(\max_{\ell_1, \dots, \ell_T} \text{Regret}_{\text{LLM}_\theta}((\ell_t)_{t\in[T]}))$ holds. 

\safevspace{6pt}
\textbf{{\color{blue}Step 2.} Proving}   $\lim_{N, k \to \infty} \cL(\theta, k, N) = h(\max_{\ell_1, \dots, \ell_T} \text{Regret}_{\text{LLM}_\theta}((\ell_t)_{t\in[T]}))$. 

Now, we will calculate $\lim_{N, k \to \infty} \cL(\theta, k, N)$. 

\begin{lemma}
\label{lem:asconvergence} For any $0 < \epsilon < 1$,   it follows that  
        $$\lim_{N, k \to \infty} \frac{\sum_{i=1}^N f(X_i, k) H(X_i) \mathbbm{1}(H(X_i) < 1 -\epsilon)}{\sum_{i=1}^N f(X_i, k) H(X_i) \mathbbm{1}(H(X_i) > 1 -\epsilon/2)} = 0$$ and 
        $$\lim_{N,k \to \infty} \frac{\sum_{i=1}^N f(X_i, k) \mathbbm{1}(H(X_i) < 1 -\epsilon)}{\sum_{i=1}^N f(X_i, k) \mathbbm{1}(H(X_i) > 1 -\epsilon/2)} = 0$$ hold with probability $1$, where $X_i$'s are i.i.d.  random variables, $\operatorname{esssup}(H(X_i)) = 1$, and $H : \RR \to \RR^+$  is a continuous non-decreasing function. 
    \end{lemma}
    \begin{proof}[Proof of \Cref{lem:asconvergence}]
    Since $f(\cdot, k), H$ are non-negative and non-decreasing functions, we have 
        \begin{align*}
 \frac{\sum_{i=1}^N f(X_i, k) H(X_i) \mathbbm{1}(H(X_i) < 1 -\epsilon)}{\sum_{i=1}^N f(X_i, k) H(X_i) \mathbbm{1}(H(X_i) > 1 -\epsilon/2)} \leq \frac{(1-\epsilon)f(H^{-1}(1-\epsilon), k) |\{i \in [N]\mid (H(X_i)< 1-\epsilon)\}|}{(1- \epsilon/2)f(H^{-1}(1-\epsilon/2), k)|\{i \in [N]\mid (H(X_i)> 1-\epsilon/2)\}|}
        \end{align*}
        and we know that \begin{align*}
            \frac{|\{i \in [N]\mid (H(X_i)< 1-\epsilon)\}|}{|\{i \in [N]\mid (H(X_i)> 1-\epsilon/2)\}|} \overset{a.s.}{\to} \frac{F(1-\epsilon)}{1-F(1-\epsilon/2)}
        \end{align*}
        as $N \to \infty$,  where $F$ is the cumulative distribution function of random variable $H(X)$. Therefore, we have 
        \begin{align*}
0 &\leq \lim_{N, k \to \infty} \frac{\sum_{i=1}^N f(X_i, k) H(X_i) \mathbbm{1}(H(X_i) < 1 -\epsilon)}{\sum_{i=1}^N f(X_i, k) H(X_i) \mathbbm{1}(H(X_i) > 1 -\epsilon/2)}
\\&\leq  \lim_{N, k \to \infty} \frac{(1-\epsilon)f(H^{-1}(1-\epsilon), k)) |\{i \in [N]\mid (H(X_i)< 1-\epsilon)\}|}{(1-\epsilon/2)f(H^{-1}(1-\epsilon/2), k))|\{i \in [N]\mid (H(X_i)> 1-\epsilon/2)\}|}
 \\
 &\underset{a.s.}{\leq} \lim_{N,k \to \infty}  \frac{(1-\epsilon)f(H^{-1}(1-\epsilon), k))}{(1-\epsilon/2)f(H^{-1}(1-\epsilon/2), k))}\frac{F(1-\epsilon)}{1-F(1-\epsilon/2)} = 0.
        \end{align*}
        By a similar argument, we have 
        $$\lim_{N,k \to \infty} \frac{\sum_{i=1}^N f(X_i, k) \mathbbm{1}(H(X_i) < 1 -\epsilon)}{\sum_{i=1}^N f(X_i, k) \mathbbm{1}(H(X_i) > 1 -\epsilon/2)} = 0$$
        with probability 1. 
    \end{proof}
    
    {One key idea in the proof above is the use of some \emph{truncation}  level $\epsilon$ for $H(X)$ {with}   $\text{esssup}(H(X)) = 1$. } 
    By \Cref{lem:asconvergence}, we have $$\lim_{N, k \to \infty} \frac{\sum_{i=1}^N f(X_i, k) H(X_i) \mathbbm{1}(H(X_i) > 1 -\epsilon)}{\sum_{i=1}^N f(X_i, k) H(X_i)} = \lim_{N, k \to \infty} \frac{\sum_{i=1}^N f(X_i, k) \mathbbm{1}(H(X_i) > 1 -\epsilon)}{\sum_{i=1}^N f(X_i, k)} = 1,$$
    since 
    \begin{align*}
        0 \leq  \frac{\sum_{i=1}^N f(X_i, k) \mathbbm{1}(H(X_i) < 1 -\epsilon)}{\sum_{i=1}^N f(X_i, k)} \leq \frac{\sum_{i=1}^N f(X_i, k) \mathbbm{1}(H(X_i) < 1 -\epsilon)}{\sum_{i=1}^N f(X_i, k) \mathbbm{1}(H(X_i) > 1 -\epsilon/2)} 
    \end{align*}
    holds with probability 1. 
    Therefore, for any $0<\epsilon < 1$, we have 
    \begin{align*}
        &\lim_{N, k \to \infty} \cL(\theta, k, N) = \EE \lim_{N, k \to \infty} \left[\frac{\sum_{j \in [N]} h(R_{\text{LLM}_\theta}((\ell_t^{(j)})_{t\in[T]})) f(R_{\text{LLM}_\theta}((\ell_t^{(j)})_{t\in[T]}), k)}{\sum_{j \in [N]}f(R_{\text{LLM}_\theta}((\ell_i^{(j)})_{t\in[T]}), k)}\right]
        \\
        &=  h\left(\max_{\ell_1, \dots, \ell_T} R_{\text{LLM}_\theta}((\ell_t)_{t\in[T]})\right)
        \\
        &\, \, \quad \times \EE \lim_{N, k \to \infty} \left[\frac{\sum_{j \in [N]} \frac{h(R_{\text{LLM}_\theta}((\ell_t^{(j)})_{t\in[T]}))}{h(\max_{\ell_1, \dots, \ell_T} R_{\text{LLM}_\theta}((\ell_t)_{t\in[T]}))}  f(R_{\text{LLM}_\theta}((\ell_t^{(j)})_{t\in[T]}), k) \mathbbm{1}(\frac{h(R_{\text{LLM}_\theta}((\ell_t^{(j)})_{t\in[T]}))}{h(\max_{\ell_1, \dots, \ell_T} R_{\text{LLM}_\theta}((\ell_t)_{t\in[T]}))} > 1-\epsilon)}{\sum_{j \in [N]}f(R_{\text{LLM}_\theta}((\ell_i^{(j)})_{t\in[T]}), k)\mathbbm{1}(\frac{h(R_{\text{LLM}_\theta}((\ell_t^{(j)})_{t\in[T]}))}{h(\max_{\ell_1, \dots, \ell_T} R_{\text{LLM}_\theta}((\ell_t)_{t\in[T]}))} > 1-\epsilon)}\right]
        \\
        &\geq (1-\epsilon) h(\max_{\ell_1, \dots, \ell_T} R_{\text{LLM}_\theta}((\ell_t)_{t\in[T]}))
    \end{align*}
which implies $\lim_{N, k \to \infty} \cL(\theta, k, N) =  h(\max_{\ell_1, \dots, \ell_T} \text{Regret}_{\text{LLM}_\theta}((\ell_t)_{t\in[T]}))$ 
since  
$$\cL(\theta, k, N) \leq h\left(\max_{\ell_1, \dots, \ell_T} \text{Regret}_{\text{LLM}_\theta}((\ell_t)_{t\in[T]})\right)$$ 
by definition of $\cL$,   {the 
 fact that $h$ is non-decreasing,} and by setting  $\epsilon \to 0$ to obtain 
 $$
 \cL(\theta, k, N) \geq h\left(\max_{\ell_1, \dots, \ell_T} \text{Regret}_{\text{LLM}_\theta}((\ell_t)_{t\in[T]})\right).
 $$
 Here, we used the fact that $(\ell_t)_{t \in [T]}$ has a continuous distribution, $\text{Regret}_{\text{LLM}_\theta}((\ell_t)_{t\in[T]})$ is a continuous function, and the non-decreasing property and continuity of $h$ (\Cref{cond}), which lead to:   
 \begin{align} 
     \text{esssup}\left(h\left(\text{Regret}_{\text{LLM}_\theta}((\ell_t)_{t\in[T]})\right)\right)= \max_{\ell_1, \dots, \ell_T}h\left(\text{Regret}_{\text{LLM}_\theta}((\ell_t)_{t\in[T]})\right) = h\left(\max_{\ell_1, \dots, \ell_T} \text{Regret}_{\text{LLM}_\theta}((\ell_t)_{t\in[T]})\right). \label{eqn:esssupmax}
 \end{align}
\Cref{eqn:esssupmax} will be used frequently in the overall proof in \Cref{appendix:mathrigor}. 

\safevspace{6pt}
\textbf{{\color{blue}Step 3.} Proving}    $\lim_{k \to \infty} \lim_{N \to \infty} \cL(\theta, k, N) = h\left(\max_{\ell_1, \dots, \ell_T} \text{Regret}_{\text{LLM}_\theta}((\ell_t)_{t\in[T]})\right)$.  

Lastly, if $N \to \infty$, similarly by dominated convergence theorem we have
\begin{align*}
    \lim_{N \to \infty} \cL(\theta, k, N) &= \EE \lim_{N \to \infty} \left[\frac{\sum_{j \in [N]} h\left(R_{\text{LLM}_\theta}\left((\ell_t^{(j)})_{t\in[T]}\right)\right) f(R_{\text{LLM}_\theta}((\ell_t^{(j)})_{t\in[T]}), k)}{\sum_{j \in [N]}f\left(R_{\text{LLM}_\theta}\left((\ell_i^{(j)})_{t\in[T]}\right), k\right)}\right] 
    \\
    &= \frac{\EE \left[h\left(R_{\text{LLM}_\theta}\left((\ell_t^{(j)})_{t\in[T]}\right)\right) f\left(R_{\text{LLM}_\theta}\left((\ell_t^{(j)})_{t\in[T]}\right),k\right)\right]}{\EE \left[f\left(R_{\text{LLM}_\theta}\left((\ell_i^{(j)})_{t\in[T]}\right), k\right)\right]}.
\end{align*}
Thus, $\lim_{N \to \infty}  \cL(\theta, k, N)$ always exists for every $k$. Now, we use the known property of double iterated limit (\Cref{lem:double-seq}), and obtain that  $\lim_{k \to \infty} \lim_{N \to \infty}\cL(\theta, k, N) =  h(\max_{\ell_1, \dots, \ell_T} \text{Regret}_{\text{LLM}_\theta}((\ell_t)_{t\in[T]}))$. 
\end{proof}

\begin{claim}[Uniform convergence of $\cL(\theta, k, N)$ (with respect to $k$ and $N$)]
\label{claim:uniform-all}
    $\cL(\theta, k, N)$ uniformly converges to $h(\max_{\ell_1, \dots, \ell_T} \emph{Regret}_{\emph{LLM}_\theta}((\ell_t)_{t\in[T]}))$ {on the domain $\Theta$}. 
\end{claim} 
\begin{proof}
We will provide a similar analysis as  \Cref{lem:asconvergence} as follows: 
    \begin{lemma}
    \label{lem:nonasymptotic-convergence}
    For any $0 < \epsilon < 1$, $0 < \delta < 1$, and $k\in\NN^+$, we  have 
    $$\frac{\sum_{i=1}^N f(X_i, k) \mathbbm{1}(H(X_i) < 1 -\epsilon)}{\sum_{i=1}^N f(X_i, k) \mathbbm{1}(H(X_i) > 1 -\epsilon)} = \tilde{\cO}\left(A(k, H, \epsilon)\left(\frac{1}{1- F_{H, X}(1-\epsilon/2)} + \frac{1}{\sqrt{N}}\right)\right)$$ with probability at least $1-\delta$, where $X_i$'s are i.i.d. random variables,  $\operatorname{esssup}(H(X_i)) = 1$, $H : \RR \to \RR^+$ is a continuous non-decreasing function,   $A(k, t, \epsilon):= 
    \frac{(1-\epsilon)f((t/\operatorname{esssup}(t(X)))^{-1}(1-\epsilon), k)}{(1-\epsilon/2)f((t/\operatorname{esssup}(t(X)))^{-1}(1-\epsilon/2), k)}$, for any non-decreasing function $t:\RR \to \RR^+$, and $F_{t, X}$ is a cumulative distribution function of random variable $t(X)/\text{esssup}(t(X))$. 
    \end{lemma}
    
    \begin{proof}[Proof of \Cref{lem:nonasymptotic-convergence}]
    With  the same  argument as the proof of \Cref{lem:asconvergence}, we have 
        \begin{align*}
 \frac{\sum_{i=1}^N f(X_i, k) \mathbbm{1}(H(X_i) < 1 -\epsilon)}{\sum_{i=1}^N f(X_i, k) \mathbbm{1}(H(X_i) > 1 -\epsilon/2)} \leq \frac{f(H^{-1}(1-\epsilon), k)|\{i \in [N]\mid (H(X_i)< 1-\epsilon)\}|}{f(H^{-1}(1-\epsilon/2), k)|\{i \in [N]\mid  (H(X_i)> 1-\epsilon/2)\}|}. 
        \end{align*}
        It holds that $\frac{1}{N} |\{i \in [N]\mid (H(X_i)< 1-\epsilon)\}| = F_{H, X}(1-\epsilon) + \tilde{\cO}(1/\sqrt{N})$ with probability at least $1-\delta/2$ due to Hoeffding's inequality (\Cref{lem:Hoeffeding}). Similarly, we have $\frac{1}{N} |\{i \in [N]\mid (H(X_i)> 1-\epsilon/2)\}| = 1-F_{H, X}(1-\epsilon/2) + \tilde{\cO}(1/\sqrt{N})$ with probability at least $1-\delta/2$. Therefore, 
        \begin{align*}
            \frac{|\{i \in [N]\mid (H(X_i)< 1-\epsilon)\}|}{|\{i \in [N]\mid (H(X_i)> 1-\epsilon/2)\}|} = \frac{F_{H, X}(1-\epsilon)}{1-F_{H, X}(1-\epsilon/2)} + \tilde{\cO}(\sqrt{1/N}) \leq \frac{1}{1- F_{H, X}(1-\epsilon/2)} +  \tilde{\cO}(\sqrt{1/N}), 
        \end{align*}
        with probability at least $1-\delta$. Finally, we have  
        \begin{align*}
 \frac{\sum_{i=1}^N f(X_i, k)  \mathbbm{1}(H(X_i) < 1 -\epsilon)}{\sum_{i=1}^N f(X_i, k)  \mathbbm{1}(H(X_i) > 1 -\epsilon)} < \frac{\sum_{i=1}^N f(X_i, k)\mathbbm{1}(H(X_i) < 1 -\epsilon)}{\sum_{i=1}^N f(X_i, k)  \mathbbm{1}(H(X_i) > 1 -\epsilon/2)} &\leq A(k, H, \epsilon) \left(\frac{1}{1- F_{H, X}(1-\epsilon/2)} +  \tilde{\cO}(\frac{1}{\sqrt{N}})\right) .
        \end{align*}
        \end{proof}
        Note that $\lim_{k \to \infty} A(k, H, \epsilon) = 0$, since $\lim_{k \to \infty} \frac{f(R_1, k)}{f(R_2, k)} = \infty \cdot \mathbbm{1}(R_1 > R_2) + \mathbbm{1}(R_1 = R_2)$. 
By \Cref{lem:nonasymptotic-convergence} with $H ( R_{\text{LLM}_\theta}((\ell_t)_{t\in[T]}))= \frac{h(R_{\text{LLM}_\theta}((\ell_t)_{t\in[T]}))}{h(\max_{\ell_1, \dots, \ell_T} R_{\text{LLM}_\theta}((\ell_t)_{t\in[T]}))}$, we have  
{\small
\begin{align*}
&\frac{\sum_{i=1}^N f(R_{\text{LLM}_\theta}((\ell_t^{(i)})_{t\in[T]}), k)  \mathbbm{1}\left(\frac{h(R_{\text{LLM}_\theta}((\ell_t^{(i)})_{t\in[T]}))}{h(\max_{\ell_1, \dots, \ell_T} R_{\text{LLM}_\theta}((\ell_t)_{t\in[T]}))} \geq 1 -\epsilon\right)}{\sum_{i=1}^N f(R_{\text{LLM}_\theta}((\ell_t^{(i)})_{t\in[T]}), k)} 
            \\
&\quad =\frac{1}{1+\frac{\sum_{i=1}^N f(R_{\text{LLM}_\theta}((\ell_t^{(i)})_{t\in[T]}), k)  \mathbbm{1}\left(\frac{h(R_{\text{LLM}_\theta}((\ell_t^{(i)})_{t\in[T]}))}{h(\max_{\ell_1, \dots, \ell_T} R_{\text{LLM}_\theta}((\ell_t)_{t\in[T]}))} < 1 -\epsilon\right)}{\sum_{i=1}^N f(R_{\text{LLM}_\theta}((\ell_t^{(i)})_{t\in[T]}), k)  \mathbbm{1}\left(\frac{h(R_{\text{LLM}_\theta}((\ell_t^{(i)})_{t\in[T]}))}{h(\max_{\ell_1, \dots, \ell_T} R_{\text{LLM}_\theta}((\ell_t)_{t\in[T]}))} \geq 1 -\epsilon\right)}}  \\
&\geq \frac{1}{1 + A(k, H, \epsilon) (\frac{1}{1- F_{H, R_{\text{LLM}_\theta}((\ell_t)_{t\in[T]})} (1-\epsilon/2)} +  \tilde{\cO}(\sqrt{1/N}))},
\end{align*} 
} 
  
\noindent where we recall the shorthand notation of $R_{\text{LLM}_\theta}=\text{Regret}_{\text{LLM}_\theta}$. Note that $A(k, H, \epsilon)=A(k, h, \epsilon)$ and $F_{H, R_{\text{LLM}_\theta}}=F_{h, R_{\text{LLM}_\theta}}$ hold by the definitions  of $F_{t, X}$ and $A(k, t, \epsilon)$ in \Cref{lem:nonasymptotic-convergence}. 
    Therefore, 
\begin{align*}
    1 &\geq \frac{\sum_{i=1}^N f(R_{\text{LLM}_\theta}((\ell_t^{(i)})_{t\in[T]}), k) \frac{h(R_{\text{LLM}_\theta}((\ell_t^{(i)})_{t\in[T]}))}{h(\max_{\ell_1, \dots, \ell_T} R_{\text{LLM}_\theta}((\ell_t)_{t\in[T]}))}}{\sum_{i=1}^N f(R_{\text{LLM}_\theta}((\ell_t^{(i)})_{t\in[T]}), k)} 
    \\&\geq \frac{\sum_{i=1}^N f(R_{\text{LLM}_\theta}((\ell_t^{(i)})_{t\in[T]}), k) \frac{h(R_{\text{LLM}_\theta}((\ell_t^{(i)})_{t\in[T]}))}{h(\max_{\ell_1, \dots, \ell_T} R_{\text{LLM}_\theta}((\ell_t)_{t\in[T]}))} \mathbbm{1}(\frac{h(R_{\text{LLM}_\theta}((\ell_t^{(i)})_{t\in[T]}))}{h(\max_{\ell_1, \dots, \ell_T} R_{\text{LLM}_\theta}((\ell_t)_{t\in[T]}))} \geq 1 -\epsilon)}{\sum_{i=1}^N f(R_{\text{LLM}_\theta}((\ell_t^{(i)})_{t\in[T]}), k) \mathbbm{1}(\frac{h(R_{\text{LLM}_\theta}((\ell_t^{(i)})_{t\in[T]}))}{h(\max_{\ell_1, \dots, \ell_T} R_{\text{LLM}_\theta}((\ell_t)_{t\in[T]}))} \geq 1 -\epsilon)} \\
    &\qquad  \times\frac{1}{1 + A(k, h, \epsilon) (\frac{1}{1- F_{h, R_{\text{LLM}_\theta}((\ell_t)_{t\in[T]})}(1-\epsilon/2)} +  \tilde{\cO}(\sqrt{1/N}))}
    \\
    &\geq \frac{1- \epsilon}{1 + A(k, h, \epsilon) (\frac{1}{1- F_{h, R_{\text{LLM}_\theta}((\ell_t)_{t\in[T]})}(1-\epsilon/2)} +  \tilde{\cO}(\sqrt{1/N}))}
\end{align*}
with probability at least $1-\delta$. 

Now, for any $\epsilon >0$ and $\delta > 0$,  we have 
\begin{align*}
    0 &\leq h\left(\max_{\ell_1, \dots, \ell_T} R_{\text{LLM}_\theta}((\ell_t)_{t\in[T]})\right) - \cL(\theta, k, N) 
    \\
    &\leq h\left(\max_{\ell_1, \dots, \ell_T} R_{\text{LLM}_\theta}((\ell_t)_{t\in[T]})\right) \left(1- \frac{(1-\delta)(1- \epsilon)}{1 + A(k, h, \epsilon) (\frac{1}{1- F_{h, R_{\text{LLM}_\theta}((\ell_t)_{t\in[T]})}(1-\epsilon/2)} +  \tilde{\cO}(\sqrt{1/N}))}\right).
\end{align*}
Note that  
$$1 - F_{h, R_{\text{LLM}_\theta}((\ell_t)_{t\in[T]})}(1-\epsilon/2) = \PP\left(h \left(\text{Regret}_{\text{LLM}_\theta}\left((\ell_t)_{t\in[T]}\right)\right) > (1-\epsilon/2)h\left(\max_{\ell_1, \dots, \ell_T} \text{Regret}_{\text{LLM}_\theta}\left((\ell_t)_{t\in[T]}\right)\right) \right)$$ 
is a continuous function of  $\theta$, since we assume $\text{LLM}_\theta$ is a continuous function {of $\theta$}, $(\ell_t)_{t \in [T]}$ has a continuous distribution, and $\text{Regret}_{\text{LLM}_\theta}((\ell_t)_{t \in [T]})$ is a continuous function of  $\text{LLM}_\theta$ and $(\ell_t)_{t \in [T]}$. Since we consider a  compact $\Theta$ (as several recent works on analyzing Transformers  \citep{bai2023transformers, lin2023transformers}), we have $p(\epsilon):= \min_{\theta \in \Theta} 1 - F_{h, R_{\text{LLM}_\theta}((\ell_t)_{t\in[T]})}(1-\epsilon/2) > 0$. Therefore, 
\begin{align}
    \left(1- \frac{(1-\delta)(1- \epsilon)}{1 + A(k, h, \epsilon) (\frac{1}{1- F_{h,R_{\text{LLM}_\theta}}(1-\epsilon/2)} +  \tilde{\cO}(\sqrt{1/N}))}\right) \leq \left(1- \frac{(1-\delta)(1- \epsilon)}{1 + A(k, h, \epsilon) (\frac{1}{p(\epsilon)} +  \tilde{\cO}(\sqrt{1/N}))}\right), \label{eqn:rhs-is-uniform}
\end{align} 
and we know that $\lim_{N, k \to \infty} 1 + A(k, h, \epsilon) (\frac{1}{p(\epsilon)} +  \tilde{\cO}(\sqrt{1/N})) = 1$, which is not dependent on $\theta$. Thus, we can conclude that $\lim_{N, k \to \infty} \sup_{\theta \in \Theta} |h(\max_{\ell_1, \dots, \ell_T} \text{Regret}_{\text{LLM}_\theta}((\ell_t)_{t\in[T]})) - \cL(\theta, k, N)| = 0${, as we can choose arbitrarily small $\epsilon,\delta$}. 
\end{proof}

\begin{claim}[Double iterated limit of supremum]
\label{claim:double-limit} It holds that:
    $$
    \lim_{N \to \infty} \lim_{k \to \infty} \sup_{\theta \in \Theta} ~~\Big| \cL(\theta, k ,N) - h\left(\max_{\ell_1, \dots, \ell_T} \emph{Regret}_{\text{LLM}_\theta}((\ell_t)_{t\in[T]})\right)\Big| = 0.$$ 
\end{claim}
\begin{proof}
    Since $h(\max_{\ell_1, \dots, \ell_T} \text{Regret}_{\text{LLM}_\theta}((\ell_t)_{t\in[T]})) \geq  \cL(\theta, k ,N)$, we will prove $$
    \lim_{N \to \infty} \lim_{k \to \infty} \sup_{\theta \in \Theta}~~  h\left(\max_{\ell_1, \dots, \ell_T} \text{Regret}_{\text{LLM}_\theta}((\ell_t)_{t\in[T]})\right) - \cL(\theta, k ,N)  = 0.
    $$ 

    \begin{lemma}\label{lem:inc}
        $\frac{\sum_{i=1}^N f(X_i, k_1) h(X_i)}{\sum_{i=1}^N f(X_i, k_1)} \leq \frac{\sum_{i=1}^N f(X_i, k_2) h(X_i)}{\sum_{i=1}^N f(X_i, k_2)} $ holds if ${0<}k_1 \leq k_2$ for any real-valued $(X_i)_{i \in [N]}$. 
    \end{lemma}
    \begin{proof}
        By multiplying $({\sum_{i=1}^N f(X_i, k_1)})({\sum_{i=1}^N f(X_i, k_2)})$ on both sides of the formula, we know that it is equivalent to $\sum_{1 \leq i\neq j \leq N} f(X_i, k_1) h(X_i) f(X_j, k_2) \leq \sum_{1 \leq i\neq j \leq N}f(X_i, k_1) h(X_j) f(X_j, k_2)$. This is equivalent to 
        \begin{align*}
            \sum_{1 \leq i\neq j \leq N} ( f(X_i, k_1) f(X_j, k_2) - f(X_j, k_1) f(X_i, k_2) )(h(X_i) - h(X_j)) \leq 0, 
        \end{align*}
        which is true since if $X_i \geq X_j$, $( f(X_i, k_1) f(X_j, k_2) - f(X_j, k_1) f(X_i, k_2) ) \leq 0$ due to the log-increasing
        difference of $f$ (\Cref{cond}), as $ \log f(X_j, k_1) - \log  f(X_j, k_2) \geq \log f(X_i, k_1) - \log  f(X_i, k_2)$ if $X_i \geq  X_j$. 
    \end{proof}
    
    Therefore, $\cL(\theta, k, N)$ is a  non-decreasing function of $k$ if $N$ is fixed, which indicates that 
    $$
    \lim_{k \to \infty} ~~\sup_{\theta \in \Theta}~~ h\left(\max_{\ell_1, \dots, \ell_T} \text{Regret}_{\text{LLM}_\theta}((\ell_t)_{t\in[T]})\right) - \cL(\theta, k ,N)
    $$  
    exists, as $\cL(\theta, k ,N)$ is also bounded. Therefore, by \Cref{lem:double-seq} and \Cref{claim:uniform-all}, we know that  
    $$
    \lim_{N \to \infty} \lim_{k \to \infty}~~ \sup_{\theta \in \Theta} ~~\Big| \cL(\theta, k ,N) - h\left(\max_{\ell_1, \dots, \ell_T} \text{Regret}_{\text{LLM}_\theta}((\ell_t)_{t\in[T]})\right)\Big|
    $$ 
    exists and this value should be 0. 
\end{proof}
\begin{claim}
\label{claim:limliminf}
    It holds that 
    $$
    \lim_{N, k \to \infty} \inf_{\theta \in \Theta} \cL(\theta, k, N) =  \lim_{N \to \infty} \lim_{k \to \infty} \inf_{\theta \in \Theta} \cL(\theta, k, N) = \inf_{\theta \in \Theta} h\left(\max_{\ell_1, \dots, \ell_T} \emph{Regret}_{\text{LLM}_\theta}((\ell_t)_{t\in[T]})\right).
    $$ 
\end{claim}
\begin{proof}
    Firstly, by \Cref{lem:change-lim-inf}, we have $\lim_{N, k \to \infty} \inf_{\theta \in \Theta} \cL(\theta, k, N)  =  \inf_{\theta \in \Theta} h(\max_{\ell_1, \dots, \ell_T} \text{Regret}_{\text{LLM}_\theta}((\ell_t)_{t\in[T]}))$. Plus, we already know that $\cL(\theta, k, N)$ is a monotonically  non-decreasing function of $k$ for any fixed  $N$ (\Cref{lem:inc}),  and it is bounded,  $\lim_{k \to \infty}\inf_{\theta \in \Theta} \cL(\theta, k, N)$ always exists. Therefore, by \Cref{lem:double-seq}
    ,  we also have $ \lim_{N \to \infty} \lim_{k \to \infty} \inf_{\theta \in \Theta} \cL(\theta, k, N) = \inf_{\theta \in \Theta} h(\max_{\ell_1, \dots, \ell_T} \text{Regret}_{\text{LLM}_\theta}((\ell_t)_{t\in[T]}))$.
\end{proof}
\subsection{Definition of the Empirical Loss Function}
{\begin{definition}[Empirical loss function] \label{def:empirical-loss}
    We define the empirical loss $\hat{\cL}$ computed with $N_T$ samples as follows:
\begin{align}
    &\hat{\cL}(\theta, k, N, N_T) :=      \frac{1}{N_T} \sum_{s = 1}^{N_T} \left[\frac{\sum_{j \in [N]} h\left(\emph{Regret}_{\emph{LLM}_\theta}((\ell_{s, t}^{(j)})_{t\in[T]})\right) f\left(\emph{Regret}_{\emph{LLM}_\theta}((\ell_{s, t}^{(j)})_{t\in[T]}), k\right)}{\sum_{j \in  [N]}f\left(\emph{Regret}_{\emph{LLM}_\theta}((\ell_{s, t}^{(j)})_{t\in[T]}), k\right)}\right]\label{eqn:empirical_regret_loss}
\end{align}
where $(\ell_{s,t}^{(j)})_{j \in [N], t \in [T]}$ denotes the $s$-th sample of $(\ell_{t}^{(j)})_{j \in [N], t \in [T]}$ for estimating $\cL(\theta, k, N)$. 
\end{definition}} 
\subsection{Deferred Proofs of \Cref{thm:generalization-error} and \Cref{coro:regret_guarantee}}
\label{appendix:proof-generalization-error}

\begin{restatable}{theorem}{gengap}
\label{thm:generalization-error}
\emph{(Generalization gap){\bf.}} Suppose $\emph{LLM}_\theta$ is Lipschitz-continuous  with respect to the model parameter $\theta$, then 
for any $0 < \epsilon < 1/2$, with  probability at least $1-\epsilon$, we have
\begin{align}
       \cL\left(\hat{\theta}_{k, N, N_T}, k, N\right) &-  \inf_{\theta \in \Theta} \cL(\theta, k, N)  \leq \tilde{\cO}\left(\sqrt{\frac{d_\theta+\log(1/\epsilon)}{N_T}}\right), \label{eqn:generalization-gap}  
    \end{align}     
    \normalsize 
    for any $N$ and sufficiently large $k$, {where $d_\theta$ is the {dimension} of the  parameter $\theta$.}
\end{restatable}  
{Through a careful use of Berge's Maximum Theorem \citep{berge1877topological}, we prove that the right-hand side of \Cref{eqn:generalization-gap} does \emph{not} depend on $k$ and $N$, which allows us to take the limit of $\lim_{N \to \infty} \lim_{k \to \infty}$ without affecting the generalization bound.}

Before proving the theorem, we remark on what LLM structure enjoys the Lipschitz-continuity. We provide two auxiliary results in the following proposition. 
 The first result is from \citep[Section J.1]{bai2023transformers}, which is about the Lipschitzness of Transformers. 
    The second result is regarding  {processing the output} of Transformers. 
    {In particular, the output of Transformers is usually not directly used, but passed through some matrix multiplication (by some matrix $A$), followed by some projection $\texttt{Operator}$ (to be specified later).} 
    
    \begin{proposition}
    \label{prop:transformer-lipschitz}
        The {$L$-layer} Transformer  $\texttt{TF}_{\theta}$  {as defined in \Cref{ssec:self-tf}} is $C_{\texttt{TF}}$-Lipschitz continuous with respect to $\theta$ with $C_{\texttt{TF}}:= L \left((1+B_{\texttt{TF}}^2)(1+B_{\texttt{TF}}^2R^3)\right)^L B_{\texttt{TF}}R(1 + B_{\texttt{TF}}R^2 + B_{\texttt{TF}}^3R^2)$, i.e., 
        \begin{align*}
            \norm{\texttt{TF}_{\theta_1}(Z) - \texttt{TF}_{\theta_2}(Z)}_{2, \infty} \leq C_{\texttt{TF}}\norm{\theta_1 - \theta_2}_{\texttt{TF}}
        \end{align*}  
                {where $\|\cdot\|_{\texttt{TF}}$ is as defined in \Cref{equ:def_TF_norm}, and $R,Z,B_{\texttt{TF}}$ are as introduced in \Cref{ssec:self-tf}.} Moreover, the function $\texttt{Operator}(A \cdot \texttt{TF}_\theta(\cdot)_{-1})$ is $\norm{A}_{\texttt{op}} C_{\texttt{TF}}$-Lipschitz continuous with respect to $\theta$, i.e.,
        \begin{align*}
            \norm{\texttt{Operator}(A \cdot \texttt{TF}_{\theta_1}(Z)_{-1}) - \texttt{Operator}(A \cdot \texttt{TF}_{\theta_2}(Z)_{-1})}_{2} \leq \norm{A}_{\texttt{op}} C_{\texttt{TF}}\norm{\theta_1 - \theta_2}_{\texttt{TF}}.
        \end{align*}  
Here,  $\texttt{Operator}$ is either the projection operator onto some convex set, or the \texttt{Softmax} function. 
    \end{proposition}
   
\begin{proof}
    The first result is from \citep[Section J.1]{bai2023transformers}. The second result comes from 
    \begin{itemize}
    \item If $\texttt{Operator}$ is a projection onto the convex set, then $\norm{\texttt{Operator}(x) - \texttt{Operator}(y)}_2 \leq \norm{x-y}_2$; 
    \item If $\texttt{Operator}$ is \texttt{Softmax}, then $\norm{\texttt{Softmax}(x) - \texttt{Softmax}(y)}_2 \leq \norm{x-y}_2$ \citep[Corollary 3]{gao2017properties}.
\end{itemize}
Note that the only condition that we require for $\texttt{Operator}$ is its non-expansiveness. 
\end{proof}
    
\begin{proof}[Proof of \Cref{thm:generalization-error}]

    Let $C_{\text{LLM}}$ denote the Lipschitz-continuity constant for $\text{LLM}_\theta$ with respect to some   norm $\norm{\cdot}_{\text{LLM}}$, where $\norm{\cdot}_{\text{LLM}}$ denotes any norm  defined on the parameter space of $\text{LLM}$ (e.g., the norm $\|\cdot\|_{\texttt{TF}}$ above in \Cref{prop:transformer-lipschitz}). Now, we prove that regret is also a Lipschitz-continuous function with respect to the LLM's parameter. 
    
    \begin{lemma}[Lipschitzness of regret]
    \label{lem:LipReg}
    The function $\text{Regret}_{\text{LLM}_\theta}$ is $C_{\texttt{Reg}}:=  B C_{\text{LLM}} T$-Lipschitz {continuous with respect to $\theta$}, i.e., 
    \begin{align*}
         \Big|\emph{Regret}_{\text{LLM}_{\theta_1}} &((\ell_t)_{t\in[T]}) - \emph{Regret}_{\text{LLM}_{\theta_2}}((\ell_t)_{t\in[T]})\Big| \leq C_{\texttt{Reg}}\norm{\theta_1 - \theta_2}_{\text{LLM}}.
    \end{align*}
    \end{lemma}
    \begin{proof}
    By definition, we have 
          \begin{align*}
        \Big|\text{Regret}_{\text{LLM}_{\theta_1}} &((\ell_t)_{t\in[T]}) - \text{Regret}_{\text{LLM}_{\theta_2}}((\ell_t)_{t\in[T]})\Big| = \left|\sum_{t =1}^T \langle \ell_t, \text{LLM}_{\theta_1}(Z_{t-1}) - \text{LLM}_{\theta_2}(Z_{t-1})\rangle \right|
        \\
        & = B\sum_{t =1}^T  \norm{\text{LLM}_{\theta_1}(Z_{t-1}) - \text{LLM}_{\theta_2}(Z_{t-1})} 
        \\
        &\leq B C_{\text{LLM}} T \norm{\theta_1 - \theta_2}_{\text{LLM}}
    \end{align*} 
where $Z_t := (\ell_1, \dots, \ell_t, c)$ for all $t \in [T]$ and $Z_0 = (c)$ where $c$ is a $d$-dimensional vector. 
\end{proof}

Now, we will prove the Lipschitzness of 
\begin{align}
    C\left((\ell_{t}^{(j)})_{t \in [T], j \in [N]}, k, \theta\right):= \frac{\sum_{j \in [N]} h(\text{Regret}_{\text{LLM}_\theta}((\ell_t^{(j)})_{t\in[T]})) f(\text{Regret}_{\text{LLM}_\theta}((\ell_t^{(j)})_{t\in[T]}), k)}{\sum_{j \in [N]}f(\text{Regret}_{\text{LLM}_\theta}((\ell_t^{(j)})_{t\in[t]}), k)}     
    \label{eqn:loss}
\end{align}
with respect to the model parameter $\theta$. 

    \begin{claim}
    \label{claim:lipshitz-general}
        For any $R>0$, there exists $\beta_R>0$  such that if $\beta > \beta_R$, we have 
        \begin{align*}
            \left| \frac{\sum_{n \in [N]} x_n  f(x_n, \beta)}{\sum_{n \in [N]} f(x_n, \beta) } - \frac{\sum_{n \in [N]} y_n f(y_n, \beta)}{\sum_{n \in [N]} f(y_n, \beta) }\right| \leq 2 \norm{x - y}_\infty
        \end{align*}
        for every $x, y \in \RR^n$ such that $|x_i| \leq R, |y_i| \leq R$ for all $i \in [N]$.
    \end{claim}
    \begin{proof}
        If $\beta = \infty$, we have 
        \begin{align*}
            \lim_{\beta \to \infty}& \left(\left| \frac{\sum_{n \in [N]} x_n  f(x_n, \beta)}{\sum_{n \in [N]} f(x_n, \beta) } - \frac{\sum_{n \in [N]} y_n f(y_n, \beta)}{\sum_{n \in [N]} f(y_n, \beta) }\right|  \Big/ \norm{x - y}_\infty \right) 
            = \frac{| \max_{n \in [N]} x_n - \max_{n \in [N]} y_n |}{\norm{x - y}_\infty} \leq 1 
        \end{align*}
        holds. Moreover, consider the following constrained optimization problem:
        \begin{align*}
             & \max_{{x,y\in\RR^n}} \quad \left(\left| \frac{\sum_{n \in [N]} x_n  f(x_n, \beta)}{\sum_{n \in [N]} f(x_n, \beta) } - \frac{\sum_{n \in [N]} y_n f(y_n, \beta)}{\sum_{n \in [N]} f(y_n, \beta) }\right|  \Big/ \norm{x - y}_\infty \right)
            \\
            &\qquad \text{subject to } \qquad |x_i| \leq R,~~~ |y_i| \leq R ~~~\text{ for all } i \in [N],
        \end{align*}
        whose optimum is denoted as $F(R,\beta)$.  
        Then, since $\norm{x}_\infty \leq R$ and $\norm{y}_\infty \leq R$ is a compact set, by Berge's Maximum Theorem \citep{berge1877topological}, we have that $F(R, \beta)$ is a continuous function for $\beta$. Moreover, we know that $F(R, \infty) \leq 1$, which indicates that we can find a large enough $\beta_R$  such that if $\beta > \beta_R$, $F(R, \beta) \leq 2$. 
    \end{proof}
    
    Note that \Cref{claim:lipshitz-general} does not hold if either $x_i$ or $y_i$ is unbounded. Now, we will apply \Cref{claim:lipshitz-general} to \Cref{eqn:loss}. We can guarantee that $\left|\text{Regret}_{\text{LLM}_\theta}((\ell_t)_{t\in[T]})\right| \leq \text{diam}(\Pi, \norm{\cdot}_2) T B $.
    
    Also, note that the domain of  $h:\RR \to \RR^+$ is effectively  \emph{constrained} to the range that $\text{Regret}_{\text{LLM}_\theta}((\ell_t)_{t\in[T]})$ can achieve, which means that we can regard $h$ as $h:[-\text{diam}(\Pi, \norm{\cdot}_2) T B, \text{diam}(\Pi, \norm{\cdot}_2) T B] \to \RR^+$. Due to the continuity of $h'$, and the fact that  $h$ has a compact domain, we know that $h(\cdot)$ is $C_h$-Lipschitz continuous {for some $C_h>0$} on this interval of $[-\text{diam}(\Pi, \norm{\cdot}_2) T B, \text{diam}(\Pi, \norm{\cdot}_2) T B]$. 
    
    \begin{lemma}[Lipschitzness of $C$ in \Cref{eqn:loss}]
        The function  $C$ in \Cref{eqn:loss} is  $C_{\texttt{cost}}:= 2 C_h  C_{\text{Reg}}$-Lipschitz continuous with respect to $\theta$,  if $k > k_{\text{diam}(\Pi, \norm{\cdot}_2) T B}$ {for some $k_{\text{diam}(\Pi, \norm{\cdot}_2) T B}>0$}, i.e., 
        \begin{align*}
            \Big|C\left((\ell_{t}^{(j)})_{t \in [T], j \in [N]}, k, \theta_1\right) - C\left((\ell_{t}^{(j)})_{t \in [T], j \in [N]}, k, \theta_2\right)\Big| \leq C_{\texttt{cost}} \norm{\theta_1 - \theta_2}_{\text{LLM}}.
        \end{align*}
    \end{lemma}
    \begin{proof}
        \begin{align*}
        \big|C((&\ell_{t}^{(j)})_{t\in [T], j \in [N]}, k, \theta_1) - C((\ell_{t}^{(j)})_{t \in [T], j \in [N]}, k, \theta_2)\big| 
        \\
        &\underset{(i)}{\leq} 2\norm{h(\text{Regret}_{\text{LLM}_{\theta_1}}((\ell_t^{(j)})_{t\in [T]})) - h(\text{Regret}_{\text{LLM}_{\theta_2}}((\ell_t^{(j)})_{t\in [T]} ))}_\infty
        \\
        &\underset{(ii)}{\leq}  2C_h \norm{\text{Regret}_{\text{LLM}_{\theta_1}}((\ell_t^{(j)})_{t\in [T]}) - \text{Regret}_{\text{LLM}_{\theta_2}}((\ell_t^{(j)})_{t\in [T]})}_\infty
        \\
        &\underset{(iii)}{\leq}  2C_h  C_{\text{Reg}} \norm{\theta_1 - \theta_2}_{\text{LLM}} = C_{\texttt{cost}} \norm{\theta_1 - \theta_2}_{\text{LLM}}.
    \end{align*}
    Here, (i) holds due to \Cref{claim:lipshitz-general}, (ii) holds since $h$ is $C_h$-Lipschitz continuous on the range of $\text{Regret}_{\text{LLM}_\theta}((\ell_t)_{t\in[T]})$, and (iii) holds due to \Cref{lem:LipReg}.
    \end{proof}
For completeness of the paper, we provide the definition of covering set and covering number.

\begin{definition}[Covering set and covering number] 
        For $\delta>0$, a metric space $(X, \norm{\cdot})$, and subset $Y \subseteq X$, set $C \subset Y$ is a $\delta$-covering {of $Y$} when $Y \subseteq \cup_{c \in C}B(c, \delta, \norm{\cdot})$ holds. $\delta$-covering number $N(\delta; Y, \norm{\cdot})$ is defined as the minimum cardinality of any covering set. 
    \end{definition}
    By \citep[Example 5.8]{wainwright2019high}, for any $r>0$, we can verify that the $\delta$-covering number $ N(\delta; B(0, r, \norm{\cdot}_{\text{LLM}}), \norm{\cdot}_{\text{LLM}})$ can be bounded by  
    \begin{align*}
        \log N(\delta; B(0, r, \norm{\cdot}_{\text{LLM}}), \norm{\cdot}_{\text{LLM}}) \leq d_{\theta} \log(1 + 2r/\delta),    
    \end{align*}
    where $d_\theta$ is the dimension of the LLM's whole parameter. For example, if we use the $\norm{\cdot}_{\texttt{TF}}$ and consider the Transformer model as defined in \Cref{ssec:self-tf}, for any $r>0$, 
        \begin{align*}
        \log N(\delta; B(0, r, \norm{\cdot}_{\text{LLM}}), \norm{\cdot}_{\text{LLM}}) \leq L (3Md^2 + 2d(dd' + 3md^2))\log(1 + 2r/\delta).
    \end{align*}
Since we consider a compact $\Theta$ (as several recent works on analyzing Transformers  \citep{bai2023transformers, lin2023transformers}), let $R_{\Theta}:= \text{diam}(\Theta, \norm{\cdot}_{\text{LLM}})$ {(which corresponds to  $B_{\texttt{TF}}$ for the Transformer models as defined in \Cref{ssec:self-tf}, with  $\norm{\cdot}_{\text{LLM}}=\norm{\cdot}_{\texttt{TF}}$)}, then there exists a set $\Theta_0$ with $\log |\Theta_0|=d_{\theta} \log(1 + 2R_\Theta/\delta)$  such that  for any $\theta \in \Theta$, there exists a $\theta_0 \in \Theta_0$ with 
    \begin{align*}
        \Big|C\left((\ell_{t}^{(j)})_{t \in [T], j \in [N]}, k, \theta\right) - C\left((\ell_{t}^{(j)})_{t \in [T], j \in [N]}, k, \theta_0\right)\Big| \leq C_{\texttt{cost}}\delta.
    \end{align*} 
    Then, by the standard result from  statistical learning theory \citep[Chapter 5]{wainwright2019high}, when trained with $N_T$ samples, for every $0 < \epsilon < 1/2$, with probability at least $1-\epsilon$, we have 
    \begin{align*}
       \cL(\hat{\theta}_{k, N, N_T}, k, N) &-  \inf_{\theta \in \Theta} \cL(\theta, k, N)  \leq \sqrt{\frac{2(\log |\Theta_0| + \log(2/\epsilon))}{N_T}} + 2C_{\texttt{cost}}\delta.
    \end{align*}
    Setting $\delta = \Omega(\sqrt{\log(\epsilon)/N_T})$, we further obtain 
    \begin{align*}
       \cL(\hat{\theta}_{k, N, N_T}, k, N) &-  \inf_{\theta \in \Theta} \cL(\theta, k, N)  \leq \tilde\cO\left(\sqrt{\frac{\log|\Theta_0| + \log(1/\epsilon)}{N_T}}\right)
    \end{align*}
    with probability at least $1-\epsilon$, completing the proof.
\end{proof} 

\corregret*
\begin{proof}
    The limit on the right-hand side of \Cref{eqn:generalization-gap} remains as $\tilde\cO\left(\sqrt{\frac{d_\theta + \log(1/\epsilon)}{N_T}}\right)$, since we firstly take  $\lim_{k \to \infty}$ and then take $\lim_{N \to \infty}$, {thanks to the fact that \Cref{thm:generalization-error} holds for large enough $k$ and any $N$.}  
    Next, we have 
    \begin{align*}
        &\lim_{N \to \infty} \lim_{k \to \infty} \bigg| \cL(\hat{\theta}_{k, N, N_T}, k, N) -  h\left(\lim_{N \to \infty}\lim_{k \to \infty} \max_{\norm{\ell_t}_\infty \leq B} \text{Regret}_{\text{LLM}_{\hat{\theta}_{k, N, N_T}}} ((\ell_t)_{t\in[T]})\right)\bigg| 
        \\
        &\quad\leq \lim_{N \to \infty} \lim_{k \to \infty} \bigg| \cL(\hat{\theta}_{k, N, N_T}, k, N) -  h\left( \max_{\norm{\ell_t}_\infty \leq B} \text{Regret}_{\text{LLM}_{\hat{\theta}_{k, N, N_T}}} ((\ell_t)_{t\in[T]})\right)\bigg|+\\
        &\quad\qquad \lim_{N \to \infty} \lim_{k \to \infty} \bigg|h\left( \max_{\norm{\ell_t}_\infty \leq B} \text{Regret}_{\text{LLM}_{\hat{\theta}_{k, N, N_T}}} ((\ell_t)_{t\in[T]})\right)-h\left(\lim_{N \to \infty}\lim_{k \to \infty} \max_{\norm{\ell_t}_\infty \leq B} \text{Regret}_{\text{LLM}_{\hat{\theta}_{k, N, N_T}}} ((\ell_t)_{t\in[T]})\right)\bigg|\\
        &\quad\leq \lim_{N \to \infty} \lim_{k \to \infty} \sup_{\theta \in \Theta }\bigg| \cL({\theta}, k, N) -  h\left(\max_{\norm{\ell_t}_\infty \leq B} \text{Regret}_{\text{LLM}_{{\theta}}} ((\ell_t)_{t\in[T]})\right)\bigg|+0 = 0,  
    \end{align*}  
    due to  {the continuity of $h$} 
    and \Cref{claim:double-limit}. 
    Finally, we have  
    \begin{align*}
        \lim_{N \to \infty} \lim_{k \to \infty} \inf_{\theta \in \Theta} \cL(\theta, k, N) = \inf_{\theta \in \Theta} h\left(\max_{\ell_1, \dots, \ell_T} \text{Regret}_{\text{LLM}_\theta}((\ell_t)_{t\in[T]})\right)
    \end{align*}
    due to \Cref{claim:limliminf}, which, combined with the fact that $h$ is non-decreasing, completes the proof. 
\end{proof}

As a result, the  coarse correlated equilibrium will emerge as the long-term interactions of multiple such learned LLMs, as stated in the following corollary. 

\begin{corollary} \emph{(Emerging behavior: Coarse correlated equilibrium)}{\textbf{.}}\label{cor:noregretcce}
For a sufficiently large \(N_T\), if each agent in the matrix game plays according to \(\text{LLM}_{\hat{\theta}_{k, N, N_T}}\), then the time-averaged policy for each agent will {constitute}  an approximate coarse correlated equilibrium of the game. 
\end{corollary} 

\begin{remark}[Dynamic-regret loss]\label{rmk:dynamic-regret}
So far, we have focused on the canonical online learning setting with regret being the metric. One can also generalize the results to the non-stationary setting, with dynamic regret being the metric.  
    Specifically, one can define the \emph{dynamic-regret-loss} function as follows:
  \begin{align} 
    &\cL(\theta, k, N) :=  \EE\left[\frac{\sum_{j \in [N]} h(\text{D-Regret}_{\text{LLM}_\theta}((\ell_t^{(j)})_{t\in[T]})) f(\text{D-Regret}_{\text{LLM}_\theta}((\ell_t^{(j)})_{t\in[T]}), k)}{\sum_{j \in [N]}f(\text{D-Regret}_{\text{LLM}_\theta}((\ell_i^{(j)})_{t\in[T]}), k)}\right].\nonumber
\end{align}    
Then, one can also establish similar results as before, since the analysis does not utilize other properties of the regret except its boundedness, and the  Lipschitz-continuity of LLM with respect to $\theta$. To be specific, \Cref{lem:LipReg} holds due to the reason that we can bound the difference of the regret with the term 
$$
\left|\sum_{t =1}^T \langle \ell_t, (\text{LLM}_{\theta_1}(Z_{t-1}) - \text{LLM}_{\theta_2}(Z_{t-1}))\rangle \right|, 
$$ 
as well as the fact that  $\inf_{\pi_i \in \Pi} \langle \ell_i, \pi_i \rangle $ will be  canceled. One can verify that all the arguments in \Cref{appendix:mathrigor} also hold for similar reasons. 
\end{remark}

\subsection{Detailed Explanation of Optimizing \Cref{eqn:regret-loss} with Single-layer Self-attention Model}
\label{appendix:detailed-softmax}
We consider the {following structure of single-layer self-attention model $g$} (see a formal introduction  in \Cref{ssec:self-tf}): 
\begin{align}
    &g(Z_t; V, K, Q, v_c, k_c, q_c) := (V\ell_{1:t} + v_c \pmb{1}_t^\intercal)  \texttt{Softmax} \left((K\ell_{1:t} + k_c \pmb{1}_t^\intercal)^\intercal \cdot (Qc + q_c)\right), \label{eqn:unconstrained-single-linear-transformer-softmax}
\end{align}
where $Z_t = (\ell_1, \dots, \ell_t, c)$ and $V, K, Q \in \RR^{d \times d}$  correspond to the value, key, and query matrices, respectively, $v_c, k_c, q_c \in \RR^{d}$ correspond to the bias terms associated with  $V, K, Q$, and $c \neq \pmb{0}_d$ is a constant vector. We then have the following result.

\begin{restatable}{theorem}{softmax}
\label{thm:single-softmax-layer-regret-minimizer}
{Consider the policy space $\Pi = B(0, R_{\Pi},\norm{\cdot})$ for some $R_{\Pi}>0$.} 
{The configuration of a single-layer self-attention model {{in \Cref{eqn:unconstrained-single-linear-transformer-softmax}}} $(V,K,Q,v_c, k_c, q_c)$ such} that $K^\intercal (Qc + q_c) = v_c = \pmb{0}_d$ and 
$
V =  -R_{\Pi}\frac{T}{\sum_{t=1}^{T-1}1/t} \Sigma^{-1}\EE\left[\bigg\|{\sum_{t=1}^T \ell_t}\bigg\| \ell_1\ell_2^\intercal  \right]\Sigma^{-1} 
$ 
is a first-order stationary point of \Cref{eqn:regret-loss} with $N=1$, $h(x) = x^2$. 
Moreover, if $\Sigma$ is a diagonal matrix, {then} plugging this configuration into \Cref{eqn:unconstrained-single-linear-transformer-softmax}, {and projecting  the output with $\texttt{Proj}_{\Pi, \norm{\cdot}}$} would perform {FTRL with an $L_2$-regularizer for the loss vectors $(\ell_t)_{t\in[T]}$.}  
\end{restatable}

{In practical training, such stationary points of the loss may be attained by first-order optimization algorithms of (stochastic)  gradient descent, the workhorse in machine learning.} 

\subsection{Deferred Proof of \Cref{thm:single-softmax-layer-regret-minimizer}}
\label{appendix:thm54}

\softmax*
\begin{proof}
{We will locally use $\cA = [d]$ without losing generality as $\cA$ is finite with $|\cA|=d$, and  will interchangeably use $\ell_{i}(j)$ and $\ell_{ij}$ for notational convenience.}  Define $a:= K^\intercal(Qc + q_c) \in \RR^{d}$ and $b_{t-1} := \beta \pmb{1}_{t-1} :=  k_c^\intercal(Qc + q_c)\pmb{1}_{t-1} \in \RR^{t-1}$. With $N=1$, $h(x)=x^2$, and the choice of $\Pi$, the loss function (\Cref{eqn:regret-loss}) can be written as follows:  
    \begin{align*}
        &f(V, a, (b_t)_{t \in [T-1]}, v_c) := \EE\left(\sum_{t = 1}^T \ell_{t}^\intercal  (V\ell_{1:t-1} + v_c\pmb{1}_{t-1}^\intercal) \texttt{Softmax} (\ell_{1:t-1}^\intercal a + b_{t-1}) + R_{\Pi}\norm{\sum_{t=1}^T \ell_{t}}_2\right)^2, 
    \end{align*}    
    {where for $t=1$,  we use the output of the single-layer self-attention as $v_c$ and we will write it as $(V\ell_{1:0} + v_c\pmb{1}_{0}^\intercal) \texttt{Softmax} (\ell_{1:0}^\intercal a + b_{0})$ for notational consistency with $t \geq 2$}. Also, we will define empty sum $\sum_{i=1}^0 a_i = 0$ for any sequence $(a_i)_{i\in\NN^+}$.

\safevspace{6pt}
\textbf{{\color{blue}Step 1}. Calculating $\frac{\partial f}{\partial a}$}. 

For $x \in [d]$, we calculate the corresponding directional derivative with the following equation for $t \geq 2$:
{\small
\begin{align*}
    &\frac{\partial}{\partial a_{x}} \ell_{t}^\intercal  (V\ell_{1:t-1} + v_c\pmb{1}_{t-1}^\intercal) \texttt{Softmax} (\ell_{1:t-1}^\intercal a + b_{t-1}) 
    \\
    &= \frac{\partial}{\partial a_{x}} \sum_{i = 1}^{t-1} \ell_{t}^\intercal  (V\ell_{1:t-1} + v_c\pmb{1}_{t-1}^\intercal)e_i \frac{\exp(e_i^\intercal (\ell_{1:t-1}^\intercal a + b_{t-1}))}{\sum_{s=1}^{t-1}\exp(e_s^\intercal (\ell_{1:t-1}^\intercal a + b_{t-1}))}
    \\
    &=  \frac{\sum_{i = 1}^{t-1} \ell_{t}^\intercal  (V\ell_{1:t-1} + v_c\pmb{1}_{t-1}^\intercal)e_i\exp(e_i^\intercal (\ell_{1:t-1}^\intercal a + b_{t-1}))\frac{\partial e_i^\intercal (\ell_{1:t-1}^\intercal a + b_{t-1})}{\partial a_{x}}(\sum_{s=1}^{t-1}\exp(e_s^\intercal (\ell_{1:t-1}^\intercal a + b_{t-1})))}{(\sum_{s=1}^{t-1}\exp(e_s^\intercal (\ell_{1:t-1}^\intercal a + b_{t-1})))^2}
    \\
    &\qquad - \frac{\sum_{i = 1}^{t-1} \ell_{t}^\intercal  (V\ell_{1:t-1} + v_c\pmb{1}_{t-1}^\intercal)e_i\exp(e_i^\intercal (\ell_{1:t-1}^\intercal a + b_{t-1}))\left(\sum_{s=1}^{t-1}\exp(e_s^\intercal (\ell_{1:t-1}^\intercal a + b_{t-1}))\frac{\partial e_s^\intercal (\ell_{1:t-1}^\intercal a + b_{t-1})}{\partial a_{x}}\right)}{(\sum_{s=1}^{t-1}\exp(e_s^\intercal (\ell_{1:t-1}^\intercal a + b_{t-1})))^2}.
\end{align*}
}

Plugging $a = \pmb{0}_{d}$ and $v_c = \pmb{0}_d$, and $(b_{t} = \beta \pmb{1}_{t})_{t \in [T-1]}$ provides 
\begin{align*}
    \frac{\partial }{\partial a_{x}} &\ell_{t}^\intercal  (V\ell_{1:t-1} + v_c\pmb{1}_{t-1}^\intercal) \texttt{Softmax} (\ell_{1:t-1}^\intercal a + b_{t-1})\bigg|_{a = \pmb{0}_{d}, v_c = \pmb{0}_d, (b_{t} = \beta \pmb{1}_{t})_{t \in [T-1]}} 
    \\
    &=  \sum_{i = 1}^{t-1} \frac{\ell_{t}^\intercal  V\ell_i \ell_{ix}}{(t-1)} - \sum_{i = 1}^{t-1} \frac{\ell_{t}^\intercal  V\ell_i\left(\sum_{s=1}^{t-1} \ell_{sx} \right)}{(t-1)^2}.
\end{align*}
For $t=1$, as $\ell_{t}^\intercal  (V\ell_{1:t-1} + v_c\pmb{1}_{t-1}^\intercal) \texttt{Softmax} (\ell_{1:t-1}^\intercal a + b_{t-1}) =  \ell_1^\intercal v_c$, $\frac{\partial }{\partial a_x}\ell_{t}^\intercal  (V\ell_{1:t-1} + v_c\pmb{1}_{t-1}^\intercal) \texttt{Softmax} (\ell_{1:t-1}^\intercal a + b_{t-1})\bigg|_{a = \pmb{0}_{d}, v_c = \pmb{0}_d, (b_{t} = \beta \pmb{1}_{t})_{t \in [T-1]}} = 0$, so we can use the same formula as $t\geq2$ with empty sum $\sum_{i=1}^{t-1}$. Using the above calculation, we can further compute  $\frac{\partial f}{\partial a_x}\bigg|_{a = \pmb{0}_{d}, v_c = \pmb{0}_d, (b_{t} = \beta \pmb{1}_{t})_{t \in [T-1]}}$ as follows:
{\small
    \begin{align}
        &\frac{\partial  f(V, a, (b_t)_{t \in [T-1]}, v_c)}{\partial a_{x}} \bigg|_{a = \pmb{0}_{d}, v_c = \pmb{0}_d, (b_{t} = \beta \pmb{1}_{t})_{t \in [T-1]}} \nonumber
        \\
        &=\EE \frac{\partial }{\partial a_x}\left(\sum_{t = 1}^T \ell_{t}^\intercal  (V\ell_{1:t-1} + v_c\pmb{1}_{t-1}^\intercal) \texttt{Softmax} (\ell_{1:t-1}^\intercal a + b_{t-1}) + R_{\Pi}\norm{\sum_{t=1}^T \ell_{t}}_2\right)^2\bigg|_{a = \pmb{0}_{d}, v_c = \pmb{0}_d, (b_{t} = \beta \pmb{1}_{t})_{t \in [T-1]}}\nonumber
        \\
        &= \EE \Bigg[ \left(\sum_{t = 1}^T \ell_{t}^\intercal  (V\ell_{1:t-1} + v_c\pmb{1}_{t-1}^\intercal) \texttt{Softmax} (\ell_{1:t-1}^\intercal a + b_{t-1}) + R_{\Pi}\norm{\sum_{t=1}^T \ell_{t}}_2\right)\bigg|_{a = \pmb{0}_{d}, v_c = \pmb{0}_d, (b_{t} = \beta \pmb{1}_{t})_{t \in [T-1]}}\nonumber
        \\
        &\qquad \qquad  \frac{\partial }{\partial a_x} \left(\sum_{t = 1}^T \ell_{t}^\intercal  (V\ell_{1:t-1} + v_c\pmb{1}_{t-1}^\intercal) \texttt{Softmax} (\ell_{1:t-1}^\intercal a + b_{t-1}) + R_{\Pi}\norm{\sum_{t=1}^T \ell_{t}}_2\right)\bigg|_{a = \pmb{0}_{d}, v_c = \pmb{0}_d, (b_{t} = \beta \pmb{1}_{t})_{t \in [T-1]}}\Bigg]\nonumber
        \\
        &= \EE \Bigg[ \left(\sum_{t = 1}^T  \ell_{t}^\intercal V  \sum_{i = 1}^{t-1} \frac{1}{t-1}\ell_i  + R_{\Pi}\norm{\sum_{t=1}^T \ell_{t}}_2\right) \sum_{t = 1}^T \left( \sum_{i = 1}^{t-1} \frac{\ell_{t}^\intercal V \ell_i \ell_{ix}}{(t-1)} -  \sum_{i = 1}^{t-1}\frac{ \ell_{t}^\intercal V  \ell_i\left(\sum_{s=1}^{t-1} \ell_{sx} \right)}{(t-1)^2}\right)\Bigg] \label{eqn:expect}
        \\
        &= 0,\nonumber
    \end{align}
    }
    
\noindent {{where we used the fact that}   $\ell_i$ is drawn from a symmetric distribution, and flipping the sign of the variable as $-\ell_i$ yields the same distribution, which leads to the following:}
\begin{align*}
    &\EE \Bigg[ \left(\sum_{t = 1}^T  \ell_{t}^\intercal V  \sum_{i = 1}^{t-1} \frac{1}{t-1}\ell_i  + R_{\Pi}\norm{\sum_{t=1}^T \ell_{t}}_2\right) \sum_{t = 1}^T \left( \sum_{i = 1}^{t-1} \frac{\ell_{t}^\intercal V \ell_i \ell_{ix}}{(t-1)} -  \sum_{i = 1}^{t-1}\frac{ \ell_{t}^\intercal V  \ell_i\left(\sum_{s=1}^{t-1} \ell_{sx} \right)}{(t-1)^2}\right)\Bigg]
    \\
    &\quad = \EE \Bigg[ \left(\sum_{t = 1}^T  \ell_{t}^\intercal V  \sum_{i = 1}^{t-1} \frac{1}{t-1}\ell_i  + R_{\Pi}\norm{\sum_{t=1}^T \ell_{t}}_2\right) \sum_{t = 1}^T \left( -\sum_{i = 1}^{t-1} \frac{\ell_{t}^\intercal V \ell_i \ell_{ix}}{(t-1)} +  \sum_{i = 1}^{t-1}\frac{ \ell_{t}^\intercal V  \ell_i\left(\sum_{s=1}^{t-1} \ell_{sx} \right)}{(t-1)^2}\right)\Bigg]. 
\end{align*}
This yields  \Cref{eqn:expect}=0.   

\safevspace{6pt}
\textbf{{\color{blue}Step 2}. Calculating $\frac{\partial f}{\partial v_c}$}.

We will use the following equation for $t \geq 2$:
\begin{align*}
    \frac{\partial }{\partial v_{c}} &\ell_{t}^\intercal  (V\ell_{1:t-1} + v_c\pmb{1}_{t-1}^\intercal) \texttt{Softmax} (\ell_{1:t-1}^\intercal a + b_{t-1}) 
    \\
    &= \frac{\partial }{\partial v_{c}} \sum_{i = 1}^{t-1} \ell_{t}^\intercal  (V\ell_{1:t-1} + v_c\pmb{1}_{t-1}^\intercal)e_i \frac{\exp(e_i^\intercal (\ell_{1:t-1}^\intercal a + b_{t-1}))}{\sum_{s=1}^{t-1}\exp(e_s^\intercal (\ell_{1:t-1}^\intercal a + b_{t-1}))} = \ell_{t}.
\end{align*}
For $t= 1$, we define $\frac{\partial }{\partial v_{c}} \ell_{1}^\intercal  (V\ell_{1:0} + v_c\pmb{1}_{0}^\intercal) \texttt{Softmax} (\ell_{1:0}^\intercal a + b_{0}) = \ell_1$, so that  we can use the same formula as $t \geq 2$.
Therefore, we can calculate $\frac{\partial f}{\partial v_c}\bigg|_{a = \pmb{0}_{d}, v_c = \pmb{0}_d, (b_{t} = \beta \pmb{1}_{t})_{t \in [T-1]}}$ as follows:  
{\small
    \begin{align*}
        &\frac{\partial  f(V, a, (b_t)_{t \in [T-1]}, v_c)}{\partial v_{c}} \bigg|_{a = \pmb{0}_{d}, v_c = \pmb{0}_d , (b_{t} = \beta \pmb{1}_{t})_{t \in [T-1]}}
        \\
        &=\EE \frac{\partial }{\partial v_c}\left(\sum_{t = 1}^T \ell_{t}^\intercal  (V\ell_{1:t-1} + v_c\pmb{1}_{t-1}^\intercal) \texttt{Softmax} (\ell_{1:t-1}^\intercal a + b_{t-1}) + R_{\Pi}\norm{\sum_{t=1}^T \ell_{t}}_2\right)^2\bigg|_{a = \pmb{0}_{d}, v_c = \pmb{0}_d, (b_{t} = \beta \pmb{1}_{t})_{t \in [T-1]}}
        \\
        &= \EE \Bigg[ \left(\sum_{t = 1}^T \ell_{t}^\intercal  (V\ell_{1:t-1} + v_c\pmb{1}_{t-1}^\intercal) \texttt{Softmax} (\ell_{1:t-1}^\intercal a + b_{t-1}) + R_{\Pi}\norm{\sum_{t=1}^T \ell_{t}}_2\right)\bigg|_{a = \pmb{0}_{d}, v_c = \pmb{0}_d, (b_{t} = \beta \pmb{1}_{t})_{t \in [T-1]}}
        \\
        &\qquad \qquad  \frac{\partial }{\partial v_c} \left(\sum_{t = 1}^T \ell_{t}^\intercal  (V\ell_{1:t-1} + v_c\pmb{1}_{t-1}^\intercal) \texttt{Softmax} (\ell_{1:t-1}^\intercal a + b_{t-1}) + R_{\Pi}\norm{\sum_{t=1}^T \ell_{t}}_2\right)\bigg|_{a = \pmb{0}_{d}, v_c = \pmb{0}_d, (b_{t} = \beta \pmb{1}_{t})_{t \in [T-1]}}\Bigg]
        \\
        &= \EE \Bigg[ \left(\sum_{t = 2}^T  \ell_{t}^\intercal V  \sum_{i = 1}^{t-1} \frac{1}{t-1}\ell_i  + R_{\Pi}\norm{\sum_{t=1}^T \ell_{t}}_2\right) \sum_{t = 1}^T \ell_{t} \Bigg]= 0.
    \end{align*}}
The last line is due to the same reason as the last part of Step 1.

\safevspace{6pt}
\textbf{{\color{blue}Step 3}. Calculating $\frac{\partial f}{\partial V}$}.

We calculate the following equation, which will be used to calculate $\frac{\partial f}{\partial V}\bigg|_{a = \pmb{0}_{d}, v_c = \pmb{0}_d, (b_{t} = \beta \pmb{1}_{t})_{t \in [T-1]}}$ for $t \geq 2$:
\begin{align*}
    \frac{\partial }{\partial V} &\ell_{t}^\intercal  (V\ell_{1:t-1} + v_c\pmb{1}_{t-1}^\intercal) \texttt{Softmax} (\ell_{1:t-1}^\intercal a + b_{t-1}) \bigg|_{a = \pmb{0}_{d}, v_c = \pmb{0}_d, (b_{t} = \beta \pmb{1}_{t})_{t \in [T-1]}} 
    \\
    &= \frac{\partial }{\partial V} \sum_{i = 1}^{t-1} \ell_{t}^\intercal  (V\ell_{1:t-1} + v_c\pmb{1}_{t-1}^\intercal)e_i \frac{\exp(e_i^\intercal (\ell_{1:t-1}^\intercal a + b_{t-1}))}{\sum_{s=1}^{t-1}\exp(e_s^\intercal (\ell_{1:t-1}^\intercal a + b_{t-1}))} \bigg|_{a = \pmb{0}_{d}, v_c = \pmb{0}_d, (b_{t} = \beta \pmb{1}_{t})_{t \in [T-1]}} 
    \\
    &= \sum_{i = 1}^{t-1} \ell_{t}\ell_i^\intercal  \frac{\exp(e_i^\intercal (\ell_{1:t-1}^\intercal a + b_{t-1}))}{\sum_{s=1}^{t-1}\exp(e_s^\intercal (\ell_{1:t-1}^\intercal a + b_{t-1}))} \bigg|_{a = \pmb{0}_{d}, v_c = \pmb{0}_d, (b_{t} = \beta \pmb{1}_{t})_{t \in [T-1]}} =   \sum_{i = 1}^{t-1} \frac{1}{t-1}\ell_{t}\ell_i^\intercal.
\end{align*}
For $t = 1$, {note that} $\frac{\partial }{\partial V} \ell_t^\intercal v_c = \pmb{O}_{d \times d}$, so we can use the same formula as $t\geq 2$ with empty sum $\sum_{i=1}^{t-1}$. 

Therefore, we have 
   \begin{align*}
        &\frac{\partial  f(V, a, (b_t)_{t \in [T-1]}, v_c)}{\partial V} \bigg|_{a = \pmb{0}_{d}, v_c = \pmb{0}_d, (b_{t} = \beta \pmb{1}_{t})_{t \in [T-1]}}
        \\
        &=\EE \frac{\partial }{\partial V}\left(\sum_{t = 1}^T \ell_{t}^\intercal  (V\ell_{1:t-1} + v_c\pmb{1}_{t-1}^\intercal) \texttt{Softmax} (\ell_{1:t-1}^\intercal a + b_{t-1}) + R_{\Pi}\norm{\sum_{t=1}^T \ell_{t}}_2\right)^2\bigg|_{a = \pmb{0}_{d}, v_c = \pmb{0}_d, (b_{t} = \beta \pmb{1}_{t})_{t \in [T-1]}}
        \\
        &= \EE \Bigg[ \left(\sum_{t = 1}^T \ell_{t}^\intercal  (V\ell_{1:t-1} + v_c\pmb{1}_{t-1}^\intercal) \texttt{Softmax} (\ell_{1:t-1}^\intercal a + b_{t-1}) + R_{\Pi}\norm{\sum_{t=1}^T \ell_{t}}_2\right)\bigg|_{a = \pmb{0}_{d}, v_c = \pmb{0}_d,(b_{t} = \beta \pmb{1}_{t})_{t \in [T-1]}}
        \\
        &\qquad \qquad  \frac{\partial }{\partial V} \left(\sum_{t = 1}^T \ell_{t}^\intercal  (V\ell_{1:t-1} + v_c\pmb{1}_{t-1}^\intercal) \texttt{Softmax} (\ell_{1:t-1}^\intercal a + b_{t-1}) + R_{\Pi}\norm{\sum_{t=1}^T \ell_{t}}_2\right)\bigg|_{a = \pmb{0}_{d}, v_c = \pmb{0}_d,(b_{t} = \beta \pmb{1}_{t})_{t \in [T-1]}}\Bigg]
        \\
        &= \EE \Bigg[ \left(\sum_{t = 1}^T \ell_{t}^\intercal V   \sum_{i = 1}^{t-1} \frac{1}{t-1} \ell_i  + R_{\Pi}\norm{\sum_{t=1}^T \ell_{t}}_2\right) \sum_{t = 1}^T \sum_{i=1}^{t-1} \frac{1}{t-1} \ell_{t}\ell_i^\intercal \Bigg]
        \\
        &= \EE \Bigg[ \left(\sum_{t = 1}^T \sum_{i = 1}^{t-1} \left(\frac{1}{t-1} \ell_{t}^\intercal V    \ell_i \right) \left(\frac{1}{t-1} \ell_{t}\ell_i^\intercal \right)+ R_{\Pi}T\norm{\sum_{t'=1}^T \ell_{t'}}_2 \ell_{t}\ell_i^\intercal\right)  \Bigg]
        \\
        &= \EE \Bigg[ \left(\sum_{t = 1}^T \sum_{i = 1}^{t-1} \sum_{x = 1}^d \sum_{y = 1}^d v_{xy} \ell_{tx} \ell_{iy} \left(\frac{1}{t-1}\right)^2 [\ell_{tz}\ell_{iw}]_{(z, w)} + R_{\Pi}T\norm{\sum_{t'=1}^T \ell_{t'}}_2 \ell_{t}\ell_i^\intercal\right)  \Bigg] 
        \\
        &= \sum_{t = 1}^T \sum_{i = 1}^{t-1} \sum_{x = 1}^d \sum_{y = 1}^d \frac{1}{(t-1)^2} [\sigma_{xz}  v_{xy}\sigma_{yw}]_{(z, w)} + \EE \Bigg[R_{\Pi}T\norm{\sum_{t'=1}^T \ell_{t'}}_2 \ell_{t}\ell_i^\intercal  \Bigg] 
        \\
        &= \left(\sum_{t = 1}^{T-1} \frac{1}{t}\right) \Sigma V \Sigma +   \EE \Bigg[R_{\Pi}T\norm{\sum_{t'=1}^T \ell_{t'}}_2 \ell_{t}\ell_i^\intercal  \Bigg].
    \end{align*}
    Therefore, if $V^\star = R_{\Pi} \frac{T}{\sum_{t=1}^{T-1}1/t} \Sigma^{-1}\EE \Bigg[\norm{\sum_{t=1}^T \ell_{t}}_2 \ell_{t}\ell_i^\intercal  \Bigg] \Sigma^{-1}$, then $\frac{\partial f}{\partial V}\bigg|_{a = \pmb{0}_{d}, v_c = \pmb{0}_d, (b_{t} = \beta \pmb{1}_{t})_{t \in [T-1]}, V=V^\star} = \bO_{d \times d}$.
    Lastly, we have

    \begin{align*}
        &\frac{\partial f}{\partial K}\big|_{K^\intercal (Qc + q_c) = v_c = \pmb{0}_d, V = V^\star} = \left(\frac{\partial f}{\partial a} \frac{\partial a}{\partial K}\right)\bigg|_{a = \pmb{0}_{d}, v_c = \pmb{0}_d, (b_{t} = \beta \pmb{1}_{t})_{t \in [T-1]}, V=V^\star} = \pmb{O}_{d  \times d}
        \\
        &\frac{\partial f}{\partial Q}\big|_{K^\intercal (Qc + q_c) = v_c = \pmb{0}_d, V = V^\star} = \left(\frac{\partial f}{\partial a} \frac{\partial a}{\partial Q}\right)\bigg|_{a = \pmb{0}_{d}, v_c = \pmb{0}_d, (b_{t} = \beta \pmb{1}_{t})_{t \in [T-1]}, V=V^\star} = \pmb{O}_{d  \times d}
        \\
        &\frac{\partial f}{\partial q_c}\big|_{K^\intercal (Qc + q_c) = v_c = \pmb{0}_d, V = V^\star} = \left(\frac{\partial f}{\partial a}\frac{\partial a}{\partial q_c}\right)\bigg|_{a = \pmb{0}_{d}, v_c = \pmb{0}_d, (b_{t} = \beta \pmb{1}_{t})_{t \in [T-1]}, V=V^\star}  = \pmb{0}_{d}
    \end{align*}
    which means that such configurations are first-order stationary points of \Cref{eqn:regret-loss} with $N=1$, $h(x) = x^2$, and $\Pi = B(0, R_{\Pi},\norm{\cdot})$. 
\end{proof}

\subsection{Deferred Proof of \Cref{thm:single-layer-regret-minimizer}}
\label{appendix:thm55}

\linear*
{\color{black}
This theorem involves the analysis of a \textit{non-convex optimization}  problem through stationary point analysis. We identified the set of stationary points. By constructing the optimization problem as shown in \Cref{eqn:calculation-under-local-opt}, we significantly reduced the candidate set for optimal points using our novel argument on the expected value of a nonnegative definite matrix.  The main challenge here was to address the global optimization problem in a non-convex setting, which required the exploitation of the particular Transformer architecture.}
\begin{proof}
The output of the single-layer linear  self-attention structure is as follows:
\begin{equation}
    \begin{aligned}
        g(&Z_t;V, K, Q, v_c, k_c, q_c) 
        \\
        &= \sum_{i=1}^t \left(V\ell_i \ell_i^\intercal (K^\intercal (Qc + q_c)) + \left(V k_c^ \intercal (Qc + q_c)+ v_c (Qc + q_c)^\intercal K\right)\ell_i + v_c k_c^\intercal(Qc + q_c)\right),
    \end{aligned} \label{eqn:original-param}
\end{equation}
which can be expressed with a larger class 
\begin{align}
    g(Z_t, \mathbb{A}, \beta, \mathbb{C}, \delta):= \sum_{i=1}^t(\mathbb{A} \ell_i \ell_i^\intercal \beta + \mathbb{C} \ell_i + \delta), \label{eqn:reparam-ABCD}
\end{align}
where $\mathbb{A} \in \RR^{d \times d}$, $\beta, \mathbb{C}, \delta \in \RR^d$. Then, if a minimizer of 
\begin{align}
    f(\mathbb{A}, \beta, \mathbb{C}, \delta): &= \EE \left( \sum_{t = 1}^T \langle \ell_t, \sum_{i=1}^{t-1} \left(\mathbb{A} \ell_i\ell_i^\intercal \beta + \mathbb{C} \ell_i+ \delta \right)   \rangle- \inf_{\pi \in \Pi } \left\langle \sum_{t=1}^T  \ell_t, \pi \right\rangle \right)^2 \nonumber
\end{align}
can be expressed as  $\mathbb{A} = V, \beta = K^\intercal (Qc + q_c), \mathbb{C} = V k_c^\intercal (Qc + q_c) + v_c(Qc+ q_c)^\intercal K, \beta = v_c k_c^\intercal (Qc + q_c)$, then we can conclude that the corresponding $V, Q, K, v_c, q_c, k_c$ are also a minimizer of  
\begin{align*}
    \EE \left( \sum_{t = 1}^T \langle \ell_t, g(Z_{t-1})  \rangle - \inf_{\pi \in \Pi } \left\langle \sum_{t=1}^T  \ell_i, \pi \right\rangle\right)^2, 
\end{align*}
since the corresponding $V, Q, K, v_c, q_c, k_c$ constitute a minimizer among a larger class. Now, since $\Pi = B(\pmb{0}_d, R_{\Pi}, \norm{\cdot})$, we can rewrite $f$ as 
\begin{align}
    f(\mathbb{A}, \beta, \mathbb{C}, \delta) = \EE\left( \sum_{t = 1}^T \langle \ell_t,\sum_{i=1}^{t-1} \left(\mathbb{A} \ell_i\ell_i^\intercal \beta + \mathbb{C} \ell_i+ \delta \right)   \rangle + R_{\Pi} \bigg\|\sum_{t=1}^T  \ell_i\bigg\|_2 \right)^2.  \label{eqn:alternative-single-layer-transformer}
\end{align}

\safevspace{6pt}
\textbf{{\color{blue}Step 1}. Finding condition for $\frac{\partial f}{\partial \delta} = 0$}.

Due to the Leibniz rule, if we calculate the partial derivative of \Cref{eqn:alternative-single-layer-transformer} w.r.t. $\delta$, we have
\begin{align}
    &\frac{\partial  f(\mathbb{A}, \beta, \mathbb{C}, \delta)}{\partial  \delta}= \frac{\partial }{\partial  \delta} \EE \left( \sum_{t = 1}^T \langle \ell_t,\sum_{i=1}^{t-1} \left(\mathbb{A} \ell_i\ell_i^\intercal \beta + \mathbb{C} \ell_i+ \delta \right)   \rangle + R_{\Pi} \norm{\sum_{t=1}^T  \ell_t}_2 \right)^2 \nonumber
    \\
    &\qquad = \EE \frac{\partial }{\partial  \delta} \left( \sum_{t = 1}^T \langle \ell_t,\sum_{i=1}^{t-1} \left(\mathbb{A} \ell_i\ell_i^\intercal \beta + \mathbb{C} \ell_i+ \delta \right)   \rangle + R_{\Pi} \norm{\sum_{t=1}^T  \ell_t}_2 \right)^2 \nonumber
    \\
    &\qquad = \EE \sum_{t = 1}^T  \ell_t   \left( \sum_{t = 1}^T \sum_{i=1}^{t-1} (t-1) \ell_t^\intercal  \left(\mathbb{A} \ell_i\ell_i^\intercal \beta + \mathbb{C} \ell_i+ \delta  \right) + R_{\Pi} \norm{\sum_{t=1}^T  \ell_t}\right). \label{eqn:derivative-delta}
\end{align}
Since the expectation of either odd-order polynomial or even-order polynomial times $\norm{\cdot}_2$  is 0,  due to that $\ell_{t}$ follows a symmetric distribution, we have 
\begin{align*}
    \EE \sum_{t = 1}^T (t-1)\ell_t  R_{\Pi} \Big\|\sum_{t=1}^T  \ell_t\Big\|_2 = 0, \qquad  \EE \sum_{t = 1}^T (t-1)\ell_t  \sum_{t = 1}^T \sum_{i=1}^{t-1} \ell_t^\intercal \mathbb{C} \ell_i= 0.  
\end{align*}
Now, we calculate 
\begin{align*}
   \EE \sum_{t = 1}^T &(t-1)\ell_t  \sum_{t = 1}^T \sum_{i=1}^{t-1} \ell_t^\intercal \mathbb{A} \ell_i\ell_i^\intercal \beta  = \EE \sum_{t_1 = 1}^T\sum_{t = 1}^T \sum_{i=1}^{t-1}  (t_1-1)\ell_{t_1} \ell_t^\intercal \mathbb{A} \ell_i\ell_i^\intercal \beta
   \\
   &\underset{(i)}{=}\EE \sum_{t = 1}^T \sum_{i=1}^{t-1}  (t-1)\ell_{t} \ell_t^\intercal \mathbb{A} \ell_i\ell_i^\intercal \beta = \EE \sum_{t = 1}^T (t-1)^2 \ell_{t} \ell_t^\intercal \mathbb{A} \Sigma \beta = \frac{1}{6}T(2T^2 - 3T +1) \Sigma \mathbb{A} \Sigma \beta,
\end{align*}
where $(i)$ holds since if $t_1 \neq t$, due to the independence of $\ell_t$, $\ell_{t_1}$, we can use $\EE \ell_t = 0$. Lastly, 
\begin{align*}
    \EE \sum_{t = 1}^T &(t-1)\ell_t  \sum_{t = 1}^T \sum_{i=1}^{t-1} \ell_t^\intercal \delta  = \EE\sum_{t_1 = 1}^T \sum_{t = 1}^T (t_1-1)(t-1)\ell_{t_1}\ell_t^\intercal \delta = \frac{1}{6}T(2T^2 - 3T +1) \Sigma \delta.
\end{align*}
Plugging the above equations into \Cref{eqn:derivative-delta}, we have 
\begin{align*}
    &\frac{\partial  f(\mathbb{A}, \beta, \mathbb{C}, \delta)}{\partial  \delta} = \frac{1}{6}T(2T^2 - 3T +1)(\Sigma \mathbb{A} \Sigma \beta + \Sigma \delta). 
\end{align*}
Due to the optimality condition, we have 
\begin{align}
    \mathbb{A} \Sigma \beta + \delta = 0. \label{eqn:local-optimality} 
\end{align}

\safevspace{6pt}
\textbf{{\color{blue}Step 2}. Plugging the optimality condition for  $\frac{\partial f}{\partial\delta}$ into \Cref{eqn:alternative-single-layer-transformer}}.

Plugging \Cref{eqn:local-optimality} to \Cref{eqn:alternative-single-layer-transformer}, $f$ can be written as  
\begin{align*}
    f(\mathbb{A}, &\beta, \mathbb{C}, -\mathbb{A} \Sigma \beta) = \EE  \left( \sum_{t = 1}^T \sum_{i=1}^{t-1} \ell_t^\intercal  \left(\mathbb{A} (\ell_i\ell_i^\intercal- \Sigma) \beta + \mathbb{C} \ell_i\right)    + R_{\Pi} \bigg\|\sum_{t=1}^T  \ell_t\bigg\|_2 \right)^2
    \\
    &= \underbrace{\EE\left(\sum_{t = 1}^T \sum_{i=1}^{t-1} \ell_t^\intercal  \mathbb{A} (\ell_i\ell_i^\intercal- \Sigma) \beta\right)^2}_{(i)} + \EE \left( \sum_{t = 1}^T \sum_{i=1}^{t-1} \ell_t^\intercal \mathbb{C} \ell_i \right)^2 + \EE \left(R_{\Pi} \bigg\|\sum_{t=1}^T  \ell_t\bigg\|_2 \right)^2
    \\
    &\quad + \underbrace{2 \EE \left(\sum_{t = 1}^T \sum_{i=1}^{t-1} \ell_t^\intercal  \mathbb{A} (\ell_i\ell_i^\intercal- \Sigma) \beta\right)\left( \sum_{t = 1}^T \sum_{i=1}^{t-1} \ell_t^\intercal \mathbb{C} \ell_i \right)}_{(ii)} 
    \\
    &\qquad  +\underbrace{2\EE \left(\sum_{t = 1}^T \sum_{i=1}^{t-1} \ell_t^\intercal  \mathbb{A} (\ell_i\ell_i^\intercal- \Sigma) \beta\right)\left(R_{\Pi} \bigg\|\sum_{t=1}^T  \ell_t\bigg\|_2 \right)}_{(iii)}
    \\
    &\qquad \quad +2\EE\left( \sum_{t = 1}^T \sum_{i=1}^{t-1} \ell_t^\intercal \mathbb{C} \ell_i \right) \left(R_{\Pi} \bigg\|\sum_{t=1}^T  \ell_t\bigg\|_2 \right). 
\end{align*}
For the part $(i)$, we have 
\begin{align}
    \EE &\left(\sum_{t = 1}^T  \sum_{i=1}^{t-1} \ell_t^\intercal  \mathbb{A} (\ell_i\ell_i^\intercal- \Sigma) \beta\right)^2 = \EE \left[ \sum_{t_1 = 1}^T \sum_{i_1=1}^{t_1-1} \sum_{t = 1}^T \sum_{i=1}^{t-1}  \beta^\intercal (\ell_{i_1} \ell_{i_1}^\intercal- \Sigma)\mathbb{A}^\intercal \ell_{t_1}  \ell_t^\intercal  \mathbb{A} (\ell_i\ell_i^\intercal- \Sigma) \beta\right] \nonumber
    \\
    &\qquad \underset{(1)}{=}  \EE \left[ \sum_{t = 1}^T  \sum_{i_1=1}^{t-1} \sum_{i=1}^{t-1}  \beta^\intercal (\ell_{i_1} \ell_{i_1}^\intercal- \Sigma)\mathbb{A}^\intercal \ell_{t}  \ell_t^\intercal  \mathbb{A} (\ell_i\ell_i^\intercal- \Sigma) \beta\right] \nonumber
    \\
    &\qquad \underset{(2)}{=}  \EE \left[ \sum_{t = 1}^T   \sum_{i=1}^{t-1}  \beta^\intercal (\ell_{i} \ell_{i}^\intercal- \Sigma)\mathbb{A}^\intercal \ell_{i}  \ell_i^\intercal  \mathbb{A} (\ell_i\ell_i^\intercal- \Sigma) \beta\right] \nonumber
    \\
     &\qquad = \frac{(T-1)T}{2}\beta^\intercal \EE \left[    (\ell_{i} \ell_{i}^\intercal- \Sigma)\mathbb{A}^\intercal\Sigma \mathbb{A} (\ell_i\ell_i^\intercal- \Sigma) \right] \beta  \label{eqn:calculation-under-local-opt}
     \\
     &\qquad =  \frac{(T-1)T}{2}\beta^\intercal \EE \left[ (\sqrt{\Sigma}A (\ell_i\ell_i^\intercal- \Sigma))^\intercal (\sqrt{\Sigma}A (\ell_i\ell_i^\intercal- \Sigma)) \right] \beta. \nonumber 
\end{align}
Here, $(1)$ holds because if $t_1 \neq t$, we know that $\EE \ell_{t_1} = \EE \ell_t = 0$, and they are independent, and $(2)$ holds because if $i_1 \neq i$, we can calculate $\EE (\ell_{i_1} \ell_{i_1}^\intercal- \Sigma) = \bO_{d \times d}$. In addition, we can easily check that $(ii)$ and $(iii)$ are $0$ as they are polynomials of odd degrees and we have $Z \overset{d}{=}-Z$. Note that  \Cref{eqn:calculation-under-local-opt} is minimized when $\PP(\sqrt{\Sigma} \mathbb{A} (\ell_i\ell_i^\intercal- \Sigma)\beta = \pmb{0}_d) = 1$. 

If $\mathbb{A} \neq \bO_{d \times d}$, suppose that the singular value decomposition of $A = U \Lambda V$ yields that $\Lambda$ is a diagonal matrix whose first diagonal element is non-zero, and $U, V$ are orthogonal matrices. Then, we want to find $\beta$ that $\sqrt{\Sigma} U \Lambda V (\ell_i \ell_i^\intercal-\Sigma) \beta = \pmb{0}_d$ for any $\ell_i$ such that $p(\ell_i) \neq 0$, where $p$ indicates the probability density function of loss vectors. Since $\Sigma$ and $U$ are invertible, we only need to consider $\Lambda V (\ell_i \ell_i^\intercal - \Sigma) \beta = \pmb{0}_d$. Since $\Lambda$'s first diagonal component is non-zero, we will consider equation $e_1^\intercal \Lambda V (\ell_i \ell_i^\intercal- \Sigma) \beta = 0$. This is equivalent to $V_1 (\ell_i \ell_i^\intercal - \Sigma) \beta = 0$,  where $V_1$ is the first row of $V$, and is a non-zero vector. 

Now, we will generally consider $a_{x,y}(v):= vv^\intercal x- y$ where $x, y, v \in \RR^d$ and $a_{x,y}: B(\pmb{0}_d, 2\epsilon_1, \norm{\cdot})\to \RR^d$ function. Then, we can check that the Jacobian of $a_{x,y}(v)$ is $v x^\intercal + (v \cdot x)I$, and we can find that the determinant of the Jacobian is nonzero when $v = \epsilon_1 x $ if $x \neq \pmb{0}_d$. Therefore, the volume of $( V_1 (\ell_i \ell_i^\intercal- \Sigma))$ for $\ell_i \in B(\pmb{0}_d, c_z, \norm{\cdot})$ is greater than the volume of $(V_1 (vv^\intercal- \Sigma))$ for $v \in B(\epsilon_1 V_1^\intercal, \epsilon_2, \norm{\cdot})$, where {$c_z$ is a constant such that $B(\pmb{0}_d, c_z, \norm{\cdot}) \subseteq \text{supp}(Z)$, and $\epsilon_1,\epsilon_2>0$}  satisfy that $\epsilon_1|V_1| + \epsilon_2 < c_z$. {Here, we define $\epsilon_2>0$ sufficiently small so that the determinant of $\text{Jacobian}(vv^\intercal V_1^\intercal - \Sigma V_1^\intercal) > 0$ for $v \in B(\epsilon_1 V_1^\intercal, \epsilon_2, \norm{\cdot})$, and $v \to vv^\intercal V_1^\intercal - \Sigma V_1^\intercal$ is a one-to-one correspondence,  by inverse function theorem.} Therefore, the volume of $( V_1 (vv^\intercal- \Sigma))$ for $v \in B(\epsilon_1 V_1^\intercal, \epsilon_2, \norm{\cdot})$ can be calculated as 
\begin{align*}
   & [\text{Volume $( V_1 (vv^\intercal- \Sigma))$ for $v \in B(\epsilon_1 V_1^\intercal, \epsilon_2, \norm{\cdot})$}] \\
   &\quad= \int_{v \in B(\epsilon_1 V_1^\intercal, \epsilon_2, \norm{\cdot})} \big|\text{det(Jacobian}( V_1 (vv^\intercal- \Sigma)))\big| dv > 0.
\end{align*}

Therefore, $ \text{Volume}( V_1 (v v^\intercal- \Sigma)) $  where $v \in B(\epsilon_1 V_1^\intercal, \epsilon_2, \norm{\cdot})$ is non-zero, {so that we can find $d$ loss vectors $\{\ell_i\}_{i\in[d]}$ such that the vectors $\{V_1 (\ell_i \ell_i^\intercal- \Sigma)\}_{i\in[d]}$ are linearly independent}.
Hence, if we want to minimize  \Cref{eqn:calculation-under-local-opt}, either $A=\bO_{d \times d}$ or $\beta =\pmb{0}_d$ should hold. In both cases, \Cref{eqn:reparam-ABCD} can be re-written as 
\begin{align*}
    g(Z_t; \mathbb{A}, \beta, \mathbb{C}, \delta):= \sum_{i=1}^t \mathbb{C} \ell_i,
\end{align*}
and this is covered by the original parametrization (\Cref{eqn:original-param}) with $K^\intercal (Qc + q_c) = v_c = \pmb{0}_d$. 

\safevspace{6pt}
\textbf{{\color{blue}Step 3}. Calculating $\frac{\partial f}{\partial \mathbb{C}}$}.

Now, we optimize over $\mathbb{C}$, by minimizing the following objective: 
\begin{align*}
    f(\mathbb{C}):&=\EE  \left( \sum_{t = 1}^T\sum_{i=1}^{t-1}   \ell_t^\intercal  \mathbb{C} \ell_i+ R_{\Pi} \norm{\sum_{t=1}^T  \ell_t}\right)^2
    \\
    &= \underbrace{\EE\left( \sum_{t = 1}^T\sum_{i=1}^{t-1}   \ell_t^\intercal  \mathbb{C} \ell_i\right)^2}_{(i)} + {2 \EE\left( \left( \sum_{t = 1}^T\sum_{i=1}^{t-1}   \ell_t^\intercal  \mathbb{C} \ell_i\right) R_{\Pi} \norm{\sum_{t=1}^T  \ell_t}\right)} + \EE  \left(R_{\Pi} \norm{\sum_{t=1}^T  \ell_t}\right)^2 
    \\
    &= \frac{T(T-1)}{2} \tr\left(    \mathbb{C}^\intercal \Sigma  \mathbb{C} \Sigma \right)+ {2 \EE\left(   B \sum_{t = 1}^T\sum_{i=1}^{t-1}   \ell_t^\intercal  \mathbb{C} \ell_i\norm{\sum_{j=1}^T \ell_j}\right)} + \EE  \left(R_{\Pi} \norm{\sum_{t=1}^T  \ell_t}\right)^2.
\end{align*}
Here, $(i)$ can be calculated as follows:
\begin{align*}
\EE&\left( \sum_{t = 1}^T\sum_{i=1}^{t-1}   \ell_t^\intercal  \mathbb{C} \ell_i\right)^2 = \EE\left(  \sum_{t_1 = 1}^T\sum_{i_1=1}^{t_1-1} \sum_{t = 1}^T\sum_{i=1}^{t-1}  \ell_{i_1}^\intercal  \mathbb{C}^\intercal \ell_{t_1}     \ell_i^\intercal  \mathbb{C} \ell_i\right)
\\
&\underset{(1)}{=} \EE\left(  \sum_{t = 1}^T\sum_{i_1=1}^{i-1} \sum_{i=1}^{t-1}  \ell_{i_1}^\intercal  \mathbb{C}^\intercal \ell_{i}     \ell_i^\intercal  \mathbb{C} \ell_i\right) = \EE\left(  \sum_{t = 1}^T\sum_{i_1=1}^{i-1} \sum_{i=1}^{t-1}  \ell_{i_1}^\intercal  \mathbb{C}^\intercal \Sigma  \mathbb{C} \ell_i\right)
\\
&\underset{(2)}{=} \EE\left(  \sum_{t = 1}^T \sum_{i=1}^{t-1}  \ell_{k}^\intercal  \mathbb{C}^\intercal \Sigma  \mathbb{C} \ell_i\right) \underset{(3)}{=} \EE \tr\left(\sum_{t = 1}^T \sum_{i=1}^{t-1}    \mathbb{C}^\intercal \Sigma  \mathbb{C} \ell_i\ell_{k}^\intercal\right) = \frac{T(T-1)}{2} \tr\left(    \mathbb{C}^\intercal \Sigma  \mathbb{C} \Sigma \right),
\end{align*}
since (1) holds because if $t_1 \neq t$, we already know that $\EE \ell_t = \EE \ell_{t_1} = 0$, (2) holds due to a similar reason, and (3) comes from $\tr(AB) = \tr(BA)$. 

We calculate $\frac{\partial  f(\mathbb{C})}{\partial \mathbb{C}}$: 
\begin{align*}
    \frac{\partial  f(\mathbb{C})}{\partial \mathbb{C}} = T(T-1)\Sigma \mathbb{C}\Sigma + 2R_{\Pi}\EE \left( \norm{\sum_{j=1}^T \ell_j}\sum_{t = 1}^T\sum_{i=1}^{t-1} \ell_t\ell_i^\intercal \right).
\end{align*}
Hence, the optimal $\mathbb{C} = -\frac{2R_{\Pi}}{T(T-1)}\Sigma^{-1} \EE \left( \norm{\sum_{j=1}^T \ell_j}\sum_{t = 1}^T\sum_{i=1}^{t-1} \ell_t\ell_i^\intercal \right) \Sigma^{-1}.$

Now, we see that for the special case of $\Sigma = I$, we have $\mathbb{C} = -R_{\Pi} \EE \left( \norm{\sum_{j=1}^T \ell_j} \ell_t\ell_i^\intercal \right) $. If we calculate the $(a,b)$-coordinate of $\mathbb{C}$, we need to calculate 
\begin{align*}
    \EE_\ell  \left[\sqrt{\sum_{o= 1}^d (\sum_{s = 1}^T \ell_{so})^2} \ell_{ia}\ell_{kb}\right]. 
\end{align*}
If $a \neq b$, then since $Z$ is symmetric, the term above becomes zero.  Therefore, we only need to consider the case when $a = b$, which is $\EE_\ell  \left[\sqrt{\sum_{o= 1}^d (\sum_{s = 1}^T \ell_{so})^2} \ell_{ia}\ell_{ka}\right]$, and it will be the same value for all $a \in [d]$ since $\ell_i$'s coordinates are independent. 

Now, we calculate the scale of  $\EE_\ell \left[\sqrt{\sum_{o= 1}^d (\sum_{s = 1}^T \ell_{so})^2} \ell_{i1}\ell_{k1}\right]$. 
We have $Z := \frac{\sum_{o= 1}^{d-1} (\sum_{s = 1}^T \ell_{so})^2}{T(d-1)} \overset{a.s.}{\to} 1$ as $d \to \infty$ (by the law of large numbers) and we define $W: = \sum_{s\neq i, k} \ell_{s1} / \sqrt{T}$ which is independent of  $\ell_{i1}$ and $\ell_{k1}$. 
{\small
\begin{align*}
    \EE_\ell &\left[\sqrt{\sum_{o= 1}^d (\sum_{s = 1}^T \ell_{so})^2} \ell_{i1}\ell_{k1}\right] = \EE_{Z, W, \ell_{i1}, \ell_{k1}}\left[\sqrt{T(d-1)Z + (\sqrt{T}W + \ell_{i1} + \ell_{k1})^2} \ell_{i1}\ell_{k1}\right]
    \\
    &= \EE_{Z, W, \ell_{i1}, \ell_{k1} \geq 0}\left[\sqrt{T(d-1)Z + (\sqrt{T}W + \ell_{i1} + \ell_{k1})^2} \ell_{i1}\ell_{k1} - \sqrt{T(d-1)Z + (\sqrt{T}W + \ell_{i1} - \ell_{k1})^2} \ell_{i1}\ell_{k1} \right]
    \\
    &= \EE_{Z, W, \ell_{i1}, \ell_{k1} \geq 0}\left[\frac{4(\sqrt{T}W + \ell_{i1})\ell_{k1}}{\sqrt{T(d-1)Z + (\sqrt{T}W + \ell_{i1} + \ell_{k1})^2}  + \sqrt{T(d-1)Z + (\sqrt{T}W + \ell_{i1} - \ell_{k1})^2}} \ell_{i1}\ell_{k1}\right].
\end{align*}
}
Taking $d \to \infty$, we have   
\begin{align*}
    \frac{\sqrt{T(d-1)Z + (\sqrt{T}W + \ell_{i1} + \ell_{k1})^2}  + \sqrt{T(d-1)Z + (\sqrt{T}W + \ell_{i1} - \ell_{k1})^2}}{2\sqrt{Td}} \overset{d}{\to} 1,
\end{align*}
which further implies   
\begin{align*}
    & \sqrt{Td} \frac{4(\sqrt{T}W + \ell_{i1})\ell_{k1}}{\sqrt{T(d-1)Z + (\sqrt{T}W + \ell_{i1} + \ell_{k1})^2}  + \sqrt{T(d-1)Z + (\sqrt{T}W + \ell_{i1} - \ell_{k1})^2}} \ell_{i1}\ell_{k1} 
    \\
    &\qquad \overset{d}{\to}  \sqrt{Td} \frac{4(\sqrt{T}W + \ell_{i1})\ell_{k1}}{2\sqrt{Td}} \ell_{i1}\ell_{k1} =  2(\sqrt{T}W + \ell_{i1})\ell_{i1}\ell_{k1}
\end{align*}
as $d \to \infty$. Therefore, 
\begin{align*}
    &\lim_{d \to \infty}  \EE_{Z, W, \ell_{i1}, \ell_{k1} \geq 0}\left[ {\sqrt{Td}}\frac{4(\sqrt{T}W + \ell_{i1})\ell_{k1}}{\sqrt{T(d-1)Z + (\sqrt{T}W + \ell_{i1} + \ell_{k1})^2}  + \sqrt{T(d-1)Z + (\sqrt{T}W + \ell_{i1} - \ell_{k1})^2}} \ell_{i1}\ell_{k1}\right]
    \\
    & = \EE_{Z, W, \ell_{i1}, \ell_{k1} \geq 0}\left[2(\sqrt{T}W + \ell_{i1})\ell_{i1}\ell_{k1}\right] = \EE_{\ell_{i1}, \ell_{k1} \geq 0}\left[ \ell_{i1}^2\ell_{k1}\right] 
\end{align*}
which is a constant. The last equality came from the fact that $W$, $\ell_{i1}$, $\ell_{k1}$ are independent random variables, and expectation of $\ell_{i1}$ is zero. Therefore, the output of the single-layer linear  self-attention provides us with online gradient descent with step-size $\Theta(R_{\Pi}/\sqrt{Td})$. In the online learning literature, we usually set the gradient step size as  $\Theta(R_{\Pi}/\sqrt{Td})$ \citep[Theorem 3.1]{hazan2016introduction}, which is consistent with the result above. 
\end{proof}

\subsection{Empirical Validation of \Cref{thm:single-softmax-layer-regret-minimizer} and \Cref{thm:single-layer-regret-minimizer}}
\label{appendix:empirical-validation}
We now provide empirical validations for \Cref{thm:single-softmax-layer-regret-minimizer} and \Cref{thm:single-layer-regret-minimizer}. We provide the training details and the results as follows. 

\subsubsection{Empirical Validation of \Cref{thm:single-softmax-layer-regret-minimizer}}

Our model architecture is defined as follows: the number of layers \(T\) is set to 30 and the dimensionality \(d\) to 32, with the loss vector $\ell_i$'s distribution \(Z\) following a standard normal distribution \(\mathcal{N}(0, 1)\). During training, we conducted 40,000 epochs with a batch size of 512. We employed the Adam optimizer, setting the learning rate to 0.001.  We initialized the value, query, and key vectors (\(v_c, q_c, k_c\)) as zero vectors.

Our empirical analysis aims to demonstrate that the optimized model inherently emulates online gradient descent. To illustrate this, we will focus on two key convergence properties: \(K^\intercal Q\) approaching the zero matrix \(\pmb{O}_{d \times d}\) and \(V\) converging to \(a\pmb{1}_d\pmb{1}^\intercal_d  + bI_{d \times d}\), where \(a\) and \(b\) are constants in \(\mathbb{R}\). The conditions \(K^\intercal Q = \pmb{O}_{d \times d}\) and \(V = a\pmb{1}_d\pmb{1}^\intercal_d  + bI_{d \times d}\) imply that the function \(g(Z_t; V, Q, K) = \sum_{i=1}^{t} (b-a) \ell_i\), effectively emulating the process of an online gradient descent method. We repeated the experiments 10 times. For verifying  \(K^\intercal Q = \pmb{O}_{d \times d}\), we will measure Frobenius norm ($\norm{\cdot}_{F}$) of $K^\intercal Q$. Also for measuring the closeness of $V$ and $a\pmb{1}_d\pmb{1}^\intercal_d  + bI_{d \times d}$, we will measure $\min_{a, b \in \RR} \norm{V - (a\pmb{1}_d\pmb{1}^\intercal_d  + bI_{d \times d})}_{F}/b$. The results are demonstrated in the first plot of \Cref{fig:emp-validation}. 

\subsubsection{Empirical Validation of \Cref{thm:single-layer-regret-minimizer}}

We now focus on two key convergence properties: \(K^\intercal (Q \pmb{1}_d + q_c) \) approaching the zero vector \(\pmb{0}_{d}\) and \(V\) converging to \(a\pmb{1}_d\pmb{1}^\intercal_d  + bI_{d \times d}\), where \(a\) and \(b\) are constants in \(\mathbb{R}\). The conditions \(K^\intercal (Q \pmb{1}_d + q_c) = \pmb{0}_{d}\) and \(V = a\pmb{1}_d\pmb{1}^\intercal_d  + bI_{d \times d}\) imply that the function \(g(Z_t; V, Q, K) = \sum_{i=1}^{t} (b-a) \ell_i\), effectively emulating the process of an online gradient descent method. We repeated 10 times. For verifying \(K^\intercal (Q \pmb{1}_d + q_c) = \pmb{0}_{d}\), we will measure the 2-norm of $K^\intercal (Q \pmb{1}_d + q_c)$. Also for measuring the closeness of $V$ and $a\pmb{1}_d\pmb{1}^\intercal_d  + bI_{d \times d}$, we will measure $\min_{a, b \in \RR} \norm{V - (a\pmb{1}_d\pmb{1}^\intercal_d  + bI_{d \times d})}_{F}/b$. The results are demonstrated in the second plot of \Cref{fig:emp-validation}. 

\begin{figure}[ht]
    \centering
    \begin{adjustwidth}{0cm}{-2cm}

\begin{subfigure}[b]{0.33\textwidth}
        \includegraphics[height =  4.2cm]{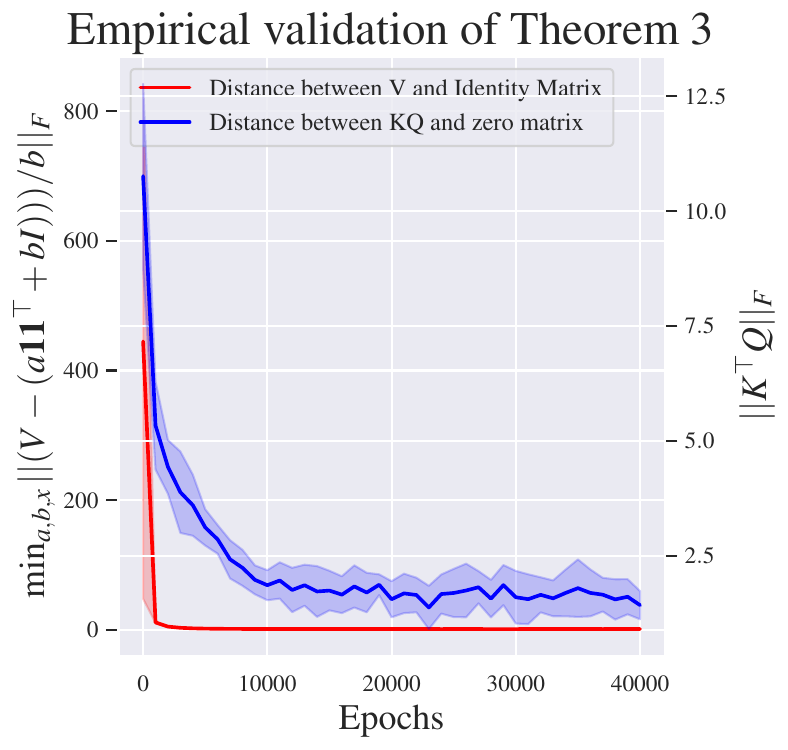}
    \end{subfigure}
     \begin{subfigure}[b]{0.33\textwidth}
        \includegraphics[height=4.2cm]{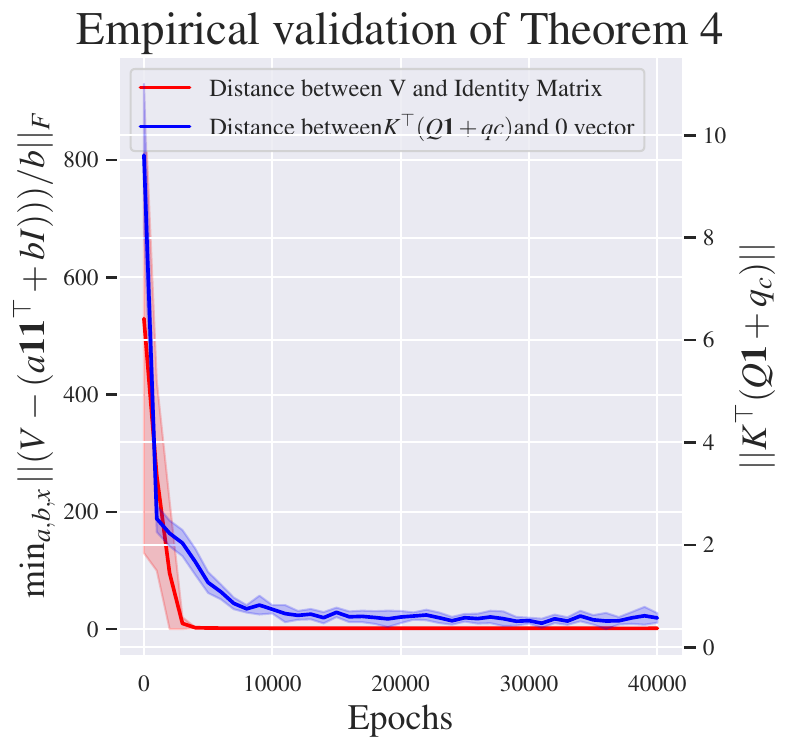} 
    \end{subfigure}
        \begin{subfigure}[b]{0.33\textwidth} 
        \includegraphics[height=4.2cm]{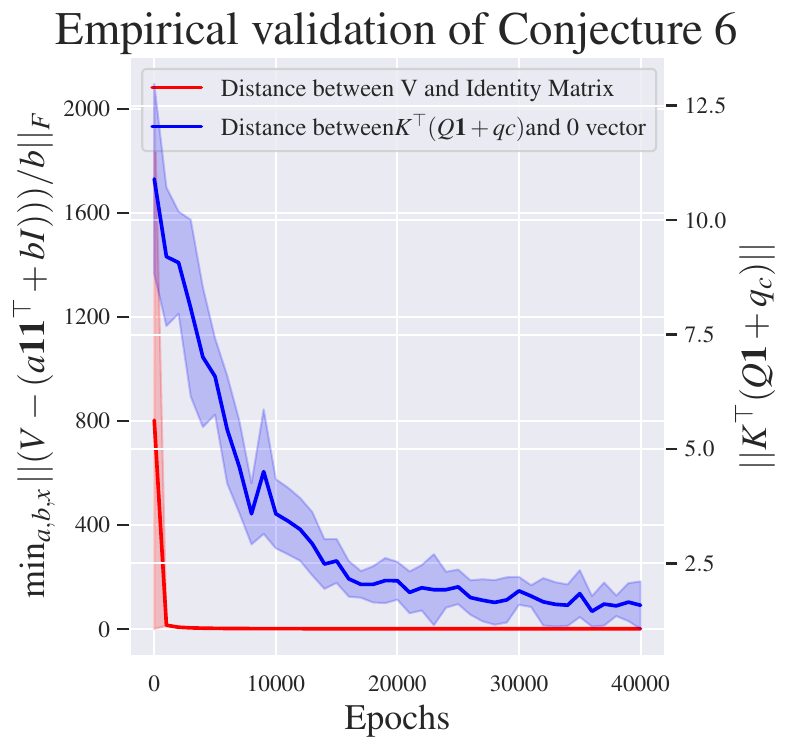}
    \end{subfigure}
\end{adjustwidth}
    \caption{Empirical validation of \Cref{thm:single-softmax-layer-regret-minimizer} (top), \Cref{thm:single-layer-regret-minimizer} (middle), and \Cref{conj:local} (bottom). The observed convergence in \Cref{thm:single-softmax-layer-regret-minimizer} and \Cref{conj:local}'s result suggests that configuration in \Cref{thm:single-softmax-layer-regret-minimizer} and \Cref{conj:local} are not only the local optimal point, but it has the potential as being  the global optimizer.}
    \label{fig:emp-validation}
    \safevspace{12pt}
\end{figure} 

\subsection{Discussions   on the Production of FTRL with Entropy Regularization}\label{appendix:conj6}
{Now, we will consider projecting a single-layer linear self-attention model into a constrained domain such as a simplex, which is more amenable to  the Experts Problem setting.}  
{To this end,} we consider the following {parameterization} by adding an additional \emph{non-linear}   structure for the single-layer linear self-attention: 
\begin{align}
    g(Z_t; V, K, Q, v_c, k_c, q_c) 
    = \texttt{Operator} \left(\sum_{i=1}^{t} (V\ell_i + v_c) ((K\ell_i + k_c))^\intercal \cdot (Qc + q_c))\right), \label{eqn:constrained-single-linear-transformer}
\end{align} 
where the $\texttt{Operator}$ denotes projection to the convex set.

\begin{restatable}{conjecture}{conj}
\label{conj:local}
Assume $\Sigma = I$. Then, the configuration that $K^\intercal (Qc + q_c) = v_c = \pmb{0}_d$ and $V =  \tilde{\Omega}\left(-\frac{1}{\sqrt{nd}}\right)I_{d\times d}$ is a first-order stationary point of \Cref{eqn:regret-loss} with $N=1$ and $h(x) = x^2$ when $\text{LLM}_\theta$ is parameterized with \Cref{eqn:constrained-single-linear-transformer}, $\texttt{Operator} = \texttt{Softmax}$, and $\Pi = \Delta(\cA)$. This configuration performs FTRL with an entropy regularizer which is a no-regret algorithm. 
\end{restatable}

{We provide an {idea} for {proving} the conjecture, together with its numerical validation. Also, we have observed in \Cref{fig:emp-validation} that \Cref{thm:single-softmax-layer-regret-minimizer} and \Cref{conj:local} might also be a global optimizer, as training results have provided  the configuration that \Cref{thm:single-softmax-layer-regret-minimizer} and \Cref{conj:local} have suggested.}

To be specific, we will consider 
{\small
\begin{align*}
    f(V, a, \beta, v_c ) &= \EE \left(\sum_{t=1}^T \sum_{s=1}^d \ell_{ts} \frac{\exp \left( e_s^\intercal \sum_{j=1}^{t-1} (V\ell_j\ell_j^\intercal a + (\beta V + v_ca^\intercal)\ell_j + v_c \beta) \right)}{\sum_{y = 1}^d \exp \left( e_y^\intercal \sum_{j=1}^{t-1} (V\ell_j\ell_j^\intercal a + (\beta V + v_ca^\intercal)\ell_j + v_c \beta) \right)} - \min_{s} \sum_{t=1}^T \ell_{ts}\right)^2
\end{align*}
}
and will try to prove that $a = \pmb{0}_d, v_c = v\pmb{1}_d, V = kI$ is a first-order stationary point.

\safevspace{6pt}
\textbf{{\color{blue}Step 1}. Calculating $\frac{\partial f}{\partial v_c}$}.

We use the following formula: for $x \in [d]$ and $t \geq 2$, we have 
{\small
\begin{align*}
    &\frac{\partial }{\partial  v_{cx}} \exp \left( e_y^\intercal \sum_{i=1}^{t} (V\ell_i\ell_i^\intercal a + (\beta V + v_ca^\intercal)\ell_i + v_c \beta) \right) \bigg|_{a = \pmb{0}_d, v_c = v\pmb{1}_d, V = kI}
    \\
    &= \exp \left( e_y^\intercal \sum_{i=1}^{t} (V\ell_i\ell_i^\intercal a + (\beta V + v_ca^\intercal)\ell_i + v_c \beta) \right)  \frac{\partial }{\partial v_{cx}}\left( e_y^\intercal \sum_{i=1}^{t} (V\ell_i\ell_i^\intercal a + (\beta V + v_ca^\intercal)\ell_i + v_c \beta) \right) \bigg|_{a = \pmb{0}_d, v_c = v\pmb{1}_d, V = kI}
    \\
    &= \exp \left( e_y^\intercal \sum_{i=1}^{t} (V\ell_i\ell_i^\intercal a + (\beta V + v_ca^\intercal)\ell_i + v_c \beta) \right)  \sum_{i=1}^{t} \left( a^\intercal \ell_i \ell_i^\intercal e_{x} + \beta \right)  \bigg|_{a = \pmb{0}_d, v_c = v\pmb{1}_d, V = kI}
    \\
    &= t\beta\exp(v\beta)\exp(\beta k \sum_{i = 1}^t \ell_{iy}), 
\end{align*}
}
and for $t = 1$, $\frac{\partial }{\partial  v_{cx}} \exp \left( e_y^\intercal \sum_{i=1}^{t} (V\ell_i\ell_i^\intercal a + (\beta V + v_ca^\intercal)\ell_i + v_c \beta) \right) \bigg|_{a = \pmb{0}_d, v_c = v\pmb{1}_d, V = kI} = 0$, so we can use the same formula with $t\geq 2$. Thus, we have
{\small
\begin{align*}
    &\frac{\partial }{\partial  v_{cx}} \left(\sum_{t=1}^T \sum_{s=1}^d \ell_{ts} \frac{\exp \left( e_s^\intercal \sum_{j=1}^{t-1} (V\ell_j\ell_j^\intercal a + (\beta V + v_ca^\intercal)\ell_j + v_c \beta) \right)}{\sum_{y = 1}^d \exp \left( e_y^\intercal \sum_{j=1}^{t-1} (V\ell_j\ell_j^\intercal a + (\beta V + v_ca^\intercal)\ell_j + v_c \beta) \right)} - \min_{s} \sum_{t=1}^T \ell_{ts}\right) \bigg|_{a = \pmb{0}_d, v_c = v\pmb{1}_d, V = kI}
    \\
    &= \beta \exp(v\beta)
    \\
    &\qquad \sum_{t=1}^T t\sum_{s=1}^d \ell_{ts} \frac{\sum_{y=1}^d \exp\left(\sum_{j=1}^{t-1} \beta k \ell_{jy} \right) \exp\left(\sum_{j=1}^{t-1} \beta k \ell_{js} \right)   -  \sum_{y=1}^d  \exp \left(  \sum_{j=1}^{t-1} \beta k \ell_{js} \right) \exp\left(\sum_{j=1}^{t-1} \beta k \ell_{jy} \right)}{\left(\sum_{y = 1}^d \exp \left( e_y^\intercal \sum_{j=1}^{t-1} \beta V\ell_j \right)\right)^2} 
    \\
    &= 0.
\end{align*}
}
Therefore, 
\begin{align*}
    &\frac{\partial  f(V, a, \beta, v_c )}{\partial v_{cx}} \bigg|_{a = \pmb{0}_d, v_c = v\pmb{1}_d, V = kI}
    \\
    &= \EE \Bigg[ \left(\sum_{t=1}^T \sum_{s=1}^d \ell_{ts} \frac{\exp \left( e_s^\intercal \sum_{j=1}^{t-1} (V\ell_j\ell_j^\intercal a + (\beta V + v_ca^\intercal)\ell_j + v_c \beta) \right)}{\sum_{y = 1}^d \exp \left( e_y^\intercal \sum_{j=1}^{t-1} (V\ell_j\ell_j^\intercal a + (\beta V + v_ca^\intercal)\ell_j + v_c \beta) \right)} - \min_{s} \sum_{t=1}^T \ell_{ts}\right) 
    \\
    &\qquad \frac{\partial }{\partial  v_{cx}} \left(\sum_{t=1}^T \sum_{s=1}^d \ell_{ts} \frac{\exp \left( e_s^\intercal \sum_{j=1}^{t-1} (V\ell_j\ell_j^\intercal a + (\beta V + v_ca^\intercal)\ell_j + v_c \beta) \right)}{\sum_{y = 1}^d \exp \left( e_y^\intercal \sum_{j=1}^{t-1} (V\ell_j\ell_j^\intercal a + (\beta V + v_ca^\intercal)\ell_j + v_c \beta) \right)} - \min_{s} \sum_{t=1}^T \ell_{ts}\right) \Bigg]\bigg|_{a = \pmb{0}_d, v_c = v\pmb{1}_d, V = kI}
    \\
    &= 0. 
\end{align*}

\safevspace{6pt}
\textbf{{\color{blue}Step 2}. Calculating $\frac{\partial f}{\partial V}$}.

The following formula will be used for calculating $\frac{\partial f}{\partial V}\bigg|_{a = \pmb{0}_d, v_c = v\pmb{1}_d, V = kI}$:  for $r, c \in [d]$, we have 
\begin{align*}
    &\frac{\partial }{\partial V_{rc}} \exp \left( e_y^\intercal \sum_{i=1}^{t} (V\ell_i\ell_i^\intercal a + (\beta V + v_ca^\intercal)\ell_i + v_c \beta) \right) \bigg|_{a = \pmb{0}_d, v_c = v\pmb{1}_d, V = kI}
    \\
    &= \exp \left( e_y^\intercal \sum_{i=1}^{t} (V\ell_i\ell_i^\intercal a + (\beta V + v_ca^\intercal)\ell_i + v_c \beta) \right)  \frac{\partial }{\partial V_{rc}}\left( e_y^\intercal \sum_{i=1}^{t} (V\ell_i\ell_i^\intercal a + (\beta V + v_ca^\intercal)\ell_i + v_c \beta) \right) \bigg|_{a = \pmb{0}_d, v_c = v\pmb{1}_d, V = kI}
    \\
    &= \exp \left( \sum_{i=1}^{t}  k \beta  \ell_{iy} + v \beta \right)  \sum_{i=1}^{t}  \beta \pmb{1}(y=r) \ell_{ic}.
\end{align*}

Therefore, 
{\small
\allowdisplaybreaks

\begin{align*}
    &\frac{\partial  f(V, a, \beta, v_c )}{\partial V_{rc}} \bigg|_{a = \pmb{0}_d, v_c = v\pmb{1}_d, V = kI}
    \\
    &= \EE \Bigg[ \left(\sum_{t=1}^T \sum_{s=1}^d \ell_{ts} \frac{\exp \left( e_s^\intercal \sum_{j=1}^{t-1} (V\ell_j\ell_j^\intercal a + (\beta V + v_ca^\intercal)\ell_j + v_c \beta) \right)}{\sum_{y = 1}^d \exp \left( e_y^\intercal \sum_{j=1}^{t-1} (V\ell_j\ell_j^\intercal a + (\beta V + v_ca^\intercal)\ell_j + v_c \beta) \right)} - \min_{s} \sum_{t=1}^T \ell_{ts}\right) 
    \\
    &\qquad \frac{\partial }{\partial V_{rc}} \left(\sum_{t=1}^T \sum_{s=1}^d \ell_{ts} \frac{\exp \left( e_s^\intercal \sum_{j=1}^{t-1} (V\ell_j\ell_j^\intercal a + (\beta V + v_ca^\intercal)\ell_j + v_c \beta) \right)}{\sum_{y = 1}^d \exp \left( e_y^\intercal \sum_{j=1}^{t-1} (V\ell_j\ell_j^\intercal a + (\beta V + v_ca^\intercal)\ell_j + v_c \beta) \right)} - \min_{s} \sum_{t=1}^T \ell_{ts}\right) \Bigg]\bigg|_{a = \pmb{0}_d, v_c = v\pmb{1}_d, V = kI}
    \\
    &=  \EE \Bigg[ \left(\sum_{t=1}^T \sum_{s=1}^d \ell_{ts} \frac{\exp \left( \sum_{j=1}^{t-1} \beta k \ell_{js} + v \beta \right)}{\sum_{y = 1}^d \exp \left( \sum_{j=1}^{t-1}  \beta V \ell_{jy} + v \beta \right)} - \min_{s} \sum_{t=1}^T \ell_{ts}\right) 
    \\
    &\qquad \Biggl(\sum_{t=1}^T \sum_{s=1}^d \ell_{ts} \frac{ \sum_{j=1}^{t-1} \beta \pmb{1}(s =r) \ell_{jc} \exp \left( \sum_{j=1}^{t-1}  \beta k \ell_{js} + v \beta \right) \sum_{y = 1}^d \exp \left( \sum_{j=1}^{t-1}  \beta k \ell_{jy} + v \beta \right)}{\left(\sum_{y = 1}^d \exp \left( \sum_{j=1}^{t-1}  \beta k \ell_{jy} + v \beta \right)\right)^2}
    \\
    &\qquad \qquad  - \sum_{t=1}^T \sum_{s=1}^d \ell_{ts} \frac{ \exp \left( \sum_{j=1}^{t-1}  \beta k \ell_{js} + v \beta \right) \sum_{y = 1}^d \left(\sum_{j=1}^{t-1}  \beta \pmb{1}(y = r) \ell_{jc} \exp \left( \sum_{j=1}^{t-1}  \beta k \ell_{jy} + v \beta \right)\right)}{\left(\sum_{y = 1}^d \exp \left( \sum_{j=1}^{t-1}  \beta k \ell_{jy} + v \beta \right)\right)^2}\Biggr) \Bigg]
\\
    &=  \beta \EE \Bigg[ \left(\sum_{t=1}^T \sum_{s=1}^d \ell_{ts} \frac{\exp \left( \sum_{j=1}^{t-1} \beta k \ell_{js} \right)}{\sum_{y = 1}^d \exp \left( \sum_{j=1}^{t-1}  \beta V \ell_{jy}  \right)} - \min_{s} \sum_{t=1}^T \ell_{ts}\right) 
    \\
    &\qquad\qquad \qquad   \Biggl( \underbrace{\frac{\sum_{t=1}^T  \sum_{j=1}^{t-1}\sum_{y = 1}^d \ell_{tr} \ell_{jc} \exp \left(\beta k\sum_{j=1}^{t-1}\ell_{jr} \right)  \exp \left(\beta k\sum_{j=1}^{t-1}\ell_{jy} \right)}{\left(\sum_{y = 1}^d \exp \left(\beta k\sum_{j=1}^{t-1}\ell_{jy} \right)\right)^2}}_{(i)}
    \\
    &\qquad \qquad \qquad \qquad -  \underbrace{\frac{\sum_{t=1}^T \sum_{j=1}^{t-1} \sum_{y=1}^d     \ell_{ty}\ell_{jc}  \exp \left(\beta k\sum_{j=1}^{t-1}\ell_{jr} \right)\exp \left(\beta k\sum_{j=1}^{t-1}\ell_{jy} \right) }{\left(\sum_{y = 1}^d \exp \left(\beta k\sum_{j=1}^{t-1}\ell_{jy} \right)\right)^2}}_{(ii)}\Biggr)  \Bigg].
\end{align*}
}
We can observe the followings: 1) if $r_1 \neq c_1$ and $r_2 \neq c_2$, $\frac{\partial f}{\partial V_{r_1c_1}}\bigg|_{a = \pmb{0}_d, v_c = v\pmb{1}_d, V = kI} = \frac{\partial f}{\partial V_{r_2c_2}}\bigg|_{a = \pmb{0}_d, v_c = v\pmb{1}_d, V = kI}$ holds, and 2) $\frac{\partial f}{\partial V_{r_1r_1}} \bigg|_{a = \pmb{0}_d, v_c = v\pmb{1}_d, V = kI}= \frac{\partial f}{\partial V_{r_2r_2}}\bigg|_{a = \pmb{0}_d, v_c = v\pmb{1}_d, V = kI}$.

\safevspace{6pt}
\textbf{{\color{blue}Step 3}. Calculating $\frac{\partial f}{\partial \beta}$}. 

The following formula will be used for calculating $\frac{\partial f}{\partial \beta}\bigg|_{a = \pmb{0}_d, v_c = v\pmb{1}_d, V = kI}$: 
\begin{align*}
    \frac{\partial }{\partial \beta} &\exp \left( e_y^\intercal \sum_{i=1}^{t} (V\ell_i\ell_i^\intercal a + (\beta V + v_ca^\intercal)\ell_i + v_c \beta) \right) \bigg|_{a = \pmb{0}_d, v_c = v\pmb{1}_d, V = kI}
    \\
    &= \exp \left( e_y^\intercal \sum_{i=1}^{t} (V\ell_i\ell_i^\intercal a + (\beta V + v_ca^\intercal)\ell_i + v_c \beta) \right)  \frac{\partial }{\partial \beta}\left( e_y^\intercal \sum_{i=1}^{t} (V\ell_i\ell_i^\intercal a + (\beta V + v_ca^\intercal)\ell_i + v_c \beta) \right) \bigg|_{a = \pmb{0}_d, v_c = v\pmb{1}_d, V = kI}
    \\
    &= tv\beta \exp \left( \sum_{i=1}^{t}  k \beta  \ell_{iy} + v \beta \right) .
\end{align*}
Further, we have 
{\small
\begin{align*}
    &\frac{\partial }{\partial  \beta} \left(\sum_{t=1}^T \sum_{s=1}^d \ell_{ts} \frac{\exp \left( e_s^\intercal \sum_{j=1}^{t-1} (V\ell_j\ell_j^\intercal a + (\beta V + v_ca^\intercal)\ell_j + v_c \beta) \right)}{\sum_{y = 1}^d \exp \left( e_y^\intercal \sum_{j=1}^{t-1} (V\ell_j\ell_j^\intercal a + (\beta V + v_ca^\intercal)\ell_j + v_c \beta) \right)} - \min_{s} \sum_{t=1}^T \ell_{ts}\right) \bigg|_{a = \pmb{0}_d, v_c = v\pmb{1}_d, V = kI}
    \\
    &= v\beta \exp(v\beta)
    \\
    &\qquad \sum_{t=1}^T t\sum_{s=1}^d \ell_{ts} \frac{\sum_{y=1}^d \exp\left(\sum_{j=1}^{t-1} \beta k \ell_{jy} \right) \exp\left(\sum_{j=1}^{t-1} \beta k \ell_{js} \right)   -  \sum_{y=1}^d  \exp \left(  \sum_{j=1}^{t-1} \beta k \ell_{js} \right) \exp\left(\sum_{j=1}^{t-1} \beta k \ell_{jy} \right)}{\left(\sum_{y = 1}^d \exp \left( e_y^\intercal \sum_{j=1}^{t-1} \beta V\ell_j \right)\right)^2} 
    \\
    &= 0.
\end{align*}
}

\safevspace{6pt}
\textbf{{\color{blue}Step 4}. Calculating $\frac{\partial f}{\partial a}$}. 

Note that 
\begin{align*}
    &\frac{\partial }{\partial a_x} \exp \left( e_y^\intercal \sum_{i=1}^{t} (V\ell_i\ell_i^\intercal a + (\beta V + v_ca^\intercal)\ell_i + v_c \beta) \right) \bigg|_{a = \pmb{0}_d, v_c = v\pmb{1}_d, V = kI}
    \\
    &= \exp \left( e_y^\intercal \sum_{i=1}^{t} (V\ell_i\ell_i^\intercal a + (\beta V + v_ca^\intercal)\ell_i + v_c \beta) \right)  \frac{\partial }{\partial a_x}\left( e_y^\intercal \sum_{i=1}^{t} (V\ell_i\ell_i^\intercal a + (\beta V + v_ca^\intercal)\ell_i + v_c \beta) \right) \bigg|_{a = \pmb{0}_d, v_c = v\pmb{1}_d, V = kI}
    \\
    &= \exp \left( e_y^\intercal \sum_{i=1}^{t} (V\ell_i\ell_i^\intercal a + (\beta V + v_ca^\intercal)\ell_i + v_c \beta) \right)  \sum_{i=1}^{t} \left( e_y^\intercal  V\ell_i\ell_i^\intercal e_x  +  e_y^\intercal v_c\ell_i^\intercal e_x \right)  \bigg|_{a = \pmb{0}_d, v_c = v\pmb{1}_d, V = kI}
    \\
    &= \exp \left( \sum_{i=1}^{t}  \beta k \ell_{iy} + v\beta \right)  \sum_{i=1}^{t} (k \ell_{iy}\ell_{ix} +   v \ell_{ix}).
\end{align*}

Therefore, 
\begin{align*}
    &\frac{\partial  f(V, a, \beta, v_c )}{\partial a_x} \bigg|_{a = \pmb{0}_d, v_c = v\pmb{1}_d, V = kI}
    \\
    &= \EE \Bigg[ \left(\sum_{t=1}^T \sum_{s=1}^d \ell_{ts} \frac{\exp \left( e_s^\intercal \sum_{j=1}^{t-1} (V\ell_j\ell_j^\intercal a + (\beta V + v_ca^\intercal)\ell_j + v_c \beta) \right)}{\sum_{y = 1}^d \exp \left( e_y^\intercal \sum_{j=1}^{t-1} (V\ell_j\ell_j^\intercal a + (\beta V + v_ca^\intercal)\ell_j + v_c \beta) \right)} - \min_{s} \sum_{t=1}^T \ell_{ts}\right) 
    \\
    &\qquad \frac{\partial }{\partial a_x} \left(\sum_{t=1}^T \sum_{s=1}^d \ell_{ts} \frac{\exp \left( e_s^\intercal \sum_{j=1}^{t-1} (V\ell_j\ell_j^\intercal a + (\beta V + v_ca^\intercal)\ell_j + v_c \beta) \right)}{\sum_{y = 1}^d \exp \left( e_y^\intercal \sum_{j=1}^{t-1} (V\ell_j\ell_j^\intercal a + (\beta V + v_ca^\intercal)\ell_j + v_c \beta) \right)} - \min_{s} \sum_{t=1}^T \ell_{ts}\right) \Bigg]\bigg|_{a = \pmb{0}_d, v_c = v\pmb{1}_d, V = kI}
    \\
    &=  \EE \Bigg[ \left(\sum_{t=1}^T \sum_{s=1}^d \ell_{ts} \frac{\exp \left( \sum_{j=1}^{t-1} \beta k \ell_{js} \right)}{\sum_{y = 1}^d \exp \left( \sum_{j=1}^{t-1}  \beta k \ell_{jy}  \right)} - \min_{s} \sum_{t=1}^T \ell_{ts}\right) 
    \\
    &\qquad \Biggl(\sum_{t=1}^T \sum_{s=1}^d \ell_{ts} \frac{ \sum_{j=1}^{t-1} (k \ell_{js} \ell_{jx} +v \ell_{jx}) \exp \left( \sum_{j=1}^{t-1}  \beta k \ell_{js} \right) \sum_{y = 1}^d \exp \left( \sum_{j=1}^{t-1}  \beta k \ell_{jy}  \right)}{\left(\sum_{y = 1}^d \exp \left( \sum_{j=1}^{t-1}  \beta k \ell_{jy}  \right)\right)^2}
    \\
    &\qquad \qquad  - \sum_{t=1}^T \sum_{s=1}^d \ell_{ts} \frac{ \exp \left(\sum_{j=1}^{t-1}  \beta k \ell_{js} \right) \sum_{y = 1}^d \left(\sum_{j=1}^{t-1} (k \ell_{jy} \ell_{jx} + v \ell_{jx}) \exp \left( \sum_{j=1}^{t-1}  \beta k \ell_{jy}  \right)\right)}{\left(\sum_{y = 1}^d \exp \left( \sum_{j=1}^{t-1}  \beta k \ell_{jy}  \right)\right)^2}\Biggr) \Bigg]
    \\
    &= \EE \Bigg[ k \left(\sum_{t=1}^T \sum_{s=1}^d \ell_{ts} \frac{\exp \left( \sum_{j=1}^{t-1} \beta k \ell_{js} \right)}{\sum_{y = 1}^d \exp \left( \sum_{j=1}^{t-1}  \beta k \ell_{jy}  \right)} - \min_{s} \sum_{t=1}^T \ell_{ts}\right) 
    \\
    &\qquad \Biggl(\sum_{t=1}^T \sum_{s=1}^d \ell_{ts} \frac{ \sum_{j=1}^{t-1} \ell_{js} \ell_{jx} \exp \left( \sum_{j=1}^{t-1}  \beta k \ell_{js} \right) \sum_{y = 1}^d \exp \left( \sum_{j=1}^{t-1}  \beta k \ell_{jy}  \right)}{\left(\sum_{y = 1}^d \exp \left( \sum_{j=1}^{t-1}  \beta k \ell_{jy}  \right)\right)^2}
    \\
    &\qquad \qquad  - \sum_{t=1}^T \sum_{s=1}^d \ell_{ts} \frac{ \exp \left(\sum_{j=1}^{t-1}  \beta k \ell_{js} \right) \sum_{y = 1}^d \left(\sum_{j=1}^{t-1} \ell_{jy} \ell_{jx} \exp \left( \sum_{j=1}^{t-1}  \beta k \ell_{jy}  \right)\right)}{\left(\sum_{y = 1}^d \exp \left( \sum_{j=1}^{t-1}  \beta k \ell_{jy}  \right)\right)^2}\Biggr) \Bigg]
\end{align*}
Note that the value does not depend on $x$, which means that $\frac{\partial f}{\partial a}\bigg|_{a = \pmb{0}_d, v_c = v\pmb{1}_d, V = kI} = \tilde c \pmb{1}_d$ for some constant $\tilde c$. 

\subsubsection{Numerical Analysis of Step 2 and Step 4}

In Steps 2 and 4 above, we were not able to show that a $k$ whose value becomes zero exists. We hence provide some empirical evidence here. First, we attach the estimated $\frac{\partial f}{\partial V_{rc}}\bigg|_{a = \pmb{0}_d, v_c = v\pmb{1}_d, V = kI}$ ($r \neq c$), $\frac{\partial f}{\partial V_{rr}}\bigg|_{a = \pmb{0}_d, v_c = v\pmb{1}_d, V = kI}$, $\frac{\partial f}{\partial a_x}\bigg|_{a = \pmb{0}_d, v_c = v\pmb{1}_d, V = kI}$  and $\frac{\partial f}{\partial a_x}\bigg|_{a = \pmb{0}_d, v_c = v\pmb{1}_d, V = kI}$ graph with respect to $k$ value when $\ell_{ts} \sim \text{Unif}([0,1])$ for all $t \in [T], s \in [d]$. While the graph of $\frac{\partial f}{\partial V}\bigg|_{a = \pmb{0}_d, v_c = v\pmb{1}_d, V = kI}$ is not stable, we can see that $k$ for $\frac{\partial f}{\partial V_{rc}}\bigg|_{a = \pmb{0}_d, v_c = v\pmb{1}_d, V = kI}= 0$, $\frac{\partial f}{\partial V_{rr}}\bigg|_{a = \pmb{0}_d, v_c = v\pmb{1}_d, V = kI}= 0$ and $\frac{\partial f}{\partial a_x}\bigg|_{a = \pmb{0}_d, v_c = v\pmb{1}_d, V = kI}=0$ is very similar in \Cref{fig:nd}. We used the Monte Carlo estimation of $1,000,000$ times. 
\vspace{-5mm}
\begin{figure}[!t]
    \centering
    \begin{adjustwidth}{-2cm}{-4cm}
        \includegraphics[width=1.3\columnwidth]{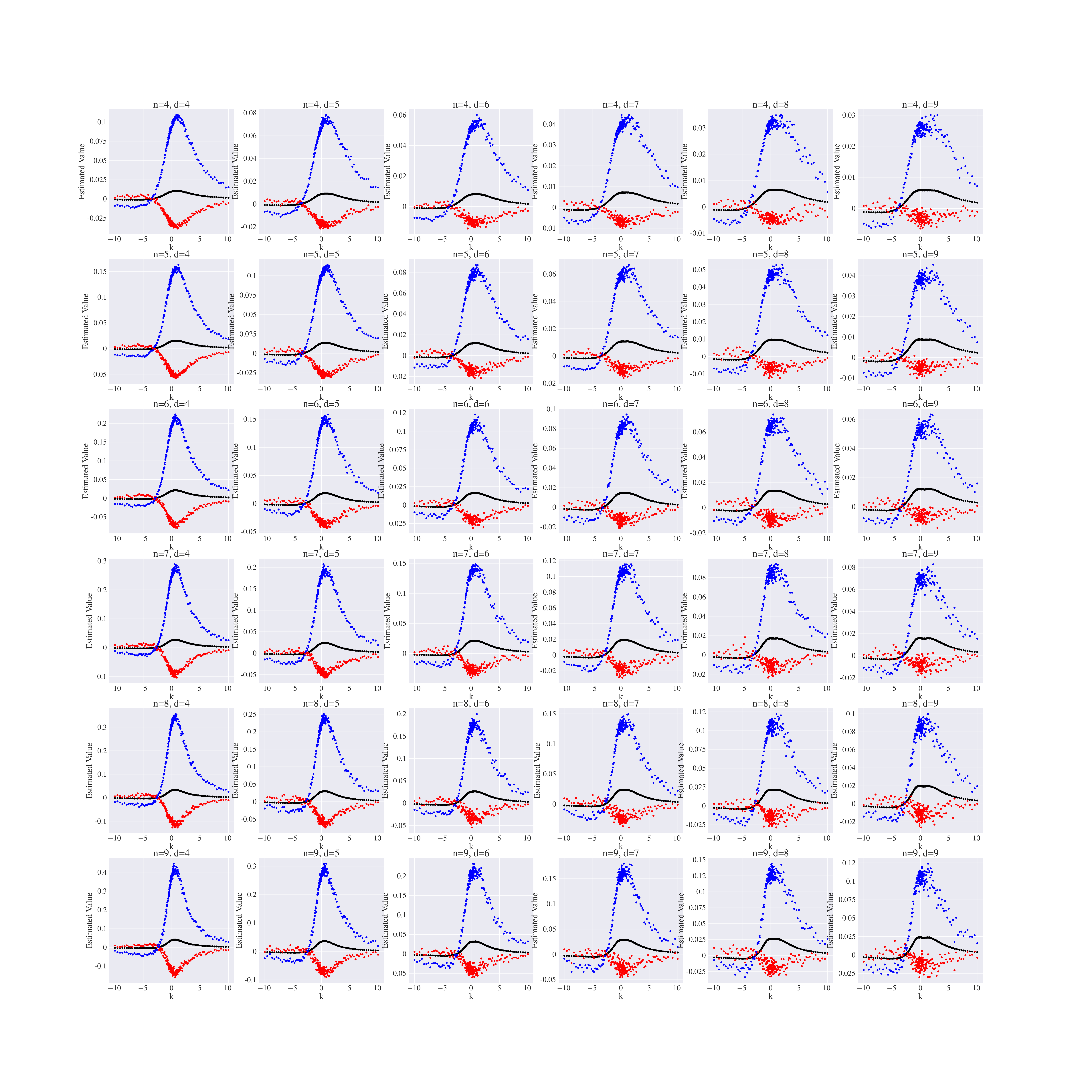}
        \safevspace{-2cm}
    \end{adjustwidth}
    \vspace{-15mm}
        \caption{Calculation of $20\frac{\partial f}{\partial V_{rc}}\bigg|_{a = \pmb{0}_d, v_c = v\pmb{1}_d, V = kI}$ ($r \neq c$)(red), $20\frac{\partial f}{\partial V_{rr}}\bigg|_{a = \pmb{0}_d, v_c = v\pmb{1}_d, V = kI}$ (blue), and $\frac{\partial f}{\partial a_x}\bigg|_{a = \pmb{0}_d, v_c = v\pmb{1}_d, V = kI}$ (black). We experimented with $n \in [4,9]$ and $d \in [4, 9]$. The figure might indicate that $\beta k$ that makes the derivative zero of $\frac{\partial f}{\partial V_{rc}}\bigg|_{a = \pmb{0}_d, v_c = v\pmb{1}_d, V = kI}$ ($r \neq c$), $\frac{\partial f}{\partial V_{rr}}\bigg|_{a = \pmb{0}_d, v_c = v\pmb{1}_d, V = kI}$, and $\frac{\partial f}{\partial a_x}\bigg|_{a = \pmb{0}_d, v_c = v\pmb{1}_d, V = kI}$ would coincide.}
        \label{fig:nd}
\end{figure}

\subsubsection{Empirical Validation}

Our model architecture is defined as follows: the number of layers \(T\) is set to 30 and the dimensionality \(d\) to 32, with the loss vector $l_i$'s distribution \(Z\) following a standard normal distribution \(\mathcal{N}(0, 1)\). During training, we conducted 40,000 epochs with a batch size of 512. We employed the Adam optimizer, setting the learning rate to 0.001. We focus on two key convergence properties: \(K^\intercal (Q \pmb{1} + q_c) \) approaching the zero vector \(\pmb{0}_{d}\) and \(V\) converging to \(a\pmb{1}_d\pmb{1}^\intercal_d  + bI_{d \times d}\), where \(a\) and \(b\) are constants in \(\mathbb{R}\). The conditions \(K^\intercal (Q \pmb{1} + q_c) = \pmb{0}_{d}\) and \(V = a\pmb{1}_d\pmb{1}^\intercal_d  + bI_{d \times d}\) imply that the function \(g(Z_t; V, Q, K) = \sum_{i=1}^{t} (b-a) l_i\), effectively emulating the process of an online gradient descent method. We repeated 10 times. For verifying \(K^\intercal (Q \pmb{1} + q_c) = \pmb{0}_{d}\), we will measure 2-norm of $K^\intercal (Q \pmb{1} + q_c)$. Also for measuring the closeness of $V$ and $a\pmb{1}_d\pmb{1}^\intercal_d  + bI_{d \times d}$, we will measure $\min_{a, b \in \RR} \norm{V - (a\pmb{1}_d\pmb{1}^\intercal_d  + bI_{d \times d})}_{2,2}/b$. The results are demonstrated in the third plot of \Cref{fig:emp-validation}.

\subsection{Comparison with In-Context-Learning Analyses in Supervised Learning}

The very recent studies by \cite{ahn2023transformers, zhang2023trained, mahankali2023one} have demonstrated that if \(Z_t = ((x_1, y_1), \dots, (x_t, y_t), (x_{t+1}, 0))\) and the ``instruction tuning'' loss (i.e., \(\mathbb{E}[\|{\hat{y}_{t+1} - y_{t+1}}\|^2]\)) is being minimized with a single-layer linear self-attention model, then a global optimizer among single-layer linear self-attention models yields the output \(\hat{y}_{n+1} = \eta \sum_{i = 1}^n y_i x_i^\intercal x_{n+1}\). This output can be interpreted as a \emph{gradient descent}  algorithm, indicating that a single-layer linear self-attention model \textbf{implicitly} performs gradient descent. However, in the online learning setting where there are no $y$-labels, such an implicit gradient descent update-rule is hard to define. Compared to the previous studies, our global optimizer among single-layer linear self-attention models is an \emph{explicit} and  \emph{online} gradient descent update for online learning. {With a different loss (regret-loss v.s. instruction-tuning-loss), the techniques to obtain the seemingly similar results are also fundamentally different.}

\subsection{Training Details of   \Cref{ssec:experiemnt}}
\label{appendix:training-detail}

We provide the training details of  \Cref{ssec:experiemnt}. 
For the multi-layer Transformer training, we used 4 layers, 1 head Transformer. For both single-layer and multi-layer, we employed the Adam optimizer, setting the learning rate to 0.001. During training, we conducted 2,000 epochs with a batch size 512. Moreover, when we trained for the loss sequences with the predictable trend, we used 4 layers, 1 head Transformer. For both single-layer and multi-layer, we employed the Adam optimizer, setting the learning rate to 0.001. During training, we conducted 9,000 epochs with a batch size of 512. 

\conf{\subsection{Detailed Experimental Settings in \Cref{ssec:experiemnt}}
\label{detailed-experiment-setting-ch5}
\subsubsection{Randomly-generated loss sequences }
 We used the same loss vector with \Cref{sec:experiments_online}'s randomly generated loss function to compare the result with GPT-4. The results show that the trained single-layer self-attention model or trained Transformer with regret-loss has comparable regret with FTRL and GPT-4's regret, and it can be checked in \Cref{fig:single-layer-ftrl1}.
\begin{figure}[h]
        \centering
        \includegraphics[width=\linewidth]{figs/figure5-4-1.png}
        \caption{Result of the randomly generated loss sequences with Gaussian (left) and uniform with truncation (right). The $p$-value and $\hat{\beta}_0$ value of the Gaussian loss sequences with trained Transformer / single-layer self-attention results were $p = 0.0, \hat{\beta}_0 = 0.4$, $p = 0.0, \hat{\beta}_0 = 0.39$, respectively.  The $p$-value and $\hat{\beta}_0$ value of uniform loss sequences with trained Transformer / single-layer self-attention results were $p = 0.0, \hat{\beta}_0 = 0.43$, $p = 0.0, \hat{\beta}_0 = 0.47$, repsectively.}
        \label{fig:single-layer-ftrl1}
\end{figure}
\subsubsection{Loss sequences with a predictable trend}\label{ssec:traininggpt2}

\cpedit{We investigate the case of loss sequences with predictable trends such as linear trends or sinusoid trends. We might expect that the performance of the trained Transformer would surpass the performance of traditional no-regret algorithms since FTRL would not be an optimal algorithm for the loss sequence with a predictable trend.} We modified the training distribution of random variable \( Z \) to follow two kinds of trends: linear and sinusoid functions. The results show that the trained single-layer self-attention model or trained Transformer with regret-loss outperformed GPT-4 in the metric of regret when the loss sequence is a linear trend, and it can be checked in \Cref{fig:trend}.

\begin{figure}[h]
        \centering
        \includegraphics[width=\linewidth]{figs/figure5-4-2.png}
        \caption{Result of the loss sequences with linear trends (left) and sinusoid trend  (right). The $p$-value and $\hat{\beta}_0$ value of the Gaussian loss sequences with trained Transformer / single-layer self-attention results were $p = 0.0, \hat{\beta}_0 = 0.51$, $p = 0.0, \hat{\beta}_0 = -0.13$, repsectively.  The $p$-value and $\hat{\beta}_0$ value of uniform loss sequences with trained Transformer / single-layer self-attention results were $p = 0.0, \hat{\beta}_0 = 0.89$, $p = 0.0, \hat{\beta}_0 = -0.9$, repsectively.}
        \label{fig:trend}
\end{figure}

\subsubsection{Repeated games}
We investigate the case with a multi-player repeated game; 2x2, 3x3x3, 3x3x3x3 games. The results show that the trained single-layer self-attention model or trained Transformer with regret-loss has a similar performance with FTRL; it can be checked in \Cref{fig:game}.

\begin{figure}
       \centering
        \includegraphics[width=\linewidth]{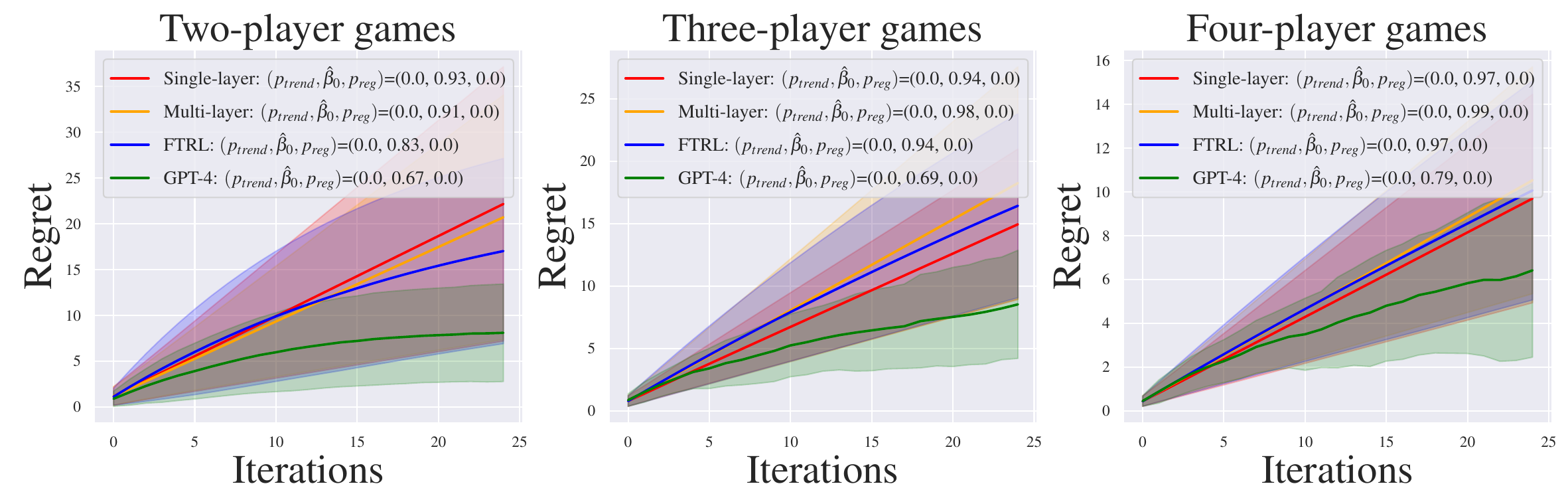}
        \caption{Result of the game with two-player (left) (($p$, $\hat{\beta}_0) = (0.0, 0.69)$), three-player (middle)  (($p$, $\hat{\beta}_0) = (0.0, 0.94)$), and four-player (right)  (($p$, $\hat{\beta}_0) = (0.0, 0.98)$).}
        \label{fig:game}
\end{figure}

\subsubsection{Two Scenarios for Regrettable Behavior of GPT-4}
 We used the same loss vector as in  \Cref{sec:experiments_online}. 
 The results show that the trained single-layer self-attention model or training Transformer with regret-loss can achieve comparable regret performance as FTRL and outperform GPT-4, and it can also be checked with \Cref{fig:counter}.}

\subsection{Ablation Study on Training \Cref{eqn:regret-loss}} \label{appendix:nvariation}
In this section, we provide an ablation study that changes $N$ and $k$ in \Cref{eqn:regret-loss}. To be specific, we will set $N = 1, 2, 4$, $f(x, k) = \max(x, 0)^k$, $h(x) = \max(x, 0)^2$, and $k = 1, 2$. For the multi-layer Transformer training, we used 4 layers and 1 head Transformer. For both single-layer and multi-layer, we employed the Adam optimizer, setting the learning rate to 0.001. During training, we conducted 2,000 epochs with a batch size of 512. We experimented on the randomly generated loss sequences. Especially, we used the uniform loss sequence ($\ell_t \sim \text{Unif}([0,10]^2)$), with the  results in  \Cref{fig:aumsa} and \Cref{fig:aumma}; and the Gaussian loss sequence ($\ell_t \sim \cN(5\cdot\pmb{1}_2, I)$), with the results in  \Cref{fig:agmsa} and \Cref{fig:agmma}.

\begin{figure}
    \centering
    \includegraphics[width=0.9\columnwidth]{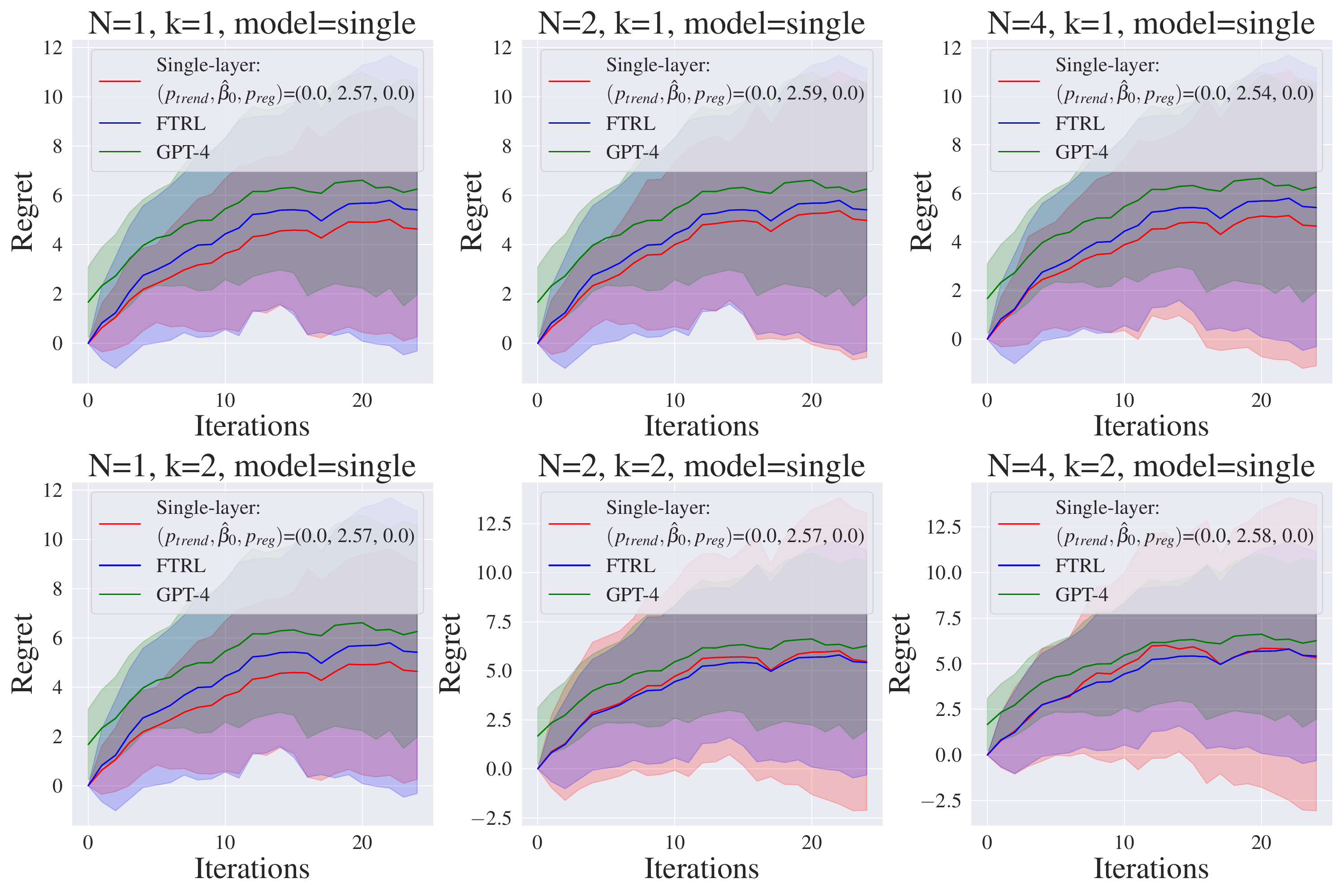}
    \caption{Ablation study for the uniform loss sequence trained with single-layer self-attention layer and  \texttt{Softmax} projection. }
    \label{fig:aumsa}
\end{figure}

\begin{figure}
    \centering
    \includegraphics[width=0.9\columnwidth]{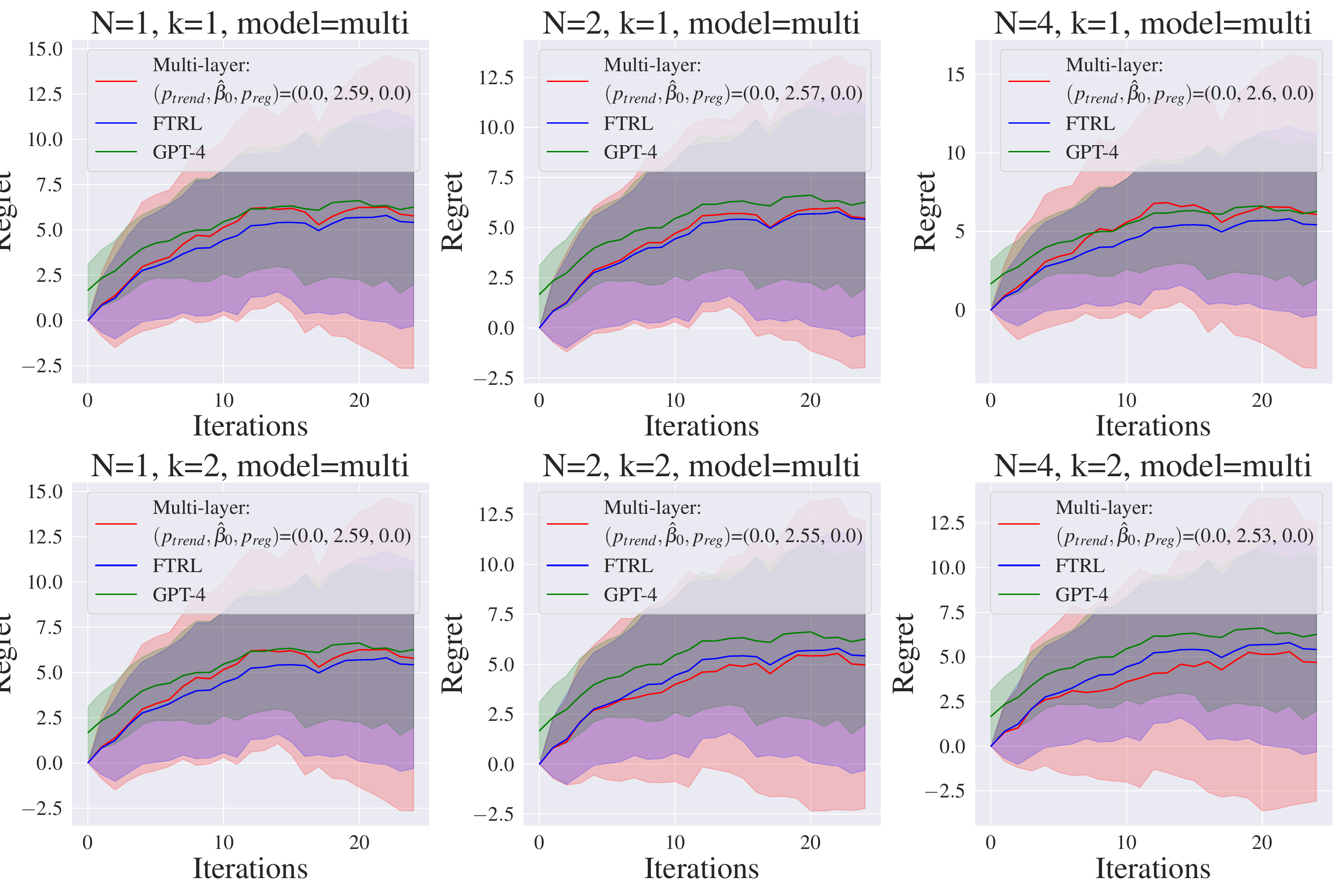}
    \caption{Ablation study for the uniform loss sequence trained with multi-layer self-attention layer and  \texttt{Softmax} projection.}
    \label{fig:aumma}
\end{figure}

\begin{figure}
    \centering
    \includegraphics[width=0.9\columnwidth]{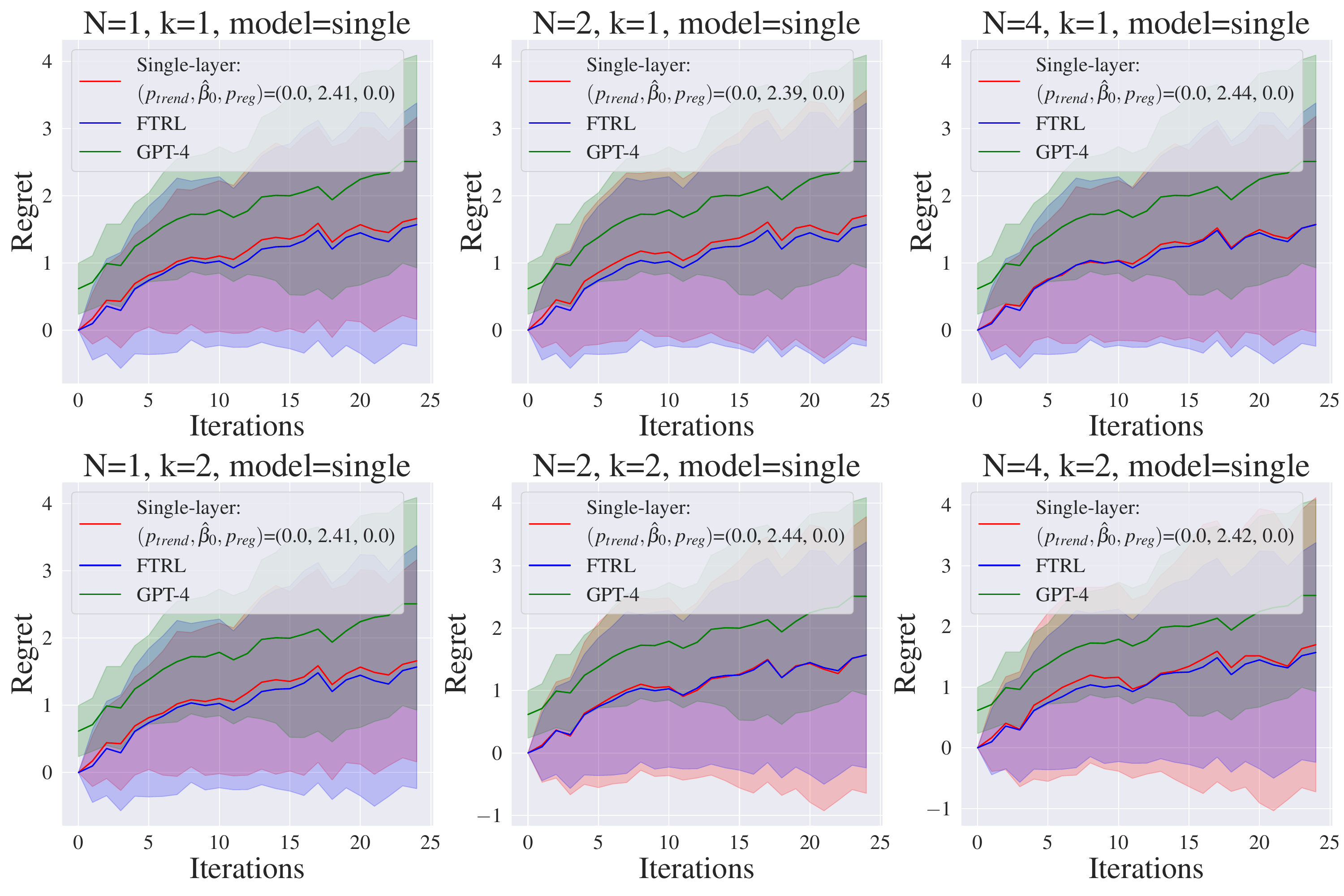}
    \caption{Ablation study for the Gaussian loss sequence trained with single-layer self-attention layer and  \texttt{Softmax} projection. }
    \label{fig:agmsa}
\end{figure}

\begin{figure}
    \centering
    \includegraphics[width=0.9\columnwidth]{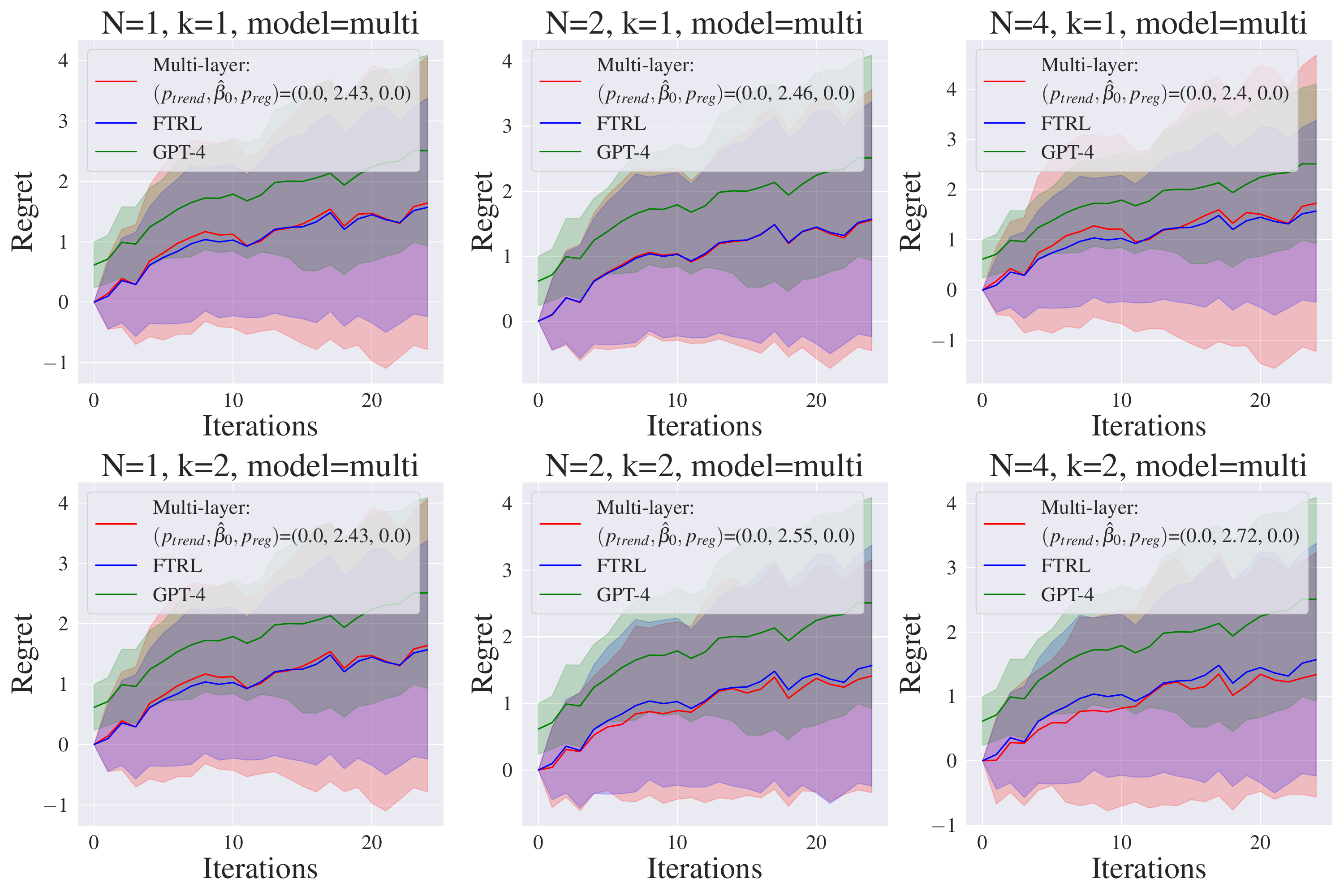}
    \caption{Ablation study for the Gaussian loss sequence trained with single-layer self-attention layer and  \texttt{Softmax} projection.}
    \label{fig:agmma}
\end{figure}

\end{document}